%% file: main.tex
\documentclass[a4paper,fontsize=10pt,
twoside=semi,
chapterprefix,
headsepline,
DIV=9,
bibliography=totoc]{scrbook}

\title{Factored Task and Motion Planning with Combined Optimization, Sampling and Learning}

\author{Joaquim Ortiz-Haro}

\date{November 2023}
\titlehead{A Thesis submitted for the degree of Doctor of Philosophy}
\publishers{School of Something\\University of Somewhere}

\usepackage[T1]{fontenc}
\usepackage{newpxtext,newpxmath}

 \usepackage{setspace}

\usepackage{amsmath,amssymb,amsthm,amsfonts}
\usepackage{stmaryrd}

\usepackage[dvipsnames]{xcolor}
  \usepackage{graphicx}
\usepackage{booktabs}
\usepackage{algorithm}
\usepackage{algpseudocode}
\usepackage{subcaption}
\usepackage{multirow}

\usepackage{inconsolata} 
\usepackage[all]{nowidow}

\usepackage{enumitem}
\newlist{todolist}{itemize}{2}
\setlist[todolist]{label=$\square$}
\usepackage{pifont}

\counterwithin{algorithm}{chapter}

\usepackage{bookmark}

\usepackage[capitalise]{cleveref}

\usepackage{parskip}

\usepackage{listings}

\usepackage{rotating}

\usepackage{fancyvrb}

\usepackage[most]{tcolorbox}

\usepackage{float}
\usepackage{siunitx}

\newcommand\gobble[1]{}%

\RedeclareSectionCommand[
  tocentryformat=\large\scshape,
  tocpagenumberbox=\gobble%
]{part}

\newif\ifcommentout

\definecolor{backcolour}{rgb}{0.95,0.95,0.92}
\lstdefinestyle{mystyle}{
    backgroundcolor=\color{backcolour},   
    basicstyle=\ttfamily\footnotesize,
    breakatwhitespace=false,         
    breaklines=true,                 
    captionpos=b,                    
    keepspaces=true,                 
    showspaces=false,                
    showstringspaces=false,
    showtabs=false,                  
    tabsize=2
}

\usepackage{xcolor}
\hypersetup{
    colorlinks,
    linkcolor={red!50!black},
    citecolor={blue!40!black},
    urlcolor={blue!80!black}
}

\newcommand{\theAnd}{{and }}

\newcommand{\namePartLearning}{{Accelerated Task and Motion Planning with Learning Methods}}
\newcommand{\namePartLearningSentence}{{Accelerated Task and Motion Planning with learning methods}}
\newcommand{\namePartMetaSolver}{{Meta-Solvers: Adaptive Combination of Sampling and Optimization Methods}}
\newcommand{\namePartMetaSolverSentence}{{Meta-solvers: Adaptive combination of sampling and optimization methods}}
\newcommand{\namePartOne}{{Integrated Planning and Optimization for Task and Motion Planning}}
\newcommand{\namePartOneSentence}{{Integrated planning and optimization for Task and Motion Planning}}
\newcommand{\nameChapterOne}{{Diverse Task Planning for Solving Logic Geometric Programs}}

\newcommand{\nameChapterTwo}{{Conflict-Based Search in Factored Logic Geometric Programs}}
\newcommand{\nameChapterThree}{{Learning Optimal Sampling Sequences for Robotic Manipulation}}

\newcommand{\nameChapterFour}{{Towards Meta-Solvers for Task and Motion Planning}}
\newcommand{\nameChapterFive}{{Deep Generative Constraint Sampling}}
\newcommand{\nameChapterSix}{{Learning Feasibility of Factored Nonlinear Programs}}

\newcommand\sbullet[1][.5]{\mathbin{\vcenter{\hbox{\scalebox{#1}{$\bullet$}}}}}

\newcommand{\seqs}{\langle s_0, \ldots, s_K\rangle}
\newcommand{\seqa}{\langle a_1, \ldots, a_K\rangle}

\newcommand{\ptask}{\Pi}
\newcommand{\ptaskPar}[1]{\ptask^{#1} = \langle \vars^{#1},\ops^{#1},\initstate^{#1},\goal^{#1} \rangle}
\newcommand{\ptaskParSTD}{\ptaskPar{}}
\newcommand{\goal}{\ensuremath{g}}
\newcommand{\initstate}{\state_{0}}

\newcommand{\ops}{\ensuremath{\mathcal A}}
\newcommand{\vars}{\ensuremath{\mathcal V}}
\newcommand{\gray}[1]{\textcolor{gray}{#1}}
\newcommand{\xs}{x}

\newcommand{\RR}{\mathbb{R}}
\newcommand{\PP}{\mathbb{P}}

\renewcommand{\AA}{\mathcal{A}}

\newcommand{\norm}[1]{\left\lVert#1\right\rVert}
\newcommand{\seq}[2]{ \ensuremath{\langle #1, \ldots ,  #2 \rangle }}

\newcommand{\prefix}[2]{#1\vert_{#2}}

\newcommand{\q}[1]{``#1''}

\newcommand{\fx}{\tilde{x}}
\newcommand{\fX}{\tilde{X}}

\newcommand{\centered}[1]{\begin{tabular}{l} #1 \end{tabular}}

\newtheorem{proposition}{Proposition}[chapter]
\newtheorem{theorem}{Theorem}[chapter]

\theoremstyle{definition}
\newtheorem{property}{Property}[chapter]

\theoremstyle{definition}
\newtheorem{definition}{Definition}[chapter]

\definecolor{color1bg}{HTML}{458588}

\definecolor{color2bg}{HTML}{67afb2}

\usepackage{tikz}
\usetikzlibrary{bayesnet}

\pgfdeclarelayer{background}
\pgfdeclarelayer{foreground}
\pgfsetlayers{background,main,foreground}

\usetikzlibrary{arrows.meta}

\tikzset{%
  >={Latex[width=2mm,length=2mm]},
            base/.style = {rectangle, rounded corners, draw=black,
                           minimum width=1cm, minimum height=1cm,
                           text centered, font=\small},
  activityStarts/.style = {base, fill=blue!30},
       startstop/.style = {base, fill=red!30},
    activityRuns/.style = {base, fill=green!30},
         inout/.style = {base, minimum width=6cm, minimum height=0cm,
                           font=\small},
         process/.style = {base, minimum width=1.5cm, 
                           font=\small},
         data/.style = {base,minimum width=1.5cm, 
                           font=\small},
         sec/.style = {base, minimum width=1.5cm, 
                           font=\color{gray}\small, draw=none,
                           fill=none },
         processX/.style = {base, minimum width=1.8cm, 
                           inner sep=5pt,
                           minimum height=.5cm,
                           font=\footnotesize},
         processR/.style = {processX, fill=red!30},
         processG/.style = {processX, fill=green!30},
         number/.style = { minimum width=1cm, minimum height=0cm,
                           text centered, font=\small},
        gnn_var/.style = {circle, draw=black,minimum height=.2cm,
           text centered, font=\small},
        gnn_con/.style = {rectangle, draw=black,minimum height=.3cm,
minimum width=.3cm,
                       text centered, font=\small, fill=color2bg},
        gnn_con2/.style = {rectangle, draw=black,minimum height=.3cm,
minimum width=.3cm,
                       text centered, font=\small, fill=color2bg},
}

\tikzstyle{decision} = [diamond, draw, fill=blue!20, text width=4.5em, 
  text badly centered, node distance=3cm, inner sep=0pt]

\tikzstyle{block} = [rectangle, draw, fill=black!20, 
  text centered, minimum height=2em]

\tikzstyle{block2} = [rectangle, draw,  text centered, 
  rounded corners, minimum height=2em]

\tikzstyle{block_lazy} = [rectangle, draw, fill=white!20, text width=3em, 
text centered, rounded corners, minimum height=2em]

\tikzstyle{line} = [draw, -latex'] \tikzstyle{cloud} = [draw, ellipse,
fill=red!20, node distance=3cm, minimum height=2em]

\renewcommand{\factoredge}[4][]{ %
  \foreach \f in {#3} { %
    \foreach \x in {#2} { %
      \path (\x) edge[-,#1] (\f) ; %
    } ;
    \foreach \y in {#4} { %
      \path (\f) edge[->, >={latex}, #1] (\y) ; %
    } ;
  } ;
}

\newcommand{\op}{a}
\newcommand{\init}{\ensuremath{s_0}}

\newcommand{\eff}{\ensuremath{\mathit{eff}}}
\newcommand{\pre}{\ensuremath{\mathit{pre}}}

\newcommand{\variables}[1]{{\mathcal V}(#1)}

\newcommand{\action}{\ensuremath{a}}

\newcommand{\state}{s}
\newcommand{\domain}{{\mathcal D}}
\newcommand{\var}{v}

\newcommand{\extrav}{\overline{\var}}
\newcommand{\noplanop}[1]{#1^e}
\newcommand{\planopdiscarded}[1]{#1^1}
\newcommand{\planopdiscard}[1]{#1^2}
\newcommand{\planopfollow}[1]{#1^3}

\newcommand{\tuple}[1]{\langle #1 \rangle}

\newcommand{\val}{\vartheta}

\input{bench/pp-bench/r0.tex}
\input{bench/pp-bench/r1.tex}
\input{bench/pp-bench/r2.tex}
\input{bench/pp-bench/r3.tex}
\input{bench/pp-bench/s0.tex}
\input{bench/pp-bench/s1.tex}
\input{bench/pp-bench/s2.tex}
\input{bench/pp-bench/s3.tex}

\DeclareMathOperator*{\E}{\mathbb{E}}

\newcommand{\R}{\mathbb{R}}
\newcommand{\M}{\mathcal{M}}

\newcommand{\feq}{{h}}

\newcommand{\f}{\phi}
\newcommand{\fineq}{{g}}

\definecolor{chaptergrey}{rgb}{0.0,0.0,0.0}

\newcommand{\outcomments}[1]{{\leavevmode}}

\newcommand{\qquote}[1]{{``#1''}}

\makeatletter
\g@addto@macro\appendix{%

  \let\oldaddcontentsline\addcontentsline
  \newcommand\hackedaddcontentsline[3]{\oldaddcontentsline{#1}{#2}{\chapapp\nobreakspace#3}}
  \let\oldchapter\chapter
  \renewcommand*\chapter[1]{%
    \let\addcontentsline\hackedaddcontentsline%
    \oldchapter{#1}%
    \let\addcontentsline\oldaddcontentsline%
  }
}
\makeatother

\usepackage{breakcites}

\begin{document}

\begin{titlepage}
  \enlargethispage{2cm}
\begin{center}
\vspace*{.5cm}
\begin{huge}\bfseries
Factored Task and Motion Planning \\ with Combined Optimization,  \\ Sampling and Learning  \par
\end{huge}
\vspace{0.5in}
\begin{large}
by
\end{large}
\begin{Large}\bfseries
  Joaquim Ortiz-Haro, M.Sc. \par
\end{Large}
\vspace{0.6in}
\begin{Large}
A thesis submitted for the degree of  \par
Doctor of Natural Sciences (Dr. rer. nat.) \par
\end{Large}
\begin{Large}
  \textbf{TU Berlin} \par
\end{Large}
\vspace{0.5in}
\begin{Large}
Faculty IV - Electrical Engineering and Computer Science \par
Learning and Intelligent Systems \par
\vspace{.5in}
Berlin, November 2023
\end{Large}
\end{center}

\vfill
{
  \begin{large}
  \rule{\textwidth}{1pt}
\vspace{2pt}
\textbf{Ph.D. Advisor}: Prof. Dr. Marc Toussaint (TU Berlin)  \\
\textbf{Doctoral Committee}: \\
Chair: Prof. Dr. Marc Alexa (TU Berlin) \\
Reviewer: Prof. Dr. Marc Toussaint (TU Berlin) \\
Reviewer: Prof. Dr. Georg Martius (University of T{\"u}bingen) \\
Reviewer: Prof. Dr. Tom{\'a}s Lozano-P{\'e}rez (MIT) 
  \end{large}

}

\end{titlepage}

\frontmatter

\renewcommand*\chapterheadstartvskip{}%

\tableofcontents

\include{acknowledgements}

\include{abstract}

\newpage

\mainmatter

\KOMAoptions{headings=big}%

\let\raggedchapter\raggedleft

\include{intro}

\include{background}

\include{factorized_structure}

\include{diverse_planning}

\include{graph_nlp_planner}

\include{mcts_keyframes}

\include{optimization_over_computation}

\include{deep_gans}

\include{learn_conflicts}

\include{conclusion}

{
\small
\bibliographystyle{apalike}
\bibliography{IEEEabrv,IEEEexample}
}

\appendix

\include{all_publications}

\end{document}

%% file: acknowledgements.tex
\chapter{Acknowledgements}

I would like to express my deepest gratitude to my supervisor, Marc Toussaint.
Thank you for your guidance and encouragement throughout my Ph.D.
Your knowledge and passion for robotics research, mastering both theory and practice, have been an inspiration to me.

I am very grateful to the members of the doctoral committee, Prof.
Georg Martius and Prof.
Tomás Lozano, whose research I have followed closely and admired, for reviewing and evaluating my work.
I am confident that your feedback and evaluation will significantly contribute to the refinement of this thesis.
Additionally, I would like to express my gratitude to Prof.
Marc Alexa for serving as the chair of my doctoral committee.

A heartfelt thanks to Erez Karpas and Michael Katz; collaborating with you has been a great learning experience and a pleasure, and I really enjoyed the short research stay at Technion in Haifa.
My sincere appreciation goes to Prof.
Georg Martius and Prof.
David Remy for their advice on the Thesis Advisory Committee of the IMPRS-IS program.

I have been lucky to work in a group of great colleagues and friends.
Danny, Ingmar, Valentin, Svetlana, Pia, Akmaral, Khaled, Wolfgang, Ilaria, Oz, Andreas, and Jung-Su, I have learned a lot from each one of you.
Collaborating on research projects has been a pleasure, and you have created a pleasant and stimulating work environment.
I look forward to meeting you again at robotics conferences all around the world!

Because research often requires time to step away and enjoy other facets of life, I would like to thank all the wonderful friends I met in Berlin.
I had a great time here with all of you, and I look forward to keeping in touch and seeing you again.

And finally, but most importantly, I would like to thank all my family, whose support has been fundamental, especially when I needed it the most.
Your love, support, and belief in me have been my pillars of strength.
I dedicate this achievement to you, with immense love and gratitude.
Thank you for always being there.

%% file: abstract.tex
\chapter{Abstract}
\vspace{-.8cm}
Modern robots excel at performing simple and repetitive tasks in controlled environments; however, future applications, such as robotic construction and assistance, will require long-term planning of physical interactions.

These problems can be formulated as Task and Motion Planning (TAMP). 
The goal is to find how the robot should move to solve complex tasks requiring multiple interactions with objects in the environment, such as building furniture or cleaning and organizing the kitchen. However, TAMP is notoriously difficult to solve because it involves a tight combination of task planning and motion planning, considering geometric and physical constraints.

In this thesis, we aim to improve the performance of TAMP algorithms from three complementary perspectives.
First, we investigate the integration of discrete task planning with continuous trajectory optimization.
Our main contribution is a conflict-based solver that automatically discovers why a task plan might fail when considering the constraints of the physical world.
This information is then fed back into the task planner, resulting in an efficient, bidirectional, and intuitive interface between task and motion, capable of solving TAMP problems with multiple objects, robots, and tight physical constraints.

Traditionally, there have been two competing approaches to solving TAMP problems: sample-based and optimization-based methods.
In the second part, we first illustrate that, given the wide range of tasks and environments within TAMP, neither sampling nor optimization is superior in all settings.
To combine the strengths of both approaches, we have designed meta-solvers for TAMP, adaptive solvers that automatically select which algorithms and computations to use and how to best decompose each problem to find a solution faster.

A third promising direction to improve TAMP algorithms is to learn from previous solutions to similar problems.
In the third part, we combine deep learning architectures with model-based reasoning to accelerate computations within our TAMP solver.
Specifically, we target infeasibility detection and nonlinear optimization, focusing on generalization, accuracy, compute time, and data efficiency.

At the core of our contributions is a refined, factored representation of the trajectory optimization problems inside TAMP.
This structure not only facilitates more efficient planning, encoding of geometric infeasibility, and meta-reasoning but also provides better generalization in neural architectures.

\chapter{Zusammenfassung}
\vspace{-.8cm}

Moderne Roboter sind hervorragend darin, einfache und wiederholte Aufgaben in kontrollierten Umgebungen auszuführen.
Zukünftige Anwendungen, wie die robotergestützte Roboterassistenz und das robotergestützte Bauen, werden jedoch eine langfristige Planung physischer Interaktionen erfordern.

Diese Probleme können als Aufgaben- und Bewegungsplanung (Task and Motion Planning, TAMP) formuliert werden.
Dabei ist das Ziel, lange Abfolgen von Roboteraktionen zu finden, um komplexe Aufgaben zu lösen, die mehrere Interaktionen mit der Umgebung erfordern und dabei geometrische und physische Beschränkungen berücksichtigen.
TAMP ist bekannterweise sehr schwer zu lösen, da es eine enge Kombination von Aufgabenplanung und Bewegungsplanung erfordert.

In dieser Arbeit zielen wir darauf ab, die Leistung von TAMP-Algorithmen aus drei komplementären Perspektiven zu verbessern.
Zuerst untersuchen wir, wie man Aufgabenplaner mit Trajektorienoptimierung integriert.
Unser Hauptbeitrag ist ein neues Framework, das automatisch entdeckt und kodiert, warum ein Aufgabenplan angesichts der Beschränkungen der physischen Welt scheitern könnte.
Dies führt zu einer effizienten und intuitiven Integration von Aufgaben- und Bewegungsplanung.

Traditionell gab es zwei konkurrierende Ansätze, um TAMP-Probleme zu lösen: Stichprobenbasierte und optimierungsbasierte Methoden.
Im zweiten Teil zeigen wir zuerst, dass angesichts der Vielzahl von Aufgaben und Umgebungen innerhalb der TAMP weder Stichproben noch Optimierung in allen Situationen überlegen sind.
Um die Stärken beider Ansätze zu kombinieren, haben wir Meta-Lösungsalgorithmen für TAMP entwickelt: adaptive Solver, die automatisch auswählen, welche Algorithmen und Berechnungen verwendet werden sollen und wie jedes Problem am besten zerlegt werden kann, um eine Lösung schneller zu finden.

Ein dritter vielversprechender Ansatz zur Verbesserung der TAMP-Algorithmen besteht darin, von früheren Lösungen ähnlicher Probleme zu lernen.
Im dritten Abschnitt schlagen wir zwei verschiedene neuronale Architekturen vor, um teure Berechnungen in unserem Lösungsalgorithmus, nämlich Unlösbarkeitserkennung und nichtlineare Optimierung, mit Hilfe von Deep-Learning-Methoden zu beschleunigen.

Im Kern unserer Beiträge steht eine verfeinerte, faktorisierte Darstellung der Trajektorienoptimierungsprobleme innerhalb von TAMP.
Diese Struktur erleichtert nicht nur eine effizientere Planung und Kodierung geometrische Unlösbarkeiten, sondern ermöglicht auch das Schlussfolgern über potenzielle Rechenoperationen und bietet eine bessere Generalisierung in neuronalen Architekturen.

\chapter{Resumen}
\vspace{-.8cm}

Los robots modernos sobresalen en la realización de tareas simples y repetitivas en entornos controlados; sin embargo, las aplicaciones futuras, como la construcción y la asistencia robótica, requerirán una planificación autónoma de diversas interacciones físicas.

Estos problemas se pueden formular como Planificación de Tareas y Movimientos (TAMP, por sus siglas en inglés).
El objetivo es encontrar cómo debe moverse el robot para resolver tareas complejas que requieren múltiples interacciones con los objetos del entorno, como por ejemplo, montar un mueble o limpiar y recoger la cocina.
Sin embargo, la resolución de TAMP es notoriamente difícil porque implica una combinación estrecha de planificación de tareas y planificación de movimientos, considerando restricciones geométricas y físicas.

En esta tesis, nuestro objetivo es mejorar el rendimiento de los algoritmos TAMP desde tres perspectivas complementarias.
Primero, investigamos la integración de la planificación de tareas discretas con la optimización continua de trayectorias.
Nuestra principal contribución es un algoritmo que descubre automáticamente por qué un plan de tareas podría fallar al considerar las restricciones del mundo físico.
Esta información retroalimenta al planificador de tareas, resultando en una interfaz bidireccional e intuitiva entre tareas y movimientos, capaz de resolver problemas con múltiples objetos, robots y restricciones físicas complejas.

Tradicionalmente, ha habido dos enfoques competitivos para resolver problemas TAMP: métodos basados en muestreo y métodos basados en optimización.
En la segunda parte, primero ilustramos que, dada la amplia variedad de tareas y entornos dentro de TAMP, ni el muestreo ni la optimización son superiores en todos los escenarios.
Para combinar las fortalezas de ambos enfoques, hemos diseñado meta-algoritmos para TAMP, algoritmos adaptativos que seleccionan automáticamente qué algoritmos y cálculos usar y cómo descomponer mejor cada problema para encontrar una solución más rápidamente.

Una tercera dirección prometedora para mejorar los algoritmos TAMP es aprender de soluciones previas a problemas similares.
En la tercera parte, combinamos arquitecturas de aprendizaje profundo con razonamiento basado en modelos para acelerar los cálculos dentro de nuestros algoritmos.
Específicamente, nos enfocamos en la detección de qué restricciones físicas son inviables y en la optimización no lineal, centrándonos en la generalización, precisión, tiempo de ejecución y la eficiencia de datos.

En el núcleo de nuestras contribuciones se encuentra una representación refinada, desglosada y fragmentada de los problemas de optimización de trayectorias dentro de TAMP.
Esta estructura no solo facilita una planificación más eficiente, análisis de restricciones geométricas y meta-algoritmos, sino que también proporciona una mejor generalización en arquitecturas neuronales.

\chapter{Resum}
\vspace{-.8cm}

Els robots moderns excel·leixen en la realització de tasques simples i repetitives en entorns controlats; tanmateix, les aplicacions futures, com la construcció i l'assistència robòtica, requeriran una planificació autònoma de diverses interaccions físiques.

Aquests problemes es poden formular com a Planificació de Tasques i Moviments (TAMP, per les seves sigles en anglès).
L'objectiu és trobar com s'ha de moure el robot per resoldre tasques complexes que requereixen múltiples interaccions amb els objectes de l'entorn, com ara muntar un moble o netejar i recollir la cuina.
La resolució de TAMP és notòriament difícil perquè requereix una combinació estreta de planificació de tasques i planificació de moviments, considerant restriccions geomètriques i físiques.

En aquesta tesi, el nostre objectiu és millorar el rendiment dels algoritmes TAMP des de tres perspectives complementàries.
Primer, investiguem la integració de la planificació de tasques discretes amb l'optimització contínua de trajectòries.
La nostra principal contribució és un algorisme que descobreix automàticament per què un pla de tasques podria fallar quan es consideren les restriccions del món físic.
Aquesta informació retroalimenta al planificador de tasques, resultant en una interfície bidireccional entre tasques i moviments, capaç de resoldre problemes amb múltiples objectes, robots i restriccions físiques.

Tradicionalment, hi ha hagut dos enfocaments competitius per resoldre problemes TAMP: mètodes basats en mostreig i mètodes basats en optimització.
A la segona part, primer il·lustrem que, donada l'àmplia varietat de tasques i entorns dins de TAMP, ni el mostreig ni l'optimització són superiors en tots els escenaris.
Per combinar les fortaleses de tots dos enfocaments, hem dissenyat meta-algorismes per a TAMP, algorismes adaptatius que seleccionen automàticament quins algorismes i càlculs utilitzar i com descompondre millor cada problema per trobar una solució més ràpidament.

Una tercera direcció prometedora per millorar els algoritmes TAMP és aprendre de solucions prèvies a problemes similars.
A la tercera part, combinem arquitectures d'aprenentatge profund amb raonament basat en models per accelerar els càlculs dins dels nostres algorismes.
Específicament, apliquem aquests models en la detecció de quines restriccions físiques són inviables, i en l'optimització no lineal, centrant-nos en la generalització, precisió, temps d'execució i eficiència de dades.

En el nucli de les nostres contribucions hi ha una representació refinada, desglossada i fragmentada dels problemes d'optimització de trajectòries dins de TAMP.
Aquesta estructura no només facilita una planificació més eficient, millor anàlisi de les restriccions geomètriques i meta-algorismes, sinó que també proporciona una millor generalització en arquitectures neuronals.

%% file: intro.tex
\chapter{Introduction}
\label{ch:intro}

\paragraph{Autonomy of robotic systems}

Nowadays, robots are ubiquitous in industrial settings such as factories and warehouses, where they perform repetitive tasks in controlled environments.
However, robotic systems still lack the robustness and autonomy needed to become useful companions in our daily lives, helping humans with construction sites, working in hazardous environments, household chores, and elderly care.

Traditionally, robots have operated in very controlled environments, such as factories for car manufacturing.
At the most basic level of autonomy, a human worker would manually provide a reference trajectory for the robot, and the robot's task would be to follow this trajectory in a repetitive manner.
At a second level of autonomy, the human operator defines a desired goal configuration for the robot, and the robot has to find a collision-free path to reach the goal.

Roughly, these two levels of autonomy correspond, respectively, to two distinct fields in robotics: optimal control
\cite{10.5555/1524151} and motion planning \cite{DBLP:books/cu/L2006}.
They both address fundamental problems in robotics, with open and interesting questions from both research and engineering perspectives.
However, optimal control and motion planning require precise task specifications and cannot handle complex and long-term tasks independently.

Thus, the degree of autonomy of these systems is insufficient for deploying robots in foreseeable future applications such as construction sites or household environments.
As human users, it is not convenient to specify short-term goals or reference trajectories.
Instead, we need robots to operate at a higher degree of autonomy, where we provide only a high-level goal such as ``clean the table'', ``stack all the blocks'' or ``build a chair'' and the robot solve the task.

These types of problems, which require long-term planning of a sequence of physical interactions with the environment, considering geometric and physical constraints, can be formalized as Task and Motion Planning (TAMP) \cite{garrett2021integrated}.
Together with other fundamental challenges in robotics, such as perception \cite{10.5555/1121596} and dexterous manipulation \cite{billard2019trends, bicchi2000robotic}, TAMP will play a central role in our quest for the autonomy of robotic systems.

Algorithmic improvements, better formulations, and a deeper understanding of Task and Motion Planning are fundamental for future robotics systems.
In this thesis, we aim to improve the performance of TAMP algorithms from three complementary perspectives: more efficient solvers that integrate classical planning with trajectory optimization, meta-solvers that automatically select how to solve and decompose the problems, and learning-based methods to accelerate expensive computations and reuse solutions to similar problems.

\paragraph{Task and Motion Planning}

In Task and Motion Planning (TAMP), robots must plan a sequence of feasible motions and interactions with other objects to achieve a desired state of the environment, considering the geometric and physical constraints of the world (e.g., \cite{toussaint2015logic,kaelbling2011hierarchical,dantam2018incremental}\footnote{A comprehensive discussion of related work in TAMP is provided later in \cref{ch:background}.
}).
It combines aspects of both task planning and motion planning and usually assumes a deterministic transition model and perfect state information.

TAMP problems can be solved by planning at two levels of abstraction: a high-level task plan (also known as an action skeleton or sequence of high-level actions) and low-level motion.
In the high-level task, the robots decide with which objects to interact, in which order, and what type of interactions, e.g., grasping, pushing, or throwing.
Such information can be encoded using discrete variables, resulting in a discrete planning problem \cite{fikes1971strips,bonet2001planning}.
Additionally, the robot must plan low-level motion that fulfills the geometric and physical constraints of the real world, such as collision avoidance, stability, reachability, and friction \cite{toussaint2018differentiable}.
Notably, in real-world applications, there are strong interdependencies between the high-level task and the low-level motion.
This prevents a naive decomposition of the problem using directly off-the-shelf task planners and motion planners, making TAMP problems notoriously difficult to solve.

TAMP includes a broad class of problems, including multimodal motion planning, sequential manipulation planning, rearrangement planning, and hybrid task planning.
Consider, for example, the problem of building a tower of blocks with two robots, as shown in \cref{fig:intro:example1}.
\cref{fig:intro:example1-sol} shows some key configurations of the solution to this TAMP problem.
Here, considering reachability and collision avoidance constraints is fundamental because two blocks are obstructing the placement of the tower, and each robot can only reach a subset of the blocks.

Similarly, in the TAMP problem in \cref{fig:intro:example2}, two robots have to put a ball on top of a tower of blocks.
The ball is out of reach, but a stick can be used to push it to the center of the table.
\cref{fig:intro:example2-sol} shows some key configurations of the solution to this TAMP problem.
Remarkably, in both examples, the human user has only provided the abstract goal, and all the motions and interactions with the environment are computed autonomously by the robots.

\begin{figure}
	\centering
	\begin{tabular}{cc}
		Initial state & Goal                      \\

		\centered{
			\includegraphics[width=.32\linewidth]{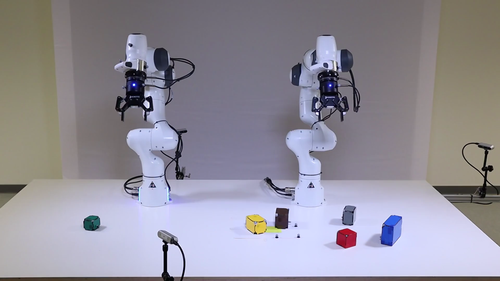}
		}             &

		\centered{ \texttt{Build tower of blocks} \\ \texttt{blue-gray-red-green} \\ \texttt{in the center of the table.
		} }
	\end{tabular}
	\caption{Task and Motion Planning -- Example \num{1}.
		\vspace{.5cm}}
	\label{fig:intro:example1}
\end{figure}
\begin{figure}
	\centering
	\setlength{\tabcolsep}{2pt} %
	\begin{tabular}{ccc}

		\includegraphics[width=.25\linewidth]{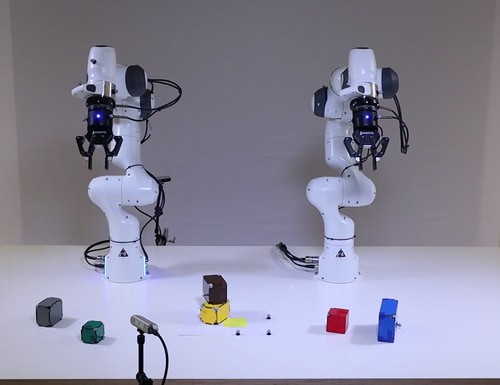} &
		\includegraphics[width=.25\linewidth]{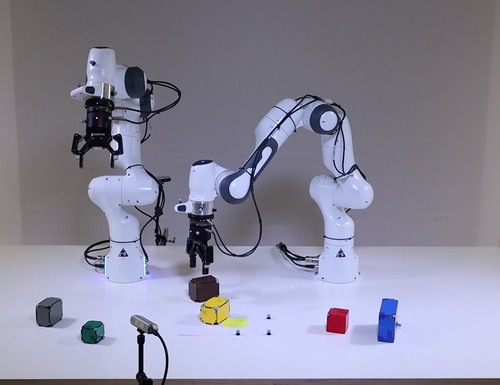} &
		\includegraphics[width=.25\linewidth]{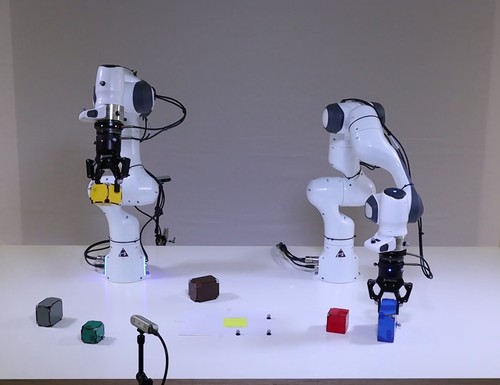}   \\

		\includegraphics[width=.25\linewidth]{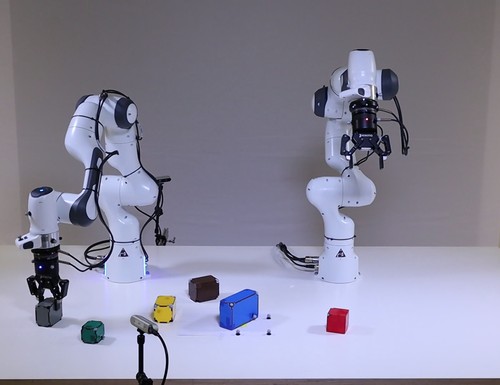} &
		\includegraphics[width=.25\linewidth]{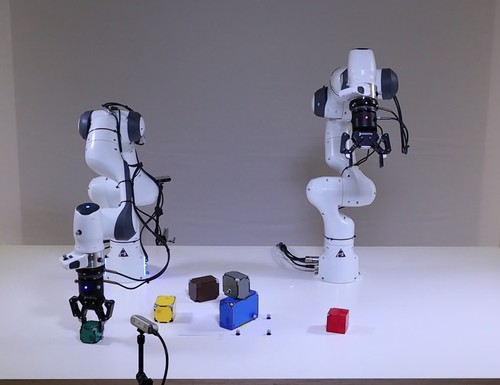} &
		\includegraphics[width=.25\linewidth]{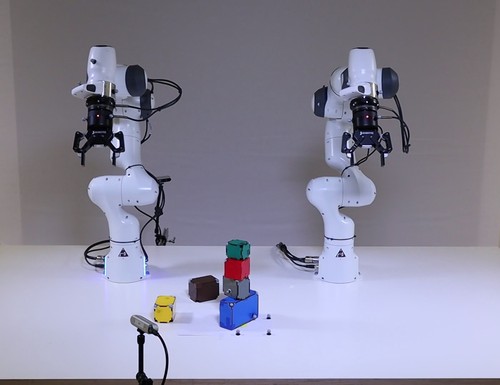}

	\end{tabular}
	\caption{
		A solution to the Task and Motion Planning problem of \cref{fig:intro:example1}.
	}
	\label{fig:intro:example1-sol}
\end{figure}

\begin{figure}[t]
	\centering
	\begin{tabular}{cc}
		Initial state & Goal           \\
		\centered{
			\includegraphics[width=.32\linewidth]{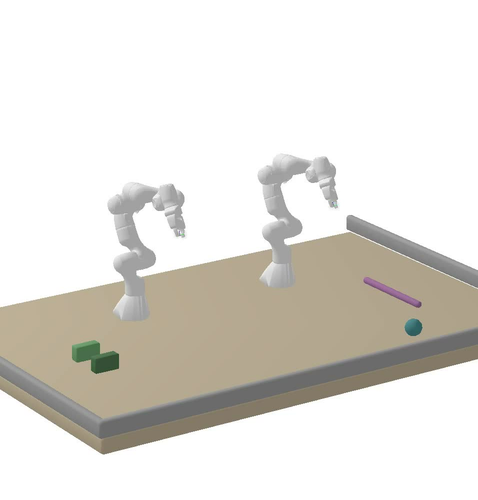}
		}             &
		\centered{
		\texttt{Build tower of blocks} \\ \texttt{light green - dark green} \\ \texttt{with blue ball on top.
		}
		}
	\end{tabular}
	\caption{Task and Motion Planning -- Example \num{2}.
		\vspace{.5cm}}
	\label{fig:intro:example2}
\end{figure}

\begin{figure}
	\centering
	\setlength{\tabcolsep}{2pt} %
	\begin{tabular}{ccc}
		\includegraphics[width=.25\linewidth]{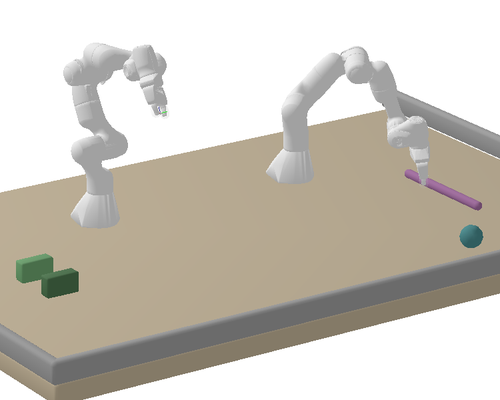} &
		\includegraphics[width=.25\linewidth]{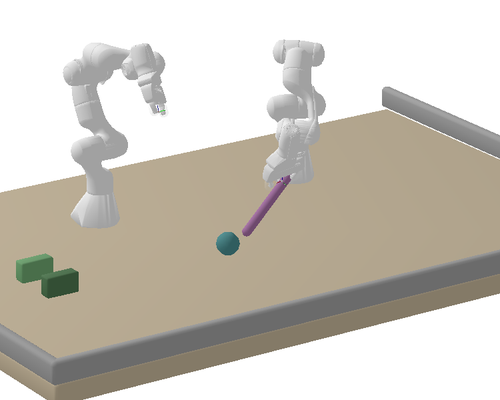} &
		\includegraphics[width=.25\linewidth]{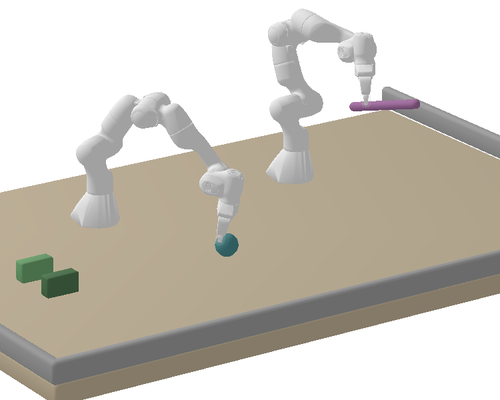}   \\
		\includegraphics[width=.25\linewidth]{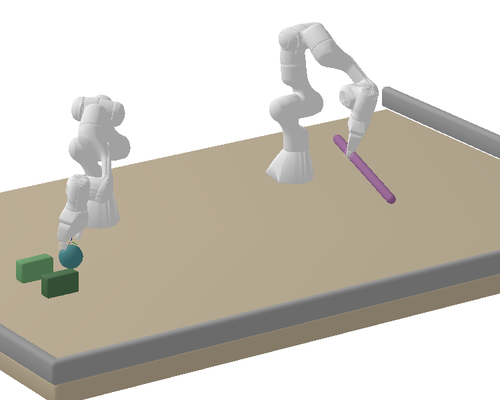} &
		\includegraphics[width=.25\linewidth]{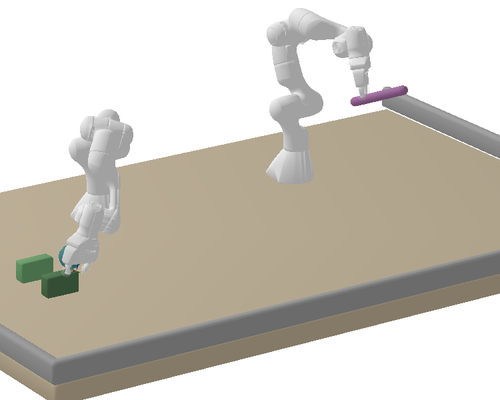} &
		\includegraphics[width=.25\linewidth]{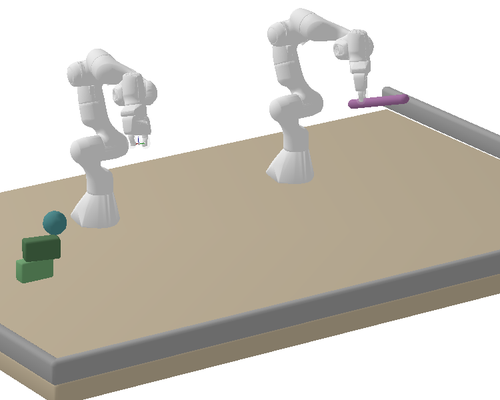}
	\end{tabular}
	\caption{
		A solution to the Task and Motion Planning problem in \cref{fig:intro:example2}.
	}
	\label{fig:intro:example2-sol}
\end{figure}

\emph{Why are TAMP problems difficult to solve?}
First, the complexity of the problem grows exponentially with the number of objects and robots in the environment, often denoted with the term ``curse of dimensionality''.
Second, generating collision-free and feasible robot trajectories is a challenging problem on its own because it requires motion planning and optimization in continuous spaces with geometric and physical nonlinear constraints, such as collision avoidance, grasping, and stability.

Finally, TAMP requires joint reasoning about motion planning and task planning, often with strong dependencies and interactions between the two levels.
Geometric constraints and physical constraints have an influence on whether a task plan is valid, and how the robots execute each intermediate step of the plan is relevant for achieving the final high-level goal.

Beyond its potential applications, Task and Motion Planning is a very interesting research field because it requires planning at two levels of abstraction while accounting for strong dependencies and without easy decompositions.
Such problems are usually easy for humans to solve, as humans excel at reasoning at different levels of abstraction and combining long-term planning with motion planning.
However, they are very challenging to formalize mathematically and to solve with computers.

From a robotics perspective, designing TAMP algorithms requires understanding and combining state-of-the-art planning algorithms, such as classical planning, trajectory optimization, and motion planning.
Using off-the-shelf solvers is usually not sufficient, as solving TAMP problems requires new interfaces and functionality to achieve a tighter integration.

Throughout this thesis, we assume perfect knowledge of the state of the environment where robots operate, as is typically done in the TAMP literature.
Thus, our robots will know the positions and the geometry of the objects they want to manipulate.
While this is a general limitation for deploying robots in the real world today--because we cannot assume this perfect information in uncontrolled environments--TAMP with perfect state knowledge is still an open and unsolved research problem.

\section{Sampling and Optimization Methods for Task and Motion Planning}

Algorithms for TAMP problems are complex systems that combine and tightly integrate different algorithmic components from task planning and motion planning.
Traditionally, there have been two alternative and competing approaches to solving these problems, which differ in the tools used to compute the motion of the robots and in how the discrete and continuous search are interleaved.

Sample-based approaches to TAMP (e.g., \cite{garrett2021integrated,srivastava2014combined}s) incrementally discretize the continuous search space and attempt to compute a solution incrementally, step by step, using constraint sampling methods and sample-based motion planning \cite{kavraki1996probabilistic,lavalle2001rapidly}.
For example, to generate the robot motion that picks up and moves one object, these methods would first generate a valid grasp and a valid placement of the object, then a robot configuration for the pick and for the place using inverse kinematics, and finally a trajectory.
Because the motion is computed incrementally in several steps using conditionally constrained sampling, these methods cannot efficiently account for joint constraints in the motion, e.g., when very precise grasps or placements are required to enable subsequent actions.

On the other hand, optimization-based approaches to TAMP first compute candidate task plans \cite{toussaint2015logic,toussaint2018differentiable} and then attempt to find a motion plan for the full task plan using trajectory optimization \cite{bertsekas1997nonlinear,betts1998survey,NoceWrig06}.
This accounts for all the constraints of the motion jointly, instead of sequentially, as in the case of sample-based approaches.
There is often a huge number of potential task plans that fail because of the geometric constraints.
Thus, a significant challenge in these approaches is to design a good interface between the task and the motion so that candidate plans are informed about geometric infeasibility.

Our work builds on Logic Geometric Programming (LGP) \cite{toussaint2015logic}, a prominent optimization-based formulation of TAMP, which represents the problem as a joint optimization over discrete and continuous variables, where different high-level task plans imply different nonlinear constraints for the robot motion.

In the first section of this thesis, we present two new optimization-based solvers for TAMP that combine and refine tools and techniques from trajectory optimization, conflict-based search, and classical planning.
Our key contributions are two new bidirectional interfaces between the discrete and continuous levels, which allow informing task planners about geometric infeasibility, resulting in very efficient TAMP solvers.

Our first solver, \textit{Diverse Planning for LGP}, interfaces trajectory optimization with state-of-the-art classical planning by detecting and encoding infeasible prefixes of the task plan (\cref{ch:diverse_planning}).

However, encoding prefix infeasibility is not enough to solve problems with multiple robots and objects, as the number of candidate high-level plans to evaluate grows exponentially fast.

In our second solver, \textit{Factored-NLP Planner}, we combine classical planning and optimization in a more precise and effective way, leveraging a novel factored representation of LGP problems.
Our framework automatically detects which nonlinear constraints fail and encodes this information back into the task planner.
This results in a very efficient interface, solving TAMP with several robots, objects, and intricate geometric constraints in just a few seconds (\cref{ch:bid}).

\paragraph{Meta-solvers: An adaptive combination of sampling and optimization}
\label{sec:intro:meta}

TAMP encompasses a wide range of problem settings, including varying numbers of robots and objects, as well as different goals, geometries, and physical constraints.
Given this diversity and complexity, it is unrealistic to expect that a single algorithm can efficiently solve all TAMP problems.

Despite recent advances in both sample-based and optimization-based approaches to TAMP, the performance of each solver is limited by the capabilities of the underlying methods used to generate the motion.
When a problem can be easily decomposed into simpler, independent subproblems -- where the motion of the robot can be computed independently -- sample-based approaches are very efficient.
In contrast, when the motion of the robot is highly constrained, and there are long-term dependencies between the actions, optimization-based approaches are more efficient.

Minor variations in the task or the environment can make a TAMP problem more suitable for one type of solver or the other.
Therefore, designing a TAMP solver that can efficiently solve all problems while strictly adhering to either the optimization or sample-based paradigm is impossible.

The second part of this thesis investigates how to design meta-solvers for Task and Motion Planning.
Intuitively, a meta-solver is a solver that can choose which type of solver to use based on the problem at hand \cite{Russell91Principles}.
In our project, the meta-solver will actively reason about how to best decompose the problem and which algorithm and strategy to use to compute the robot motions, two vital questions that are fixed by design in current TAMP solvers.

We first investigate the problem of finding the keyframe configurations for a fixed high-level task plan (\cref{ch:mcts}).
This is a fundamental subproblem of TAMP that already exposes the trade-off between choosing either sampling or joint nonlinear optimization.
Our algorithm will learn how to best decompose the problem to maximize the number of solutions found in a fixed computational time.
The discovered optimal strategies are adaptive hybrid combinations of optimization and sampling, which outperform user-predefined decompositions as well as full joint optimization.

In a subsequent project, we tackle the comprehensive TAMP problem.
To bridge the gap between optimization and sample-based approaches, we first define a novel computational state that extends the traditional notion of discrete-continuous state with free states subject to constraints.
Based on this formulation, we introduce a TAMP meta-solver, a hybrid solver that automatically uses a flexible combination of optimization and sampling methods to solve for the high-level task plan and the low-level motion in a TAMP problem (\cref{ch:meta-solver}).

\section{Accelerating Model-Based Solvers with Deep Learning}

The traditional approach to solving a Task and Motion Planning (TAMP) problem in the real world can be described in three stages:
First, we create a model of the world using our knowledge of physics and geometry.
Second, we use this model, together with a TAMP solver, to compute the best robot actions.
Third, we apply these actions in the real-world system.

Triggered by the success of Deep Learning \cite{Goodfellow-et-al-2016} in other fields, such as computer vision \cite{krizhevsky2012imagenet, mildenhall2021nerf} and natural language processing \cite{hochreiter1997long, vaswani2017attention}, there has been an explosion of data-driven approaches that leverage data rather than first-principles physics models to solve robotic problems.

However, in the context of TAMP, where long-term planning is required, abstract and physics models are very valuable as they enable scaling and generalization to different types of problems and scenarios.
Instead of replacing world models and planning algorithms, we investigate how to combine data-driven approaches with model-based approaches to improve the efficiency of TAMP solvers.
Specifically, we are interested in using a dataset of solutions to similar problems to accelerate planning algorithms, making expensive computations more tractable and enabling real-time solutions to combinatorial and large-scale optimization problems.

The combination of learning and model-based approaches has received a lot of attention in recent years, as it can potentially combine the best of both worlds, but it presents several challenges.
A fundamental open research question is how to represent the problem and the solutions to achieve strong generalization and accuracy.
Useful neural models should be applicable in a wide range of new, unseen problems, instead of just memorizing the training data as if it were a table of problems and solutions.

In this thesis, we investigate how to use deep learning to accelerate TAMP solvers, focusing on two computationally expensive operations in our TAMP solvers: solving nonlinear optimization programs and determining which constraints are infeasible in an overconstrained optimization problem.

First, we propose deep generative models \cite{goodfellow2014generative, kingma2013auto} to provide a good initial warm start for nonlinear optimization in very challenging optimization problems (\cref{ch:gans}).
In the second project, we train a graph-based classifier \cite{kipf2016semi} using the structure of the nonlinear program to directly predict which constraints are infeasible (\cref{ch:learn-feas}).
In both projects, we present a unique contribution toward understanding how to leverage structural knowledge of the TAMP problem for creating more accurate and generalizable neural models.

\section{The Factored Structure of Task and Motion Planning}

A deep study and analysis of the factored structure of Task and Motion Planning problems are at the core of all our contributions.

Using the underlying factored structure of large problems has been key in other robotics fields.
For example, in Simultaneous Localization and Mapping (SLAM), the factored structure is used to solve sparse optimization problems faster to estimate the position of a robot while creating a map of the environment \cite{dellaert2017factor}.
In classical planning, the factorization of the state space using a set of variables is key to designing domain-independent heuristics that guide heuristic search algorithms \cite{bonet2001planning}.

Similarly, a correct and precise formulation of TAMP problems can unveil significant structure that can be exploited to solve the problem more efficiently.
First, we have different objects and robots in the scene, which implies a natural factorization of the continuous configuration space and the high-level discrete representation.
Second, there is a temporal structure that arises from the temporal dimension of a planning problem.
Furthermore, even when the required motions are long and involve multiple robots and objects, each constraint and physical interaction usually involves a few objects at a time.

In this thesis, we use a refined, factored representation of the trajectory optimization problems within TAMP, which exposes the local dependencies between variables and constraints.
For instance, when a robot picks up an object, the new interaction constrains the motion of the object with respect to the gripper but does not create additional constraints between other objects in the environment.
Additionally, we show how the nonlinear constraints of the motion are defined by the high-level task plan.
For instance, a pushing interaction will imply very different constraints on the robot's and object's motion than a pick, place, or handover action.

A factored representation of TAMP is required to identify which are the building blocks of all the TAMP problems, understand how they are connected, and unveil how these combinations of local structures can compose a very rich set of TAMP problems.
For instance, consider again the two TAMP problems of \cref{fig:intro:example1,fig:intro:example2}, with solutions shown in \cref{fig:intro:example1-sol,fig:intro:example2-sol}.
They both share the same local structure, with multiple objects and robots, pick and place interactions (with additional pushing interactions in the second example).
What information can we reuse from one problem to the other?
In what sense can we use the solution of one problem to inform the other?
How can we decompose the problem?
How can we share information about different high-level task plans in the same scenario, or even across problems?

During this thesis, we will incrementally tackle all these questions by using a factored representation of the trajectory optimization problems within TAMP.
In our new TAMP solvers, we introduce a refined factored Logic Geometric Program formulation, which, combined with technical algorithmic contributions, can be used to detect why task plans fail and share this information back into the high-level task planner (\cref{ch:bid}).

In the second part of the thesis, the factored structure is used to reason about good problem decompositions and sequences of small sampling operations inside a meta-solver (\cref{ch:mcts}).

Finally, the factored structure is also fundamental in our learning-based approaches (\cref{ch:gans,ch:learn-feas}).
Together with suitable neural architectures, the factored structure introduces a strong inductive and relational bias in our models, reducing the sample complexity during training, improving accuracy, and enabling generalization to new and diverse problems and environments.

\section{Reading Guide and Statement of Contributions}

\cref{ch:background} introduces the main concepts and tools used as a foundation throughout the thesis: nonlinear programming, classical planning, and the Logic Geometric Programming formulation of Task and Motion Planning, together with a literature review of solvers and problem formulations for TAMP.

\cref{sec:bg:structure} provides an intuitive explanation of the factored structure inherent in trajectory optimization problems within TAMP, illustrated with multiple examples.

The core of the thesis is divided into six chapters, organized into three parts: \textup{1 - \namePartOne}, \textup{2 - \namePartMetaSolver}, and \textup{3 - \namePartLearning}.
In the conclusion (\cref{ch:conclusion}), we summarize the contributions and discuss open challenges in the field of TAMP.
\begin{figure}
	\centering
	\small
	\begin{tikzpicture}[node distance = 7em, auto]
		\node [block2, minimum width=7cm,text width=7cm,align=left] (n1) {\textbf{Ch. 4} - \nameChapterOne ~ \textcolor{gray}{(3)}};
		\node [block2, below of=n1,  minimum width=7cm,text width=7cm,align=left] (n2) {\textbf{Ch. 5} - \nameChapterTwo ~ \textcolor{gray}{(4)}};
		\node [block2, below of=n2,minimum width=7cm,text width=7cm,align=left] (n3) {\textbf{Ch. 6} - \nameChapterThree ~  \textcolor{gray}{(1)}};
		\node [block2, below of=n3,minimum width=7cm,text width=7cm,align=left] (n4) {\textbf{Ch. 7} - \nameChapterFour ~ \textcolor{gray}{(6)}};
		\node [block2, below of=n4,minimum width=7cm,text width=7cm,align=left] (n5) {\textbf{Ch. 8} - \nameChapterFive ~ \textcolor{gray}{(2)}};
		\node [block2, below of=n5,minimum width=7cm,text width=7cm,align=left] (n6) {\textbf{Ch. 9} - \nameChapterSix ~ \textcolor{gray}{(5)}};
		\node [block, fill=SkyBlue,above left of=n2,yshift =-.5cm, align=left,text width=7cm,xshift=1cm] (p1) {\textbf{Part I} - { \footnotesize \namePartOne }};
		\node [block,fill=SkyBlue, above left of=n4,yshift =-.5cm,align=left,text width=7cm,xshift=1cm] (p2) {\textbf{Part II} - {\footnotesize \namePartMetaSolver}};
		\node [block,fill=SkyBlue, above left of=n6,yshift =-.5cm,align=left,text width=7cm,xshift=1cm] (p3) {\textbf{Part III} - {\footnotesize \namePartLearning}};

		\node [block, right of=n3, xshift=5cm, minimum width=3cm, text width=3cm, fill=Salmon] (f1) {Subsets of Infeasible Constraints};
		\node [block, below of=f1, minimum width=3cm,  text width=3cm,fill=SeaGreen] (f2) {Factored Structure};
		\node [block, below of=f2, minimum width=3cm,text width=3cm, fill=GreenYellow]  (f3) {Sampling Keyframes};

		\begin{pgfonlayer}{background}

			\path[color=GreenYellow,line width=.8pt]
			(f3.west) edge (n5.east)
			(f3.west) edge (n3.east);

			\draw[color=Salmon,line width=.8pt]
			(f1.west) edge (n2.east)
			(f1.west) edge (n6.east);

			\draw[color=SeaGreen,line width=.8pt]
			(f2.west) edge (n2.east)
			(f2.west) edge (n3.east)
			(f2.west) edge (n5.east)
			(f2.west) edge (n6.east);

			\draw[color=SkyBlue,line width=.8pt]
			(p1) edge (n1)
			(p1) edge (n2)

			(p2) edge (n3)
			(p2) edge (n4)

			(p3) edge (n5)
			(p3) edge (n6);

		\end{pgfonlayer}
	\end{tikzpicture}
	\caption{Reading Guide -- A graphical overview.}
	\label{fig:reading-guide}
\end{figure}
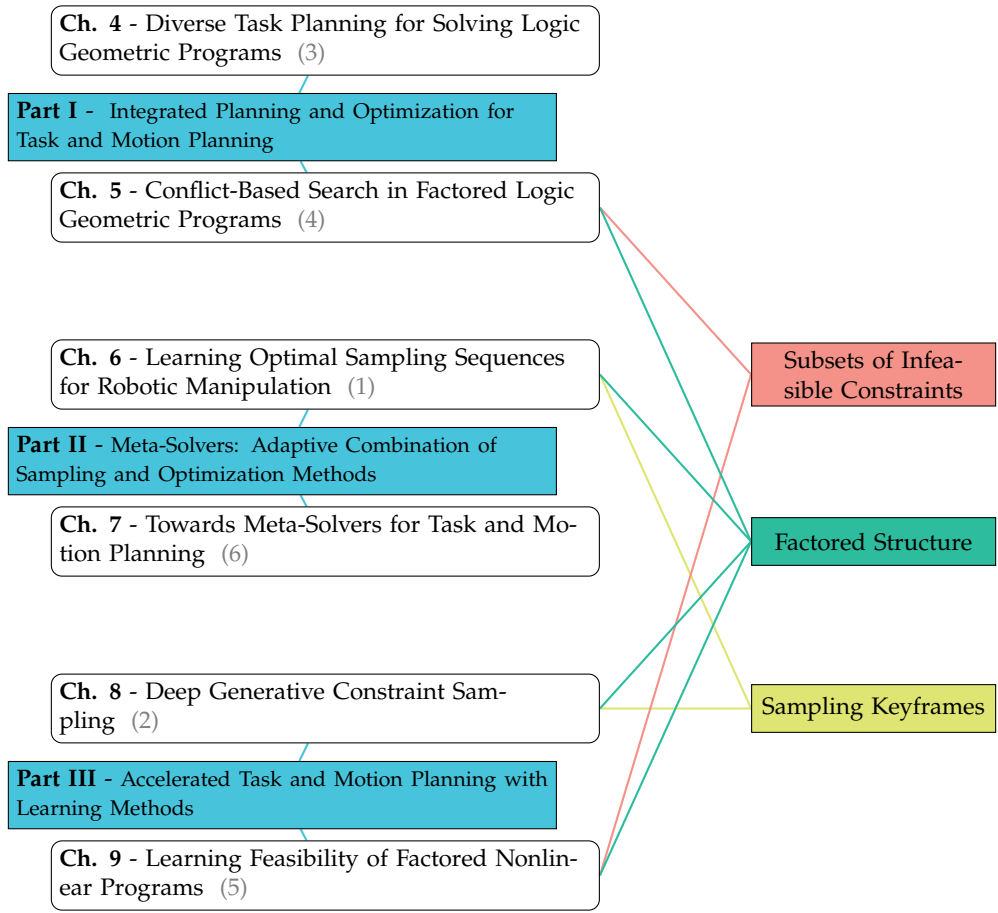

Despite this clear organization, several interconnections and synergies emerge between different chapters, offering a comprehensive view of TAMP from multiple perspectives.
A graphical overview is provided in \cref{fig:reading-guide}.
Chapters are represented with white rectangles and are connected with other chapters based on the shared contributions, methodologies, or problem settings they address, depicted with distinct color-coded boxes.
For instance, the analysis of the factored structure of the trajectory optimization problems, displayed in red, is a central theme in four out of the six chapters.

Each chapter essentially reflects a different research project.
The chronological order of these projects, marked with a gray number, illuminates the progression of my research journey.
It aims to show how diverse perspectives and fields can converge on novel ideas and valuable contributions, often requiring one to revisit similar methodologies and problem settings with new tools and insights.

\textit{Part I - \namePartOne}

\begin{itemize}
	\item In \cref{ch:diverse_planning}, \textit{\nameChapterOne}, we present a systematic interface between discrete task planning and trajectory optimization for solving TAMP.
	      Our solver detects geometric conflicts in the form of prefixes of task plans that are infeasible and blocks these prefixes in the task planner.
	      This chapter is based on the publication \cite{ortiz22conflict},
	      \begin{itemize}
		      \item \underline{Ortiz-Haro, J.
		            }, Karpas, E., Toussaint, M., \theAnd Katz, M. (2022).
		            Conflict-Directed Diverse Planning for Logic-Geometric Programming.
		            In Proceedings of the International Conference on Automated Planning and Scheduling (Vol.
		            32, pp. 279-287).
	      \end{itemize}

	\item In \cref{ch:bid}, \textit{\nameChapterTwo}, we present a new factored formulation of TAMP and a second TAMP solver that uses this factored representation as a bidirectional interface between task and motion.
	      Here, the solver can detect and encode infeasible subsets of nonlinear constraints, resulting in a more efficient interface.
	      This chapter is based on the publication \cite{ortiz2022conflictinterface},
	      \begin{itemize}
		      \item \underline{Ortiz-Haro, J.
		            }, Karpas, E., Katz, M., \theAnd Toussaint, M. (2022).
		            A Conflict-Driven Interface Between Symbolic Planning and Nonlinear Constraint Solving.
		            IEEE Robotics and Automation Letters, 7(4), (pp. 10518-10525).
	      \end{itemize}
\end{itemize}

\textit{Part II - \namePartMetaSolver}

\begin{itemize}
	\item In \cref{ch:mcts}, \textit{\nameChapterThree}, we present a meta-algorithm to solve a key subproblem of TAMP: finding the keyframe configurations for a fixed task plan.
	      The meta-algorithm combines sampling and optimization to minimize the computational time required to generate diverse solutions.
	      This work is based on the publication \cite{ortiz2021learning},
	      \begin{itemize}
		      \item \underline{Ortiz-Haro, J.
		            }, Hartmann, V.
		            N., Oguz, O.
		            S., \theAnd Toussaint, M.
		            (2021).
		            Learning Efficient Constraint Graph Sampling for Robotic Sequential Manipulation.
		            IEEE International Conference on Robotics and Automation (ICRA) (pp.
		            4606-4612).
	      \end{itemize}
	\item In \cref{ch:meta-solver}, \textit{\nameChapterFour}, we present a meta-solver\footnote{
		      We plan to extend and submit the content of this chapter to a robotics or planning conference, for instance, IROS, ICRA, or ICAPS.
		      This research has been conducted in collaboration with Erez Karpas and Marc Toussaint.
	      },
	      for the comprehensive TAMP problem.
	      This solver combines search on the task and motion levels and determines the best way to decompose the TAMP problem, deciding automatically whether it is better to use constrained sampling or joint nonlinear optimization.
\end{itemize}

\textit{Part III - \namePartLearning}

\begin{itemize}
	\item In \cref{ch:gans}, \textit{\nameChapterFive}, we present a new method to generate samples on constraint manifolds using a combination of deep generative models and nonlinear optimization.
	      We apply our framework to compute faster keyframe configurations of a fixed task plan.
	      This work is based on the publication \cite{ortiz2022structured},
	      \begin{itemize}
		      \item \underline{Ortiz-Haro, J.
		            }, Ha, J.
		            S., Driess, D., \theAnd Toussaint, M.
		            (2022).
		            Structured Deep Generative Models for Sampling on Constraint Manifolds in Sequential Manipulation.
		            In Conference on Robot Learning (pp.
		            213-223).
		            PMLR.
	      \end{itemize}

	\item In \cref{ch:learn-feas}, \textit{\nameChapterSix}, we present a graph-neural model that predicts which constraints of a factored nonlinear program are infeasible.
	      In the context of TAMP, our model can generalize to different scenes, longer manipulation sequences, more robots and objects than the example data seen during training.
	      This chapter is based on the publication \cite{ortiz2023learning},
	      \begin{itemize}
		      \item \underline{Ortiz-Haro, J.
		            }, Ha, J.
		            S., Driess, D., Karpas, E., \theAnd Toussaint, M.
		            (2023).
		            Learning Feasibility of Factored Nonlinear Programs in Robotic Manipulation Planning.
		            IEEE International Conference on Robotics and Automation (ICRA) (pp.
		            3729-3735).
	      \end{itemize}
\end{itemize}

Several of my publications as a Ph.D.
student at TU Berlin have been excluded from this thesis to maintain a focus on the core contributions to Task and Motion Planning.
The complete list of publications during my Ph.D.
is shown in \cref{app:publications}.

%% file: background.tex
\chapter{Background}
\label{ch:background}

This chapter introduces the background material necessary to understand the scientific contributions that will be presented later.
Our work builds on Logic Geometric Programming, an optimization-based formulation of Task and Motion Planning (TAMP), which combines nonlinear programming and classical planning.

Thus, we start with a brief presentation of nonlinear programs, including a short discussion on how to solve them and examples in the context of robotics (\cref{sec:bg:nlp}).
Then, we present classical planning, introducing the factored formulation and solvers that we will use later in the thesis.
We also present the Blocksworld domain, which provides a high-level abstraction of the TAMP problem, ignoring the geometric constraints (\cref{bg:sec:planning}).

In \cref{sec:bg:lgp}, we present the Logic Geometric Programming formulation together with the state-of-the-art solver for this formulation.
We conclude the background chapter with a literature review on Task and Motion Planning (\cref{sec:bg:related-work}).

\section{Nonlinear Programs in Robotics}
\label{sec:bg:nlp}

\paragraph{Nonlinear programs}
A nonlinear program (NLP) is an optimization problem of the form:
\begin{subequations}\label{eq:nlp-opt}
	\begin{align}
		\min_{x \in \RR^n} \quad & f(x)\,,         \\
		\text{s.t.
		} \quad                  & h(x) = 0 \,,    \\
		                         & g(x) \leq 0 \,,
	\end{align}
\end{subequations}
where \(x \in \RR^n \) is an $n$-dimensional continuous vector variable, \(f: \RR^n \to \RR\) is the cost function, \(h : \RR^n \to \RR^{l}\) are the equality constraints, and \(g : \RR^n \to \RR^m\) are the inequality constraints.
All functions \(f\), \(h\), and \(g\) are smooth and differentiable.
The abbreviation $\text{s.t.
	}$ stands for ``subject to''.

The optimal solution minimizes the objective function while fulfilling the equality and inequality constraints.
It must fulfill the first-order necessary conditions of optimality, known as the Karush–Kuhn–Tucker (KKT) conditions.
KKT conditions state that if a point \(x^* \in \mathbb{R}^n\) is a local minimum, then there exist vectors \(\lambda^* \in \RR^{l}\) and \(\mu^* \in \RR^{m}\), called Lagrange multipliers or dual variables, such that the following conditions hold:
\begin{subequations}
	\label{eq:kkt}
	\begin{align}
		\nabla f(x^*) + D h(x^*)^T \lambda^{*} + D g (x^*)^T \mu^* = 0 \, , \\
		h(x^*) = 0 \, ,                                                     \\
		g(x^*) \leq 0 \, ,                                                  \\
		\mu^* \geq 0 \, ,                                                   \\
		g(x^*)^T \mu = 0 \,,
	\end{align}
\end{subequations}
where \(\nabla f(x^*)\) is the gradient of \(f\), and \(D h(x^*)\) and \(D g(x^*)\) are the Jacobians of \(h\) and \(g\) evaluated at $x^*$.
The KKT conditions are sufficient conditions for optimality if the NLP is a convex optimization problem, with \(f\) convex, \(h\) being linear, and \(g\) being convex (under some regularity conditions; we refer to \cite{NoceWrig06} for more technical and precise definitions and proofs).

If a vector \(x_{\text{feas}} \in \mathbb{R}^n\) fulfills the constraints, i.e., \(h(x_{\text{feas}}) = 0\), \(g(x_{\text{feas}}) \leq 0\), it is called a feasible solution.
The set of feasible solutions of \eqref{eq:nlp-opt} defines a nonlinear manifold
\begin{equation}
	\mathcal{M} = \{ x \in \mathbb{R}^n \mid h(x) = 0, \, g(x) \leq 0 \} \,.
\end{equation}

An NLP is infeasible if there are no feasible solutions.
Throughout this thesis, we often omit the cost term and focus on the feasibility problem,
\begin{subequations}
	\label{eq:nlp-feas}
	\begin{align}
		\text{find} & ~ x \in \mathbb{R}^n, \\
		\text{s.t.
		}           & ~ h(x) = 0 \,,        \\
		            & ~ g(x) \leq 0 \,.
	\end{align}
\end{subequations}
For simplicity, we will refer to both \eqref{eq:nlp-opt} and \eqref{eq:nlp-feas} as NLPs.
In fact, a practical and robust way to generate one feasible solution in \eqref{eq:nlp-feas}
is to choose a reference value \(x_{\text{ref}} \in \mathbb{R}^n\) and solve the NLP:
\begin{subequations}
	\label{eq:nlp-proj}
	\begin{align}
		\min_{x \in \mathbb{R}^n} & ~ || x - x_{\text{ref}}||^2 \\
		\text{s.t.
		}                         & ~ h(x) = 0 \,,              \\
		                          & ~ g(x) \leq 0 \,.
	\end{align}
\end{subequations}

\newpage

\paragraph{Methods for solving a nonlinear program}

Nonlinear programs do not have a closed-form solution, and they are usually solved with iterative methods.
An extensive review is available in \cite{NoceWrig06}.

\textit{Unconstrained Optimization:}
First, let's consider the unconstrained optimization problem,
\begin{align}
	\min_{x \in \mathbb{R}^n} \quad & f(x) \,.
\end{align}
We can optimize this function with a broad family of local iterative algorithms that generate a sequence \( x_{k+1} = x_k + \alpha_k d_k \), starting from an initial value \( x_0 \).
The step direction \( d_k \) can be computed using the gradient \( \nabla f(x_k) \) and the Hessian \( \nabla^2 f(x_k) \) (second-order derivatives of the function \( f \) at the current point \( x_k \)).

For instance, in gradient descent, \( d_k \) is given by \( -\nabla f(x_k) \).
Using second-order information, \( d_k \) is given by \( -(\nabla^2 f(x_k))^{-1} \nabla f(x_k) \) in the Newton method, or \( -(\nabla^2 f(x_k) + \beta_k I_d)^{-1} \nabla f(x_k) \), with \( \beta_k>0 \), for a regularized Newton method.
Quasi-Newton methods use \( -B^{-1} \nabla f(x_k) \), where \( B \) approximates the Hessian.

The step size \( \alpha_k \) is either fixed or chosen adaptively to ensure that the function decreases enough at each iteration, using a line search algorithm, e.g., a backtracking line search that finds the step size that fulfills the Armijo or Wolfe Conditions \cite{NoceWrig06}.

In general, the sequence \( x_0,x_1,\ldots,x_k \) converges to a stationary point, i.e., a point where the gradient is zero.
Note that some precautions must be taken when computing the step direction and step size, see \cite{NoceWrig06}.

However, the point of convergence can be a local minimum, instead of a global minimum \( x^* = \min_x f(x) \).
When the function \( f \) is convex, a local minimum is guaranteed to be the global minimum, but otherwise, this does not hold in the general non-convex case.
Thus, the choice of the initial guess \( x_0 \) is very important, as it can lead to different local minima.

The convergence rate is linear for gradient descent and quadratic for the Newton method.
The computational complexity of each iteration, with respect to the size of the vector variable problem (i.e., \( n \)), is linear in gradient descent and cubic in the Newton method, which requires solving a linear system of dimension \( n \) in each iteration.

In robotics problems, Newton and Quasi-Newton methods converge much faster than gradient descent methods and are often preferred.
Gradient descent performs poorly when the cost function has different curvatures in different directions (a very large disparity between the largest and smallest eigenvalues of the Hessian matrix).

\textit{Constrained Optimization:}
The most popular algorithms to solve \eqref{eq:nlp-opt} are the Augmented Lagrangian algorithm, sequential quadratic programming, interior point methods, and penalty methods.
Similar to the unconstrained case, they are iterative methods that generate a sequence of points \( x_k \) that converges to a stationary point.

All these methods try to find a point that fulfills the first-order necessary conditions for optimality \eqref{eq:kkt}.
Similar to the unconstrained case,
these methods perform only local optimization.
In general, convergence to the optimum is guaranteed only if the NLP is a convex optimization problem with both a convex objective function and a convex feasible set.

In practice, in the nonlinear case, these methods might converge to a feasible local optimal point or to an infeasible point that does not fulfill the constraints, depending on the initial guess and the nonlinearity of the constraints.

Throughout this thesis, we solve NLPs using the Augmented Lagrangian algorithm.
The Augmented Lagrangian algorithm solves \eqref{eq:nlp-opt} by solving a sequence of unconstrained optimizations.
Starting from an initial guess \( (x_0, \lambda_0 , \mu_0) \), we update the primal and the dual variables in an iterative two-step process.
At iteration \( k \), primal variables are updated by solving the unconstrained optimization problem (starting from the initial guess \( x_{k-1} \))
\begin{equation}
	\min_x \mathcal{L}_A(x, \lambda_{k-1} , \mu_{k-1}, \rho ),
\end{equation}
where \( \mathcal{L}_A \) is the Augmented Lagrangian, defined as:
\begin{equation}
	\begin{split}
		\mathcal{L}_A(x, \lambda , \mu, \rho ) = f(x) + &\lambda^T h(x) + \mu^T g(x) + \\ &\frac{\rho}{2} \sum_{j=1}^{l} h_j(x)^2 + \frac{\rho}{2} \sum_{j=1}^m [g_j(x) \ge 0 \lor \mu_j > 0 ] g_j(x)^2 \,,
	\end{split}
\end{equation}
with the penalty parameter \( \rho > 0 \) and where the subscript $j$, e.g., $h_j(x)$, indicates the $j$-th component of a vector.
Using this formulation (slight variations of the term for the inequalities are also possible), an inequality constraint acts as an equality constraint if it is not fulfilled or its dual variable is strictly positive.

After the primal variables $x_k$ have been updated, the dual variables are updated with:
\begin{equation}
	\lambda_{k} \gets \lambda_{k-1} + \rho h(x_{k})\,, \quad
	\mu_{k} \gets \text{max}( 0, \mu_{k-1} + \rho g(x_{k})).
\end{equation}
Additionally, it is often convenient to increase the penalty parameter \( \rho \) if the constraints are not fulfilled to a desired amount.
Detailed analyses of convergence and practical implementations are provided, e.g., in \cite{NoceWrig06, andreani2008augmented, conn2013lancelot}.

\newpage

\paragraph{Examples of nonlinear programs in robotics}

Nonlinear programs are used in robotics for a wide range of applications, including motion planning, optimal control, and inverse kinematics.

The variables in an NLP represent the configuration of the robot.
For instance, the joint angles of a 7-DOF manipulator can be represented with a vector variable \(x \in \mathbb{R}^7\).
A finite-dimensional vector can also represent a trajectory.
For example, the trajectory of a manipulator \(q: [0,1] \to \mathbb{R}^7\) can be represented by a cubic polynomial \(q(t) = a + bt + ct^2 + dt^3\), resulting in the variable in the NLP \(x = [a, b, c, d] \in \mathbb{R}^{4 \cdot 7}\).

The nonlinear constraints in the NLP can model various constraints, including collision avoidance, kinematic, grasping, and contact constraints.
The cost function can represent objectives such as energy, time, and control effort.
We now present two examples of NLPs in robotics relevant to this thesis.

\emph{Single Keyframe Optimization:}
Suppose we have a 7-DOF manipulator, and our objective is to generate a configuration that picks up a box using a top grasp along the x-direction (assuming, for simplicity, that the box is aligned with the world axis).
This problem can be formulated as the following NLP, where the variable \(x \in \mathbb{R}^7\) represents the joint angles of the robot.
\begin{subequations}
	\label{eq:top-grasp}
	\begin{align}
		\min_{x \in \mathbb{R}^7} \quad & || x - x_\text{ref} ||^2 \,,                                       \\
		\text{s.t.} \quad               & p_z(x) - b_z  = b/2  \,,                                           \\
		                                & a/2 \leq p_x(x) - b_x  \leq a/2  \,,                               \\
		                                & R(x) = I_d  \,,                                                    \\
		                                & p_y(x) - b_y  = 0  \,,                                             \\
		                                & q_{\text{lb}} \leq x \leq q_{\text{ub}}  \,,                       \\
		                                & \text{sdf}(P_j(x), \text{env} ) \geq 0 ~     &  & j  = 1, \ldots,J \\
		                                & \text{sdf}(P_j(x), \text{block} ) \geq 0 ~   &  & j = 1, \ldots, J
	\end{align}
\end{subequations}
where \( p(x)  : \mathbb{R}^7 \to \mathbb{R}^3 \) is the position of the end-effector as a function of the joint values, and \( R(x) \) is the rotation matrix representing the orientation of the end-effector.
Here, we assume that a successful grasp requires the end-effector to have the same orientation as the box, thus \(R(x) = I_d\), but more complex grasp constraints are used later throughout the thesis.
Vectors \( q_{\text{lb}} \) and \( q_{\text{ub}} \) are the lower and upper bounds of the joint angles.
The position of the block is given by \( (b_x, b_y, b_z) \), with \( a \) and \( b \) denoting its length and height.
\( x_{\text{ref}} \) is a reference configuration for the robot.

\( P_j(x) \) represents the collision shape of part \(j\) of the robot, the position and orientation of which depend on the joint values \( x \).
The robot has \( J \) collision parts (often, one per link).
The signed-distance function, denoted with \textup{sdf}, is a function that returns the minimum distance between the collision shape and either the environment (env) or the block (block).
If the objects are in collision, \textup{sdf} returns a negative value corresponding to the penetration distance.
A solution to this NLP is depicted in \cref{fig:top-grasp}.

\begin{figure}
	\centering
	\includegraphics[width=.2\textwidth]{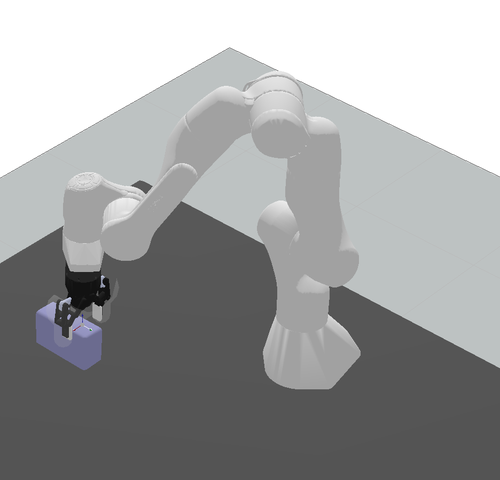}
	\caption{Solution to the NLP in \cref{eq:top-grasp}.}
	\label{fig:top-grasp}
\end{figure}

\emph{Trajectory Optimization:}
In the second example, our goal is to generate a trajectory for grasping the object starting from the configuration \( q_0 \).
We choose to parameterize the trajectory using \( N \) waypoints.
Our variable is now \( x = [q_1 , \ldots , q_N ] \in \mathbb{R}^{7 \cdot N} \), where \( q_i \in \mathbb{R}^7 \) is the joint values at waypoint \( i \).
The optimization problem is:
\begin{subequations}
	\label{eq:top-grasp-trajectory}
	\begin{align}
		\min_{ [q_1 , \ldots , q_N ] \in \mathbb{R}^{7\cdot N}} \quad &
		|| q_1 - q_{0} ||^2 + \sum_{i=2}^N || q_i - 2 q_{i-1} + q_{i-2} ||^2 \,,                                                                                 \\
		\text{s.t.
		} \quad                                                       & p_z(q_N) - b_z = b/2 \,,                                                                 \\
		                                                              & a/2 \leq p_x(q_N) - b_x \leq a/2 \,,                                                     \\
		                                                              & R(q_N) = Id \,,                                                                          \\
		                                                              & p_y(q_N) - b_y = 0 \,,                                                                   \\
		                                                              & q_{\text{lb}} \leq q_i \leq q_{\text{ub}} \,, & i = 1,\ldots, N                          \\
		                                                              & \text{sdf}(P_j(q_i), \text{env} ) \geq 0 \,,  & i = 1,\ldots, N &  & j = 1 , \ldots, J   \\
		                                                              & \text{sdf}(P_j(q_i), \text{block}) \geq 0 \,. & i = 1,\ldots, N &  & j = 1 , \ldots , J
	\end{align}
\end{subequations}
The chosen cost function minimizes the sum of squared accelerations, computed using second-order backward finite differences.
The constraints on the last configuration \( q_N \) are the same as in the previous NLP \eqref{eq:top-grasp}.
Joint limits and collisions are assessed at each waypoint.
In this example, we only evaluate collisions at each waypoint, but collisions could also be assessed at intermediate trajectory points, which can be determined through linear interpolation.
The solution is shown in \cref{fig:top-grasp-trajectory}.

\begin{figure}
	\centering
	\includegraphics[width=.19\textwidth]{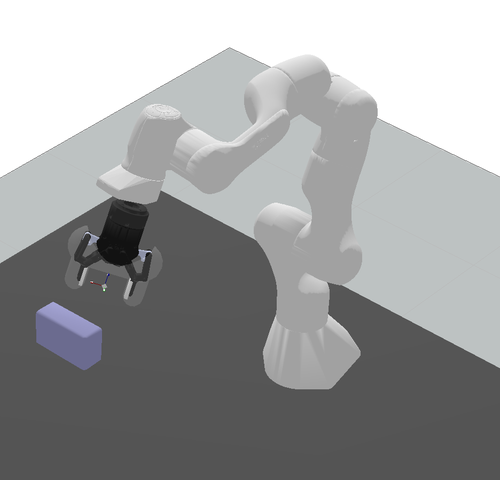}
	\includegraphics[width=.19\textwidth]{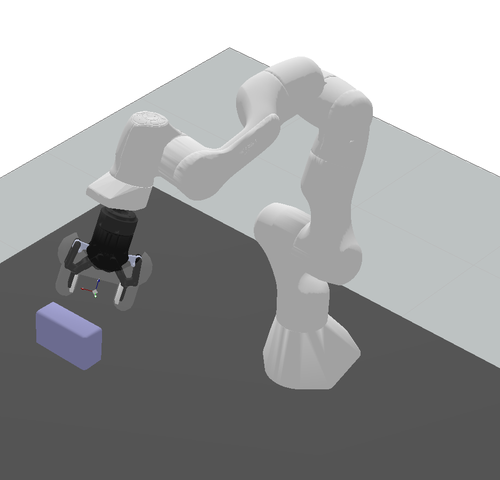}
	\includegraphics[width=.19\textwidth]{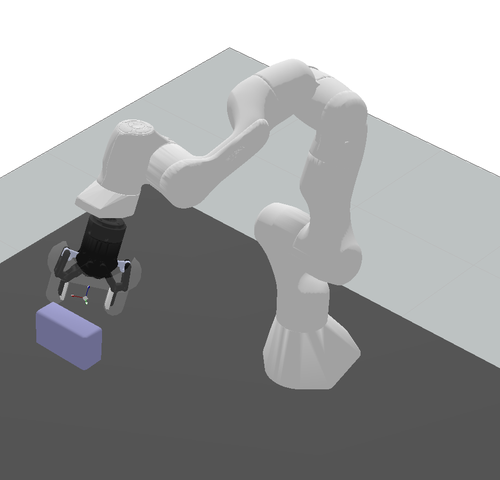}
	\includegraphics[width=.19\textwidth]{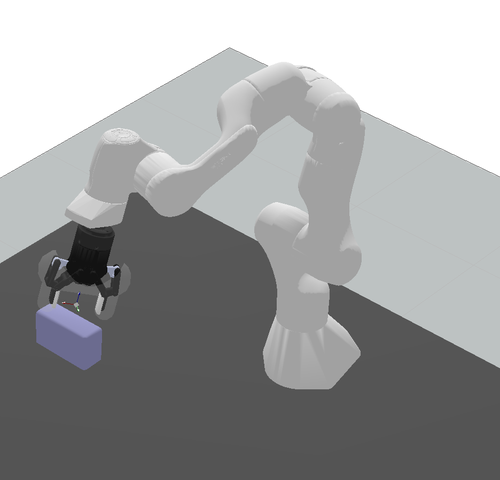}
	\includegraphics[width=.19\textwidth]{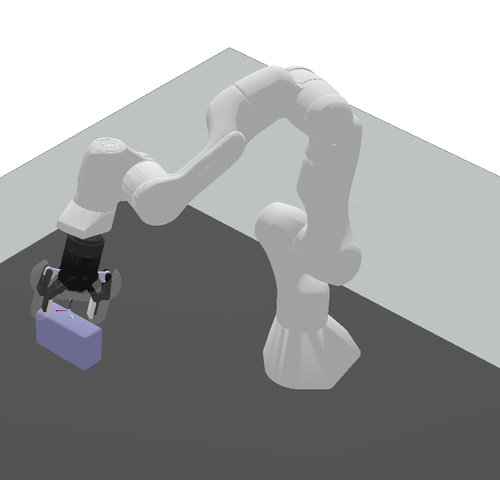}
	\includegraphics[width=.19\textwidth]{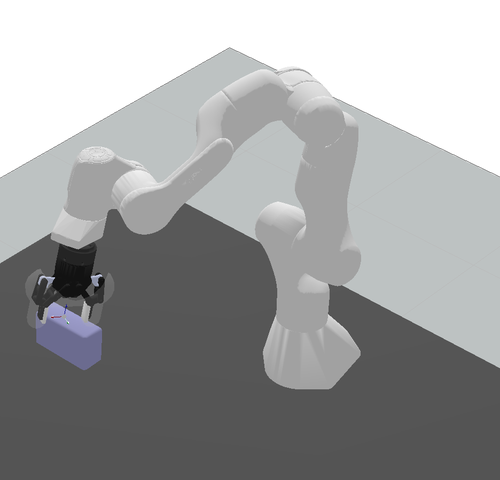}
	\includegraphics[width=.19\textwidth]{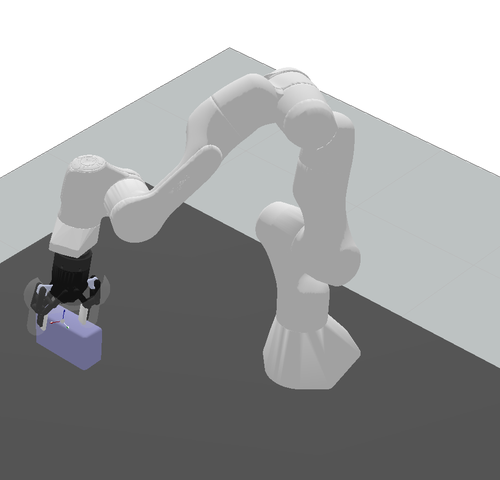}
	\includegraphics[width=.19\textwidth]{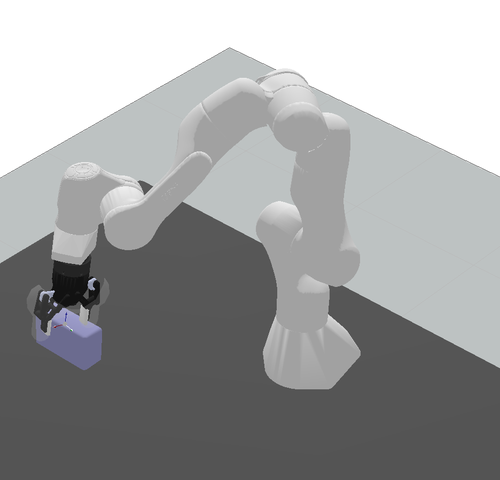}
	\includegraphics[width=.19\textwidth]{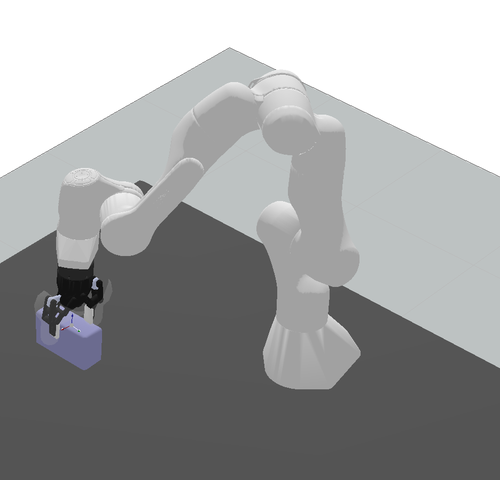}
	\includegraphics[width=.19\textwidth]{figs_old_computer/background/0021.png}
	\caption{Solution to the NLP in \cref{eq:top-grasp-trajectory} using \num{10} waypoints (\( N=10 \)).
		Each image shows a waypoint of the trajectory (from top left to bottom right).
	}
	\label{fig:top-grasp-trajectory}
\end{figure}

\newpage
\section{Classical Planning}
\label{bg:sec:planning}

Classical planning involves finding a sequence of actions to achieve a goal from an initial state.
It assumes that states and actions are discrete and finite, the state is fully observable, and actions have known deterministic effects.

A classical planning model \( \Pi = \left\langle \mathcal{S}, s_0, \mathcal{S}_G, \mathcal{A} \right\rangle \) comprises:
\begin{itemize}
	\item A finite and discrete set of states \( \mathcal{S} \), representing the state space.
	\item An initial state \( s_0 \in \mathcal{S} \).
	\item A set of goal states \( \mathcal{S}_G \subseteq \mathcal{S} \).
	\item A set of actions \( \mathcal{A} \).
	      The subset of actions applicable in state \( s \) is given by \( \mathcal{A}(s) \subseteq \mathcal{A} \), and executing action \( a \) in state \( s \) results in the successor state \( s' = \text{succ}(s,a) \).
\end{itemize}

A feasible solution to a planning problem is a sequence \( s_0, a_1, s_1, \ldots, a_K, s_K \) that transforms the initial state \( s_0 \) into a goal state \( s_K \in \mathcal{S}_G \), where \( a_k \in \mathcal{A}(s_{k-1}) \) and \( s_k = \text{succ}(s_{k-1}, a_k) \).
The optimal solution minimizes the number of actions (assuming each action has a uniform cost).

In large problems, explicitly enumerating the state space is not feasible.
In such cases, we use factored representations, in which states are factored into variables.
A state is now represented as a complete value assignment to a set of variables with finite and discrete domains.
The set of applicable actions \(\mathcal{A}(s)\) and the successor function \(\text{succ}(s,a)\) are now defined in terms of conditions and effects on these variables.

In this thesis, we use the SAS+ encoding of the classical planning problem, as referenced by \cite{backstrom-nebel-compint1995}.
This provides a more compact and intuitive representation for TAMP problems than the original STRIPS formulation \cite{fikes1971strips}.
A notable difference is that in SAS+, variables have a discrete, finite domain, whereas in STRIPS, they are boolean.

A factored classical planning task is a tuple \( \ptaskParSTD \) where:

\begin{itemize}
	\item \( \vars \) is a set of state variables.
	      Each state variable \( \var \in \vars \) has a finite domain \( \domain(\var) \).
	      A fact is a pair \(\langle \var, \val \rangle \) of a variable \( \var \in \vars \) and its value \( \val \in \domain(\var) \).
	\item An assignment to all the variables in \(\mathcal{V}\) is called a state \(s\).
	      The set of all such states is denoted as \( \mathcal{S} \).
	      A partial state \( p \) is a value assignment to only a subset of the variables in \( \vars \).
	      We view a partial state \( p \) as a set of facts (i.e., a set of variable-value pairs) and use \( p[\var] \) to denote the value of variable \( \var \) in \( p \) (i.e., \( p[\var] = \val \) if and only if \( \langle \var, \val \rangle \in p \)).
	      For any partial state \( p \), \( \variables{p} \subseteq \vars \) indicates the state variables instantiated by \( p \).
	      A partial state \( p \) is consistent with a state \( s \) if \( p \subseteq s \).
	\item \( \ops \) is a set of actions.
	      Each action \( \action \) is a pair of partial states called preconditions \( \pre(\action) \) and effects \( \eff(\action) \).
	      An action \( \action \) is valid in a state \( \state \) if \( \pre(\action) \subseteq \state \).
	      The set of all applicable actions in \( \state \) is given by \( \AA(s) \).
	      Applying \( \action \) on state \( s \) results in the next state \( \state' = \text{succ}(\state, \action) \), where the value of variables \( \var \in \variables{\eff(\action)} \) has changed to \( \eff(\action)[\var] \).
	\item An initial state \( s_0 \).
	\item A partial state \( g \) that defines the goal.
\end{itemize}

Similarly to the unstructured case, an action sequence
\(\pi = \langle a_1, \ldots, a_K \rangle\) is a valid plan if each action is applicable in the previous state (\(s_k = \text{succ}(s_{k-1}, a_k)\),
starting from \(\init\)),
and the final state satisfies the goal, that is \(\goal \subseteq s_K\).

The Planning Domain Definition Language (PDDL) \cite{mcdermott1998pddl} has become a standard for defining planning problems, as it is the language for the International Planning Competitions (IPC)
(e.g., \cite{long20033rd}, \cite{coles2012survey}).
In PDDL, a planning problem is defined as a pair comprising a planning domain and a problem instance.
The domain defines a class of problems by specifying the set of valid action schemas and predicates in this domain.
The problem instance defines the objects, on which predicates are evaluated, the initial state, and the goal state.
For solving a PDDL problem, classical planners first transform the action schemas and objects into a propositional representation like STRIPS or SAS+.

The classical planning community has developed domain-independent planners that leverage the factored representation of the planning problem.
A prominent approach is heuristic search, where the distance to the goal is estimated by solving a relaxed (simplified) version of the original problem.
The relaxed problems, often based on the so-called delete relaxation, can be solved efficiently in polynomial time on the number of variables (instead of exponential) \cite{bonet2001planning, helmert2006fast, hoffmann2001ff}.

Alternative approaches for solving classical planning problems include reductions to other formalisms such as Boolean Satisfiability (SAT) \cite{biere2009handbook} and Constraint Satisfaction Problems (CSP) \cite{rossi2006handbook}.

\paragraph{Blocksworld}
\label{sec:bg:pddl}

\begin{figure}[t]
	\centering
	\includegraphics[width=0.5\textwidth]{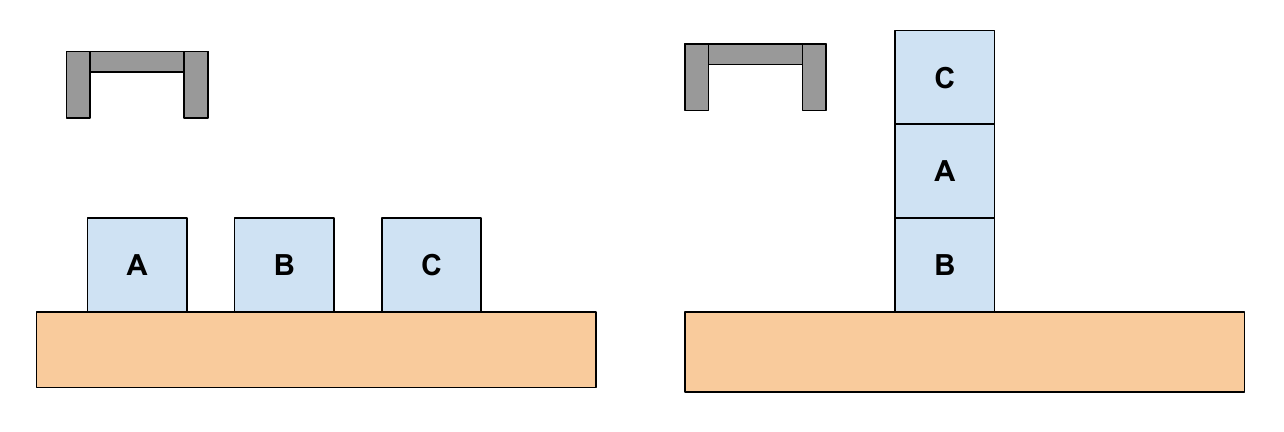}
	\caption{Start and goal states in the Blocksworld problem.
		PDDL files are shown in \cref{fig:blocks-world-all}.
		\vspace{.5cm}}
	\label{fig:blocks-world-goal}
\end{figure}

\begin{figure}[t]
	\centering
	\includegraphics[width=0.18\textwidth]{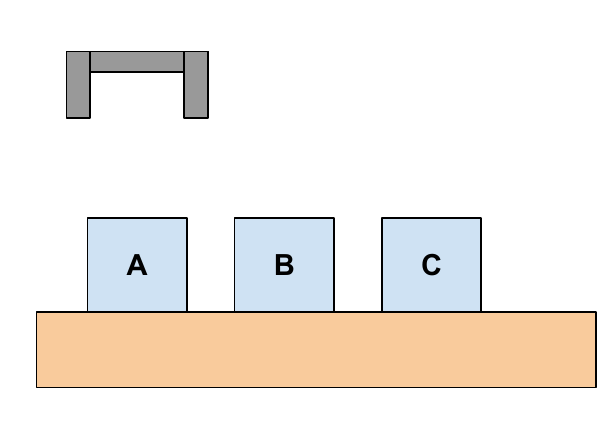}
	\includegraphics[width=0.18\textwidth]{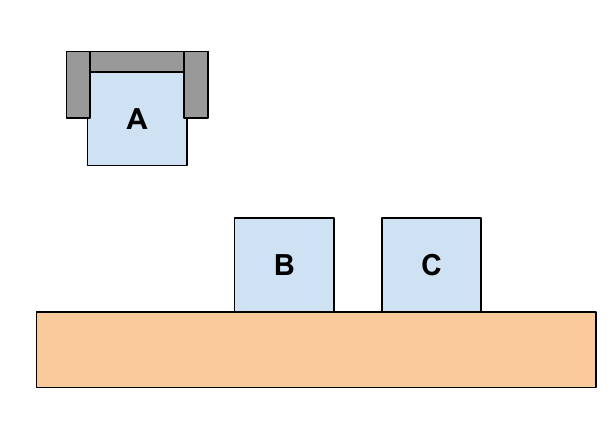}
	\includegraphics[width=0.18\textwidth]{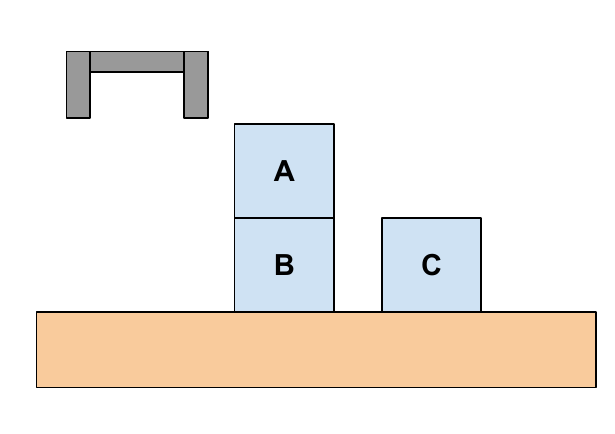}
	\includegraphics[width=0.18\textwidth]{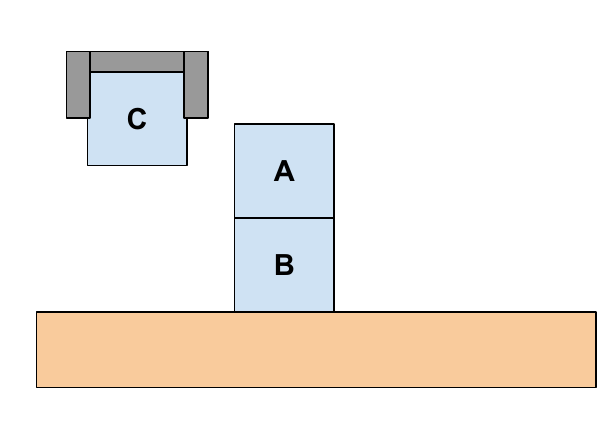}
	\includegraphics[width=0.18\textwidth]{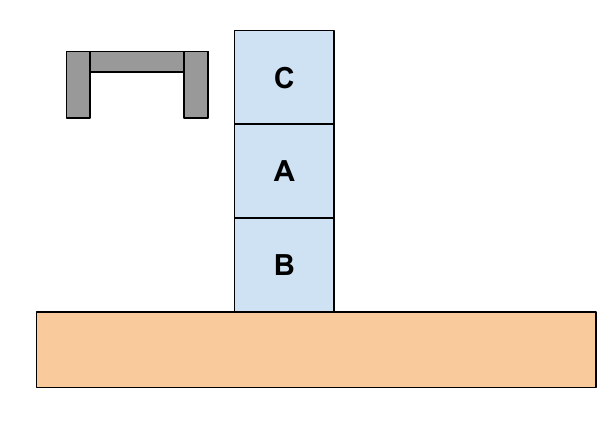}
	\caption{Solution to the Blocksworld problem in \cref{fig:blocks-world-goal}.}
	\label{fig:blocks-world-solution}
\end{figure}

Blocksworld is a classical planning problem that involves a set of blocks on a table, where the goal is to build a particular stack of blocks.

The Blocksworld PDDL domain file is shown in \cref{fig:blocks-world-domain}.
An example problem instance is shown in \cref{fig:blocks-world-problem} (file) and \cref{fig:blocks-world-goal} (graphical visualization).
The solution to this problem is shown in \cref{fig:blocks-world-solution}.

Blocksworld is a high-level planning problem that only deals with logical relationships between blocks.
On the other hand, Task and Motion Planning (TAMP) in robotics goes beyond abstract reasoning to include physical constraints.
TAMP considers both the high-level discrete actions (like picking up an object) and the low-level continuous motions required to perform those tasks (such as path planning, reachability, and collision avoidance).
This makes TAMP much more complex and connected to the real physical world, whereas Blocksworld remains a more abstract and simplified problem.

The success of classical planning in solving large-scale problems stems from the analysis of the factored structure inherent in most planning problems.
This factorization, exemplified by the Blocksworld domain, is also readily available in Task and Motion Planning (TAMP) problems and will be exploited throughout this thesis (see \cref{sec:bg:structure} for a first introduction).

\begin{figure}[!ht]
	\centering
	\begin{subfigure}[b]{.9\textwidth}
		\begin{lstlisting}
(define (domain blocksworld)
  (:predicates (on ?x ?y) (ontable ?x) (clear ?x)
               (handempty) (holding ?x))
  (:action pick-up
   :parameters (?x)
   :precondition (and (clear ?x) (ontable ?x) (handempty))
   :effect (and (not (ontable ?x)) (not (clear ?x))
                (not (handempty))  (holding ?x)))
  (:action put-down
   :parameters (?x)
   :precondition (holding ?x)
   :effect (and (not (holding ?x)) (clear ?x)
                (handempty) (ontable ?x)))
  (:action stack
   :parameters (?x ?y)
   :precondition (and (holding ?x) (clear ?y) (not (= ?x ?y)))
   :effect (and (not (holding ?x)) (not (clear ?y)) (clear ?x)
                (handempty) (on ?x ?y)))
  (:action unstack
   :parameters (?x ?y)
   :precondition (and (on ?x ?y) (clear ?x) (handempty) (not (= ?x ?y)))
   :effect (and (holding ?x) (clear ?y) (not (clear ?x))
                (not (handempty)) (not (on ?x ?y))))
)
\end{lstlisting}
		\caption{Domain.}
		\label{fig:blocks-world-domain}
	\end{subfigure}
	\vspace{1cm}

	\begin{subfigure}[b]{.9\textwidth}
		\begin{lstlisting}
(define (problem blocksworld-problem)
  (:domain blocksworld)
  (:objects a b c )
  (:init (handempty) (ontable a) (ontable b) (ontable c)
         (clear a) (clear b) (clear c))
  (:goal (and (clear c) (ontable b) (on c a) (on a b)))
)
\end{lstlisting}
		\caption{Example of a problem instance.}
		\label{fig:blocks-world-problem}
	\end{subfigure}
	\caption{Blocksworld in PDDL.}
	\label{fig:blocks-world-all}
\end{figure}

\clearpage

\section{Logic Geometric Programming}
\label{sec:bg:lgp}

\textit{Can we formulate TAMP as a continuous optimization problem?}
Logic Geometric Programming (LGP) is an optimization-based formulation of Task and Motion Planning (TAMP).
To motivate the LGP formulation, we first show that Task and Motion Planning can be written as a single continuous-time optimization program, which can later be discretized into a nonlinear program.
However, proceeding without introducing a discrete abstraction renders the problem unsolvable.

Let $\mathcal{X} = \mathbb{R}^n \times SE(3)^m$ be the configuration space of an $n$-dimensional robot and $m$ rigid objects, initially at pose $x_0 \in \mathcal{X}$.
The trajectory of the robot and the movable objects can be represented with a continuous function $x(t) : [0,T] \to \mathcal{X}$, with $T$ being the terminal time.
A Task and Motion Planning problem is then formulated as a continuous trajectory optimization with,
\begin{subequations}
	\label{eq:only-low-level}
	\begin{align}
		\min_{x(t), T} & \int_0^T f_{\text {path }}(\bar{x}(t)) d t \,,                                                      \\
		\text { s.t.
		}              & x(0)=x_0,                                                                                           \\ & h_{\text {goal }}(x(T))=0, ~ g_{\text {goal }}(x(T)) \leq 0,                       \\
		               & h_{\text {path}} ~ (\bar{x}(t))=0,~ g_{\text {path }}(\bar{x}(t)) \leq 0, \quad \forall t \in[0, T]
	\end{align}
\end{subequations}
where $\bar{x}(t) = (x(t), \dot{x}(t) , \ddot{x}(t))$ includes the position, velocity, and acceleration of the trajectory.
Path constraints $h_{\text{path}}$ and $g_{\text{path}}$ represent the physical constraints of the physical world.
For instance, objects can only move when grasped or pushed by the robot, and objects and robots should not collide with the environment.
Goal constraints $\{h,g\}_{\text{goal}}$ impose constraints on the last state and are used to represent the desired final state, such as stacking objects in a tower.
The cost function $f_{\text{path}}$ typically represents control effort, smoothness, or energy.

Converting the continuous-time optimization to an NLP requires a finite-dimensional representation of the trajectory.
For instance, this could be achieved using a finite sequence of waypoints, and checking the path constraints on a finite set of points (e.g., the waypoints themselves), similarly to the example in \cref{eq:top-grasp-trajectory}.

Unfortunately, this NLP formulation is unsolvable, even for short-term horizons and few objects in the scene.
The two main challenges are as follows:
\begin{enumerate}
	\item The non-convexity of the constraints used to model all possible physics interactions (e.g., using complementarity constraints \cite{posa2014direct}) leads to disconnected and non-convex feasible sets.
	\item The lack of meaningful gradients, which also stems from the generic nature of the constraints.
	      If the initial guess for the trajectory does not interact with movable objects, the derivatives of the path and goal constraints either lack relevant information
	      for long-term planning, or only greedily move the robot towards every movable object.
\end{enumerate}
Therefore, a local optimization method will get trapped in an infeasible local optimum and fail to find a feasible solution to \eqref{eq:only-low-level}.
In fact, we observe that general complex behavior that requires reasoning about interacting with multiple objects cannot be computed using only local optimization methods.

\begin{figure}[!t]
	\centering
	\begin{subfigure}[b]{.14\textwidth}
		\includegraphics[width=\textwidth]{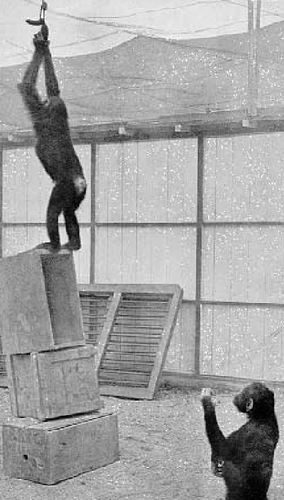}
		\caption{}
	\end{subfigure}
	\begin{subfigure}[b]{.26\textwidth}
		\includegraphics[width=\textwidth]{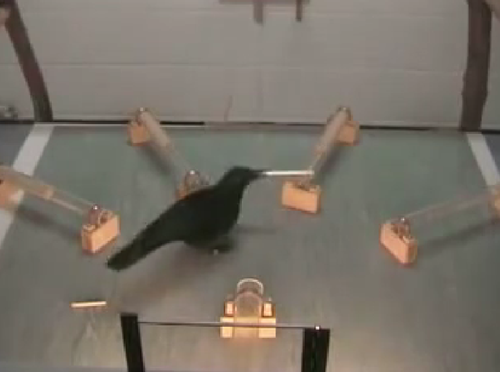}
		\caption{}
	\end{subfigure}
	\begin{subfigure}[b]{.26\textwidth}
		\includegraphics[width=\textwidth]{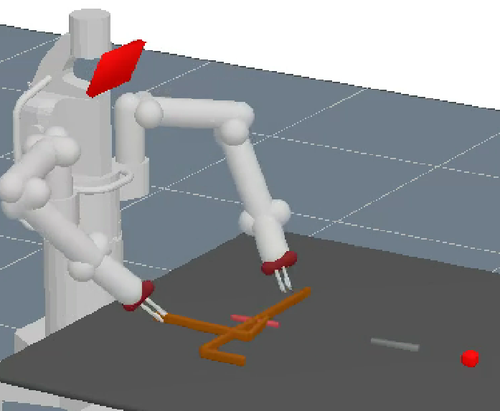}
		\caption{}
	\end{subfigure}
	\begin{subfigure}[b]{.26\textwidth}
		\includegraphics[width=\textwidth]{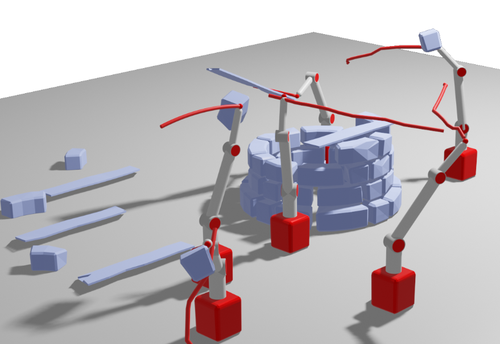}
		\caption{}
	\end{subfigure}
	\caption{Logic Geometric Programming. \textit{(a,b)}
		Apes and crows solving physical puzzles that require precise manipulation and long-term planning of physical interactions.
		\textit{(c)}
		In robotics, LGP can be used to model tool-use, and diverse physical interactions, such as grasping, pushing, and throwing.
		\textit{(d)}
		The LGP framework and the original Multi-Bound Tree Search algorithm have been extended to solve problems in the construction domain \cite{hartmann2022long}.
		Images are reproduced from \cite{toussaint2018differentiable, hartmann2022long}.
	}
	\label{fig:lgp:showcase}
\end{figure}
\emph{High-level and low-level abstractions in TAMP}:
Introducing a high-level abstraction that defines the sequence of interactions between the robot and the objects is essential to solving the TAMP problem.

In fact, the problem can be formulated with two levels of abstraction: low-level motion planning and high-level task planning.
The high-level task planning can be formulated as a classical planning task, introducing a set of discrete states and actions.
But now, in addition to planning in a discrete domain, choices on the discrete level impose nonlinear constraints on the continuous trajectory, with a different sequence of high-level actions implying a different set of nonlinear constraints on the trajectory.

Therefore, instead of a single universal physics constraint, we now operate with a set of more approachable nonlinear constraints based on each discrete state and action.
Though these constraints remain nonlinear, they tend to be smooth and informative, and typically guide the optimizer towards feasible solutions, as demonstrated by the diverse set of TAMP problems solved using the LGP formulation \cite{toussaint2017multi, toussaint2018differentiable} (\cref{fig:lgp:showcase}).

\paragraph{Logic Geometric Program}
Formally, consider a continuous configuration space $\mathcal{X}$ (e.g., $\mathbb{R}^n \times SE(3)^m$ for an $n$-dimensional robot and $m$ rigid objects), a finite set of discrete states $\mathcal{S}$, and a finite set of discrete actions $\mathcal{A}$.
Let $x_0 \in \mathcal{X}$ be the initial configuration, $s_0 \in \mathcal{S}$ be the initial discrete state, and $\mathcal{S}_g \subseteq \mathcal{S}$ be the set of discrete goal states.
A Logic Geometric Program is a combined optimization problem over the continuous trajectory $x(t):[0,KT]\to \mathcal{X}$ and the sequence of discrete states and actions $s_0, a_1, s_1, \ldots, a_K, s_K$,
\begin{subequations}\label{eq:lgp}
	\begin{align}
		\min _{x(t),a_{1:K},s_{1:K}} & \sum_k \int_{t_k}^{t_{k+1}} f_{\text {path }}(\bar{x}(t), s_k) \, dt\,,                                   \\
		\text {s.t.
		} ~                          & x(0)=x_0 \,,                                                                                              \\
		                             & h_{\text {path }}\left(\bar{x}(t), s_{k}\right)=0\,, \quad t\in[t_k, t_{k+1}],      & k = 0, \ldots , K-1 \\
		                             & g_{\text {path }}\left(\bar{x}(t), s_{k}\right) \leq 0\,,\quad t\in[t_k, t_{k+1} ],
		                             & k = 0, \ldots , K-1                                                                                       \\
		                             & h_{\text {switch }}\left(\bar{x}\left( t_k \right), s_{k-1}, s_k \right)=0\,,       & k = 1,\ldots, K     \\
		                             & g_{\text {switch }}\left(\bar{x}\left( t_k \right), s_{k-1}, s_k \right) \leq 0\,,  & k = 1,\ldots, K     \\
		                             & s_k = \operatorname{succ}\left(s_{k-1}, a_k\right)\,,                               & k = 1,\ldots,K      \\
		                             & s_k \in \mathcal{S}, \quad a_k \in \mathcal{A},                                     & k = 1,\ldots,K      \\
		                             & s_K \in \mathcal{S}_g \,.
	\end{align}
\end{subequations}
Here, $\bar{x}(t) = \left(x(t), \dot{x}(t), \ddot{x}(t)\right)$, $t_k = Tk$ is the start time of step $k$, and $s_{1:K}, a_{1:K}$ are short notations for $\langle s_1,\ldots, s_K \rangle$ and $\seqa$.
The functions $(h, g)_{\text{path}}$, $(h, g)_{\text{switch}}$, and $f_{\text{path}}$ are smooth and differentiable in the continuous configuration when conditioned on the discrete states.
The discrete function \text{succ}$(s, a)$ indicates the successor discrete state after applying action $a$ to state $s$.
The length of the sequences $K \in \mathbb{N}$ is also subject to optimization.
We remark that the discrete component of an LGP corresponds to a classical planning problem \cref{bg:sec:planning}.

The sequence $a_{1:K}$ is called a task plan, action-skeleton, or sequence of high-level actions in the TAMP literature.
Throughout this thesis, we use the term \emph{task plan}.
The continuous motion $x(t)$ is now divided into $K$ phases of duration $T \in \mathbb{R}$ (for simplicity, we assume $T$ is fixed, but it could also be optimized), with different nonlinear constraints on the configuration for each phase, based on $s_k$ (or the pair $s_{k-1},s_k$ for switch constraints).

The key difference from the unstructured problem \eqref{eq:only-low-level} is that the nonlinear functions $(f,g,h)$ are now conditioned on the discrete state $s_k$ and provide informative gradients for local optimization, as opposed to using a unique universal constraint based on complementarity or differentiable contact models.
This means that the constraints $h_{\text{switch}}$ and $g_{\text{switch}}$ will be different functions of $\bar{x}$ depending on which object is grasped, pushed, or placed, as indicated by the discrete states.

\paragraph{Continuous space and discrete space in LGP}

Before delving into solving the Logic Geometric Program \eqref{eq:lgp}, we will clarify the meanings of the continuous space, discrete space, and the nonlinear constraints in the context of TAMP.

\emph{Discrete Level}: The discrete states $s \in \mathcal{S}$ in LGP are used to encode the structure of the kinematic tree, which models which objects are in contact with each other.
For instance, an object can be attached either to the gripper or to the table.
This information is discrete because it only contains a parent's name or identifier, but it does not consider the continuous relative pose between them.
Apart from special discrete symbols to represent the initial position of the objects, LGP does not introduce additional discrete symbols to represent the continuous state of the world, such as concrete values of the positions of the objects, grasps, or robot configurations, as done in other sample-based TAMP formulations (e.g., \cite{garrett2020pddlstream}).
Equivalently, we can think of the discrete states as modeling the discrete contact states of the world, specifying which objects are in contact with others.
From the perspective of multimodal motion planning, discrete states in an LGP represent the different motion modes.

Discrete actions $a \in \mathcal{A}$ are the high-level actions of the task plan, such as pick, place, push, which change the kinematic structure or contact status.
Discrete actions do not include any continuous parametrization.
For instance, a discrete action \textit{pick object A with Q from table} does not model the continuous grasp, e.g., the relative transformation between the gripper and the object.

\emph{Continuous Level}: The continuous configuration $\mathcal{X} = \mathbb{R}^n \times SE(3)^m$ represents the configuration of the robot (joint values) and the objects (position and orientation).
Collision avoidance, reachability, grasping, pushing, and placement constraints are modeled using the nonlinear constraints inside $(h,g)_{(\text{path}, \text{switch})}$, which vary in each motion step depending on the discrete decisions.

For instance, concerning grasping constraints, we use geometric constraints, such as aligning the end-effector with a particular axis of a box or ensuring that a point near the end-effector's palm touches the object's surface while maintaining the correct orientation.
In practice, for boxes, balls, and sticks, this typically implies the existence of a stable grasp, which is then represented as a constant relative transformation until placement.
For pushing constraints, we introduce additional decision variables for the forces and point-of-attack between the two interacting objects.
Motion and forces are then constrained by physics equations as outlined in \cite{20-toussaint-RAL}.

\newpage
\paragraph{How to solve a Logic Geometric Program?}
If we fix the task plan $\seqa$ in a LGP, the resulting subproblem, denoted as $\textup{Trajectory-NLP}(\langle a_1,\ldots,a_K \rangle)$, is a continuous optimization problem:
\begin{subequations}\label{eq:lgp-nlp}
	\begin{align}
		\min_{x(t)} & \sum_k \int_{t_k}^{t_{k+1}} f_{\text{path}}(\bar{x}(t), s_k) dt \,,                                        \\
		            & \text{s.t.
		} \nonumber                                                                                                              \\
		            & x(0)=x_0,                                                                                                  \\
                & h_{\text{path}}\left(\bar{x}(t), s_{k}\right)=0, \quad t \in[t_k, t_{k+1}], & & k = 0, \ldots , K-1,     \\
                & g_{\text{path}}\left(\bar{x}(t),  s_k\right) \leq 0, \quad  t \in[t_k, t_{k+1}], & & k = 0, \ldots , K-1, \\
                & h_{\text{switch}}\left(\bar{x}\left(t_k\right), s_{k-1},s_k \right)=0, & & k = 1,\ldots,K,               \\
                & g_{\text{switch}}\left(\bar{x}\left(t_k\right), s_{k-1}, s_k \right) \leq 0, & & k = 1,\ldots,K.
	\end{align}
\end{subequations}
where $\seqs$ is uniquely defined by $\seqa$ and the fixed $s_0$ (using the discrete successor function).
This optimization can be converted directly into an NLP by using a finite-dimensional representation of the trajectory, and thus, we refer to $\eqref{eq:lgp-nlp}$ directly as an NLP.
Throughout this thesis, we represent the trajectory with a finite set of waypoints, similar to the example NLP in \cref{eq:top-grasp-trajectory} of \cref{sec:bg:nlp}.

A task plan $\seqa$ is said to be geometrically infeasible when \eqref{eq:lgp-nlp}
is infeasible, i.e., it has no feasible solution.
This frequently occurs in TAMP, where many candidate high-level plans define constraints for the motion that can never be satisfied.
For instance, the task plan $\langle\textit{pick object A with robot Q from A\_init}\rangle$ fails if the object is too far away or if there is an obstacle blocking the grasp.
A task plan $\langle$\textit{pick object A with robot Q from A\_init}, \textit{place object A with robot Q on red table}$\rangle$ can fail if the table is out of reach or other objects are already on the table.
These two examples of motion failures are shown in \cref{fig:examples-fail}.

\begin{figure}[t]
	\centering
	\begin{subfigure}{.9\textwidth}
		\centering
		\includegraphics[width=.8\textwidth]{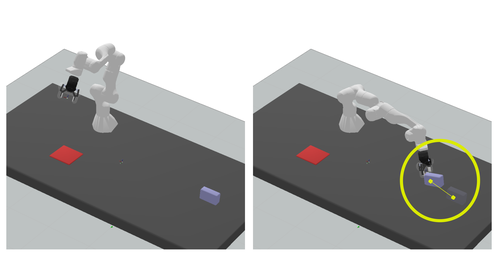}
		\caption{Trajectory-NLP for the task plan $\langle\textit{pick object A with robot Q from A\_init}\rangle$ is infeasible.}
	\end{subfigure}\\
	\begin{subfigure}{.9\textwidth}
		\centering
		\includegraphics[width=.9\textwidth]{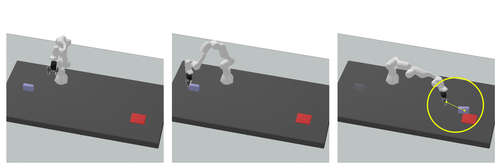}
		\caption{Trajectory-NLP for the task plan $\langle$\textit{pick object A with robot Q from A\_init}, \textit{place object A with robot Q on red table}$\rangle$ is infeasible.}
	\end{subfigure}
	\caption{Two examples of infeasible Trajectory-NLPs for two task plans in two example environments.
		The yellow circle highlights why the motion is infeasible.
		In \emph{(a)}, the block suddenly ``jumps'' from the start configuration to the gripper.
		In \emph{(b)}, the block ``jumps'' from the gripper to the red table.
	}
	\label{fig:examples-fail}
\end{figure}

We can define relaxed versions of $\textup{Trajectory-NLP}(\seqa)$ that can be used to quickly test the feasibility while being computationally simpler.
Relaxed problems remove constraints from the original problem and act as a lower bound.
Thus, if the relaxed problem is infeasible, the original problem is also infeasible.

The \textup{Keyframes-NLP} (\textit{keyframes or sequence bound}) optimizes only a single configuration per phase instead of a continuous trajectory.

The optimization variables are $\{x_k \equiv x(t_k) \mid k=1,\ldots,K \}$, which are optimized jointly, accounting for their interdependencies but without considering the continuous path between them.
That is, we only evaluate the constraints at the beginning and end of each phase.
This sequence of discrete configurations is called \textit{keyframes} (the term we use in this thesis) or mode-switches and is very informative in manipulation planning, as keyframes are usually the more constrained configurations and capture with high accuracy whether a high-level task plan is geometrically feasible.
The $\textup{Keyframes-NLP}(\seqa)$ is:
\begin{subequations}\label{eq:keyframesbound}
	\begin{align}
		\min_{ x_1,\ldots,x_K } & \sum_{k} \tilde{f}_0( x_k ) + \tilde{f}_1(x_{k-1}, x_k)\,,                          \\
		\text{s.t.
		}                       & \tilde{h}_{\text{path}} ( x_k , s_k )  = 0,                   & k = 1, \ldots, K    \\
		                        & \tilde{g}_{\text{path}} ( x_k , s_k ) \leq 0,                 & k = 1, \ldots, K    \\
		                        & \tilde{h}_{\text{switch}} (x_k, x_{k+1},s_k, s_{k+1}) = 0,    & k =  0, \ldots, K-1 \\
		                        & \tilde{g}_{\text{switch}} (x_k, x_{k+1},s_k, s_{k+1}) \leq 0, & k =  0, \ldots, K-1
	\end{align}
\end{subequations}
where $(\tilde{h}, \tilde{g})_{(\text{path},\text{switch})}$ model the path and switch constraints but are evaluated only on the keyframe configurations instead of the full trajectory.

Importantly, in the Keyframes-NLP, we optimize the full manipulation sequence jointly.
This means that we can discover, for instance, grasp locations that are good for both the pick and place keyframe, picking in places that allow for a handover, or placing an object close to the other robot for later manipulation.

In practice, before solving \eqref{eq:lgp-nlp}, it is recommended to always solve first the Keyframes-NLP \eqref{eq:keyframesbound}, and use its solution to warm-start the trajectory in \eqref{eq:lgp-nlp}.

A looser relaxation is to optimize a single keyframe for each phase independently, without considering the interdependencies between them.
The
$\textup{Pose-NLP}(\seqa)$ (\textit{pose bound}) is a set of $k=1,\ldots,K$ independent optimization problems, one for each keyframe $x_k$,
\begin{subequations} \label{eq:posebound}
	\begin{align}
		k=1,\ldots,K \quad & \min_{x_k} \tilde{f}_0( x_k , s_k)  \,,        \\
		                   & \text{s.t.} ~  x(0) = x_0\,,                   \\
		                   & \tilde{h}_{\text{path}} ( x_k, s_k ) = 0\,,    \\
		                   & \tilde{g}_{\text{path}} ( x_k, s_k ) \leq 0\,.
	\end{align}
\end{subequations}
This is the most computationally efficient relaxation to test the feasibility of a sequence of actions but is also the least informative, as it does not consider the interdependencies between the different motion steps.

\subsection{Multi-Bound Tree Search}
\label{sec:multibound}

Multi-Bound Tree Search (MBTS) \cite{toussaint2017multi} is a state-of-the-art approach to solve an LGP \eqref{eq:lgp}.
The discrete state and action space of the LGP formulation is explored with a search tree starting from $s_0$, where each branch represents a different sequence of discrete actions.
The leaf nodes $s_g \in S_g$ are solutions to the discrete planning problem and are potential candidates for a solution to the LGP problem.
Each node can be tested for feasibility by solving the continuous optimization
problem induced by the sequence of actions from the root to the current node.
Therefore, to find a solution to the LGP problem, we have to identify a leaf node $s_g \in S_g$
and find a feasible solution to its corresponding $\textup{Trajectory-NLP}(a_{1:K})$ (which might be infeasible for most candidate plans).

However, trying to solve the trajectory optimization is expensive, and the number of candidate NLPs is generally too high.
To alleviate this issue, MBTS solves relaxed versions of \eqref{eq:lgp} incrementally.
The feasibility of each relaxed problem is a necessary condition for the feasibility of the Trajectory-NLP, i.e., these act as lower bounds while being computationally faster.
Specifically, the two bounds are the \textit{keyframes bound} (Keyframes-NLP \eqref{eq:keyframesbound}) and the \textit{pose bound} (Pose-NLP \eqref{eq:posebound}), which consider only a subset of variables and constraints of the full trajectory optimization problem.

The search is organized around four queues using a simple round-robin policy to process the next element of each queue.
There is a queue for intermediate discrete states (explored in a breadth-first search order), two queues for computing the relaxations using the Keyframes-NLP and the Pose-NLP, and a queue for solving the Trajectory-NLP.

If a node in the tree reaches the goal, it is promoted to the relaxation queues.
If a node fulfills the two relaxations, it is moved to the trajectory optimization queue.
Additionally, we can also use the relaxations to prune the search tree.
If the pose, keyframes, or trajectory NLP of a candidate or intermediate node fails, we can identify the prefix that is infeasible and utilize it to prune sub-branches in the search tree.

\section{Related Work in Task and Motion Planning}
\label{sec:bg:related-work}

Solvers for Task and Motion Planning can be categorized based on two complementary criteria \cite{garrett2021integrated}.
First, based on how motion planning and task planning are combined, we distinguish between \emph{1) continuous-first} (where partial motions are first computed, then combined into a full TAMP solution), \emph{2) interleaved} (with simultaneous search at both the continuous and discrete levels), and \emph{3) discrete-first} (where candidate high-level task plans are computed first).

The second criterion focuses on the methods used to compute the motions.
These are primarily \emph{1) predefined discretization}, \emph{2) sampling}, or \emph{3) optimization methods}.
In practice, the two criteria are closely interconnected; solvers using \textit{optimization} typically employ a \textit{discrete-first} search, while most \textit{sampling} methods use \emph{continuous-first} or \emph{interleaved} search.

Prominent examples of sample-based methods are PDDLStream \cite{garrett2020pddlstream} and TAMP in the Now \cite{kaelbling2011hierarchical}.
TAMP in the Now adopts a hierarchical and interleaved search approach between the discrete and continuous levels, while
PDDLStream integrates constrained samplers for generating the continuous motion (e.g., grasps, collision-free paths, and inverse kinematic solutions) within PDDL-like planning.
On the other hand, the study in \cite{ferrer2017combined} showcases a pre-discretized continuous-first approach, where a set of valid configurations is integrated into task planning through precompilation.

Some sampling-based TAMP solvers reason explicitly about geometric conflicts.
For example, a set of predefined predicates such as ``is reachable'' is used in \cite{srivastava2014combined} to combine a black-box task planner with a motion planner.
The constraint-based approach in \cite{dantam2016incremental} incorporates information about geometric infeasibility by blocking the full task plan or, in special cases, a pair of a discrete (partial) state and an action.
Geometric infeasibility can also be evaluated efficiently using linear constraint propagation \cite{lagriffoul2014efficiently}.

In this work, we focus on optimization-based formulations of TAMP, where Logic Geometric Programming stands as a leading general formulation \cite{toussaint2015logic,toussaint2018differentiable}.
Optimization-based solvers leverage nonlinear optimization to compute motions that satisfy all geometric and physical constraints, while taking into account the interdependencies in the motion.
A state-of-the-art general solver for LGP is the Multi-bound Tree Search \cite{toussaint2017multi}, which combines discrete search with relaxed optimization problems to efficiently evaluate geometric feasibility.
Other optimization-based methods, e.g., \cite{migimatsu2020object,zhao2021sydebo,zimmermann2020multi,hadfield2016sequential}, address more specific TAMP settings (e.g., rearrangement) or a TAMP subproblem (e.g., only motion planning).

Task and Motion Planning can also be formulated as multimodal motion planning.
Indeed, the difference between TAMP and multimodal motion planning is mainly a naming convention used by different authors.
In practice, the concepts of \textit{motion-modes} and \textit{mode-transitions} in multimodal motion planning correspond to the high-level abstraction in TAMP problems.
The naming convention traditionally highlights a slight difference in target applications, with TAMP focusing more on planning long manipulation sequences with multiple objects, while multimodal planning emphasizes more on problems with shorter sequences with more challenging motion planning.

Multi-Modal-PRM, proposed in the seminal work \cite{hauser2010multi}, builds a probabilistic roadmap (PRM, \cite{kavraki1996probabilistic}) in different motion modes and connects these modes by sampling configurations belonging to two modes simultaneously, known as mode-switch configurations.
The original multimodal framework can be extended to problems with an infinite number of modes \cite{Hauser2011}, asymptotic optimal planning \cite{Vega2020}, and heuristics to bias the search towards useful mode transitions \cite{Kingston2020a}.

Problems similar to TAMP or multimodal motion planning also appear under a third distinct name in robotics literature: manipulation planning, which typically assumes contact modes using only a stable grasp (i.e., prehensile manipulation).

Similar to multimodal motion planning, most methods extend the tools of sample-based motion planning to manipulation problems.
For instance, manipulation planning can be formulated as a search over a sequence of transit paths (where the robot moves freely) and transfer paths (where the robot moves while holding an object), using probabilistic roadmaps \cite{Simeon2004}.
More recently, an asymptotically optimal manipulation planner has been proposed in \cite{schmitt2017optimal}, and the Manipulation-RRT \cite{lamiraux2021prehensile} extends the classical RRT algorithm (\cite{lavalle2001rapidly}) to plan across different contact and manipulation modes.
Some algorithms focus only on specific settings within manipulation planning, such as rearrangement planning \cite{Ota2004, krontiris2016efficiently, huang2019large}, or navigation among movable obstacles \cite{stilman2007manipulation}.

The multimodal nature of Task and Motion Planning (TAMP) arises from creating and breaking contacts with the environment.
Such problems also appear in legged locomotion, where mixed-integer formulations can be used to optimize foot placement, gait, and joint movement \cite{deits2014footstep}.
To avoid the explicit combinatorial search, an alternative approach is to use local optimization methods, which, in turn, raise concerns about local optima and feasibility.
Differentiable contact models \cite{todorov2011convex}, contact invariant optimization \cite{mordatch2012discovery}, nonlinear programming \cite{posa2014direct}, and convex relaxations \cite{song2021solving} have been used to optimize trajectories and contacts simultaneously for locomotion.

From a different perspective, and within a different research community, Task and Motion Planning can be formulated as a classical planning problem with additional numerical variables.
Classical AI planners have been extended to support planning with numerical constraints on action preconditions (e.g., Metric-FF \cite{koehler1998planning}), and recent versions of the Planning Domain Definition Language (PDDL) include temporal planning with numerical variables \cite{fox2006modelling, piotrowski2016heuristic, scala2016interval}.

For instance, the COLIN planner \cite{DBLP:journals/jair/ColesCFL12} includes continuous linear changes of numerical variables (e.g., fixed velocities) and encodes the temporal and state evolution constraints implied by a sequence of actions as a linear program.
The Scotty planner \cite{fernandez2018scottyactivity} adds support for general convex constraints, combining discrete search with convex optimization, and the planner in
\cite{haslum2018extending} extends classical planning with general state constraints.
However, these general planning formulations have not been used to tackle general TAMP problems, where the dimensionality and complexity of the continuous space pose significant challenges and often require tools from motion planning and nonlinear trajectory optimization.

%% file: factorized_structure.tex
\chapter{Factored Structure of Task and Motion Planning}
\label{sec:bg:structure}

In this chapter, we analyze the factored structure that appears in Task and Motion Planning (TAMP).
Specifically, we study the factorization of the nonlinear trajectory optimization problem for a fixed task plan, denoted as \(\textup{Trajectory-NLP}(\seqa)\) in \cref{sec:bg:lgp} (\cref{eq:lgp-nlp}).

A similar analysis of the factorization of the TAMP problem is foundational in modern sample-based TAMP solvers, where it is used to define effective conditional sampling operations \cite{garrett2018sampling, garrett2021integrated} and to propagate information about feasibility \cite{lagriffoul2014efficiently}.
In contrast, we analyze the factored structure from the perspective of optimization-based solvers for TAMP, providing complementary insights and ideas.

One of the contributions of this thesis is the formalization of TAMP problems using our novel planning formulation, \q{Planning with Nonlinear Transition Constraints}, which offers a factored view on LGPs.
While the formal definition, details, and comprehensive analysis will be presented later in \cref{ch:bid}, this chapter provides an approachable, intuitive view of the inherent structure that appears naturally in trajectory optimization for TAMP problems.

The chapter is organized as follows: First, we introduce the factored nonlinear program formulation.
Second, we discuss the factored optimization problem for the Pick and Place task plan, which will later be used as a building block in more sophisticated manipulation tasks.
Third, we examine more complex examples that showcase the main benefits of our factored representation, namely, the temporal structure, sparse factorization, and composition.

\newpage
\section{Factored-NLP -- Definition and Properties}
\label{sec:bg:factored-nlp}

A factored nonlinear program (Factored-NLP) is a nonlinear program \eqref{eq:nlp-opt} where the vector variable is factored into a set of smaller vector variables, and
the cost function, the equality, and inequality constraints are also decomposed into a set of smaller cost terms and constraint functions, each depending on a small subset of the variables.

Given a set of $N$ vector variables \( X = \{x_1,\ldots, x_N \} \) with \( x_i \in \mathbb{R}^{n_i} \), a set of nonlinear cost functions \( F = \{ f_1, \ldots, f_B \} \) with \( f_b: \mathbb{R}^{m_b} \to \mathbb{R} \), and a set of vector-valued constraint functions \( \Phi = \{\phi_1,\ldots,\phi_A\} \) with \( \phi_a: \mathbb{R}^{m_a} \to \mathbb{R}^{m_a'} \) that include both equality and inequality constraints, a Factored-NLP is the optimization problem,
\begin{subequations}
  \label{eq:factored-nlp}
  \begin{align}
    \min_{x_1, \ldots, x_N} & \sum_{f_b \in F} f_b{(X_b)} \,, &  &                          \\
    \text{s.t.
    } ~
               & x_i \in \mathbb{R}^{n_i} \,,    &  & i  = 1, \ldots, N        \\
    ~          & \phi_a(X_a) ~ \{\leq,=\} ~ 0\,, &  & \forall \phi_a \in \Phi
  \end{align}
\end{subequations}
where each cost function \( f_b \) and constraint \( \phi_a \) may depend on different subsets of variables \( X_b \subseteq X, \, X_a \subseteq X \) (e.g., \( X_a = \{x_1, x_3\} \) for some \( a \)).
The notation \( \{\leq,=\} \) indicates that a constraint can be either an equality or an inequality constraint, which are handled differently by nonlinear solvers.

In this thesis, we often use Factored-NLPs to reason about the infeasibility of optimization problems and to generate one or more feasible solutions that fulfill the constraints.
Thus, we consider Factored-NLPs without a cost term, resulting in the feasibility problem,
\begin{subequations}
  \label{fac:feas}
  \begin{align}
    \text{find} & ~ x_i \in \mathbb{R}^{n_i},   &  & i=1,\ldots,N             \\
    \text{s.t.
    } ~         & \phi_a(X_a) ~ \{\leq,=\} ~ 0, &  & \forall \phi_a \in \Phi.
  \end{align}
\end{subequations}

A Factored-NLP is a structured representation similar to factor graphs
\cite{frey1997factor,dellaert2017factor}, constraint systems \cite{rossi2006handbook}, and graphical models \cite{koller2009probabilistic}.
It can be represented with a bipartite graph with two types of vertices, variables and constraints, where edges represent the dependency relations between them.

We can represent a Factored-NLP with variables \( X_G = \{ x_1, \ldots ,x_N\} \) and constraints \(\Phi_G= \) \(\{ \phi_1 ,\ldots, \phi_A\} \) as a graph \( G =(V_G,E_G) \) with vertices \( V_G \) and edges \( E_G \),
\begin{subequations}
  \label{eq:factpred-nlp-graph}
  \begin{align}
     & V_G=X_G \cup \Phi_G \,,                                                                                                        \\
     & E_G = \{ ( x_i, \phi_a ) \mid ~ \text{constraint} ~ \phi_a \in \Phi_G ~ \text{depends on} ~ \text{variable} ~ x_i \in X_G \} .
  \end{align}
\end{subequations}

Subproblems of Factored-NLPs are defined as \emph{subgraphs} of the original graph \( G \).
Namely, a subset of variables and constraints is a subgraph \( M \subseteq G \) of the original Factored-NLP.
Likewise, a superset of variables and constraints is a supergraph \( \tilde{G} \supseteq G \).

A Factored-NLP \( G \) is feasible, \( \mathcal{F}(G) = 1 \), if the optimization problem \eqref{eq:factored-nlp} has a solution (i.e., if there exists a value assignment for all variables that fulfills all constraints),
\begin{equation}
  \label{eq:fac:feas}
  \mathcal{F}(G) = 1 \quad \textup{iff} \quad \exists ~ x_i ~\in \mathbb{R}^{n_i} ~ i=1,\ldots,N, ~ \textup{such that} ~ \phi_a(X_a) ~ \{\leq,=\} ~ 0, \quad \forall \phi_a \in \Phi \,.
\end{equation}
Otherwise, it is infeasible, with \( \mathcal{F}(G) = 0 \).
Note that if a Factored-NLP \( G \) is feasible, then all its subgraphs \( M \subseteq G \) are feasible.
Conversely, if a Factored-NLP is infeasible, then all its supergraphs \( \tilde{G} \supseteq G \) are infeasible.

Similar to unstructured NLPs (\cref{sec:bg:nlp}), one can solve a Factored-NLP using
joint nonlinear optimization of all variables and constraints.
If desired, the structured representation can be used for faster matrix factorization and multiplication using sparse matrices.
In fact, leveraging the temporal structure of motion planning, we can reduce the computational complexity of second-order methods (Newton/Gauss-Newton) from cubic to linear in the temporal dimension.

Additionally, since we now have a set of variables and constraints, we can solve for any subset of them, ignoring the other variables and constraints. In particular, we can attempt to solve the complete problem by computing subsets of variables sequentially in any order, conditioned on the variables that have already been determined, as later discussed in \cref{ch:mcts}.

\newpage

\begin{figure}[!t]
  \centering
  \begin{subfigure}[t]{.4\textwidth}
    \centering
    \includegraphics[width=.6\textwidth]{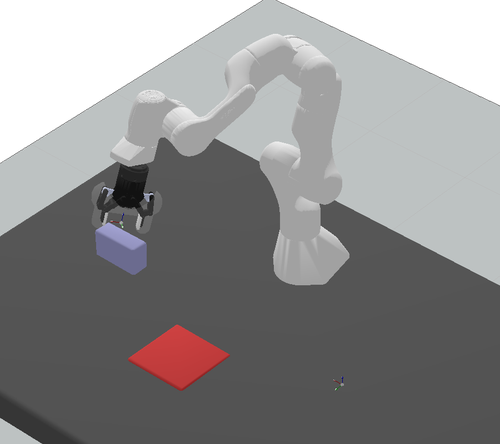}
    \caption{Environment with a robot $Q$, an object $B$, and a red table. }
  \end{subfigure}
  \begin{subfigure}[t]{.55\textwidth}
    \centering
    \raisebox{2cm}{
      \centered{
        \texttt{\small pick object B with robot Q from B\_init}, \\
        \texttt{\small place object B with robot Q on red table}.
      }
    }
    \caption{Task plan.}
  \end{subfigure} \\

  \vspace{.5cm}

  \begin{subfigure}[t]{.9\textwidth}
    \centering
    \includegraphics[width=.8\textwidth]{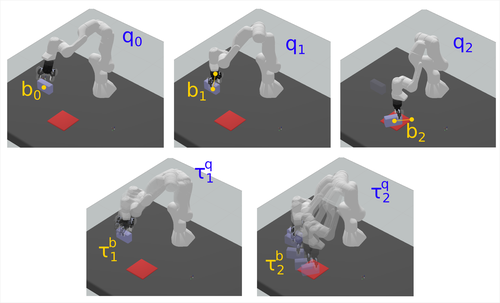}

    \caption{Trajectory optimization problem.
      The top row illustrates the keyframes, and the bottom row shows the trajectory between keyframes.
      The initial configuration is \( q_0 \) for the robot and \( b_0 \) for the object.
      Variables \( q_1 \) and \( b_1 \) are the robot and object configurations in the pick keyframe, and \( q_2 \) and \( b_2 \) are the configurations in the place keyframe.
      Variables \( \tau_{\{1,2\}}^q \) and \( \tau_{\{1,2\}}^b \) indicate the trajectories of the robot and the object between keyframes.
      The images show a valid solution.
    }
  \end{subfigure}
  \caption{Pick and Place -- Environment, task plan, and trajectory optimization problem.}
  \label{fig:build_block}
\end{figure}

\section{Pick and Place -- The Basic Building Block}
\label{sec:bg:pick-and-place}

We analyze the trajectory optimization problem for the Pick and Place task plan, which is the basic building block in manipulation planning.
The problem is shown in \cref{fig:build_block}.
(a) Shows the environment with the robot and the object, (b) shows the chosen Pick and Place task plan, and (c) represents the trajectory optimization problem that we have to solve to compute a robot motion for this task plan.

The graphical representation of the Factored-NLP is shown in \cref{fig:factor_graph_example}.
Circles represent variables, squares represent constraints, and edges indicate the dependencies between variables and constraints.
In the following, we explain the meaning of the variables and constraints that appear in the Factored-NLP in \cref{fig:factor_graph_example}.

\paragraph{Variables}
Variables in the optimization problem represent the robot and object configurations in each keyframe (\(q, b\)) and the trajectories between the keyframes (\(\tau^b, \tau^q\)).
The subscript indicates the time step.
Variable \(q_0 \in \mathbb{R}^7\) is the initial configuration of the robot, \(q_1 \in \mathbb{R}^7\) is the configuration in the pick keyframe, and \(q_2 \in \mathbb{R}^7\) is the configuration in the place keyframe.

The variable \(\tau_1^q \in \mathbb{R}^{7 \cdot 20}\) is the trajectory from the start configuration to the pick keyframe, and \(\tau_2^q \in \mathbb{R}^{7 \cdot 20}\) is the trajectory from the pick keyframe to the place keyframe.
The trajectories are represented using a finite set of waypoints, e.g., \num{20}, thus \(\tau_1^q, \tau_2^q \in \mathbb{R}^{7 \cdot 20}\).

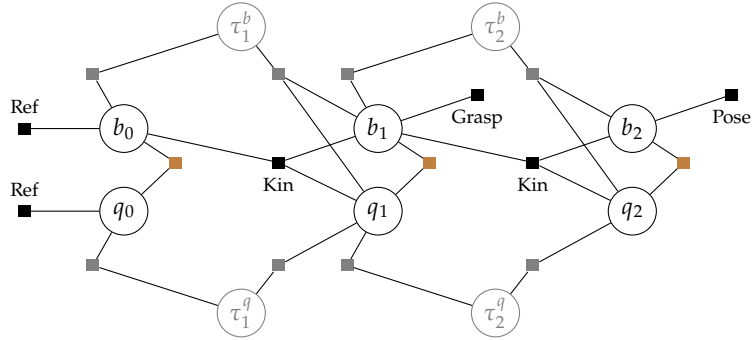
\begin{figure}[!t]
  \centering
  \input{details/factore_nlp_pick_and_place.tex}

  \caption{Factored-NLP of a Pick and Place task plan (see \cref{fig:build_block}).
    Circles represent variables and squares represent constraints.
    The vertical slices represent the initial state, the two keyframes of the manipulation plan, and the trajectories between them.
    Variables \( q \), \( b \), \( \tau_q \), and \( \tau_b \) are, respectively, the robot configuration, the object pose, the robot trajectory, and the object trajectory.
    The subindices denote the time step.
    The meaning of each constraint (\textit{Ref}, \textit{Kin}, \textit{Grasp}, \textit{Pose}) is explained in the main text.
    Brown squares represent collision avoidance constraints between the object, robot, and environment.
    Gray squares indicate the boundary value constraints between trajectories and keyframes.
    For clarity, joint limits on \( q \) and additional constraints on the trajectories, such as zero velocity or collision avoidance within trajectories, are not drawn, but they are detailed in the main text.
  }
  \label{fig:factor_graph_example}
\end{figure}

The variables \( b \) for the object pose represent the relative transformation of the object with respect to its parent frame, as defined by the task plan.
Variable \( b_0 \in SE(3) \) is the initial pose with respect to the world frame, \( b_1 \in SE(3) \) is the object pose with respect to the gripper in the pick keyframe, and \( b_2 \in SE(3) \) is the object pose with respect to the table in the place keyframe.
We also introduce two variables for the trajectories of the object, \( \tau_1^b \) and \( \tau_2^b \), relative to their parent frames.
Note that in the case of prehensile manipulation, such trajectories will be constrained to have zero velocity.

\paragraph{Constraints}
The constraints of the optimization problem are:
\begin{itemize}
  \item \(\textit{Kin}(b_0, b_1, q_1)\): when the robot picks up an object, the object pose with respect to the gripper \(b_1\) is determined by the position of the end-effector \(p(q_1)\) and the absolute position of the object \(b_0\).
        This is implemented with: \( b_0 = p(q_1) \oplus b_1, \)
        where \(p(q) : \mathbb{R}^7 \to SE(3)\) is the pose of the end-effector with respect to the world frame, as a function of the robot configuration, and \(\oplus\) is the addition operator in \(SE(3)\).

  \item \(\textit{Kin}(b_1, b_2, q_2)\): similarly, when the robot places the object, the pose with respect to the table \(b_2\) is determined by the position of the end-effector \(p(q_2)\) and the relative grasp \(b_1\).

  \item  \(\textit{Grasp}(b_1)\): To ensure a stable grasp, the pose of the object with respect to the gripper must be such that the object does not fall.
        This can be represented with a nonlinear function based on different grasp models (e.g., we can constrain it to be a top grasp, a grasp that encapsulates the object between the two fingers, or a simpler touch-grasp).

  \item  \(\textit{Pose}(b_2)\): Similarly, the placement of the object in the place keyframe is constrained to be stable.
        This can be represented with a nonlinear function based on different placement models (e.g., a top placement of a block on a surface).

  \item \(\textit{Ref}(q_0)\): The robot configuration in the initial state is fixed to the given value.

  \item  \(\textit{Ref}(p_0)\): The object pose in the initial state is fixed to the given value.

  \item Collision constraints (brown squares in \cref{fig:factor_graph_example}) are evaluated by calculating the distance between the collision shapes of the robot, movable object, and static environment.
        Collision avoidance constraints are also added to the trajectory variables, applying the constraint to each waypoint.

  \item Joint limits on the robot configurations \(q_1,q_2\) and the robot trajectories \(\tau_1^q, \tau_2^q\) are also added to the optimization problem (not shown in \cref{fig:factor_graph_example}).

  \item Boundary value constraints (gray squares in \cref{fig:factor_graph_example}) tie the keyframes and the trajectories together.

  \item Zero Velocity: the object pose is always stable with respect to its parent frame, either the world frame when it lies in the start position, or the gripper when held by the robot.
        This is represented with the equality constraints \(\textit{Vel0}(\tau_1^b), \textit{Vel0}(\tau_2^b)\) (not shown in \cref{fig:factor_graph_example}).
\end{itemize}

Adding constraints between keyframes directly, e.g., \(\textit{Kin}(b_0,b_1,q_1)\), is essential for efficient TAMP, as it results in very informative relaxations of the Factored-NLP when removing the trajectory variables.
In the case of a stable grasp, we can add equality constraints (\textit{Kin}) directly between the variables in the keyframes, which do not depend on the intermediate trajectory.
For pushing interactions, we can add inequality constraints between keyframes that over-approximate the reachability and affordances of a pushing motion.

\paragraph{Cost}
The Factored-NLP in \cref{fig:factor_graph_example} does not show any cost term, but we can add convex quadratic costs with two purposes: to favor solutions with short and smooth trajectories, and to act as a regularization in the optimization process, improving the convergence behavior and success of the optimizer.
For instance,

\begin{itemize}
  \item Squared acceleration cost on the robot trajectories.
        Because trajectories are represented by a finite set of waypoints, we can
        compute the acceleration at each waypoint using finite differences.
  \item Distance between robot configurations in the keyframes, e.g., \( \| q_1 - q_2 \|^2 \).
  \item Regularization on robot configurations and object poses.
        For instance,
        \( \| q_1 - q_{\text{ref}} \|^2 \) with a reference configuration \( q_{\text{ref}} \), that avoids singular configurations.
\end{itemize}

\section{Complex Manipulation Sequences}
\label{sec:bg:example}

In this section, we analyze the Factored-NLP for longer manipulation sequences in environments with more objects and robots.
Our goal is to illustrate the scalability and unique properties of this representation, which contribute to the design of efficient planning and learning algorithms throughout the thesis.
The three essential properties of our Factored-NLP formulation are:

\begin{itemize}
  \item \emph{Temporal Markov structure:} Constraints in the graph only connect variables from consecutive time steps, maintaining the sequential temporal structure of a planning problem.

  \item \emph{Sparse factorization:}
        Each constraint depends only on a small subset of variables.
        This becomes particularly clear in scenarios involving multiple robots and objects.
        For instance, when a robot picks up an object, additional constraints are added between that specific robot and object, but without affecting the other objects or robots in the scene.

  \item \emph{Repeatable local structure:}
        The optimization problem for a new manipulation task contains several small building blocks that are repeated across different task plans.
        For instance, when comparing the new Factored-NLPs in \cref{fig:bf:example2:plan1,fig:graph-nlp:plan3} to the Pick and Place example, similar structures with small variations are used to model handovers or the placement of one object on top of another.
\end{itemize}

To achieve these three properties in the Factored-NLP, it is fundamental to use a \emph{non-minimal} representation of the optimization problem.
Here, non-minimal means that, given a task plan, there often exists a smaller Factored-NLP equivalent to our formulation, using fewer variables and constraints.
For instance, note that we have added the fixed initial state as a variable (along with constraints) to more clearly expose the temporal structure of the problem.
Furthermore, when an object does not move, we also add new variables, which are constrained to be equal.
While our non-minimal formulation is superior for planning and learning algorithms, solving the Factored-NLP with a non-minimal formulation incurs a small runtime penalty.
If required, before employing a nonlinear solver to solve for the full Factored-NLP or a subset, one can remove and merge fixed or equal variables using a preprocessing step.

\begin{figure}
  \centering
  \includegraphics[width=.22\textwidth]{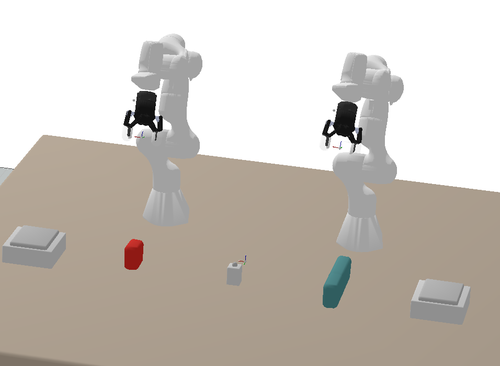}
  \includegraphics[width=.22\textwidth]{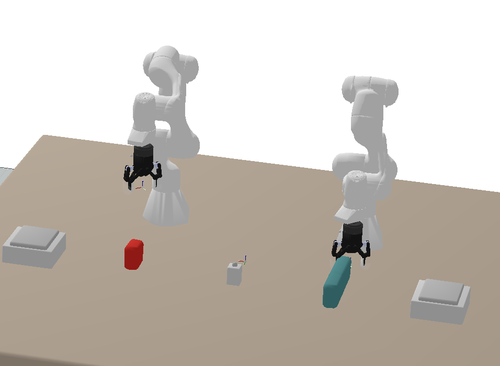}
  \includegraphics[width=.22\textwidth]{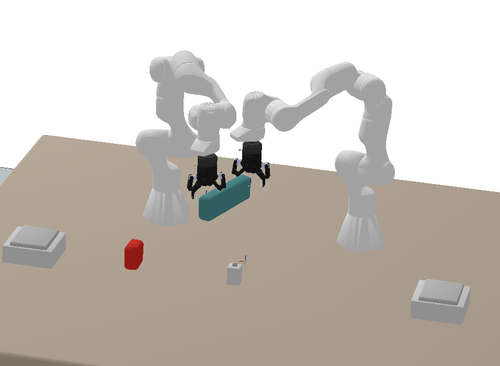}
  \includegraphics[width=.22\textwidth]{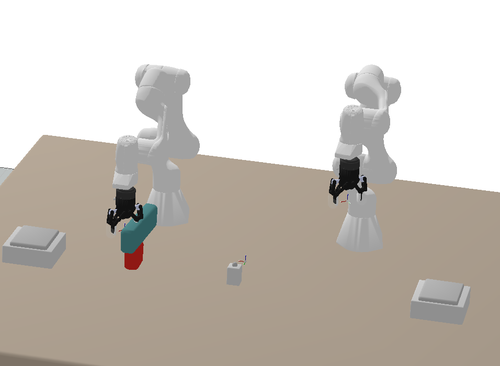}\\
  \caption{Example domain with two objects and two robots.}
  \label{fig:example-domain-background}
\end{figure}

\paragraph{A domain with two robots and two objects}
\label{sec:bg:example2}

In our second example, we consider an environment with two movable objects, \textit{A} and \textit{B}, initially at \textit{A\_init} and \textit{B\_init}, and two robot manipulators, \textit{Q} and \textit{W}.
The high-level goal is to stack \textit{A} on top of \textit{B}.

We discuss the structure of the Factored-NLP for two candidate task plans that can potentially achieve the high-level goal.
\cref{fig:bf:example2:plan1} shows the Factored-NLP for the task plan \( \langle\textit{pick~B~with~Q~from~B\_init},~\textit{pick~B~with~W~from~Q},~\textit{place~B~with~W~on~A}\rangle \).

In each vertical slice of the Factored-NLP, we have continuous variables \( \{ a , b , q , w\} \) for keyframe configurations,
and \( \{ \tau^a , \tau^b , \tau^q , \tau^w \} \) for trajectories.
Variables
\( a,b \) are the poses of the two objects with respect to the parent frame in the kinematic chain, and \( q,w \) are the robot joint configurations.
Variables \( \tau^a , \tau^b , \tau^q , \tau^w \in \RR^{20 \cdot 7} \) are the corresponding trajectories during each motion phase (represented with 20 waypoints).
In comparison to the Pick and Place example, we now have four keyframe variables per vertical slice (instead of 2), and 4 vertical slices (instead of 3) because the task plan is longer.
A feasible solution of the Factored-NLP is shown in \cref{fig:example-domain-background}.

Interestingly, most of the constraints are the same as in the Pick and Place example.
For instance, the collision avoidance, grasping, and positioning constraints are the same.
However, this problem contains some unique, but related structures,

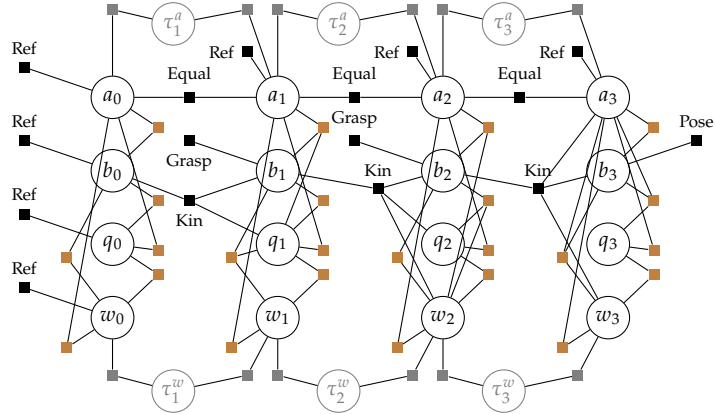
\begin{figure}
  \input{details/factore_nlp_handover.tex}
  \centering
  \caption{Factored-NLP for the task plan
    \( \langle \)\textit{pick~B~with~Q~from~B\_init}, ~\textit{pick~B~with~W~from~Q}, ~\textit{place~B~with~W~on~A}\(\rangle \).
    We display all variables for the keyframe configurations \((a,b,q,w)\), and the trajectory variables only for \((\tau^a,\tau^w)\).
    We omit
    the variables \( \tau^b,\tau^q \), and the constraints for zero-velocity and collisions between trajectories to keep the illustration cleaner.
    Brown squares represent collision avoidance constraints.
    Gray squares represent boundary constraints between trajectories and keyframes.
  }
  \label{fig:bf:example2:plan1}
\end{figure}

\begin{itemize}
  \item In the second step, robots perform a handover.
        We use a \textit{Kin} constraint, which is essentially similar to the \textit{Kin} constraint in the Pick and Place example, but now it connects the two robots with object \( B \) because the robot \textit{W} picks the object from the other robot \textit{Q}.
  \item The \textit{Kin} constraint when placing the object is also different because now the object is placed on top of another object.
  \item Finally, there is a novel \textit{Equal} constraint between consecutive variables for objects that are not manipulated (in this case, object \( A \)), to ensure they remain still.
\end{itemize}

In \cref{fig:graph-nlp:plan3}, we show the Factored-NLP for an alternative plan:
\( \langle \)\textit{pick~B~with~Q~from~B\_init}, ~\textit{place~B~with~Q~on~A}\(\rangle \).
Note how the Factored-NLP has three columns instead of four, and a different global structure, while retaining many common local structures and relationships between variables and constraints.

Our novel planning formulation, \textit{Planning with Nonlinear Constraints}, dictates which constraints appear in the graph, depending on the high-level task plan.
A formal definition and analysis will be provided later in \cref{ch:bid}.

We emphasize that a single Factored-NLP cannot represent the full TAMP problem because the Factored-NLP is conditioned on the task plan, which is a priori unknown and should also be optimized within a TAMP problem.
In difficult TAMP problems, most Factored-NLPs are infeasible since the candidate task plan fails when considering collisions and physics constraints.
This is illustrated in our example: from the two Factored-NLPs shown in \cref{fig:bf:example2:plan1,fig:graph-nlp:plan3}, only the first one is feasible in the environment shown in \cref{fig:example-domain-background}.
The second plan would fail because the robot \( Q \) cannot place object \( B \) correctly on top of object \( A \) because \( A \) is too far.

\begin{figure}
  \centering
  \input{details/factore_nlp_handover_plan2.tex}
  \caption{Factored-NLP for the task plan \( \langle\)\textit{pick~B~with~Q~from~B\_init}, ~\textit{place~B~with~Q~on~A}\(\rangle \) in the domain referenced in \cref{fig:example-domain-background}.
    See main text and caption of \cref{fig:bf:example2:plan1}.
  }
  \label{fig:graph-nlp:plan3}
\end{figure}
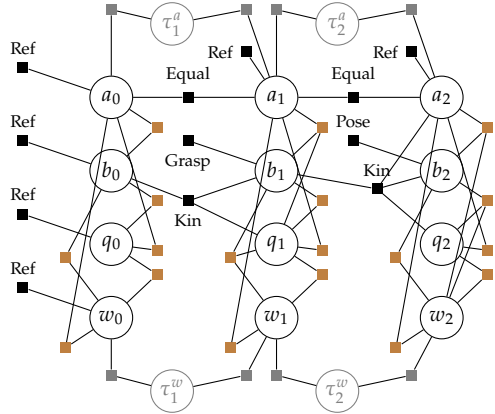

%% file: details/factore_nlp_pick_and_place.tex
		\begin{tikzpicture}[scale=0.9,every node/.style={transform shape}]

			\node[latent] (b0) {$b_0$} ;
			\node[latent,below=.5  of b0] (q0) {$q_0$} ;
			\node[latent,right=1  of q0,yshift=-1.5cm,draw=gray] (tauq0) {\gray{ $\tau^q_1$}} ;
			\node[latent,right=1 of b0,yshift=1.5cm,draw=gray] (taub0) {\gray{ $\tau^b_1$}} ;

			\node[latent,right=3 of b0] (b1) {$b_1$} ;
			\node[latent,below=.5 of b1] (q1) {$q_1$} ;

			\node[latent,right=3 of b1] (b2) {$b_2$} ;
			\node[latent,below=.5 of b2] (q2) {$q_2$} ;

			\node[latent,right=1 of q1,yshift=-1.5cm,draw=gray] (tauq1) {\gray{ $\tau^q_2$}} ;
			\node[latent,right=1 of b1,yshift=1.5cm,draw=gray] (taub1) {\gray{ $\tau^b_2$}} ;

			\factor[left=1 of q0] {trajp0} { Ref } {q0} {};

			\factor[left=1 of b0] {trajp0} { Ref } {b0} {};

			\factor[right=1 of b1, yshift=0.5cm] {trajp0} {below:Grasp} {b1} {};

			\factor[right=1 of b2, yshift=0.5cm] {} {below:Pose} {b2} {};

			\factor[left=1 of b1, yshift=-.5cm] {trajp0} {below:Kin} {b0, q1,b1} {};
			\factor[left=1 of b2, yshift=-.5cm] {} {below:Kin} {b1, q2,b2} {};

			\factor[left=0 of q0,  yshift=-0.8cm,color=gray] {} {} {q0,tauq0} {};
			\factor[left=1 of q1, yshift=-0.8cm,color=gray] {} {} {tauq0,q1} {};
			\factor[left=1 of q2, yshift=-0.8cm,color=gray] {} {} {tauq1,q2} {};

			\factor[left=1 of b1,yshift=.8cm,color=gray] {} {} {taub0,b1,q1} {};
			\factor[left=0 of b0,  yshift=+0.8cm,color=gray] {} {} {b0,taub0} {};

			\factor[left=0 of q1,  yshift=-0.8cm,color=gray] {} {} {q1,tauq1} {};
			\factor[left=1 of b2,yshift=.8cm,color=gray] {} {} {taub1,b2,q2} {};

			\factor[left=0 of b1,yshift=.8cm,color=gray] {} {} {taub1,b1} {};

			\factor[right=.3 of b0, yshift=-0.5cm,color=brown] {} {} {b0,q0} {};
			\factor[right=.3 of b1, yshift=-0.5cm,color=brown] {} {} {b1,q1} {};
			\factor[right=.3 of b2, yshift=-0.5cm,color=brown] {} {} {b2,q2} {};

		\end{tikzpicture}

%% file: details/factore_nlp_handover.tex
\begin{tikzpicture}[scale=0.8,every node/.style={transform shape}]

        \node[latent] (a0) {$a_0$} ;

        \node[latent,above=.5  of a0, xshift=1cm,draw=gray] (taua0) {
            \gray{$\tau^a_1$}} ;

        \node[latent,below=.5 of a0  ] (b0) {$b_0$} ;
        \node[latent,below=.5  of b0] (q0) {$q_0$} ;
        \node[latent,below=.5  of q0] (w0) {$w_0$} ;

        \node[latent,below=.5  of w0,xshift=1cm,draw=gray] (tauw0) {

          \gray{ $\tau^w_1$}} ;

        \node[latent,right=2 of a0  ] (a1) {$a_1$} ;
        \node[latent,below=.5 of a1  ] (b1) {$b_1$} ;
        \node[latent,below=.5 of b1] (q1) {$q_1$} ;
        \node[latent,below=.5 of q1] (w1) {$w_1$} ;
        \node[latent,above=.5  of a1,xshift=1cm,draw=gray] (taua1) {\gray{$\tau^a_2$}} ;

        \node[latent,below=.5  of w1,xshift=1cm,draw=gray] (tauw1) {\gray{$\tau^w_2$}} ;

        \node[latent,right=2 of a1  ] (a2) {$a_2$} ;
        \node[latent,below=.5 of a2  ] (b2) {$b_2$} ;
        \node[latent,below=.5 of b2] (q2) {$q_2$} ;
        \node[latent,below=.5 of q2] (w2) {$w_2$} ;

        \node[latent,above=.5  of a2,xshift=1cm,draw=gray] (taua2) {\gray{$\tau^a_3$}} ;

        \node[latent,below=.5  of w2,xshift=1cm,draw=gray] (tauw2) {\gray{$\tau^w_3$}} ;

        \node[latent,right=2 of a2  ] (a3) {$a_3$} ;
        \node[latent,below=.5 of a3  ] (b3) {$b_3$} ;
        \node[latent,below=.5 of b3] (q3) {$q_3$} ;
        \node[latent,below=.5 of q3] (w3) {$w_3$} ;

      \factor[left=1 of w0, yshift=0.5cm] {trajp0} { Ref } {w0} {};
      \factor[left=1 of q0, yshift=0.5cm] {trajp0} { Ref } {q0} {};
      \factor[left=1 of a0, yshift=0.5cm] {trajp0} { Ref } {a0} {};
      \factor[above=.3 of a1,xshift=-.5cm] {trajp0} { left:Ref } {a1} {};
      \factor[above=.3 of a2, xshift=-.5cm] {trajp0} { left:Ref } {a2} {};

      \factor[left=1 of b0, yshift=0.5cm] {trajp0} { Ref } {b0} {};
      \factor[left=1 of b1, yshift=0.5cm] {trajp0} { below:Grasp } {b1} {};
      \factor[left=1 of b2, yshift=0.5cm] {trajp0} { Grasp } {b2} {};
      \factor[right=1 of b3, yshift=0.5cm] {trajp0} { Pose } {b3} {};
      \factor[above=.3 of a3,xshift=-.5cm] {trajp0} {left:Ref} {a3} {};

      \factor[left=1 of b1, yshift=-.5cm] {trajp0} {below:Kin} {b0, q1,b1} {};

      \factor[right=1.2 of b1, yshift=-.3cm] {trajp0} {Kin} {b1, q2,b2,w2} {};

      \factor[left=.7 of b3, yshift=-.3cm] {trajp0} {Kin} {w3,b3,a3,b2} {};

      \factor[left=1 of a3] {trajp0} {Equal} {a2,a3} {};

      \factor[left=1 of a1] {trajp0} {Equal} { a0, a1} {};

      \factor[left=1 of a2] {trajp0} {Equal} { a1, a2} {};

      \factor[right=.3 of a0, yshift=-0.5cm,color=brown] {} {} {a0,b0} {};
      \factor[right=.3 of b0, yshift=-0.5cm,color=brown] {} {} {b0,q0} {};
      \factor[right=.3 of q0, yshift=-0.1cm,color=brown] {} {} {a0,q0} {};

      \factor[right=.3 of q0, yshift=-0.5cm,color=brown] {} {} {q0,w0} {};
      \factor[left=.3 of w0, yshift=-0.5cm,color=brown] {} {} {a0,w0} {};
      \factor[left=.3 of w0, yshift=+1cm,color=brown] {} {} {w0,b0} {};

      \factor[right=.3 of a1, yshift=-0.5cm,color=brown] {} {} {a1,b1,q1} {};
      \factor[right=.3 of b1, yshift=-0.5cm,color=brown] {} {} {b1,q1} {};
      \factor[right=.3 of q1, yshift=-0.1cm,color=brown] {} {} {a1,q1} {};

      \factor[right=.3 of q1, yshift=-0.5cm,color=brown] {} {} {q1,w1} {};
      \factor[left=.3 of w1, yshift=-0.5cm,color=brown] {} {} {a1,w1} {};
      \factor[left=.3 of w1, yshift=+1cm,color=brown] {} {} {w1,b1,q1} {};

      \factor[right=.3 of a2, yshift=-0.5cm,color=brown] {} {} {a2,b2,w2} {};
      \factor[right=.3 of b2, yshift=-0.5cm,color=brown] {} {} {b2,q2,w2} {};
      \factor[right=.3 of q2, yshift=-0.1cm,color=brown] {} {} {a2,q2} {};

      \factor[right=.3 of q2, yshift=-0.5cm,color=brown] {} {} {q2,w2} {};
      \factor[left=.3 of w2, yshift=-0.5cm,color=brown] {} {} {a2,w2} {};
      \factor[left=.3 of w2, yshift=+1cm,color=brown] {} {} {w2,b2} {};

      \factor[right=.3 of a3, yshift=-0.5cm,color=brown] {} {} {a3,b3} {};
      \factor[right=.3 of b3, yshift=-0.5cm,color=brown] {} {} {b3,q3,a3} {};
      \factor[right=.3 of q3, yshift=-0.1cm,color=brown] {} {} {a3,q3} {};

      \factor[right=.3 of q3, yshift=-0.5cm,color=brown] {} {} {q3,w3} {};
      \factor[left=.3 of w3, yshift=-0.5cm,color=brown] {} {} {a3,w3} {};
      \factor[left=.3 of w3, yshift=+1cm,color=brown] {} {} {w3,b3,a3} {};

      \factor[below=.5 of w0, color=gray] {} {} {w0,tauw0} {};
      \factor[below=.5 of w1,xshift=-.5cm, color=gray] {} {} {tauw0,w1} {};

      \factor[below=.5 of w1,color=gray] {} {} {w1,tauw1} {};
      \factor[below=.5 of w2, xshift=-0.5cm,color=gray] {} {} {tauw1,w2} {};

      \factor[below=.5 of w2,color=gray] {} {} {w2,tauw2} {};
      \factor[below=.5 of w3, xshift=-0.5cm,color=gray] {} {} {tauw2,w3} {};

      \factor[above=1 of a0,color=gray] {} {} {a0,taua0} {};
      \factor[above=1 of a1, xshift=-0.5cm,color=gray] {} {} {taua0,a1} {};

      \factor[above=1 of a1,color=gray] {} {} {a1,taua1} {};
      \factor[above=1 of a2, xshift=-0.5cm,color=gray] {} {} {taua1,a2} {};

      \factor[above=1 of a2, color=gray] {} {} {a2,taua2} {};
      \factor[above=1 of a3, xshift=-0.5cm,color=gray] {} {} {taua2,a3} {};

  \end{tikzpicture}

%% file: details/factore_nlp_handover_plan2.tex
\begin{tikzpicture}[scale=0.8,every node/.style={transform shape}]

        \node[latent] (a0) {$a_0$} ;

        \node[latent,above=.5  of a0, xshift=1cm,draw=gray] (taua0) {
            \gray{$\tau^a_1$}} ;

        \node[latent,below=.5 of a0  ] (b0) {$b_0$} ;
        \node[latent,below=.5  of b0] (q0) {$q_0$} ;
        \node[latent,below=.5  of q0] (w0) {$w_0$} ;

        \node[latent,below=.5  of w0,xshift=1cm,draw=gray] (tauw0) {

          \gray{ $\tau^w_1$}} ;

        \node[latent,right=2 of a0  ] (a1) {$a_1$} ;
        \node[latent,below=.5 of a1  ] (b1) {$b_1$} ;
        \node[latent,below=.5 of b1] (q1) {$q_1$} ;
        \node[latent,below=.5 of q1] (w1) {$w_1$} ;
        \node[latent,above=.5  of a1,xshift=1cm,draw=gray] (taua1) {\gray{$\tau^a_2$}} ;

        \node[latent,below=.5  of w1,xshift=1cm,draw=gray] (tauw1) {\gray{$\tau^w_2$}} ;

        \node[latent,right=2 of a1  ] (a2) {$a_2$} ;
        \node[latent,below=.5 of a2  ] (b2) {$b_2$} ;
        \node[latent,below=.5 of b2] (q2) {$q_2$} ;
        \node[latent,below=.5 of q2] (w2) {$w_2$} ;

      \factor[left=1 of w0, yshift=0.5cm] {trajp0} { Ref } {w0} {};
      \factor[left=1 of q0, yshift=0.5cm] {trajp0} { Ref } {q0} {};
      \factor[left=1 of a0, yshift=0.5cm] {trajp0} { Ref } {a0} {};
      \factor[above=.3 of a1,xshift=-.5cm] {trajp0} { left:Ref } {a1} {};
      \factor[above=.3 of a2, xshift=-.5cm] {trajp0} { left:Ref } {a2} {};

      \factor[left=1 of b0, yshift=0.5cm] {trajp0} { Ref } {b0} {};
      \factor[left=1 of b1, yshift=0.5cm] {trajp0} { below:Grasp } {b1} {};
      \factor[left=1 of b2, yshift=0.5cm] {trajp0} { Pose } {b2} {};

      \factor[left=1 of b1, yshift=-.5cm] {trajp0} {below:Kin} {b0, q1,b1} {};

      \factor[right=1.2 of b1, yshift=-.3cm] {trajp0} {Kin} {b1, q2,b2,a2} {};

      \factor[left=1 of a1] {trajp0} {Equal} { a0, a1} {};

      \factor[left=1 of a2] {trajp0} {Equal} { a1, a2} {};

      \factor[right=.3 of a0, yshift=-0.5cm,color=brown] {} {} {a0,b0} {};
      \factor[right=.3 of b0, yshift=-0.5cm,color=brown] {} {} {b0,q0} {};
      \factor[right=.3 of q0, yshift=-0.1cm,color=brown] {} {} {a0,q0} {};

      \factor[right=.3 of q0, yshift=-0.5cm,color=brown] {} {} {q0,w0} {};
      \factor[left=.3 of w0, yshift=-0.5cm,color=brown] {} {} {a0,w0} {};
      \factor[left=.3 of w0, yshift=+1cm,color=brown] {} {} {w0,b0} {};

      \factor[right=.3 of a1, yshift=-0.5cm,color=brown] {} {} {a1,b1,q1} {};
      \factor[right=.3 of b1, yshift=-0.5cm,color=brown] {} {} {b1,q1} {};
      \factor[right=.3 of q1, yshift=-0.1cm,color=brown] {} {} {a1,q1} {};

      \factor[right=.3 of q1, yshift=-0.5cm,color=brown] {} {} {q1,w1} {};
      \factor[left=.3 of w1, yshift=-0.5cm,color=brown] {} {} {a1,w1} {};
      \factor[left=.3 of w1, yshift=+1cm,color=brown] {} {} {w1,b1,q1} {};

      \factor[right=.3 of a2, yshift=-0.5cm,color=brown] {} {} {a2,b2,w2} {};
      \factor[right=.3 of b2, yshift=-0.5cm,color=brown] {} {} {b2,q2,w2} {};
      \factor[right=.3 of q2, yshift=-0.1cm,color=brown] {} {} {a2,q2} {};

      \factor[right=.3 of q2, yshift=-0.5cm,color=brown] {} {} {q2,w2} {};
      \factor[left=.3 of w2, yshift=-0.5cm,color=brown] {} {} {a2,w2} {};
      \factor[left=.3 of w2, yshift=+1cm,color=brown] {} {} {w2,b2} {};

      \factor[below=.5 of w0, color=gray] {} {} {w0,tauw0} {};
      \factor[below=.5 of w1,xshift=-.5cm, color=gray] {} {} {tauw0,w1} {};

      \factor[below=.5 of w1,color=gray] {} {} {w1,tauw1} {};
      \factor[below=.5 of w2, xshift=-0.5cm,color=gray] {} {} {tauw1,w2} {};

      \factor[above=1 of a0,color=gray] {} {} {a0,taua0} {};
      \factor[above=1 of a1, xshift=-0.5cm,color=gray] {} {} {taua0,a1} {};

      \factor[above=1 of a1,color=gray] {} {} {a1,taua1} {};
      \factor[above=1 of a2, xshift=-0.5cm,color=gray] {} {} {taua1,a2} {};

  \end{tikzpicture}

%% file: diverse_planning.tex
\part{\namePartOne}

\chapter{\nameChapterOne}

\label{ch:diverse_planning}

\section{Introduction}

In this chapter\footnote{This chapter is based on the paper: Ortiz-Haro, J., Karpas, E., Toussaint, M., and Katz, M.
  (2022).
  Conflict-directed Diverse Planning for Logic-Geometric Programming.
  In Proceedings of the International Conference on Automated Planning and Scheduling (Vol.
  32, pp. 279-287).},
we combine state-of-the-art diverse classical planning with trajectory optimization within the Logic Geometric Program (LGP) formulation.

A central challenge in Task and Motion Planning (TAMP) is the integration of information and tools from both the discrete and continuous domains.
Efficient solvers for either the discrete or the continuous levels of TAMP problems are readily available, such as trajectory optimization using constrained optimization or classical planning with heuristic search.
However, these tools are not directly applicable to the full problem, and their integration is neither trivial nor direct.

This limitation has led to the creation of custom TAMP solvers that can better reason about the dependencies between logic and geometry.
However, these solutions often fall short in performance compared to mature, state-of-the-art solvers for each of the subproblems.
The Multi-Bound Tree Search (MBTS) serves as a prominent example, being a leading solver for TAMP problems.
It employs a custom tree search algorithm that explores the search space in a breadth-first manner, combining information from both geometry and logic with a multi-bound strategy.

Our new solver, \emph{Diverse Planning for LGP}, interfaces high-level task planning with low-level trajectory optimization by identifying geometric conflicts in the form of infeasible plan prefixes, and employs a new multi-prefix forbidding compilation to transmit this information back into the task planner.
Additionally, we leverage diverse planning with a new novelty criterion for selecting candidate plans based on prefix novelty, and a metareasoning approach which attempts to extract only useful conflicts by leveraging the information gathered in the course of solving the given problem.

While Multi-Bound Tree Search can use a similar notion of geometric prefix-conflicts internally in the custom tree search, our approach enables the combination of off-the-shelf state-of-the-art PDDL planning with trajectory optimization.
This combination allows us to solve problems that require complex reasoning at both the continuous and discrete levels more quickly.
For instance, the manipulation task shown in \cref{fig:push} requires reasoning about tool-use, pushing, and pick and place actions with two robot manipulators.

The enhancement in the discrete search makes the solver faster, facilitating the generation of multiple candidate plans and allowing for the incorporation of new ideas and techniques from the planning community such as diverse planning and metareasoning.

We demonstrate that our combination not only accelerates the discrete search but also reduces the number of optimization problems solved in the continuous layer, resulting in faster solution times.

\begin{figure}[t]
  \centering
  \includegraphics[width=.2\linewidth]{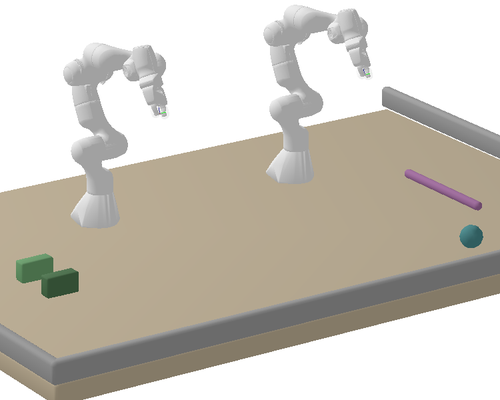}
  \includegraphics[width=.2\linewidth]{pics_icaps21/png/crop2/0010.ppm.png.c.png}
  \includegraphics[width=.2\linewidth]{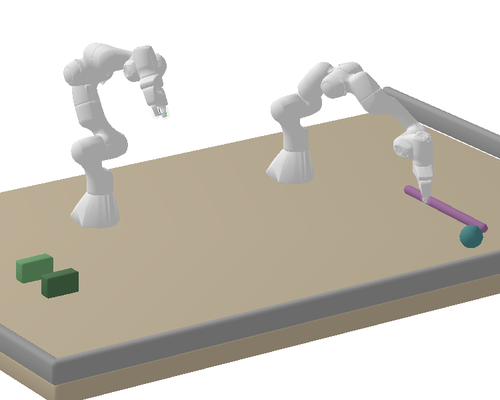} \\

  \includegraphics[width=.2\linewidth]{pics_icaps21/png/crop2/0030.ppm.png.c.png}
  \includegraphics[width=.2\linewidth]{pics_icaps21/png/crop2/0040.ppm.png.c.png}
  \includegraphics[width=.2\linewidth]{pics_icaps21/png/crop2/0050.ppm.png.c.png}  \\

  \includegraphics[width=.2\linewidth]{pics_icaps21/png/crop2/0060.ppm.png.c.png}
  \includegraphics[width=.2\linewidth]{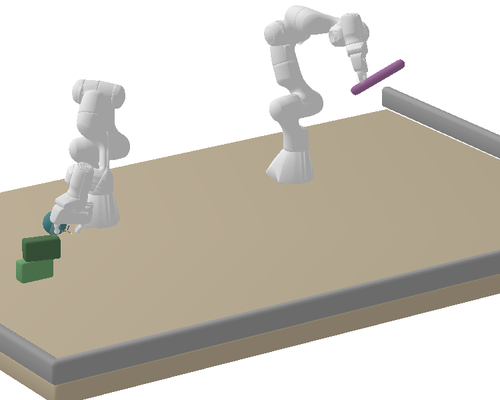}
  \includegraphics[width=.2\linewidth]{pics_icaps21/png/crop2/0074.ppm.png.c.png}
  \caption{
    Example of a TAMP problem solved by our algorithm.
    The objective is to compute the task plan and the continuous motion that achieves the high-level goal \textit{(ball on block A) and (block A on block B)} (bottom right) from the initial configuration (top left).
    This requires combined logic and geometric reasoning about tool-use, pushing, and pick and place actions with two robot manipulators.
  }
  \label{fig:push}
\end{figure}

\section{Related Work}

A comprehensive review of related work in Task and Motion Planning is available in \cref{sec:bg:related-work}.
The closest approaches to ours are sample-based TAMP solvers that attempt to identify geometric conflicts and encode this information back into the discrete descriptions, either using a set of predefined predicates \cite{srivastava2014combined} or by adding additional constraints to an incremental constraint satisfaction solver \cite{dantam2018incremental}.

We follow the LGP formulation of Task and Motion Planning and provide an alternative to the Multi-Bound Tree Search \cite{toussaint2017multi} algorithm.
The interface between the high-level task plan and the low-level motion, focusing on prefixes of task plans that are infeasible, is consistent across both solvers.
However, the encoding techniques and tools used in the discrete search are more advanced in our approach, leading to a more efficient TAMP solver, as our experiments demonstrate.

Our solver employs tools from classical planning and conflict-based search to solve an LGP.
These ideas and tools are further refined in the \textit{Factored-NLP Planner} (\cref{ch:bid}), resulting in a more efficient bidirectional interface that exploits the factorization of TAMP problems.

\section{Factorization of the Discrete State Space}
\label{sec:factorization-discrete}

\textit{Diverse Planning for LGP} builds on top of Logic Geometric Programming (\cref{sec:bg:lgp}) and classical planning (\cref{bg:sec:planning}).
We refer to \cref{sec:bg:lgp} for an extensive introduction to the LGP formulation and notation, which we adopt in this chapter.

In order to use off-the-shelf discrete planners, we first need a factorized representation of the discrete state space in terms of discrete variables.
As exemplified in the \emph{Blocksworld} domain (\cref{bg:sec:planning}), this factorization is readily available if we consider an object- and robot-centric representation.

To transform the discrete search component of the LGP into a classical planning problem using the
SAS+ encoding (\cref{sec:bg:pddl}), we introduce a discrete variable for each movable object and robot.
Variables for objects, e.g., \( \texttt{parent\_A} \) for block \( A \), indicate where the object is in a discrete sense, e.g., on the table, on top of another object, or held by the gripper; but without defining the continuous relative transformation. Specifically, these variables indicate the ``parent'' in the kinematic tree and (implicitly) the type of interaction (e.g., stable pick by a gripper or push by a stick). Consequently, we use the term ``parent'', e.g., the discrete variable for object \( A \) is called \( \texttt{parent\_A} \).
Variables for robots indicate whether the robots are interacting with any object (e.g., the gripper can be full or empty).

Consider the scenario in \cref{fig:push} with two robots (\( Q \) and \( W \)), two blocks (\( A \) and \( B \)), a ball, and a stick, where robots can pick up the blocks and use the stick to push the ball.
We can represent the discrete state using six discrete variables:
\{\texttt{parent\_A},
\texttt{parent\_B},
\texttt{parent\_Stick},
\texttt{robot\_Q},
\texttt{robot\_W},
\texttt{parent\_Ball}\}.
The sets of possible values for each variable are:
\begin{itemize}
  \item[--] \texttt{parent\_A} $\in$ \{ \texttt{table},~\texttt{block\_B}, ~\texttt{robot\_Q}, ~\texttt{robot\_W}, ~\texttt{A\_init} \},
  \item[--]
    \texttt{parent\_B} $\in$ \{ \texttt{table}, ~\texttt{block\_A}, ~\texttt{robot\_Q}, ~\texttt{robot\_W}, ~\texttt{B\_init} \},
  \item[--] \texttt{parent\_Stick} $\in$ \{ ~\texttt{table}, ~\texttt{robot\_Q}, ~\texttt{robot\_W}, ~\texttt{Stick\_init} \},
  \item[--] \texttt{robot\_Q} $\in$ \{ ~\texttt{free}, ~\texttt{full} \},
  \item[--] \texttt{robot\_W} $\in$ \{ ~\texttt{free}, ~\texttt{full} \},
  \item[--] \texttt{parent\_Ball} $\in$ \{ ~\texttt{table}, ~\texttt{block\_A}, ~\texttt{block\_B} , ~\texttt{robot\_Q}, ~\texttt{robot\_W}, \\ ~\texttt{Ball\_init}, ~\texttt{Stick} \}.
\end{itemize}

A discrete state $s \in \mathcal{S}$ is a value assignment to all the variables.
For instance, the initial state $s_0$ in \cref{fig:push} is:

\begin{center}
  \begin{tabular}{ll}
    \texttt{parent\_A = A\_init},         & \texttt{parent\_B = B\_init},       \\
    \texttt{parent\_Stick = Stick\_init}, & \texttt{robot\_Q = free},           \\
    \texttt{robot\_W = free},             & \texttt{parent\_Ball = Ball\_init}.
  \end{tabular}
\end{center}

In a more concise way, we can also represent the initial state as a set of predicates that are true: 

\begin{center}
  (\texttt{A on A\_init}),
  (\texttt{B on B\_init}), 
  (\texttt{Stick on Stick\_init}), \\
  (\texttt{Q free}), (\texttt{W free}), (\texttt{Ball on Ball\_init}),
\end{center}

where, for example,
\((\texttt{A on A\_init})\) means that \texttt{parent\_A = A\_init}.
Similarly, we can represent every discrete action \(a \in \mathcal{A}\) in the LGP formulation as a pair of conditions and effects on the discrete variables.
For instance:

\begin{itemize}
  \item[--] Action: ~\texttt{pick block B with robot Q from B\_init} \\
    Conditions:
    ~\texttt{parent\_B = B\_init},
    ~\texttt{robot\_Q = free}.
    \\
    Effects:
    ~\texttt{robot\_Q = full},
    ~\texttt{parent\_B = robot\_Q}.
  \item[--]
    Action:
    \texttt{place block B with robot Q on block A} \\
    Conditions:
    ~\texttt{parent\_B = robot\_Q}. 
    ~\texttt{robot\_Q = full}. \\
    Effects:
    ~\texttt{parent\_B = block\_A},
    ~\texttt{robot\_Q = free}.
\end{itemize}

\section{Diverse Task Planning for LGP}

\begin{figure}
  \centering
  \begin{tikzpicture}[node distance=70pt,
      every node/.style={fill=white, font=\small}, align=center]
    \node (planner)     [process]          {Task Planner \\ + Diversity};
    \node (plan)     [right of=planner, xshift=-.1cm ]          {Task \\ Plan};
    \draw[-]             (planner) -- (plan);
    \node (motionplanner)   [process, right of=planner, xshift=3cm]          {Motion Planner \\ + Conflict Extraction};
    \draw[->]             (plan) -- (motionplanner);
    \node (solution)   [ right of=motionplanner, xshift=1cm]          {Solution};
    \node (ptask)   [below of=planner, yshift = .5cm]          {Discrete Planning \\ Task};
    \node (reformulation)   [process, below of=planner,yshift=-1cm]          {Reformulation};
    \node (infeas)   [ right of=reformulation,xshift=3cm]          {Infeasible Prefix};
    \draw[->]             (infeas) -- (reformulation);
    \draw[->]             (motionplanner) -- (solution);
    \draw[->]             (motionplanner) -- (infeas);
    \draw[-]             (reformulation) -- (ptask);
    \draw[->]             (ptask) -- (planner);
  \end{tikzpicture}
  \caption{Overview of our approach \emph{Diverse Planning for LGP}.
    We combine a discrete task planner to generate task plans with a motion planner to compute the continuous trajectory.
    If a plan is not geometrically feasible, we extract a conflict, namely a prefix of discrete actions of the task plan, and reformulate the discrete planning task to block this prefix.
  }
  \label{fig:div:flowchart}
\end{figure}
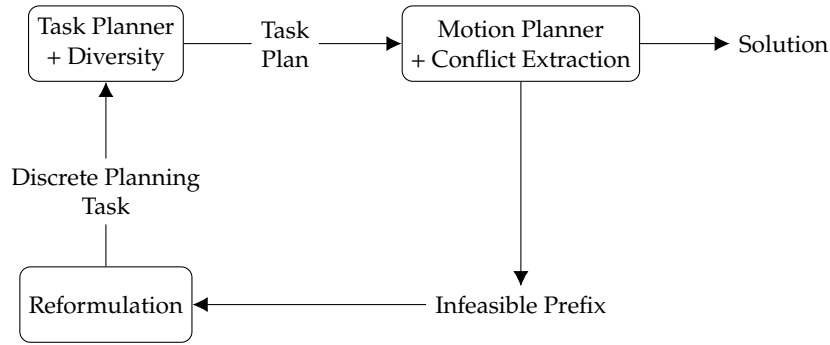

The fundamental contribution of this chapter is the application of diverse planning at the task level of a Logic Geometric Programming (LGP) problem, yielding a varied sequence of task plans.
\textit{Diverse Planning for LGP} is an iterative Task and Motion Planning (TAMP) solver, where candidate task plans are checked for geometric feasibility by solving a trajectory optimization problem.
Crucially, if a task plan is not geometrically feasible, we extract a conflict and use this conflict to reformulate the planning task.
Specifically, we build upon and extend the iterative plan forbidding approach \cite{katz2020reshaping} to generate diverse and conflict-free candidate plans.

In this chapter, we address conflicts in the form of task plan \textit{prefixes} -- that is, a sequence of discrete actions \( \pi = \langle a_1, \ldots, a_K \rangle \) which is applicable from the initial state at the discrete level but lacks a feasible geometric trajectory.
To exploit such conflicts, in a manner akin to conflict-directed clause learning \cite{SilvaS99Search} or conflict-directed A* \cite{WilliamsR07Conflict}, two steps are required.
First, we must be able to efficiently extract a conflict from a task plan that is not geometrically feasible.
Second, we must prevent our task planner from generating plans that contain the identified conflict as a prefix.
Both of these will be elaborated on later.

The flowchart in \cref{fig:div:flowchart} provides a graphical description, and \cref{alg:pseudocode}, which appears later, shows the pseudocode for our approach.
Fig.
\ref{fig:example} provides an example of the execution of our solver in a simplified setting.

\subsection{Prefixes as Conflicts}

To begin the detailed discussion of our approach, we discuss why we choose to identify prefixes as conflicts, and not a more general restriction on plans.

\begin{definition}
  Given a sequence of discrete actions \( \pi = \langle a_1, \ldots, a_K \rangle \), the prefix of length \( k \leq | \pi | = K \) is denoted as,
  \begin{equation}
    \pi|_k = \langle a_1, \ldots, a_k \rangle\,.
  \end{equation}
\end{definition}

\begin{theorem}
  \label{thm:prefix}
  Let \( \pi = \langle a_1, \ldots, a_K \rangle \) be a sequence of discrete actions, such that \( \pi \) is not geometrically feasible from the initial state.
  Then any sequence of actions \( \pi' \) which contains \( \pi \) as a prefix is not geometrically feasible from the initial state.
\end{theorem}
\begin{proof}
  Recall that a sequence \( \pi = \langle a_1, \ldots, a_K \rangle \) of \( K \) actions is not geometrically feasible if the nonlinear optimization \( \textup{Trajectory-NLP}(\pi) \) (\cref{eq:lgp-nlp}) is infeasible.

  For the sake of contradiction, assume that, given an infeasible prefix \( \pi \), there exists a longer task plan \( \pi' = \langle \pi, a_{1:J}' \rangle \) that is feasible.
  If \( \pi' \) is feasible, then there exists a geometric path \( x(t), \; t \in [0,(K+J)T] \) that is feasible for \( \textup{Trajectory-NLP}(\langle \pi, a_{1:J}' \rangle) \).
  This would imply that \( x(t), \; t \in [0,KT] \) is also a feasible solution for \( \textup{Trajectory-NLP}(\pi) \), because variables and constraints in the time interval \( t \in [0, KT] \) are the same in \( \textup{Trajectory-NLP}(\langle \pi, a_{1:J}' \rangle) \) and \( \textup{Trajectory-NLP}(\pi) \).
  This is a contradiction because \( \pi \) is said to be infeasible.
  Therefore, \( \pi' = \langle \pi, a_{1:J}' \rangle \) cannot be feasible.
\end{proof}

For example, consider an LGP task in which a single robot \( Q \) can pick and place objects \( A \) and \( B \) on a cluttered table.
Suppose the starting sequence:
\begin{center}
  \( \langle \texttt{pick B with Q from B\_init}, \texttt{place B with Q on table} \rangle \),
\end{center}
is infeasible from the initial discrete state \([ \texttt{parent\_A=A\_init}, \texttt{parent\_B=B\_init}, \texttt{robot\_Q}=\texttt{free} ]\).
Then it is safe to infer that any action sequence beginning with this sequence is also infeasible, e.g., \( \langle \) \texttt{pick B with Q from B\_init}, \texttt{place B with Q on table}, \texttt{pick A with Q from A\_init} \(\rangle \).

However, it is \textit{not} safe to infer that \( \langle \texttt{pick B with Q from B\_init}, \texttt{place B with Q on table} \rangle \)
can never be applied.
For example, it is possible that object \( A \) in the initial position is obstructing the placement of object \( B \).
Thus, the sequence of actions:
\begin{center}
  \( \langle \) \texttt{pick A with Q from A\_init}, \texttt{place A with Q on table}, \texttt{pick B with Q from B\_init}, \texttt{place B with Q on table} \( \rangle \)
\end{center}
might be geometrically feasible.
In this example, \( \texttt{pick A} \) and \( \texttt{place B} \) form a causal link \cite{Tate77Generating}, as \( \texttt{pick A} \) supports \( \texttt{place B} \).
However, this causal link does not appear at the discrete level but only at the geometric level.
It is not possible to infer a stronger conflict than prefixes without a deeper analysis of geometric feasibility (\cref{ch:bid}).

Prefix forbidding is a general and sound way to encode information from the geometric level back into the discrete level.
It does not rely on hand-crafted additional predicates or checks and is therefore applicable to any sequence of actions, independently of the underlying physics or geometry model.

\subsection{Forbidding Plans by Prefixes}

\label{sec:forbid-plan-prefixes}

This section describes how to prevent a discrete planner from returning task plans which begin with a given set of prefixes, found to be geometrically infeasible in a previous iteration.

Our approach builds upon previous work \cite{katz18novel}, which has suggested a \textit{plan forbidding reformulation}, a method of constructing a planning task with a set of valid plans being reduced by precisely the given plan.
The suggested construction follows the execution of the given task plan \( \langle a_1 ,\ldots, a_K \rangle \) and allows one to achieve the (modified) goal only once an action different from \( a_k \) is applied at step \( k \).

We modify this construction in two ways.
First, instead of forbidding \( \langle a_1 ,\ldots ,a_K \rangle \) as a plan, we forbid it as a prefix.
Thus, applying the starting sequence \( \langle a_1,\ldots,a_K \rangle \) in the reformulated task leads to a dead end.
Consequently, there are no plans for the reformulation with the prefix \( \langle a_1,\ldots,a_K \rangle \).

Second, we simultaneously forbid multiple prefixes.
While this effect can be achieved by sequentially forbidding a single prefix, the simultaneous forbidding approach yields a much more compact compilation.

The key to the simultaneous forbidding approach is building a \textit{prefix tree} that contains all (non-dominated) prefixes.
A prefix \( \tilde{\pi} \) is dominated by prefix \( \pi \) if \( \pi \) is a prefix of \( \tilde{\pi} \); in this case, it is sufficient to forbid \( \pi \) and not \( \tilde{\pi} \).
We construct a tree \( T=(N,E) \) where each node corresponds to a prefix, and there is an edge from node \( \pi \) to node \( \pi' \) if we can add one action to \( \pi \) to yield \( \pi' \).
Given a set of prefixes, this tree can be efficiently constructed by adding the nodes from each prefix iteratively.
Given a prefix tree, Definition \ref{def:forbid} shows how to construct a planning task that forbids exactly these prefixes.

\begin{definition}
  \label{def:forbid}
  Let \( \ptaskParSTD \) be a planning task using the SAS+ encoding (\cref{sec:bg:pddl}), \( T=(N,E) \) be a prefix tree with \( L\subseteq N \) being the leaf nodes, and \( \ops(T) \) be the set of discrete actions that appear on the prefixes in \( T \).
  The task \( \ptask^{-}_{T} =
  \langle\vars',\ops',\initstate',\goal'\rangle \) is defined as follows:
  \begin{itemize}
    \item $\vars' = \vars \cup \{\extrav\} \cup \{\extrav_s \mid s\in N\}$, with
          $\extrav_s$ being binary variables and $\extrav$ being a ternary variable,
    \item $\ops' = \noplanop{\ops} \cup \planopdiscarded{\ops} \cup \planopdiscard{\ops} \cup \planopfollow{\ops}$, where \\
          $\noplanop{\ops} = \{ \noplanop{\op} \mid \op \in\ops\setminus\ops(T)\}$,
          $\planopdiscarded{\ops} = \{\planopdiscarded{\op} \mid\op\in\ops\}$,
          $\planopdiscard{\ops} = \{\planopdiscard{\op} \mid\op\in\ops(T)\}$, and
          $\planopfollow{\ops} = \{\planopfollow{\op_{(s,t)}} \mid (s,t)\in E\}$ with
          \[
            \begin{split}
              \noplanop{\op} &= \tuple{\pre(\op)\cup\{\tuple{\extrav,1}\}, \eff(\op)\cup \{\tuple{\extrav, 0} \}}\\
              \planopdiscarded{\op} &=
              \tuple{\pre(\op)\cup\{\tuple{\extrav,0}\},\eff(\op)}\\
              \planopdiscard{\op} &= \tuple{\pre(\op)\cup \{\tuple{\extrav,1}\} \cup \{
                \tuple{\extrav_{s},0} \mid (s,t)\in E^{\op} \},\\
                & \hspace*{3cm} \eff(\op)\cup \{\tuple{\extrav,0} \}}\\
              \planopfollow{\op_{(s,t)}} &= \tuple{\pre(\op_{(s,t)})\cup \{\tuple{\extrav,1},
                \tuple{\extrav_{s},1} \}, \\
                & \hspace*{1.5cm}\eff(\op_{(s,t)})\cup \{
                \tuple{\extrav_{s},0}, \tuple{\extrav_t, 1} \}} \mbox{ if } t\not\in L, \\
              \planopfollow{\op_{(s,t)}} &= \tuple{\pre(\op_{(s,t)})\cup \{\tuple{\extrav,1},
                \tuple{\extrav_{s},1} \}, \{\tuple{\extrav, 2} \}} \mbox{ if } t\in L, %
            \end{split}
          \]
    \item $\initstate'[\var] = \initstate[\var]$ for all
          $\var\in\vars$, $\initstate'[\extrav] = 1$, $\initstate'[\extrav_{\initstate}] = 1$,
          and $\initstate'[\extrav_s] = 0$ for all $s\in N\setminus \{\initstate\}$, and
    \item $\goal'[\var]\!
            =\!\goal[\var]$ for all $\var\!\in\!\vars$ s.t. $\goal[\var]$ defined, and $\goal'[\extrav]\! =\! 0$.
  \end{itemize}
\end{definition}

The proof of the correctness of the compilation is similar to the proof of Theorem 6 in \cite{katz18novel}.
The main difference between the two reformulations is in the application of \( \planopfollow{\op_{(s,t)}} \) for \( t \in L \), which leads to a dead-end state.

Finally, we remark that the conflicts we extract can be encoded as PDDL 3 trajectory constraints \cite{gerevini2009deterministic}.
These can be compiled away \cite{baier2006planning}, and the above-mentioned compilation is a special case of such a compilation.

\subsection{Feasibility Checking}
\label{sec:feas_check}

If a candidate task plan \( \pi = \langle a_1,\ldots,a_K \rangle \) is found to be geometrically feasible, then we have found a solution to our LGP task and can terminate.
Otherwise, we can return the full plan \( \langle a_1,\ldots,a_K \rangle \) as a conflict.
We will refer to doing this as \textit{lazy} conflict extraction.

Alternatively, we can also search for a stronger conflict, in the form of a shorter prefix of \( \pi \) that is not geometrically feasible, which we refer to as \textit{eager} conflict extraction.
\textit{Eager} conflict extraction searches for the strongest possible conflict we can extract from \( \pi \), that is, the shortest prefix \( \prefix{\pi}{k} \) which is geometrically infeasible, i.e.
\begin{equation}
  \min ~ k ~ \text{s.t.} ~ \text{Feas}(\prefix{\pi}{k}) = 0 ,
\end{equation}
where \( \text{Feas}(\pi) \) is a binary function that returns 1 if \( \text{Trajectory-NLP}(\pi) \) is feasible or 0 otherwise.
By Theorem \ref{thm:prefix}, \( \text{Feas}(\prefix{\pi}{k}) \ge \text{Feas}(\prefix{\pi}{k+1}) \).
Therefore, we can find the strongest conflict with a binary search for the length \( k \) of this prefix.

Initially, the lower bound \( l \) is initialized to \num{0}, and the upper bound \( u \) is initialized to \( K \).
The evaluation of \( \text{Feas}(\prefix{\pi}{m}) \), for the midpoint \( m = \lfloor \frac{l + u}{2} \rfloor \), corresponds to checking with the motion planner whether the prefix up to \( m \) is geometrically feasible.

As geometric feasibility checking is the most expensive computational action we perform, we cache every prefix we check and whether it is feasible or not.
This cache is helpful in speeding up feasibility checking, as different discrete plans might still share a common prefix.
Additionally, this cache serves as a dataset that captures the history of computational actions performed so far, which will be useful for (a) metareasoning about feasibility checking, and (b)
guiding our choice of which plan to check for feasibility next when we use diverse planning.
These are described in the subsequent sections.
Additionally, we can use the \textit{pose} and \textit{keyframes} bounds (\cref{eq:keyframesbound,eq:posebound}) to accelerate conflict extraction (see also \cref{sec:diverse:exp}).

\section{Metareasoning for Conflict Extraction}
\label{sec:metareasoning}

A middle-ground approach between \textit{lazy} conflict extraction (which does not perform any reasoning to extract conflicts) and \textit{eager} (which attempts to find the minimal conflict) is to use metareasoning.
Metareasoning \cite{Russell91Principles} can be used to balance the cost (the computational effort spent on extracting a conflict) and the reward (the benefits from having a stronger conflict).
As the most expensive operation in our algorithm is geometric feasibility checking, we measure both the reward and the cost in terms of the number of geometric feasibility checks -- either required to extract the conflict or saved by having the conflict.

We now describe the metareasoning problem we face in deciding when to stop looking for a conflict, and the overall utility we can expect to obtain.
Let \( \tau = \langle a_1,\ldots,a_k \rangle \) be some sequence of actions.
We will denote by \( r(\tau) \) the reward from adding the conflict \( \tau \).
Of course, this is an unknown quantity, and we will describe ways to estimate it later.
Recall that during the binary search for a minimal conflict, we have a discrete plan \( \pi \), and a range \( [l,u] \) such that \( \prefix{\pi}{u} \) is not geometrically feasible, while \( \prefix{\pi}{l} \) is.
Thus, we can define the metareasoning problem for a given plan \( \pi \) of length \( |\pi| \) as a Markov Decision Process (MDP) \cite{Bellman:DynamicProgramming} with states \( S_\pi = \{ \langle l, u \rangle \mid l \leq u = 0,\ldots,|\pi| \} \) -- that is, each state describes the current range of the search.

The terminal states are those where the search has converged, that is, \( \{ \langle l, l \rangle \mid l = 0,\ldots,|\pi| \} \).
The reward in state \( \langle l, l \rangle \) is the reward from adding the conflict \( \prefix{\pi}{l} \), that is, \( r( \prefix{\pi}{l} ) \).
The reward from all other states is 0.
As we are sure to reach a terminal state, there is no need to introduce a discount factor (that is, \( \gamma=1 \)).

The possible actions at state \( \langle l, u \rangle \) are either to stop searching or continue searching.
The decision to stop searching yields a deterministic transition to the state \( \langle u, u \rangle \) -- that is, we terminate and add the conflict \( \prefix{\pi}{u} \), obtaining reward \( r(\prefix{\pi}{u}) \).

Although the binary search always continues the search by checking the middle node (\( \lfloor (l+u)/2 \rfloor \)), using the metareasoning MDP allows us to consider any of the nodes between \( l \) and \( u \) as the next node to check.
Thus, we have \( u-l+1 \) possible actions -- one for each node in the range.

Let us denote the probability of a sequence of actions \( \tau \) being geometrically feasible by \( p_f(\tau) \).
Then by continuing the search to node \( m \) (representing \( \prefix{\pi}{m} \)), we reach the state \( \langle m, u \rangle \) with probability \( p_f(\prefix{\pi}{m}) \), and state \( \langle l, m \rangle \) with probability \( 1-p_f(\prefix{\pi}{m}) \).

Due to the structure of this MDP, which lacks any loops, we can compute the optimal values using simple dynamic programming, starting with the terminal states, and computing the optimal value function for states with an increasing gap between the lower and upper bounds -- that is, we compute the value for states \( \{ \langle l, l+1 \rangle \mid l = 1,\ldots,|\pi|-1 \} \), then for \( \{ \langle l, l+2 \rangle \mid l = 1 ,\ldots, |\pi|-2 \} \), and so on.

Of course,
we still do not know the exact rewards or transition probabilities.
In the following, we describe a data-driven method to estimate these, which allows us to compute an optimal policy for an approximate MDP.

\paragraph{Data-driven estimation}

Although computing \( r(\tau) \) exactly is not tractable, it should be commensurate with the number of discrete plans that would be pruned by introducing the conflict \( \tau \).
While we do not know this number, we can estimate the fraction of plans that would be pruned by conflict \( \tau \) by the fraction of plans we have discovered that have \( \tau \) as a prefix.
Thus, we can define the estimator,
\begin{equation}
  \hat{r}(\tau) := \frac{|\{ \pi' \mid \pi' \in C \text{~and~} \tau \text{ is a prefix of } \pi' \}|}{|C|},
\end{equation}
where \( C \) is the set of prefixes in our cache.
As the number of matching prefixes in the numerator might be 0 (especially early on in the process), we actually add 1 to both the numerator and the denominator.

The probability of a prefix being feasible or not can also be estimated from the history of previous feasibility checks.
Recall that we cache every prefix we check for feasibility.
We use this cache to estimate \( p_f \).
We follow a type system approach \cite{Lelis13Active} and define a set of simple features for each prefix.
Specifically, we use the length of the prefix as its only feature, and
keep track of how many feasibility checks were performed for each prefix length, and how many of these turned out to be feasible -- the ratio between these is our estimate of \( p_f \), denoted \( \hat{p_f} \).
Combining \( \hat{r} \) and \( \hat{p_f} \), we can define our MDP.
As our empirical results will show, this approach results in a reduction of the runtime of our solver.

\section{Diversity Criteria and Complete Algorithm}
\label{sec:diverse}

\begin{figure*}[!t]
  \centering
  \begin{minipage}{.8\textwidth}
    \centering
    \begin{algorithm}[H]
      \caption{Pseudocode for \textit{Diverse Planning for LGP.}}\label{alg:pseudocode}
      \begin{algorithmic}[1]
        \State Input: LGP problem $\Pi_{\text{LGP}}$
        \State Parameters: $N$ \Comment{{\color{gray} \small Number of plans to generate at each iteration.}}

        \State $\Pi := $ discrete component of $\Pi_{\text{LGP}}$ \Comment{{\color{gray} \small Discrete component of an LGP, encoded in SAS+}}
        \State $T:=\emptyset$ \Comment{{\color{gray} \small Set of tried plans}}
        \State $LP:=\emptyset$ \Comment{{\color{gray} \small Set of found discrete plans}}
        \State $MC:=\emptyset$ \Comment{{\color{gray} \small Set of found conflicts}}

        \While{not solved}
        \State $\Pi^f := \mbox{FORBID}(\Pi, LP \cup MC)$ \Comment{{\color{gray} \small Forbid found plans and conflicts (\cref{sec:forbid-plan-prefixes})}}
        \State $LP := LP \cup \mbox{Diverse-Plan}(\Pi^f, N)$ \Comment{{\color{gray} \small Call a diverse planner to find $N$ new plans}}
        \State $\pi := \mbox{SELECT}(LP, T)$ \Comment{{\color{gray}\small Select a plan to try (\cref{sec:diverse})}}
        \State $\mbox{feasible}, \mbox{traj}:= \mbox{MOTION-Feasible?
          }(\Pi_{\text{LGP}}, \pi)$ \Comment{{\color{gray} \small Check task plan $\pi$ for geometric feasibility}}
        \If{\mbox{feasible}}
        \Return $\pi,\mbox{traj}$ \Comment{{\color{gray} \small Return trajectory and discrete plan}}
        \Else
        \State $T := T \cup \{\pi\}$
        \State $\mbox{conflict} := \mbox{FIND-CONFLICT}(\pi)$  \Comment{{\color{gray} \small Find a prefix of $\pi$ that is infeasible (\cref{sec:feas_check,sec:metareasoning})}}
        \State $MC := MC \cup \{ \mbox{conflict} \}$
        \EndIf
        \EndWhile
      \end{algorithmic}
    \end{algorithm}
  \end{minipage}
\end{figure*}

To explore candidate task plans more rapidly, we use a diverse planning approach \cite{katz2018novel}.
The key idea here is to generate multiple plans at each iteration, and then choose one of them for geometric feasibility checking.

Generating a set of plans is done by applying the forbidding compilation (Definition \ref{def:forbid}) iteratively, as done in previous diverse planning approaches \cite{katz2018novel}.

The main question here is how to choose which plan to test next.
Our approach is driven by prefixes, so it makes sense to choose a plan that has the longest novel prefix, as even if that plan fails, there is a higher chance that we will extract a short conflict from it.
Furthermore, choosing a plan with a novel prefix encourages our approach to explore the space of discrete plans, thereby covering diverse high-level approaches to the task that imply different nonlinear programs for the continuous trajectory.

Thus, we define the novelty of a plan \( \pi \) with respect to a set of plans \( LP \) as,
\begin{equation}
  np(\pi, LP) := - \min \{ k \in \mathbb{N} \mid \forall \pi' \in LP, \prefix{\pi'}{k} \neq \prefix{\pi}{k} \}.
\end{equation}

We then choose to test the plan \( \pi \) which maximizes the novelty with respect to the set of plans that were already tested for geometric feasibility, breaking ties randomly.
We remark that this notion of novelty is different from previous ones, e.g.,
\cite{Lipovetzky21Width,
  tuisov2021fewer}, and serves as a greedy selection criterion for choosing the next plan.

To summarize, \cref{alg:pseudocode} describes our complete solver in pseudocode, and \cref{fig:example} shows an illustrative example of the execution of our algorithm in a simplified setting.
We can now state the theorem proving that our approach is sound and complete.

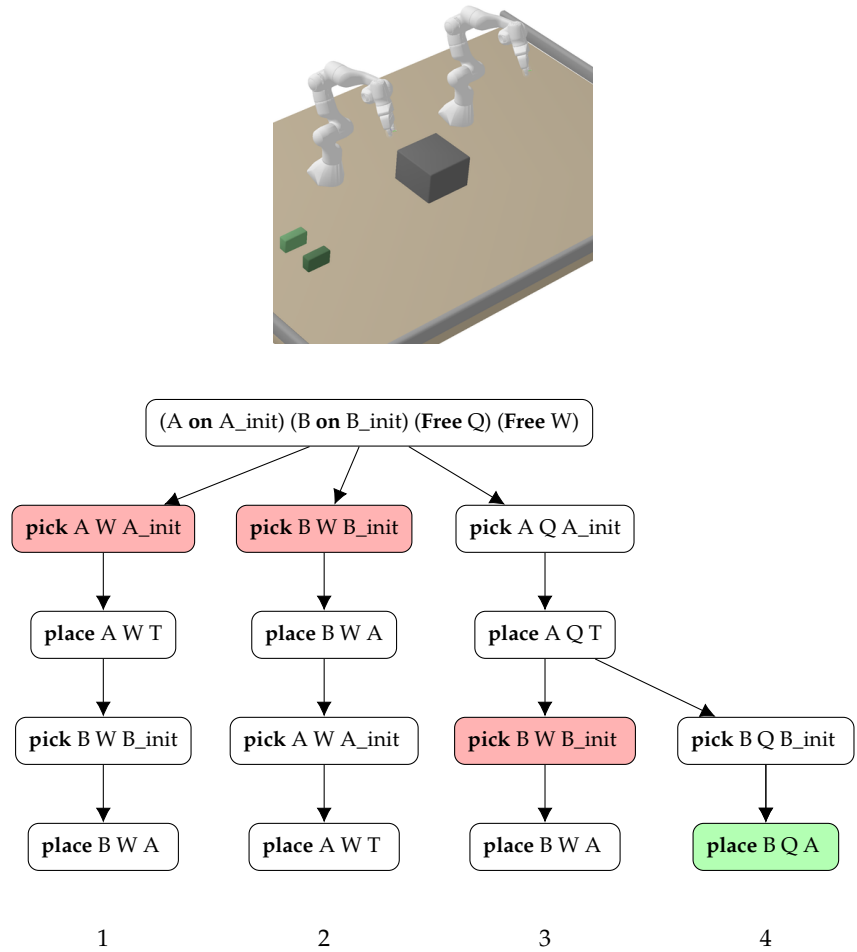
\begin{figure}[t]
  \centering
  \includegraphics[width=.3\linewidth]{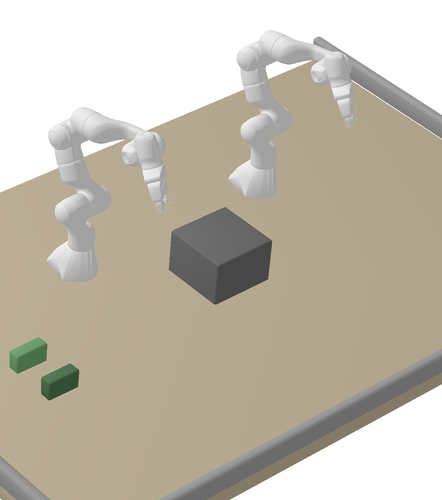} \\[20pt]
  \begin{tikzpicture}[node distance=40pt,
      every node/.style={fill=white, font=\small}, align=center]
    \node (s0)     [processX]          { (A \textbf{on}
      A\_init) (B \textbf{on} B\_init) (\textbf{Free} Q) (\textbf{Free} W)};
    \node (a11) [processR, below of=s0, xshift=-3.5cm] { \textbf{pick} A W A\_init};
    \node (a12) [processX, below of=a11] { \textbf{place} A W T};
    \node (a13) [processX, below of=a12] { \textbf{pick} B W B\_init};
    \node (a14) [processX, below of=a13] { \textbf{place} B W A };
    \node (o1) [number, below of=a14, yshift=.2cm] {1};
    \node (a21) [processR, right of=a11, xshift=1.5cm] { \textbf{pick} B W B\_init};
    \node (a22) [processX, below of=a21] { \textbf{place} B W A};
    \node (a23) [processX, below of=a22] { \textbf{pick} A W A\_init };
    \node (a24) [processX, below of=a23] { \textbf{place} A W T };
    \node (o2) [number, below of=a24, yshift=.2cm] {2};
    \node (a31) [processX, right of=a21, xshift=1.5cm] { \textbf{pick} A Q A\_init};
    \node (a32) [processX, below of=a31] { \textbf{place} A Q T};
    \node (a33) [processR, below of=a32] { \textbf{pick} B W B\_init };
    \node (a34) [processX, below of=a33] { \textbf{place} B W A };
    \node (o3) [number, below of=a34,yshift=.2cm] {3};
    \node (a43) [processX, right of=a33, xshift=1.5cm ] { \textbf{pick} B Q B\_init };
    \node (a44) [processG, below of=a43] { \textbf{place} B Q A };
    \node (o4) [number, below of=a44,yshift=.2cm] {4};
    \draw[->] (s0) -- (a11);
    \draw[->] (a11) -- (a12);
    \draw[->] (a12) -- (a13);
    \draw[->] (a13) -- (a14);
    \draw[->] (s0) -- (a21);
    \draw[->] (a21) -- (a22);
    \draw[->] (a22) -- (a23);
    \draw[->] (a23) -- (a24);
    \draw[->] (s0) -- (a31);
    \draw[->] (a31) -- (a32);
    \draw[->] (a32) -- (a33);
    \draw[->] (a33) -- (a34);
    \draw[->] (a43) -- (a44);
    \draw[->] (a32) -- (a43);
    \draw[->] (a43) -- (a44);
  \end{tikzpicture}
  \caption{Illustrative example of the execution of our algorithm (with \( N=1 \) and \textit{eager} conflict extraction).
    The scene contains two movable objects, \textit{A} and \textit{B}, a table, \textit{T}, and two robots, \textit{Q} and \textit{W}, that can \texttt{pick} and \texttt{place} the objects.
    The goal is to stack the blocks on the table: \texttt{(A on T)} and \texttt{(B on A)}.
    In each iteration, the task planner has produced a task plan (\num{1}, \num{2}, \num{3}, and \num{4} in this order) that has been tested for feasibility.
    The motion planner returned the minimal prefix of infeasible discrete actions (highlighted in red), which is used to reformulate the planning task for the subsequent iterations.
    Plan number \num{4} (in green) is geometrically feasible.
  }
  \label{fig:example}
\end{figure}

\begin{theorem}
  If the underlying task planner is sound and complete, and the motion planner always finds a feasible trajectory for problems \eqref{eq:lgp-nlp} if such a trajectory exists, then Algorithm \ref{alg:pseudocode} is sound and complete.
\end{theorem}
\begin{proof}
  The proof follows from the fact that we only identify prefixes which cannot appear at the beginning of geometrically feasible plans (\cref{thm:prefix}), and from the correctness of the forbidding compilation (\cref{def:forbid}).
\end{proof}

An important technical point is that some planners perform a relevance analysis and discard actions or state variables which they consider to be useless or redundant.
For example, two actions might have the same discrete effects, and thus the planner might decide to keep only one of them.
However, these actions might lead to different geometric constraints, and it may be the case that one of them is feasible while the other is not.
Thus, such preprocessing techniques must be disabled when solving the discrete planning task.

\clearpage

\section{Empirical Evaluation}
\label{sec:diverse:exp}

\begin{figure}[t]
  \centering
  \includegraphics[width=.3\linewidth]{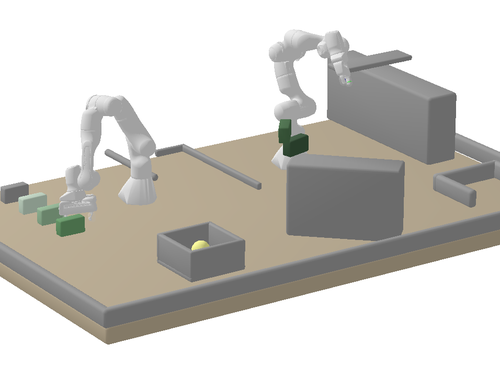}
  \includegraphics[width=.3\linewidth]{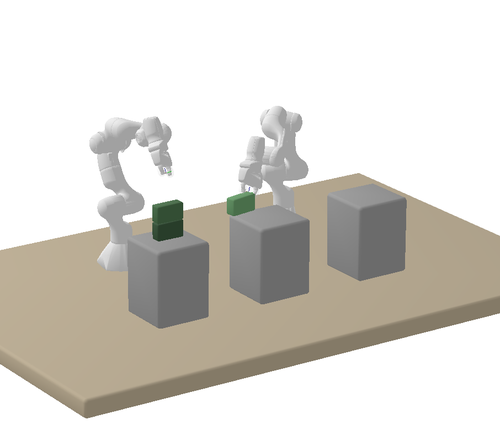}
  \includegraphics[width=.28\linewidth]{pics_icaps21/png/crop2/0020.ppm.png.c.png}
  \caption{Three domains used to evaluate our algorithm.
    From left to right: \textit{Blocks}, \textit{Hanoi} and \textit{Push}.
  }
  \label{fig:three-scenarios}
\end{figure}

\subsection{Benchmarks}

We use three different domains, all with two 7-DOF robotic arms (\cref{fig:three-scenarios}).
\begin{description}
  \item[Blocks] The robots can execute pick and place actions to construct a specified tower of blocks, similarly to the classical Blocksworld domain -- except that the planner must also come up with motion plans.
    Robots can hold a stack of blocks, move the boxes, and place several objects on top of other objects.
    See \cref{fig:blocksAndHanoi}.
  \item[Hanoi] The robots can execute pick and place actions to solve a Tower of Hanoi problem with objects of equal size and three tables.
    Only the top object of each tower can be picked, and at most one tower is allowed on each table.
    From a logical point of view, this is more challenging than \textit{Blocks} and requires longer action sequences, but the instances we use have fewer movable objects in the scene.
    See \cref{fig:blocksAndHanoi}.
  \item[Push] The robots can pick and place blocks and balls, pick up sticks, and use them as tools to push balls.
    The goal is to move balls and blocks to a desired discrete state, for example, stacking blocks and placing the ball on top.
    See \cref{fig:push}.
\end{description}
For each domain, we generate different problems (e.g., \textit{Blocks-\{0,1,2,3,4,5\}}) by modifying the goal and the number of objects, increasing complexity at both the discrete and geometric levels.
Our benchmark \footnote{Project website: \url{https://quimortiz.github.io/ConflictPlanningLGP/}} contains six problems in the domain \textit{Blocks}, three in \textit{Hanoi}, and eleven in \textit{Push}.

\subsection{Baselines}

We compare our new approach, \textit{Diverse Planning for LGP}, against three variations of Multi-Bound Tree Search: \textit{MBTS-\{0,1,2\}} (See \cref{sec:multibound}).

\textit{MBTS-0} does not perform geometric checks on intermediate discrete nodes; that is, it waits until a full candidate task plan is found before solving all bounds and the complete trajectory optimization problem. \textit{MBTS-1} and \textit{MBTS-2} check the \emph{pose} and \emph{keyframes} bounds (\cref{eq:keyframesbound,eq:posebound}) respectively, before expanding a discrete node in the breadth-first search.

Geometric checks during node expansion in \textit{MBTS-1} and \textit{MBTS-2} prune partial plans that are infeasible.
This reduces the branching factor of the search and subsequent node expansions but increases the computational time spent on solving NLPs for action sequences that do not lead to the goal.

\begin{sidewaystable}[ph!]
  \small
  \centering
  \begin{tabular}{lrrrrrrrrrrrr}
    \toprule
    {}       & \multicolumn{3}{c}{MBTS-0}                   & \multicolumn{3}{c}{N=1 \textit{Eager}} & \multicolumn{3}{c}{N=4 \textit{Eager}} & \multicolumn{3}{c}{N=4 \textit{Meta}}                                                                                                                                                                                                                                                                                                                                    \\
    \cmidrule(lr){2-4}
    \cmidrule(lr){5-7}
    \cmidrule(lr){8-10}
    \cmidrule(lr){11-13}
    {}       & {time}                                       & {pose}                                 & {key}                                  & {time}                                       & {pose}                             & {key}                              & {time}                                       & {pose}                             & {key}                              & {time}                                        & {pose}                             & {key}                             \\
    \midrule
    Blocks-0 & \bfseries 19.4{\scriptsize \color{gray} 1.0} & 12.0{\scriptsize \color{gray} 0.0}     & 3.0{\scriptsize \color{gray} 0.0}      & 43.7{\scriptsize \color{gray} 1.8}           & 19.9{\scriptsize \color{gray} 0.9} & 6.5{\scriptsize \color{gray} 0.5}  & 41.8{\scriptsize \color{gray} 4.3}           & 17.5{\scriptsize \color{gray} 1.9} & 6.5{\scriptsize \color{gray} 0.9}  & 41.3{\scriptsize \color{gray} 4.4}            & 17.5{\scriptsize \color{gray} 1.9} & 2.8{\scriptsize \color{gray} 0.6} \\
    Blocks-1 & -                                            & -                                      & -                                      & 44.8{\scriptsize \color{gray} 1.3}           & 18.0{\scriptsize \color{gray} 0.0} & 5.0{\scriptsize \color{gray} 0.0}  & \bfseries 44.0{\scriptsize \color{gray} 5.6} & 17.1{\scriptsize \color{gray} 2.3} & 4.8{\scriptsize \color{gray} 0.9}  & 46.5{\scriptsize \color{gray} 6.4}            & 17.1{\scriptsize \color{gray} 2.3} & 1.7{\scriptsize \color{gray} 0.3} \\
    Blocks-2 & -                                            & -                                      & -                                      & 82.6{\scriptsize \color{gray} 10.5}          & 17.0{\scriptsize \color{gray} 0.0} & 4.0{\scriptsize \color{gray} 0.0}  & \bfseries 60.4{\scriptsize \color{gray} 8.6} & 12.6{\scriptsize \color{gray} 1.1} & 2.4{\scriptsize \color{gray} 0.5}  & 70.9{\scriptsize \color{gray} 11.8}           & 12.6{\scriptsize \color{gray} 1.1} & 1.4{\scriptsize \color{gray} 0.2} \\
    Blocks-3 & -                                            & -                                      & -                                      & 111{\scriptsize \color{gray} 5.8}            & 17.0{\scriptsize \color{gray} 1.0} & 3.4{\scriptsize \color{gray} 0.4}  & 104{\scriptsize \color{gray} 19.7}           & 21.7{\scriptsize \color{gray} 2.8} & 5.4{\scriptsize \color{gray} 1.1}  & \bfseries 80.2{\scriptsize \color{gray} 12.0} & 20.8{\scriptsize \color{gray} 3.1} & 2.1{\scriptsize \color{gray} 0.4} \\
    Blocks-4 & -                                            & -                                      & -                                      & 200{\scriptsize \color{gray} 33.1}           & 27.1{\scriptsize \color{gray} 8.2} & 6.1{\scriptsize \color{gray} 2.8}  & 160{\scriptsize \color{gray} 17.0}           & 19.3{\scriptsize \color{gray} 3.1} & 3.3{\scriptsize \color{gray} 1.1}  & \bfseries 139{\scriptsize \color{gray} 22.6}  & 17.8{\scriptsize \color{gray} 2.7} & 1.3{\scriptsize \color{gray} 0.2} \\
    Hanoi-0  & 10.4{\scriptsize \color{gray} 0.4}           & 13.0{\scriptsize \color{gray} 0.0}     & 4.0{\scriptsize \color{gray} 0.0}      & \bfseries 7.0{\scriptsize \color{gray} 0.2}  & 7.0{\scriptsize \color{gray} 0.0}  & 3.0{\scriptsize \color{gray} 0.0}  & 10.0{\scriptsize \color{gray} 1.9}           & 8.0{\scriptsize \color{gray} 0.9}  & 3.7{\scriptsize \color{gray} 0.8}  & 9.1{\scriptsize \color{gray} 1.8}             & 8.6{\scriptsize \color{gray} 1.1}  & 2.9{\scriptsize \color{gray} 0.5} \\
    Hanoi-1  & 34.7{\scriptsize \color{gray} 0.7}           & 34.0{\scriptsize \color{gray} 0.0}     & 6.0{\scriptsize \color{gray} 0.0}      & 27.0{\scriptsize \color{gray} 0.6}           & 17.0{\scriptsize \color{gray} 0.0} & 8.0{\scriptsize \color{gray} 0.0}  & 18.7{\scriptsize \color{gray} 3.1}           & 13.5{\scriptsize \color{gray} 1.0} & 5.1{\scriptsize \color{gray} 0.6}  & \bfseries 13.8{\scriptsize \color{gray} 2.2}  & 14.0{\scriptsize \color{gray} 1.0} & 3.4{\scriptsize \color{gray} 0.4} \\
    Push-1   & 41.9{\scriptsize \color{gray} 0.8}           & 55.8{\scriptsize \color{gray} 0.2}     & 1.0{\scriptsize \color{gray} 0.0}      & \bfseries 17.1{\scriptsize \color{gray} 0.4} & 14.0{\scriptsize \color{gray} 0.0} & 4.0{\scriptsize \color{gray} 0.0}  & 24.4{\scriptsize \color{gray} 1.4}           & 17.3{\scriptsize \color{gray} 1.1} & 5.3{\scriptsize \color{gray} 0.5}  & 24.9{\scriptsize \color{gray} 1.7}            & 18.7{\scriptsize \color{gray} 1.2} & 3.8{\scriptsize \color{gray} 0.4} \\
    Push-2   & 50.0{\scriptsize \color{gray} 1.0}           & 64.0{\scriptsize \color{gray} 0.0}     & 1.0{\scriptsize \color{gray} 0.0}      & 49.5{\scriptsize \color{gray} 0.9}           & 37.0{\scriptsize \color{gray} 0.0} & 13.2{\scriptsize \color{gray} 0.1} & 37.1{\scriptsize \color{gray} 1.1}           & 23.2{\scriptsize \color{gray} 0.9} & 7.2{\scriptsize \color{gray} 0.4}  & \bfseries 34.3{\scriptsize \color{gray} 1.6}  & 24.3{\scriptsize \color{gray} 0.8} & 3.2{\scriptsize \color{gray} 0.2} \\
    Push-3   & 27.7{\scriptsize \color{gray} 0.9}           & 38.0{\scriptsize \color{gray} 0.0}     & 1.0{\scriptsize \color{gray} 0.0}      & \bfseries 14.4{\scriptsize \color{gray} 0.2} & 11.0{\scriptsize \color{gray} 0.0} & 3.0{\scriptsize \color{gray} 0.0}  & 26.1{\scriptsize \color{gray} 3.2}           & 17.9{\scriptsize \color{gray} 2.0} & 5.8{\scriptsize \color{gray} 0.9}  & 21.1{\scriptsize \color{gray} 1.8}            & 17.3{\scriptsize \color{gray} 1.8} & 2.9{\scriptsize \color{gray} 0.3} \\
    Push-4   & 75.9{\scriptsize \color{gray} 1.5}           & 104{\scriptsize \color{gray} 0.0}      & 2.0{\scriptsize \color{gray} 0.0}      & 71.3{\scriptsize \color{gray} 7.8}           & 41.2{\scriptsize \color{gray} 2.8} & 15.2{\scriptsize \color{gray} 1.5} & 32.6{\scriptsize \color{gray} 4.1}           & 20.3{\scriptsize \color{gray} 2.2} & 5.9{\scriptsize \color{gray} 1.0}  & \bfseries 30.6{\scriptsize \color{gray} 2.7}  & 21.3{\scriptsize \color{gray} 2.4} & 3.1{\scriptsize \color{gray} 0.4} \\
    Push-5   & 111{\scriptsize \color{gray} 1.6}            & 144{\scriptsize \color{gray} 0.1}      & 1.0{\scriptsize \color{gray} 0.0}      & \bfseries 20.4{\scriptsize \color{gray} 0.3} & 17.0{\scriptsize \color{gray} 0.0} & 5.0{\scriptsize \color{gray} 0.0}  & 30.8{\scriptsize \color{gray} 2.5}           & 23.7{\scriptsize \color{gray} 2.2} & 7.4{\scriptsize \color{gray} 0.9}  & 29.4{\scriptsize \color{gray} 2.4}            & 24.4{\scriptsize \color{gray} 2.3} & 3.2{\scriptsize \color{gray} 0.4} \\
    Push-6   & 117{\scriptsize \color{gray} 1.5}            & 142{\scriptsize \color{gray} 0.0}      & 1.0{\scriptsize \color{gray} 0.0}      & 64.5{\scriptsize \color{gray} 1.2}           & 50.0{\scriptsize \color{gray} 0.0} & 17.1{\scriptsize \color{gray} 0.1} & 45.4{\scriptsize \color{gray} 1.0}           & 29.2{\scriptsize \color{gray} 0.7} & 9.2{\scriptsize \color{gray} 0.4}  & \bfseries 45.1{\scriptsize \color{gray} 1.2}  & 31.8{\scriptsize \color{gray} 1.2} & 4.6{\scriptsize \color{gray} 0.3} \\
    Push-7   & -                                            & -                                      & -                                      & \bfseries 68.6{\scriptsize \color{gray} 4.6} & 51.3{\scriptsize \color{gray} 3.5} & 19.0{\scriptsize \color{gray} 1.6} & 79.1{\scriptsize \color{gray} 5.4}           & 52.6{\scriptsize \color{gray} 4.0} & 18.0{\scriptsize \color{gray} 1.6} & 70.4{\scriptsize \color{gray} 2.6}            & 53.0{\scriptsize \color{gray} 2.7} & 7.4{\scriptsize \color{gray} 0.5} \\
    Push-8   & 78.3{\scriptsize \color{gray} 1.2}           & 92.0{\scriptsize \color{gray} 0.0}     & 1.0{\scriptsize \color{gray} 0.0}      & \bfseries 17.0{\scriptsize \color{gray} 0.4} & 13.0{\scriptsize \color{gray} 0.0} & 3.0{\scriptsize \color{gray} 0.0}  & 32.4{\scriptsize \color{gray} 3.6}           & 26.0{\scriptsize \color{gray} 3.0} & 7.8{\scriptsize \color{gray} 1.1}  & 32.8{\scriptsize \color{gray} 3.7}            & 28.2{\scriptsize \color{gray} 3.6} & 4.1{\scriptsize \color{gray} 0.6} \\
    Push-9   & 248{\scriptsize \color{gray} 40.1}           & 423{\scriptsize \color{gray} 67.2}     & 2.5{\scriptsize \color{gray} 0.5}      & 63.3{\scriptsize \color{gray} 1.3}           & 45.0{\scriptsize \color{gray} 0.0} & 16.0{\scriptsize \color{gray} 0.0} & \bfseries 46.2{\scriptsize \color{gray} 6.2} & 32.5{\scriptsize \color{gray} 3.7} & 10.4{\scriptsize \color{gray} 1.4} & 49.8{\scriptsize \color{gray} 11.9}           & 39.7{\scriptsize \color{gray} 9.5} & 5.3{\scriptsize \color{gray} 1.7} \\
    Push-10  & \bfseries 12.7{\scriptsize \color{gray} 0.5} & 16.0{\scriptsize \color{gray} 0.0}     & 1.0{\scriptsize \color{gray} 0.0}      & 12.8{\scriptsize \color{gray} 0.5}           & 9.0{\scriptsize \color{gray} 0.0}  & 3.0{\scriptsize \color{gray} 0.0}  & 13.8{\scriptsize \color{gray} 1.6}           & 10.4{\scriptsize \color{gray} 1.3} & 3.3{\scriptsize \color{gray} 0.6}  & 13.2{\scriptsize \color{gray} 1.4}            & 10.5{\scriptsize \color{gray} 1.3} & 1.6{\scriptsize \color{gray} 0.2} \\
    Push-11  & -                                            & -                                      & -                                      & 61.1{\scriptsize \color{gray} 9.3}           & 25.5{\scriptsize \color{gray} 2.3} & 10.7{\scriptsize \color{gray} 1.4} & \bfseries 26.0{\scriptsize \color{gray} 2.0} & 13.4{\scriptsize \color{gray} 0.5} & 3.8{\scriptsize \color{gray} 0.4}  & 30.5{\scriptsize \color{gray} 5.4}            & 16.7{\scriptsize \color{gray} 2.5} & 2.6{\scriptsize \color{gray} 0.6} \\
    \midrule
    Total    & 827                                          & 1138                                   & 24.5                                   & 976                                          & 437                                & 145                                & 833                                          & 376                                & 115                                & 783                                           & 394                                & 57.4                              \\
    \bottomrule
  \end{tabular}
  \caption{Summary of the experimental results.
    We report the computational time in seconds (\textit{time}), and the number of calls to the motion planner for checking the pose bound (\textit{pose}) and the keyframes bound (\textit{key}),
    with the mean over \num{10} randomized runs in black and the standard deviation of the mean estimator in grey.
    ``Total'' is the sum of the columns (note that the sum for \textit{MBTS-0} is over fewer problems).
    A dash ``–'' denotes that the problem was not solved in all \num{10} runs.
  }
  \label{tab:theTable}
\end{sidewaystable}

\begin{figure}[!t]
  \centering

  \includegraphics[width=.18\linewidth]{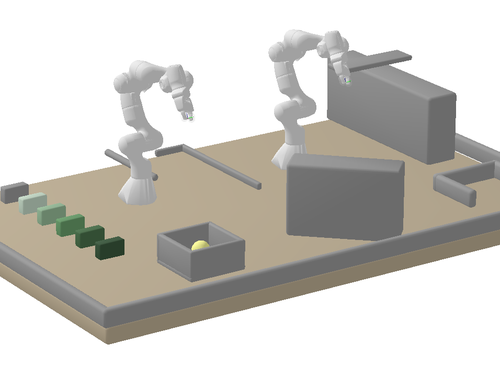}
  \includegraphics[width=.18\linewidth]{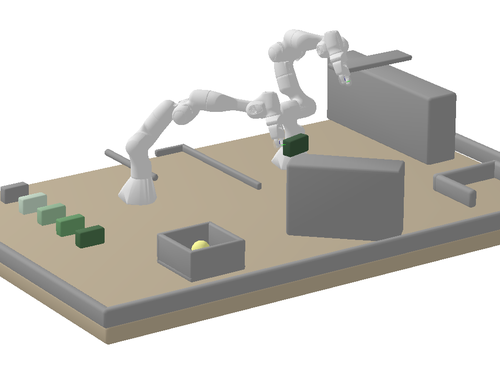}
  \includegraphics[width=.18\linewidth]{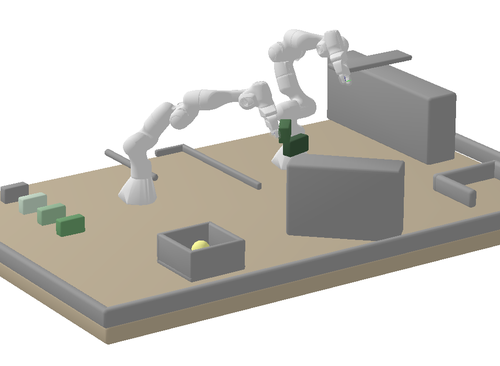}
  \includegraphics[width=.18\linewidth]{pics_icaps21/4png/crop/0050.ppm.png}
  \includegraphics[width=.18\linewidth]{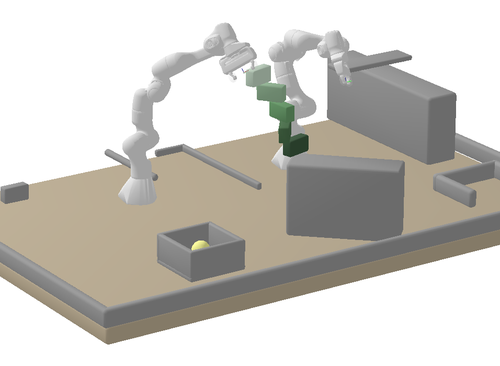}

  \vspace{1em}

  \includegraphics[width=.18\linewidth]{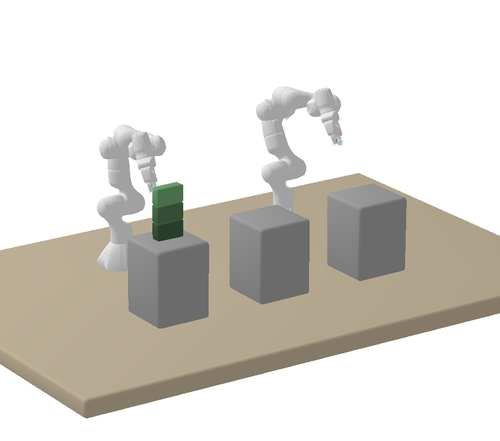}
  \includegraphics[width=.18\linewidth]{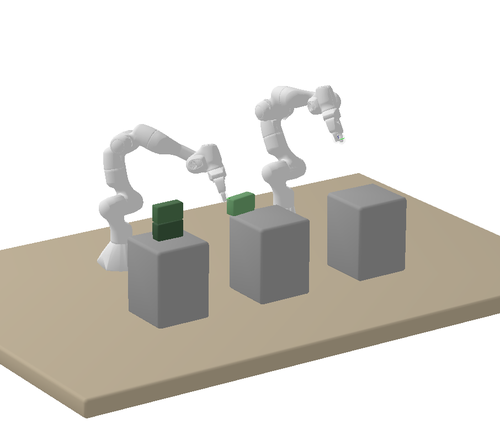}
  \includegraphics[width=.18\linewidth]{pics_pddl_divers_hanoiV/0030.crop.png}
  \includegraphics[width=.18\linewidth]{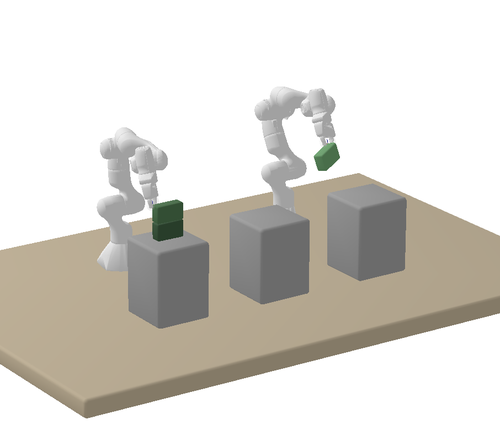}
  \includegraphics[width=.18\linewidth]{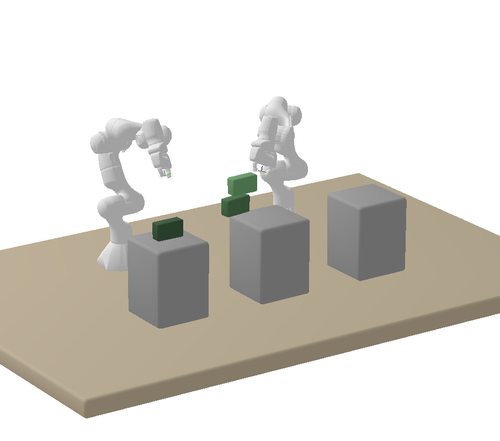}  \\

  \caption{Examples of solutions to the problems in the \textit{Blocks} (top) and \textit{Hanoi} (bottom) domains.
    A solution in the \textit{Push} domain is shown in \cref{fig:push}.
  }
  \label{fig:blocksAndHanoi}
\end{figure}

\subsection{Results}

We use the first iteration of LAMA \cite{richter2010lama} as our underlying task planner.
We ran a set of experiments comparing several versions of our approach to the baselines -- all experiments were run on an AMD Ryzen 9 5980HS CPU with a 600-second time limit per run.
Results are shown in Table \ref{tab:theTable}.
We omit problems \textit{Hanoi-2} and \textit{Blocks-5}, which were not solved by any algorithm or baseline.

\paragraph{Comparison to baselines}
Hypothesis: \textit{“Our basic novel approach (N=1, Eager conflict extraction) will be faster and solve more problems than any of the MBTS baselines”}.
Our method with “\(N=1\), \textit{Eager}” solves more problems (\num{18} vs.
\num{12} out of \num{20}) and is faster (\num{16} vs. \num{2}) than all the baselines \textit{MBTS-\{0,1,2\}}.
In Table \ref{tab:theTable}, we only report \textit{MBTS-0}, which shows better performance than the other baselines.

MBTS-0 does not solve problems that require long action sequences or where the branching factor of the tree is very high (for example, the domain \textit{Blocks} contains \num{12} movable objects).
Due to the uninformed behavior of Breadth-First Search, it only finds a few task plans (sometimes none), none of which are geometrically feasible.
Instead, our method leverages state-of-the-art task planning to compute action sequences efficiently even in large discrete spaces, and geometric information is encoded incrementally in the planning task through our prefix forbidding reformulation.

\paragraph{Analysis of diverse planning}
Hypothesis: \textit{“Diverse planning with a novelty measure will improve over incremental plan generation”}.
We compare “\textit{N=1, Eager}” (the planner produces a single plan, which is evaluated by the motion planner) and
“\textit{N=4, Eager}” (the planner produces four plans in each iteration, which are stored in a buffer; the motion planner evaluates the plan in the buffer that maximizes our novelty criteria).

\(N=4\) reduces both the overall computational time and the number of tested plans.
Choosing a plan from a set of candidates with our criteria is beneficial, as it enforces novelty-based exploration in the space of candidate discrete plans.
The role of prefixes and orderings in an LGP is captured accurately by our novelty measure, which outperforms alternative plan similarity metrics like action set similarity,
which is inaccurate in the context of LGP, where action ordering and precedence cannot be neglected.

\paragraph{Analysis of conflict extraction}
Hypothesis: \textit{“Metareasoning is faster than Eager and Lazy conflict extraction”}.
For \textit{N=4}, we compare three different methods for extracting prefix conflicts:
\textit{Eager} (finds the minimal prefix using the \textit{keyframes} bound (\eqref{eq:keyframesbound})),
\textit{Lazy-pose} (an enhancement of \textit{Lazy} that checks only the \textit{pose} bound \eqref{eq:posebound} to try to extract a conflict), and
\textit{Meta} (a metareasoning approach for conflict extraction).

Our metareasoning approach delivers a speedup across problems (\textit{Meta} is better in 12 vs.
\textit{Eager} 6).
The \textit{Lazy-pose} sometimes provides small infeasible prefixes with the \textit{pose} bound \cref{eq:posebound}, but is slower than \textit{Meta} and \textit{Eager}.
Finally, note that relaxation bounds of feasible NLPs are very fast to compute.
This explains why, in some problems, \textit{Eager} is faster than \textit{Meta}, even if it performs more geometric checks in total.

\section{Limitations}
\label{sec:diverse:limitations}

\textit{Diverse Planning for LGP} shares the main limitations of the underlying LGP formulation and the previous MBTS solver, namely, the local convergence of nonlinear optimization methods.
Optimization methods converge only to local optima, which might prevent finding a solution even if a problem is feasible.
One way to mitigate convergence to bad local optima is to use random restarts, as solving the same problem with different initializations can improve the success rate.

To integrate random restarts into our conflict-based formulation, we can use a soft-conflict formulation.
Instead of blocking prefixes in the task planner, we can penalize task plans that contain plan prefixes where the optimizer failed to find a solution.
Another possible practical implementation is to use a probabilistic hard-conflict formulation, where, in each call to the task planner, we block a prefix with a probability proportional to the number of times the optimizer failed to solve the corresponding optimization problem.

Further limitations for deploying the algorithms in the real world, as often encountered in the TAMP literature, are the assumptions of accurate world information (i.e., geometry and position of the objects), a perfect forward model used for planning, and the simple geometric shapes of the objects.

\section{Conclusions}

In this chapter, we propose the first systematic interface between state-of-the-art task planners and nonlinear constrained path optimization methods to solve Logic Geometric Programs.
A key idea of our approach is to efficiently identify geometric conflicts in the form of minimal infeasible action prefixes and incorporate this information back into the task planner through a multi-prefix forbidding compilation.
Based on this general interface, we further develop a metareasoning strategy to minimize the number of calls to the motion planner and a new novelty criterion for selecting plans from a set of candidates.
Our approach systematically outperforms the baseline LGP solver, solving more problems and faster, especially when the solution requires long action sequences.

This work lays the foundations for the more efficient Factored-NLP Planner, presented in \cref{ch:bid}, which also combines a discrete planner with trajectory optimization through an interface based on detecting and encoding geometric conflicts.
However, instead of relying on infeasible prefixes, the Factored-NLP Planner uses a more powerful interface based on detecting and blocking subsets of infeasible nonlinear constraints in the optimization problems, which results in an order-of-magnitude improvement with respect to the Multi-Bound Tree Search.

Our results suggest that incorporating a PDDL planner into Task and Motion Planning (TAMP) solvers is crucial for enhanced performance and scalability.
Besides the Factored-NLP Planner in \cref{ch:bid}, our TAMP meta-solver in \cref{ch:bid} also employs a PDDL solver to compute a lower bound on the number of discrete actions required to reach the goal.

%% file: graph_nlp_planner.tex
\chapter{\nameChapterTwo}
\label{ch:bid}

\section{Introduction}

Despite recent advances in Task and Motion Planning (TAMP) solvers, current algorithms struggle with high-dimensional configuration spaces (e.g., multiple robots), long-horizon planning, and constrained environments that require joint optimization.
A promising approach to planning in such challenging settings is to efficiently interface state-of-the-art solvers on both sides, particularly incorporating information about infeasibility from continuous solvers back to the task level.

In \cref{ch:diverse_planning}, we illustrate how geometric conflicts in the form of task plan prefixes could be encoded back into a discrete planner.
Building on this foundation, we now aim to find smaller conflicts to create a more efficient interface between the task level and the motion level.
To this end, in this chapter\footnote{
	This chapter is based on the publication: Ortiz-Haro, J., Karpas, E., Katz, M., and Toussaint, M.
	(2022).
	A Conflict-Driven Interface Between Symbolic Planning and Nonlinear Constraint Solving.
	IEEE Robotics and Automation Letters, 7(4), (pp. 10518-10525).
}, we present a second iterative, conflict-based TAMP solver that combines discrete planning with nonlinear optimization with a novel bidirectional interface.

Our approach is based on identifying minimal subsets of nonlinear constraints that guarantee the infeasibility of the continuous trajectory optimization problems.
This information is encoded back into the high-level task planner with a special blocking reformulation.
This new interface provides a powerful enhancement, as now one conflict can directly block multiple different candidate plans--specifically, those that would generate a trajectory optimization problem containing the infeasible constraints.
In contrast, our previous prefix-forbidding solver,
Diverse Planning for LGP, only blocked plans with a matching task plan prefix
(\cref{ch:diverse_planning}).
Given that the number of candidate high-level task plans grows exponentially with the number of robots and objects, an efficient interface is vital for success and scalability, as demonstrated in our evaluation.

The design of our new solver, called the Factored-NLP Planner, has required the introduction of several innovative techniques and contributions: an efficient conflict detection algorithm, a conflict-blocking reformulation of the discrete planning problem, and a precise formulation of the TAMP problem.

As a foundation for this algorithm, we first introduce an abstract Planning with Nonlinear Transition Constraints (PNTC)
formulation, where a discrete task plan implies a factored nonlinear program as a sub-problem, and logical predicates of the task plan can be related to factors of the NLP.
This formulation clarifies the concepts and exact assumptions our algorithm builds on and formally defines
the explicit bidirectional relation between the discrete and the continuous components of the problem, which is
exploited in our solver.

The PNTC formulation ensures both the correct relationship between the discrete and continuous levels, and the appropriate structure in the trajectory optimization problem.
In the context of TAMP, it can be viewed as a factored variant of a Logic Geometric Program (LGP) (\cref{sec:bg:lgp}).
However, in comparison to LGP, PNTC explicitly defines a factored structure of the implied NLP and a bidirectional mapping between symbols and constraint factors in the resulting NLP.
This is exactly the structure we need
to better inform the discrete search and is naturally available in TAMP, making our solver directly applicable to solve TAMP and LGP problems.

We evaluate our method on three robotic TAMP scenarios that present complex intrinsic logic-geometric dependencies requiring long action sequences.
Generating solutions within seconds, our approach clearly outperforms previous optimization-based solvers for TAMP.
We further validate the framework through real-world experiments, demonstrating that the solver computes full task and motion plans in a few seconds.

\begin{figure}[ht]

	\centering
	\begingroup
	\begin{tabular}{cc}

		\includegraphics[width=.27\linewidth]{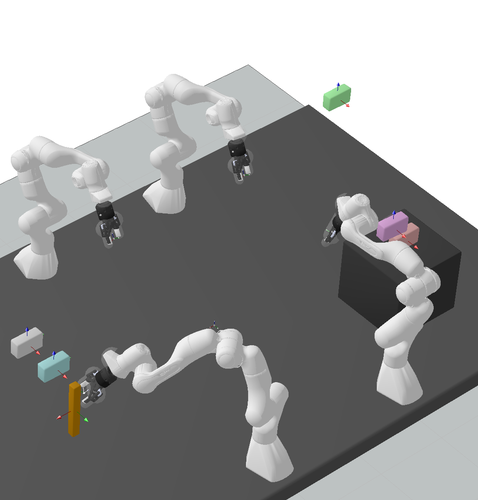}   &
		\includegraphics[width=.27\linewidth]{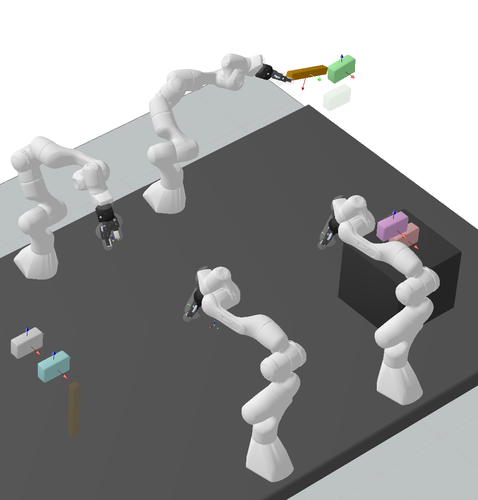}     \\[10pt]

		\includegraphics[width=.27\linewidth]{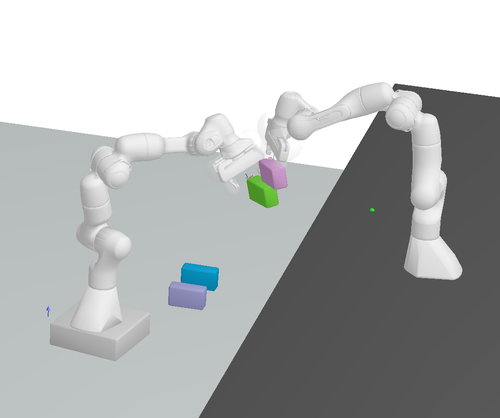}   &
		\includegraphics[width=.27\linewidth]{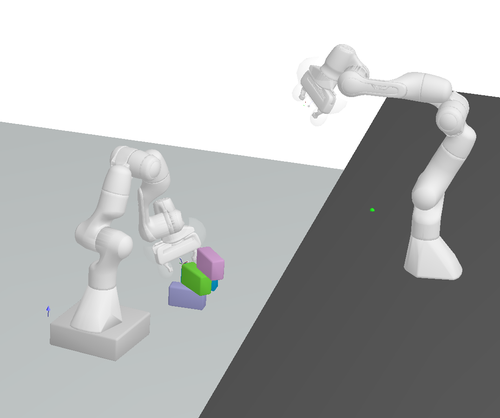}     \\[10pt]

		\includegraphics[width=.27\linewidth]{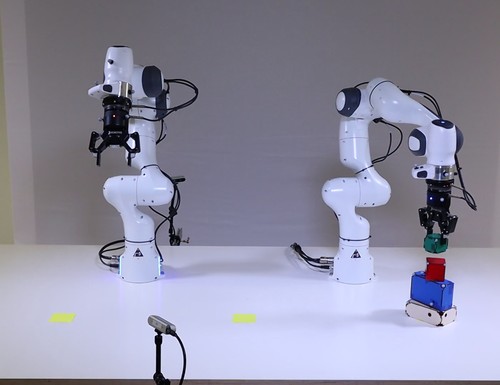} &
		\includegraphics[width=.27\linewidth]{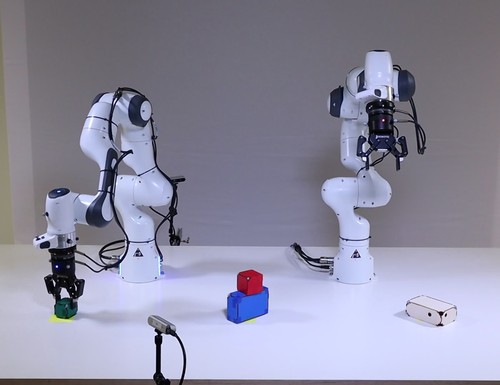}   \\

	\end{tabular}
	\endgroup
	\caption{Task and Motion Planning problems solved by our framework. \textit{Top row}: Four robot manipulators use a stick as a tool to reach a distant block.
		\textit{Middle row}: A heterogeneous team of robots builds a tower. \textit{Bottom row}: Two real 7-DOF manipulators solve the Tower of Hanoi puzzle.}
	\label{fig:showcase}
\end{figure}

\section{Related Work}

A comprehensive review of related work on Task and Motion Planning is provided in \cref{sec:bg:related-work}.
We now briefly position our work with respect to the most closely related approaches.

From the perspective of classical planning, the most closely related work includes \cite{haslum2018extending}, which extends classical planning with general state constraints, and \cite{fernandez2018scottyactivity}, which combines discrete search with convex optimization.
In contrast, the constraints of PNTC are nonlinear, defined by a sequence of discrete states and evaluated on consecutive continuous variables.
This implies a nonlinear program for the whole sequence of continuous variables, which can model the complex continuous constraints in the TAMP problem efficiently.

In comparison to TAMP solvers, our method is related to conflict-based solvers, such as \cite{srivastava2014combined, dantam2016incremental}.
These methods either use a set of predefined predicates such ``is reachable'' \cite{srivastava2014combined} to incorporate information about geometry, or block full plans or pairs of state-actions \cite{dantam2016incremental}.
Alternatively, our framework can encode any type of continuous infeasibilities that potentially involve several motion phases.
In fact, instead of enumerating possible geometric failure cases, we define nonlinear constraints to model the motion and geometry and let the solver detect which intrinsic subset is jointly infeasible.

Compared to optimization-based solvers such as the Multi-Bound tree search \cite{toussaint2017multi} and \textit{Diverse planning for LGP} (\cref{ch:diverse_planning}), we design a much more efficient interface between task planning and motion planning, as demonstrated in our experiments.

\section{Problem Formulation}
\label{sec:planner:formulation}

In this section, we introduce \textit{Planning with Nonlinear Transition Constraints} (PNTC).
PNTC is similar to the Logic Geometric Program formulation for TAMP; the main difference lies in that it provides an explicit modeling of the factorization at the discrete level and within the trajectory optimization problem.
For a comprehensive introduction to the Logic Geometric Program formulation, we direct the reader to \cref{sec:bg:lgp}.

\cref{sec:bg:structure} presents an intuitive and insightful explanation of the natural factorization of trajectory optimization that occurs in Task and Motion Planning.
PNTC will now formalize the required interface between logic and geometry to generate such structured representations.
This will expose three key properties--time structure, local composition, and sparse factorization--that will be used in our solver.

\paragraph{Planning with Nonlinear Transition Constraints}

A \textit{Planning with Nonlinear Transition Constraints (PNTC)} problem is a 7-tuple
$\langle \mathcal{V}, \mathcal{A},s_0,g, \Pi, \mathcal{H}, X_v \rangle$ that includes a discrete component $\langle \mathcal{V}, \mathcal{A},s_0,g\rangle$ and a continuous component $\langle \mathcal{H}, X_v \rangle$, coupled through an interface $\Pi$.

\begin{itemize}

	\item \textit{Discrete Component}:
	      The discrete component $\langle \mathcal{V}, \mathcal{A}, s_0, g\rangle$ corresponds to a classical planning problem encoded in SAS+ (\cref{sec:bg:lgp}).

	\item \textit{Continuous Component}:
	      \(X_v\) is a finite set of \(n\) continuous variables \(\{x^1, \ldots, x^n\}\).
	      Each variable takes a value in a continuous space \(\text{dom}(x^i) = \mathcal{X}^i\) (e.g., \(\mathcal{X}^i = \RR^{n_i}\)).
	      A continuous state \(x \in \mathcal{X}^1 \times \ldots \times \mathcal{X}^n = \mathcal{X}\) is a value assignment to all variables.
	      In the planning problem, we use the notation \(x_k\) to denote the state at step \(k\), and \(x_k^i\) to denote the variable \(i\) at step \(k\).
	      \(\mathcal{H}\) is a finite set of nonlinear, piece-wise differentiable constraint functions that are evaluated on pairs of subsets of continuous variables, \(\mathcal{H} = \{  \phi_b :  \mathcal{X}^{b_{0}} \times \mathcal{X}^{b_{1}}   \to \RR^{n_b}    \}\).
	      The index sets \(b_0, b_1 \subseteq \{1, \ldots, n\}\) indicate on which subsets of variables the function \(\phi_b\) depends.
	      These functions define nonlinear constraints \(\phi_b(x^{b_0}, \tilde{x}^{b_1}) \, \{\le, =\} \, 0\) on a pair of subsets of continuous variables \((x^{b_0} , \tilde{x}^{b_1})\) (e.g., \(x^{b_0} = \{x^1, x^2\}, \tilde{x}^{b_1} = \{\tilde{x}^3\}\)).

	\item \textit{Interface}:
	      The discrete and continuous components of a \textit{PNTC} are coupled through the mapping \(\Pi\).
	      Let \(\mathcal{P}\) be the set of all possible partial discrete states.
	      The mapping \(\Pi: \mathcal{P} \times \mathcal{P} \to \mathcal{H} \cup \emptyset\) with \(\langle p, \tilde{p}\rangle \mapsto \phi_b(x^{b_0}, \tilde{x}^{b_1})\), maps a pair of discrete partial states \(\langle p, \tilde{p} \rangle\) to a nonlinear constraint function \(\phi_b\) that is evaluated on subsets of continuous variables \(x^{b_0}, \tilde{x}^{b_1}\).
	      The empty set \(\emptyset\) indicates that some pairs \(\langle p, \tilde{p} \rangle\) do not generate constraints.
	      This formulation also includes constraints acting on a single state \(\Pi(p) \to \phi_b(x^{b_0})\).

\end{itemize}

A solution to a PNTC is a sequence of discrete and continuous states \( \langle (s_0, x_0), \ldots, (s_K, x_K) \rangle \) and discrete actions \( \langle a_1, \ldots, a_K \rangle \) (starting from the fixed \( s_0 \) and \( x_0 \)), such that,
\begin{subequations}
	\begin{align}
		 & s_k \in \mathcal{S},                                 &  & k=0,\ldots,K                                                                                                            \\
		 & a_k \in \mathcal{A}(s_{k-1}),                        &  & k=1,\ldots,K                                                                                                            \\
		 & s_k = \textup{succ}(s_{k-1},a_k),                    &  & k=1,\ldots,K                                                                                                            \\
		 & g \subseteq s_K,                                     &  &                                                                                                                         \\
     &    x_k^i \in \mathcal{X}^i, &                                 & k= 0, \ldots, K,~ i=1, \ldots, n \\
		 & \phi_b(x_k^{b_0}, x_{k+1}^{b_1}) \, \{\le, =\} \, 0, &  & \phi_b \equiv \Pi(p, \tilde{p})\,, \forall p \subseteq s_k \,, \forall \tilde{p} \subseteq s_{k+1}\,, k=0, \ldots, K-1.
	\end{align}
\end{subequations}

Given a fixed task plan \(\langle a_1, \ldots, a_K \rangle\), the sequence of discrete states is \(\langle s_0, \ldots, s_K \rangle\).
The continuous states can be computed by solving the continuous feasibility program,
\begin{subequations}
	\label{eq:nlp}
	\begin{align}
		                                                          & \text{find} \, x_k^i \in \mathcal{X}^i, &                                                                                                                       & k= 0, \ldots, K,~ i=1, \ldots, n \\
		                                                          & \text{s.t.
		} \, \phi_b(x_k^{b_0}, x_{k+1}^{b_1}) \, \{\le, =\} \, 0, &                                         & \phi_b \equiv \Pi(p, \tilde{p}) \,,\forall p \subseteq s_k, \forall \tilde{p} \subseteq s_{k+1}, \, k=0, \ldots, K-1.
	\end{align}
\end{subequations}
Therefore, a valid task plan is only a necessary condition for the existence of a full discrete and continuous solution and, in practice, valid discrete plans often fail at the continuous level.

\paragraph{PNTC and LGP for Task and Motion Planning}

When comparing PNTC and LGP, we observe that PNTC is a factored formulation.
The discrete state space is now factored: instead of an unstructured discrete state space \(\mathcal{S}\) as in LGP \(\eqref{eq:lgp}\), we now have a set of discrete variables \(\mathcal{V}\).
The continuous space in PNTC is also factorized, resulting in a Factored-NLP formulation of the trajectory optimization problems to compute the motion of robots and objects.
As highlighted in \cref{sec:bg:pddl,sec:factorization-discrete}, this factorization is naturally available in TAMP problems that involve multiple objects and robots.

A technical difference with respect to LGP is that here we introduce special variables to represent trajectories between keyframes within the continuous state.
Thus, a continuous state in PNTC includes both the keyframe configuration (i.e., the configuration exactly at the phase of the transition) and the trajectory from the last keyframe, while LGP uses the original configuration space as the continuous space (e.g., the joint values of the robot or the object pose for a single configuration).
This modification allows us to define pairs of discrete and continuous states and generates a beneficial structure in the Factored-NLP for conflict-based planning, which can be shown to be equivalent to the LGP formulation.
The Factored-NLP in PNTC can also be viewed as a combination of the full trajectory optimization problem and the keyframes bound of LGP within a single more structured optimization problem.

Notably, PNTC decomposes the nonlinear constraints that appear in the LGP problem
into a set of small constraints and introduces an explicit mapping \(\Pi\) that defines which parts of the discrete plan generate which constraints.

Constraints of the form \(h_{\text{path}}(x,s)\), \(h_{\text{switch}}(x;s,s')\), and
\(\tilde{h}_{\text{switch}}(x,x';s,s')\) in LGP (\cref{eq:lgp,eq:keyframesbound}) are decomposed into a set of smaller constraints \( \{ \phi_b \mid \phi_b \equiv \Pi(p,\tilde{p}), ~ \forall p \subseteq s, \tilde{p} \subseteq s' \}\) in PNTC \eqref{eq:nlp}.

\section{Factored-NLP: a Bidirectional Interface Between Task and Motion}
\label{sec:graphnlp}

Given a fixed sequence of discrete states \( \langle s_0 ,\ldots ,s_K \rangle \), we represent the optimization problem over the sequence of continuous variables \eqref{eq:nlp} as a Factored-NLP, denoted by \( G(\langle s_0, \ldots ,s_K \rangle) \).

\cref{sec:bg:structure} provides the basic definitions of Factored-NLPs, a detailed example for a Pick and Place task plan, and a discussion on scalability, generalization, and properties of this representation.
All Factored-NLPs previously shown in \cref{sec:bg:structure} have been generated using the PNTC formulation that we have formally defined here.

The set of variables and constraints of the Factored-NLP \( G(\langle s_0, \ldots ,s_K \rangle) \) is:
\begin{subequations}
	\begin{align}
		 & X_G = \{ x_k^i \mid ~ k = 0,\ldots, K, ~ i = 1 ,\ldots, n \},                                                                            \\
		 & \Phi_G = \{ \phi_b \mid \phi_b \equiv \Pi(p,\tilde{p}), ~ \forall p \subseteq s_k, \tilde{p} \subseteq s_{k+1}, ~ k = 0, \ldots ,K-1 \}.
	\end{align}
\end{subequations}
The factored structure of this Factored-NLP stems from the object and time factorization of the variable set \( X_G \) and the structured dependency between variables and constraints.

Factored-NLPs from the PNTC formulation have the following properties:

\begin{property}
	\label{property:localtime}
	(Local time connectivity)
	A variable vertex \( x_k^i \) is connected to constraints that are evaluated on
	variables with time index \( k \), \( k-1 \), or \( k+1 \).
\end{property}

\begin{property}
	\label{prop:subgraph-intro}
	(Factor time invariance)
	A sequence of partial states \(\seq{p_0}{p_L}\) induces a subgraph \( M(\seq{p_0}{p_L}) = (X_M \cup \Phi_M, E_M) \) with:
	\begin{subequations}
		\begin{align}
			 & \Phi_M = \{ \phi_b \mid \phi_b \equiv \Pi(p,\tilde{p}), ~ \forall p \subseteq p_l , \tilde{p} \subseteq p_{l+1}, ~ l = 0,\ldots, L-1 \},              \\
			 & X_M = \{ x^i_l \, \mid \exists \phi_b \in \Phi_M  ~\text{such that} ~ \phi_b ~\text{depends on}~ x^i_l, ~ l  = 0, \ldots, L ,~  i = 1 ,\ldots, I \}.
		\end{align}
	\end{subequations}
	Because the mapping from partial states to nonlinear constraints does not depend explicitly on the time index, if a sequence of states \(\seq{s_0}{s_K}\) contains a sequence of partial states \(\seq{p_0}{p_L}\) starting at any time index, that is,
	\begin{equation}
		\exists k \in \{0, \ldots, K-L\} ~\text{such that}~ p_l \subseteq s_{k+l}, ~ l=0,\ldots, L,
	\end{equation}
	then the Factored-NLP for the sequence of partial states
	\(M(\seq{p_0}{p_L})\) is a subgraph of the Factored-NLP of the full sequence of states \(G(\seq{s_0}{s_K})\),
	\begin{equation}
		M(\seq{p_0}{p_L}) \subseteq G(\seq{s_0}{s_K}).
	\end{equation}
\end{property}

\begin{definition}
	An infeasible subgraph of a Factored-NLP (i.e., a subset of variables and constraints) is minimal if, when removing any variables or constraints, the resulting optimization problem is feasible.
\end{definition}
\begin{property}
	\label{propo:minimal-connected}
	The minimal infeasible subgraph is connected.
	If the Factored-NLP \( G \) is not connected, the NLP associated with each connected component \( G_j \) can be solved independently, and \( \text{Feas}(G) = \bigwedge_j \text{Feas}(G_j) \).
\end{property}

\cref{property:localtime} and \cref{propo:minimal-connected} are used later for detecting minimal infeasible sets of constraints efficiently. \cref{prop:subgraph-intro} is essential in our conflict-based algorithm, as it ensures the correctness of the task reformulation, by allowing us to prune multiple candidate task plans with a single conflict.

\paragraph{Example domain}
\label{sec:example}

\begin{figure}
	\centering
	\includegraphics[width=.24\textwidth]{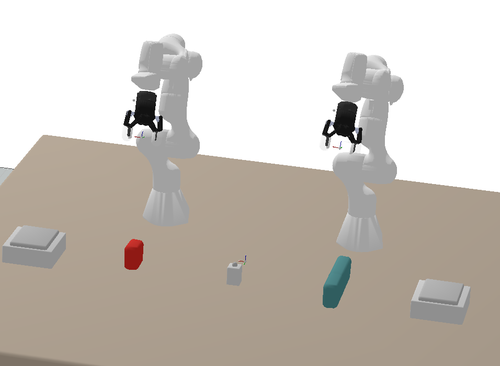}
	\includegraphics[width=.24\textwidth]{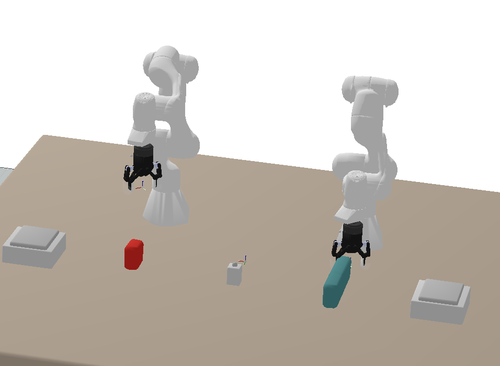}
	\includegraphics[width=.24\textwidth]{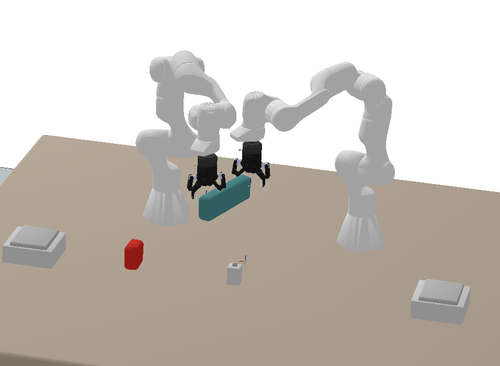}
	\includegraphics[width=.24\textwidth]{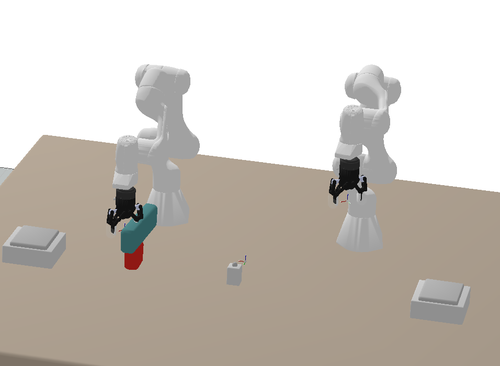}\\
	\caption{Example domain with two objects and two robots.}
	\label{fig:domain:graphnlpplanner}
\end{figure}

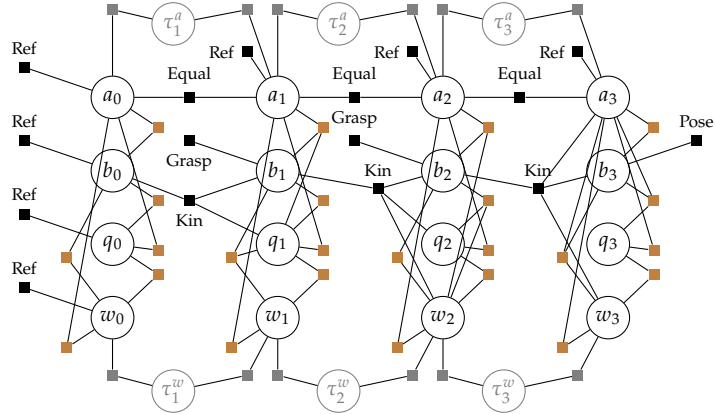
\begin{figure}
	\centering
	\input{details/factore_nlp_handover.tex}

	\caption{Factored-NLP of the example domain in~\cref{fig:domain:graphnlpplanner}.
		Circles represent variables and squares represent constraints.
		We display the variables for all the keyframe configurations \((a,b,q,w)\), and the trajectories \((\tau^a,\tau^w)\) (omitting
		\(\tau^b,\tau^q\) and factors that represent collisions between trajectories to keep the illustration cleaner).
		Brown squares are collision avoidance constraints.
		Gray squares are boundary constraints between trajectories and keyframes.
	}
	\label{fig:graph-in-planner-q}
\end{figure}

In this section, we revisit the example domain shown in \cref{sec:bg:example2}, but now provide a more formal and detailed description of the task plan, the constraints that appear in the Factored-NLP, and the relationships between them,
as defined by the PNTC formulation.

Thus, we consider again a domain with a \textit{Table} and two movable objects, \textit{A} and \textit{B}, initially on \textit{A\_init} and \textit{B\_init}, and two robot manipulators, \textit{Q} and \textit{W}, with the goal of stacking \textit{A} on top of \textit{B}.

The discrete state space is factorized into four variables: \texttt{parent\_A}, \texttt{parent\_B}, \texttt{robot\_Q}, \texttt{robot\_W}.
The set of possible actions is defined by two action operators: \texttt{pick} and \texttt{place}.
The initial discrete state is \( s_0 = \)
[\texttt{parent\_A = A\_init},
		\texttt{parent\_B = B\_init},
		\texttt{robot\_Q = free},
		\texttt{robot\_W = free}] (see \cref{ch:diverse_planning}).
A Factored-NLP is defined by a sequence of discrete states.
In this example, we choose the sequence of actions,
\def \aaONE {\texttt{Pick B with Q from B\_init}}
\def \aaTWO {\texttt{Pick B with W from Q}}
\def \aaTHREE {\texttt{Place B with W on A}}
\begin{itemize}
	\item[--] \( a_1 \): \aaONE,
	\item[--] \( a_2 \): \aaTWO,
	\item[--] \( a_3 \): \aaTHREE.
\end{itemize}
Applying these actions from the initial discrete state \( s_0 \) results in the state sequence,
\begin{itemize}

	\item[--]
		$s_0$:
		[\texttt{parent\_A = A\_init},
		\texttt{parent\_B = B\_init},
		\texttt{robot\_Q = free},
		\texttt{robot\_W = free}],

	\item[--]
		$s_1$:
		[\texttt{parent\_A = A\_init},
		\texttt{parent\_B = Q},
		\texttt{robot\_Q = full},
		\texttt{robot\_W = free}],

	\item[--]
		$s_2$:
		[\texttt{parent\_A = A\_init},
		\texttt{parent\_B = W},
		\texttt{robot\_Q = free},
		\texttt{robot\_W = full}],

	\item[--]
		$s_3$:
		[\texttt{parent\_A = A\_init},
		\texttt{parent\_B = A},
		\texttt{robot\_Q = free},
		\texttt{robot\_W = free}].

\end{itemize}

The Factored-NLP is shown again in \cref{fig:graph-in-planner-q}.
We refer to \cref{sec:bg:pick-and-place,sec:bg:example} for an explanation of variables and constraints, and we discuss here which sequences of partial discrete states imply which constraints.
Constraints operate on pairs of consecutive continuous variables, and the constraints that are applied depend on the values of the discrete variables.

First, note that the variables for the continuous initial state \( x_0 \) (i.e., \( a_0, b_0, q_0, w_0 \)) are also added in the Factored-NLP, together with constraints \textit{Ref} that fix their value.
The mapping \( \Pi: (p,\tilde{p}) \mapsto \phi \) in PNTC can be implemented as a set of rules that, given a fixed task plan, analyzes all pairs of partial states and generates the constraints in the Factored-NLP.
For instance:

\begin{itemize}
	\item
  \texttt{Parent\_B = B\_init and Parent\_B' = Q} \(\to\) \textit{Kin}\((b,b',q')\) 
	      This transition occurs in \( s_0 \to s_1 \) and generates the constraint
	      \textit{Kin} between variables \( b_0, b_1, q_1 \) in the Factored-NLP.

    \item \texttt{Parent\_A = A\_init} \(\to\) \textit{Ref}\((a)\)
	      In all discrete states \( s_0, s_1, s_2, s_3 \), the variable \texttt{Parent\_A} has the value \texttt{A\_init}; the constraint \textit{Ref} is applied to variable \( a \) in all time steps: \( a_0, a_1, a_2 \) and \( a_3 \).

    \item \texttt{Parent\_A = A\_init  and Parent\_A' = A\_init} \(\to\) \textit{Equal}\((a,a')\).
	      In this example, object A is always in the start position, and thus we add the constraints \textit{Equal}\((a_0,a_1)\), \textit{Equal}\((a_1,a_2)\), and \textit{Equal}\((a_2,a_3)\).
	      Note that in this case, such constraints are redundant with \textit{Ref}\((a_1)\), \textit{Ref}\((a_2)\), \textit{Ref}\((a_3)\).
	      However, \textit{Equal} constraints are necessary, e.g., when placing an object on the table, to ensure it remains still, or when holding an object for multiple time steps.

    \item \texttt{Parent\_B = Q } \(\to\)  \textit{Grasp}\((b)\).
	      In \( s_1 \), \texttt{Parent\_B = Q}.
	      Therefore, we add the constraint \textit{Grasp}\((b_1)\).

	\item  Collision constraints (brown squares) are added at all time steps.
	      They account for the different structures in the kinematic chain (e.g., whether the object is held by the robot or is on the table) -- resulting in a slight variation of dependencies in each vertical slice of the grasp.

	\item  Boundary value constraints (gray squares) tie the keyframes and the trajectories at all time steps.
	      For instance, \( \tau_1^q \) is constrained to start at \( q_0 \), and end at \( q_1 \).

\end{itemize}

\section{Overview: Factored-NLP Planner}
\label{sec:overview}

\cref{fig:flowchart} provides an overview of the Factored-NLP Planner for solving a PNTC, which we will introduce in the subsequent sections.
To simplify the presentation, we briefly outline the steps of the algorithm, which are run iteratively:

\begin{enumerate}

	\item
	      We leverage
	      a state-of-the-art discrete PDDL planner to find a sequence of discrete states that are valid for the current discrete planning task.

	\item
	      We generate the Factored-NLP that represents the continuous optimization problem for the continuous variables and the nonlinear constraints associated with the chosen candidate task plan.

	\item
	      An NLP solver attempts to solve the Factored-NLP.
	      If this NLP is feasible, the algorithm terminates, and the output is a solution containing a sequence of discrete and continuous states.
	      Otherwise, a minimal conflict in the form of a minimal infeasible subgraph (i.e., a subset of the Factored-NLP) is extracted, and all evaluated subgraphs are stored in a database as either feasible or infeasible subgraphs.

	\item
	      Finally, we reformulate the planning task to forbid all plans that would generate a Factored-NLP containing any subgraph previously determined to be infeasible.

\end{enumerate}

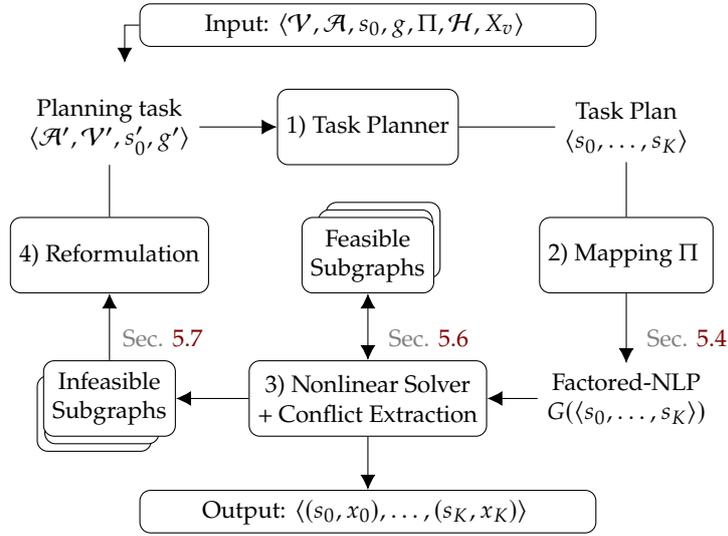
\begin{figure}
	\centering
	\begin{tikzpicture}[node distance=1.9cm,
			every node/.style={fill=white, font=\small}, align=center]
		\node (input)     [inout]          {Input: $\langle \mathcal{V}, \mathcal{A},s_0,g, \Pi, \mathcal{H},  X_v \rangle$ };
		\node (planner)     [process,below of=input,yshift=.55cm]          {1) Task Planner};

		\node (task)   [left of=planner,xshift=-1.5cm]          {Planning task \\
			$\langle \mathcal{A}', \mathcal{V}', s_0', g' \rangle$

		};

		\node (plan)     [right of=planner,xshift=1.5cm]          {Task Plan\\ $\langle s_0,\ldots, s_K \rangle$};

		\node (formulation)     [process,below of=plan,yshift=.2cm]          {2) Mapping $\Pi$   };

		\node (graphlgp)     [below of=formulation]          {Factored-NLP \\ $ G( \langle s_0 ,\ldots, s_K \rangle)$ };

		\node (buu)   [data, below of=planner,yshift=+0.4cm, xshift=.2cm]          {};
		\node (buu)   [data, below of=planner,yshift=+0.3cm, xshift=.1cm]          {};

		\node (dataFeas)   [process,below of=planner, yshift=.2cm]          {Feasible  \\Subgraphs};

		\node (motionplanner)   [process, below of=dataFeas]          {3) Nonlinear Solver\\ + Conflict Extraction };

		\node (reformulation)   [process, below of=task, yshift=.2cm ]          {4) Reformulation};

		\node (ref_feas)   [sec, below of=dataFeas, yshift=+.8cm, xshift=.8cm]          {Sec. \ref{sec:memory}};

		\node (ref_form)   [sec, below of=formulation, yshift=+.8cm, xshift=.8cm]          {Sec. \ref{sec:graphnlp}};

		\node (output)   [inout, below of=motionplanner,yshift=+.4cm]   {Output: $\langle  (s_0,\xs_0), \ldots , (s_K,\xs_K)  \rangle $ };

		\draw[-]             (planner) -- (plan);

		\node (buu)   [data, below of=reformulation,yshift=-0.2cm, xshift=-.2cm]          {};
		\node (buu)   [data, below of=reformulation,yshift=-0.1cm, xshift=-.1cm]          {};

		\node (dataInfeas)   [process,below of=reformulation]          {Infeasible \\ Subgraphs };

		\draw[->]             (motionplanner) -- (output);

		\draw[<->]             (motionplanner) -- (dataFeas);
		\draw[->]             (motionplanner) -- (dataInfeas);
		\draw[->]             (input) -- ++(-3.2,0) -- ++(0,-.5)  ;
		\draw[-]             (plan) -- (formulation);

		\draw[->]             (formulation) -- (graphlgp);
		\draw[->]             (graphlgp) -- (motionplanner);

		\node (ref_sec)   [sec, below of= reformulation, yshift=+.8cm, xshift=.7cm]          {Sec. \ref{sec:reformulate}};

		\draw[->]              (dataInfeas) -- (reformulation);

		\draw[-]             (reformulation) -- (task);
		\draw[->]             (task) -- (planner);
	\end{tikzpicture}
	\caption{Overview of the Factored-NLP Planner for solving a PNTC problem  $\langle \mathcal{V}, \mathcal{A},s_0,g, \Pi,  \mathcal{H} ,
			X_v   \rangle$.
		The solution is a sequence of discrete and continuous states $\langle (s_0,\xs_0), \ldots , (s_K,\xs_K) \rangle $.
	}
	\label{fig:flowchart}
\end{figure}

\section{Finding Small Infeasible Subgraphs}
\label{sec:find-minimal}

In this section, we discuss how to detect a minimal subset of infeasible constraints from a Factored-NLP (Step \num{3} of the Factored-NLP Planner, \cref{fig:flowchart}).
In the worst case, finding an infeasible subgraph of \emph{minimum cardinality} requires solving an NLP for each subset of constraints, \( O(2^{|\Phi_G|}) \) \cite{shoukry2018smc}.
Conversely, a \emph{minimal} infeasible subgraph can be found by solving a linear number of problems \cite{amaldi1999some}.
This search can be accelerated with a divide-and-conquer strategy, with complexity \( O(\log |\Phi_G|) \) \cite{junker2004preferred}.
Recently, \cite{shoukry2018smc} presented a technique for finding an approximately minimal subgraph in a convex optimization problem by solving one convex program with slack variables.

Inspired by these works, we propose an algorithm for finding small minimal infeasible subgraphs that exploits the particular structure of the Factored-NLP in our setting, namely the time structure and the semantic information contained within them, as well as the convergence point of the nonlinear optimizer.

\paragraph{Double binary search on the time index}

The first key insight is to exploit the time connectivity of our Factored-NLP (\cref{property:localtime}).
Given an infeasible Factored-NLP \( G(\langle s_0, \ldots, s_K \rangle) \), we can find a minimal temporal sequence \( \langle s_f, \ldots, s_l \rangle, ~ 0 \le f \le l \le K \) such that \( G(\langle s_f, \ldots, s_l \rangle) \) is infeasible with a double binary search that executes \( O(\log K) \) calls to a nonlinear optimizer.
Specifically, we first compute the minimum upper index \( l \) such that \( G(\langle s_0, \ldots, s_l \rangle) \) is infeasible.
After fixing \( l \), we compute the maximum lower index \( f \) such that \( G(\langle s_f, \ldots, s_l \rangle) \) is infeasible.

\paragraph{Relaxations}
\label{sec:relax}

Binary search on time exploits the local connectivity in the temporal dimension but does not detect the infeasible factors within an infeasible temporal sequence.
To address this issue, we propose solving a set of relaxations of the Factored-NLP that evaluate only a subset of variables and constraints.
Each relaxation corresponds to a subgraph of the Factored-NLP and is, therefore, a necessary condition for feasibility.
The algorithm stores the infeasible relaxations as candidates for the minimal subgraph.

The relaxations depend on the semantic information of the variables and constraints and are problem-independent but domain-specific.
Intuitively, we are looking for relaxations that make the graph sparser, smaller, and potentially disconnected, while keeping those constraints that define the infeasible subgraph.
\cref{sec:relax_tamp} presents informative relaxations in the context of Task and Motion Planning.

\paragraph{Leveraging the convergence point of the optimizer}

A powerful heuristic to discover a smaller infeasible subset of variables and constraints is to check the convergence point of the optimizer in an infeasible graph.

Typical optimization methods also converge for infeasible Factored-NLP \( G \), and we can use the convergence point as a heuristic guess to find a subgraph of \( G \) that is infeasible.
Specifically, we test the subgraph spanned by the constraints violated at the convergence point, i.e., \( M' = ( X' \cup \Phi' , E' ) \) where \( \Phi' \subseteq \Phi_G \) is the set of constraints not fulfilled, and \( X' = \{ x_i^k \in X_G \mid \exists \phi_b \in \text{Neigh}(x_i^k) ~ \text{s.t.
} ~ \phi_b \in \Phi' \} \).
If \( M' \) is also infeasible, we consider only \( M' \) as a candidate for the minimal infeasible subgraph.

\paragraph{The complete algorithm}

We combine these three ideas into one algorithm to find an infeasible subgraph.
In this algorithm, we alternate between applying relaxations (each relaxation considers only a subset of variables and constraints) that potentially break the full problem into disconnected components, and computing the minimal infeasible time slice inside each connected component (with a double binary search).
The convergence point of the optimizer is used to reduce the size of the output infeasible subgraph.
The algorithm will return the first infeasible subgraph it finds, and therefore it is best to try the relaxations in a \textit{loose} to \textit{tight} order, as this will likely result in a smaller infeasible subgraph (see Alg.
\num{1} of Appendix A in our paper for the implementation details).

Deciding whether a relaxation should be applied before or after the binary search on the time index is rather arbitrary.
To this end, a relevant observation is that solving a small NLP that is feasible is usually an order of magnitude faster than checking that a larger NLP is infeasible.
Thus, we try to solve numerous small and feasible problems first.

\paragraph{Database of feasible subgraphs}

\label{sec:memory}

The graph structure of the Factored-NLP is a suitable representation to share information about feasibility between different sequences of discrete states.
Factored-NLPs of different task plans contain common subgraphs, which correspond to sequences of partial states that appear in both plans (potentially at different time indices).

During the execution of the Factored-NLP Planner (see \cref{fig:flowchart}), all solved subgraphs are stored either in a feasible or an infeasible database.
Before solving a new nonlinear program, we check if it corresponds to a subgraph of any graph in the feasible database.
This check requires a graph isomorphism test \cite{cordella2004sub}, based on the adjacency structure and semantic information of the variable-vertices (the variable index \( i=1, \ldots, n \) and the name of the constraint \( \phi \in \mathcal{H} \), without considering the time index).
Given the available semantic information, the test is fast in practice (with complexity closer to \( O\left(\left(nK\right)^2\right) \) instead of the worst-case exponential).

\paragraph{Infeasible subgraphs in TAMP}
To conclude the section, \cref{fig:subgraphs} provides two examples of possible infeasible subgraphs of the Factored-NLP of the example domain (\cref{fig:cg_example2}), together with an intuitive explanation of the underlying reason for the continuous infeasibility.

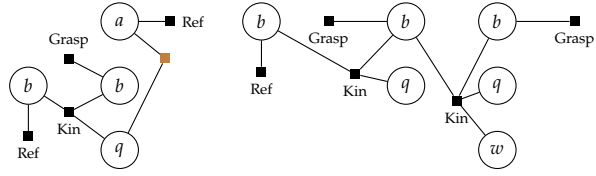
\begin{figure}
	\centering

	\begin{tabular}{cc}

		\begin{tikzpicture}[scale=0.7,every node/.style={transform shape}]
			\node[latent] (b0) {$b$} ;
			\node[latent,right=1 of b0  ] (b1) {$b$} ;
			\node[latent,above=.5 of b1] (a1) {$a$} ;
			\node[latent,below=.5 of b1] (q1) {$q$} ;
			\factor[left=.5 of b1, yshift=0.5cm] {trajp0} {above:Grasp} {b1} {};
			\factor[left=.5 of b1, yshift=-.5cm] {trajp0} {below:Kin} {b0, q1,b1} {};
			\factor[below=1 of b0, yshift=0.5cm] {trajp0} {below:Ref} {b0} {};
			\factor[right=.4 of a1,yshift=-.7cm,color=brown] {} {} {a1,q1} {};
			\factor[right=.5 of a1] {trajp0} { right:Ref } {a1} {};
		\end{tikzpicture} &
		\begin{tikzpicture}[scale=0.7,every node/.style={transform shape}]
			\node[latent] (b0) {$b$} ;

			\node[latent,right=2 of b0] (b1) {$b$} ;
			\node[latent,below=.5 of b1] (q1) {$q$} ;

			\node[latent,right=1 of b1] (b2) {$b$} ;
			\node[latent,below=.5 of b2] (q2) {$q$} ;
			\node[latent,below=.5 of q2] (w2) {$w$} ;

			\factor[below=1 of b0, yshift=0.5cm] {trajp0} {below:Ref} {b0} {};
			\factor[left=1 of b1 ] {trajp0} {below:Grasp} {b1} {};
			\factor[right=1 of b2] {trajp0} {below:Grasp} {b2} {};

			\factor[left=.5 of b1, yshift=-1cm] {trajp0} {below:Kin} {b0, q1,b1} {};
			\factor[right=.5 of b1, yshift=-1.5cm] {trajp0} {below:Kin} {b1, q2,b2,w2} {};
		\end{tikzpicture}

	\end{tabular}

	\caption{
		Two examples of possible infeasible subgraphs of the Factored-NLP shown in \cref{fig:graph-in-planner-q}.
		\textit{Left}: the robot \textit{Q} cannot pick up object \textit{B} from its initial position if \textit{A} is also in the initial position, i.e., \textit{A} blocks the grasp of \textit{B}.
    \textit{Right}: It is not possible to pick up object \textit{B} with the robot \textit{Q} and then do a handover to robot \textit{W}, e.g., due to kinematic constraints, robot \textit{Q} can only pick up the object in a certain way that prevents a handover later.
	}
	\label{fig:subgraphs}
\end{figure}

\section{Reformulation of the Discrete Planning Task}

\label{sec:reformulate}

In this section, we discuss how to reformulate the planning task with information about the continuous infeasibility (Step 4 of the Factored-NLP Planner, \cref{fig:flowchart}).
Specifically,
given an infeasible subgraph \( M \), we modify the discrete planning task \( \langle \mathcal{V}, \mathcal{A}, s_0, g \rangle \) to ensure that the discrete planner will never generate plans whose Factored-NLP contains \( M \).
The mapping is achieved through a two-step process:

First, we translate the infeasible subgraph \( M = (X_M \cup \Phi_M, E_M) \) into a sequence of discrete partial states \( \langle p_0, \ldots, p_L \rangle \).
Recall that each constraint \( \phi \in \Phi_M \) was generated by the mapping \( \Pi: (p, \tilde{p}) \mapsto \phi \).
We now trace this mapping back to obtain \( (p, \tilde{p}) \) which generated \( \phi \), maintaining the relative order of the partial states.
Importantly, we use \( p_0 \) to denote the first partial state that appears in the conflict, which could correspond to any step \( k=1,\ldots,K \) of the task plan.

Given an infeasible sequence of partial states \seq{p_0}{p_L}, we introduce a compilation that eliminates plans containing \seq{p_0}{p_L} starting at any time index, similarly to the plan forbidding compilation \cite{katz2018novel}.
Our compilation introduces binary variables \( l=0, \ldots, L \) to indicate whether the path from \( s_0 \) to \( s_K \) contains the infeasible sequence of partial states.
Specifically, \( b_l = 1 \) indicates that the current path contains the first \( l+1 \) elements of the infeasible sequence.

Given a planning task \( \langle \mathcal{V}, \mathcal{A}, s_0, g \rangle \) and an infeasible sequence \seq{p_0}{p_L},
the new discrete planning task is \( \langle \mathcal{V'}, \mathcal{A'}, s_0', g' \rangle \), where:
\begin{itemize}
	\item \( \mathcal{V'} = \mathcal{V} \cup \{ b_0, \ldots, b_L \} \),
	\item \( s_0' = s_0 \cup \{ b_l = 0 \mid l=1, \ldots, L \} \cup \{ b_0 = 1 ~ \text{if} ~ p_0 \subseteq s_0 ; ~ b_0 = 0 ~ \text{otherwise} \} \),
	\item \( g' = g \cup \{ b_L = 0 \} \),
	\item \( \mathcal{A'} = \{ a' = \text{mod}(a), ~ a \in \mathcal{A} \} \),
\end{itemize}
where \( a' = \text{mod}(a) \) modifies action \( a \) by adding conditional effects to ensure that if action \( a \) is executed when \( b_{l-1} = 1 \), and executing \( a \) makes \( p_l \) true, then \( a \) sets \( b_l = 1 \) and \( b_{l-1} = 0 \).
Alternatively, if \( a \) is executed when \( b_{l-1} = 1 \), and it does not make \( p_l \) true, then \( a \) sets \( b_{l-1} = 0 \).
The last binary variable \( b_L \) cannot transition from \( 1 \) to \( 0 \) (i.e., \( b_L = 1 \) is a dead end).
The formal reformulation \( a' = \text{mod}(a) \) is shown in Appendix B of our paper.

We can now state the proposition which shows that this compilation eliminates exactly all solutions which satisfy \seq{p_0}{p_L}.
\begin{proposition}
	\label{thm:forbid}
	Let \( T =
	\langle \mathcal{V}, \mathcal{A}, s_0, g \rangle \) be a SAS+ planning task, \seq{p_0}{p_L} be some infeasible sequence, and \( T' = \langle \mathcal{V'}, \mathcal{A'}, s_0', g' \rangle \) be the reformulation described above.
	A plan \( \pi \) is a solution of \( T' \) if and only if \( \pi \) is a solution of \( T \) and the states along \( \pi \) do not contain any sequence of states \( \langle s'_k, \ldots, s'_{k+L} \rangle \) such that \( p_l \subseteq s'_{k+l} \) for \( l = 0, \ldots, L \), starting at any \( k \).
\end{proposition}

Multiple infeasible sequences are forbidden by iterative reformulation.
We are now ready to discuss the properties of our Factored-NLP Planner (\cref{fig:flowchart}):
\begin{theorem}
	\label{thr:thetheorem}
	If the underlying classical planner is sound and complete, and the nonlinear optimizer always finds a feasible solution if one exists, then the Factored-NLP Planner is sound and complete.
\end{theorem}

\textit{Proof Sketch:}
The proof follows from the fact that any sequence we forbid cannot be part of any feasible solution (because it generates a Factored-NLP with a subgraph found to be infeasible, \cref{prop:subgraph-intro}), together with the fact that our compilation eliminates only plans which contain these sequences (\cref{thm:forbid}).
The completeness of the algorithm does not require the infeasible subgraphs to be minimal, nor the mapping \( \Pi \).
Nevertheless, these properties are desirable for an efficient algorithm.

\section{Experimental Results}

\label{sec:experimental_results}

\begin{figure}
	\centering
	\includegraphics[width=.3\textwidth]{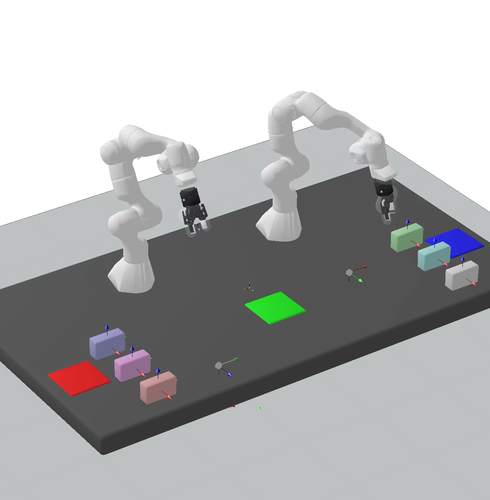}
	\includegraphics[width=.3\textwidth]{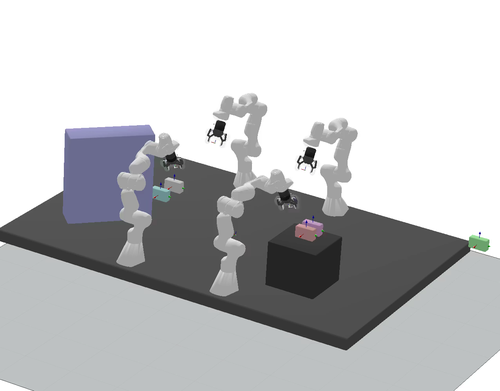}
	\includegraphics[width=.3\textwidth]{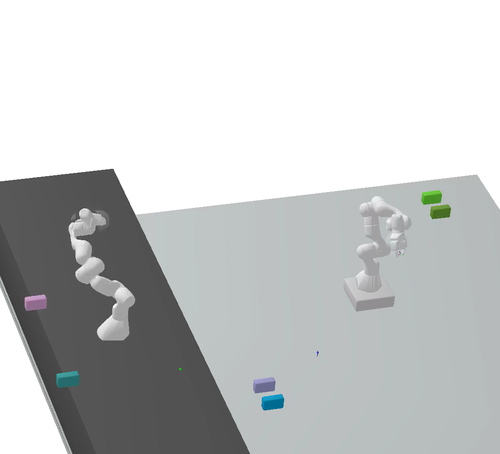}
	\caption{Three environments used in the benchmark. \textit{Left}: \textit{Laboratory}. \textit{Center}: \textit{Workshop}. \textit{Right}: \textit{Field}.}
	\label{fig:planner:results}
\end{figure}

Our algorithm is evaluated in three different simulated scenarios, where the goal is to move obstacles, rearrange, and stack up to six blocks to build towers with several robots.
The evaluation on real robots is reported in \cref{sec:real}.
\begin{enumerate}
	\item \textit{Laboratory (Lab)}: Two 7-DOF manipulator arms execute pick and place actions to build a tower.
	      The solution requires handovers, regrasping, and removing obstacles.
	      It is based on a real-world setting, \cref{fig:real}.
	\item \textit{Workshop (Work)}: An extension of the \textit{Laboratory} scenario that includes four robots and a stick that can be grasped and used as a tool to reach blocks,  \cref{fig:showcase}.
	\item \textit{Field}: Contains a fixed 7-DOF manipulator and a mobile 7-DOF manipulator, with two additional action operators: \textit{start-move} and \textit{end-move}, for moving the base of the mobile robot on the floor, \cref{fig:showcase}.
\end{enumerate}

For each scenario, we generate five different problems (e.g., \textit{Lab\_\{1,2,3,4,5\}}) by modifying the discrete goal and the initial configuration to increase complexity at both the discrete and geometric levels, while keeping the number of movable objects constant (except for the easier versions \textit{Work\_1} and \textit{Field\_1}).

\subsection{Relaxations for Finding Infeasible Subgraphs}
\label{sec:relax_tamp}

The formulation \textit{PNTC} and the solver are general and domain-independent.
Domain knowledge is introduced through the relaxations used for extracting minimal conflicts (\cref{sec:relax}).
The following relaxations (applicable in any TAMP problem) are used in the benchmark scenarios:
\begin{itemize}
	\item \emph{Removal of trajectories}: The remaining graph only contains variables for the keyframes, considerably reducing the dimensionality of the nonlinear program while still detecting most of the geometric infeasibilities.
  We note that this is the most important relaxation because the dimensionality of the underlying nonlinear program is reduced considerably (a trajectory is represented with, e.g., \num{20} waypoints, while a keyframe variable corresponds to a single waypoint).

	\item
	      \emph{Removal of collision constraints}:
	      Collision constraints connect all robot configuration and object pose variables at the same time step, resulting in a densely connected graph (see \cref{fig:cg_example2}).
	      Without collisions, the graph becomes sparse, and object and robot variables are only connected by grasping, kinematics, and placement constraints.
	\item \emph{Removal of time consistency}:
	      Time-consistency constraints (\textit{Equal} in \cref{fig:cg_example2}) appear when objects are not modified by a discrete action.
	      This relaxation does not consider the long-term dependencies of the manipulation sequence and creates a sparse time structure.
	\item \emph{Removal of robot variables}:
	      The remaining graph considers only the variables for the objects, detecting infeasible placements due to collisions between objects.
\end{itemize}

\begin{sidewaystable}[ph!]
	\small
	\caption{Number of NLP evaluations and CPU time, averaged over 10 randomly seeded runs, with standard deviations in gray.}
	\begin{tabular}{lllrrrrrrrrrrrr}
		\toprule
		{}       & \multicolumn{2}{c}{\textit{length}} & \multicolumn{2}{c}{\textit{One-way}} & \multicolumn{2}{c}{\textit{MBTS}}  & \multicolumn{2}{c}{\textit{FNPP\_t}} & \multicolumn{2}{c}{\textit{FNPP\_tr}} & \multicolumn{2}{c}{\textit{FNPP\_trn}} & \multicolumn{2}{c}{\textit{FNPP\_trng}}                                                                                                                                                                                                                                                                                                \\
		\cmidrule(lr){2-3}
		\cmidrule(lr){4-5}
		\cmidrule(lr){6-7}
		\cmidrule(lr){8-9}
		\cmidrule(lr){10-11}
		\cmidrule(lr){12-13}
		\cmidrule(lr){14-15}
		{}       & {$N_0$}                             & {$N$}                                & {\textit{NLP}}                     & {\textit{time}}                      & {{\textit{NLP}}}                      & {\textit{time}}                        & {{\textit{NLP}}}                        & {\textit{time}}                   & {\textit{NLP}}                      & {\textit{time}}                             & {\textit{NLP}}                      & {\textit{time}}                             & {\textit{NLP}}                     & {\textit{time}}                             \\
		\midrule
		Work\_1  & 2                                   & 4                                    & 77.0{\scriptsize \color{gray} 0.0} & 8.2{\scriptsize \color{gray} 1.0}    & 371{\scriptsize \color{gray} 58.5}    & 23.5{\scriptsize \color{gray} 4.4}     & 50.8{\scriptsize \color{gray} 12.0}     & 5.7{\scriptsize \color{gray} 1.2} & 53.0{\scriptsize \color{gray} 12.6} & 5.7{\scriptsize \color{gray} 1.3}           & 60.8{\scriptsize \color{gray} 10.5} & 5.6{\scriptsize \color{gray} 0.9}           & 55.2{\scriptsize \color{gray} 1.4} & \textbf{5.3}{\scriptsize \color{gray} 0.1}  \\
		Work\_2  & 4                                   & 6                                    & -                                  & -                                    & -                                     & -                                      & -                                       & -                                 & 113{\scriptsize \color{gray} 40.4}  & 22.7{\scriptsize \color{gray} 9.2}          & 94.7{\scriptsize \color{gray} 1.7}  & 19.7{\scriptsize \color{gray} 6.0}          & 86.8{\scriptsize \color{gray} 1.5} & \textbf{16.7}{\scriptsize \color{gray} 2.9} \\
		Work\_3  & 4                                   & 6                                    & -                                  & -                                    & -                                     & -                                      & -                                       & -                                 & 105{\scriptsize \color{gray} 37.0}  & \textbf{19.1}{\scriptsize \color{gray} 5.3} & 95.6{\scriptsize \color{gray} 1.0}  & 22.1{\scriptsize \color{gray} 5.9}          & 86.1{\scriptsize \color{gray} 1.5} & 21.4{\scriptsize \color{gray} 6.1}          \\
		Work\_4  & 8                                   & 10                                   & -                                  & -                                    & -                                     & -                                      & -                                       & -                                 & 282{\scriptsize \color{gray} 0.0}   & \textbf{53.1}{\scriptsize \color{gray} 5.6} & 307{\scriptsize \color{gray} 1.9}   & 56.4{\scriptsize \color{gray} 7.6}          & 270{\scriptsize \color{gray} 1.8}  & 55.4{\scriptsize \color{gray} 9.0}          \\
		Work\_5  & 8                                   & 11                                   & -                                  & -                                    & -                                     & -                                      & -                                       & -                                 & -                                   & -                                           & 355{\scriptsize \color{gray} 4.9}   & \textbf{76.0}{\scriptsize \color{gray} 7.0} & 309{\scriptsize \color{gray} 5.3}  & 76.8{\scriptsize \color{gray} 9.4}          \\
		Lab\_1   & 2                                   & 3                                    & 25.0{\scriptsize \color{gray} 0.0} & 7.1{\scriptsize \color{gray} 1.0}    & 25.0{\scriptsize \color{gray} 0.0}    & 4.1{\scriptsize \color{gray} 0.5}      & 21.0{\scriptsize \color{gray} 0.0}      & 4.5{\scriptsize \color{gray} 0.2} & 25.0{\scriptsize \color{gray} 0.0}  & \textbf{3.2}{\scriptsize \color{gray} 0.2}  & 30.0{\scriptsize \color{gray} 0.0}  & 3.3{\scriptsize \color{gray} 0.1}           & 28.0{\scriptsize \color{gray} 0.0} & 3.3{\scriptsize \color{gray} 0.2}           \\
		Lab\_2   & 2                                   & 3                                    & 12.0{\scriptsize \color{gray} 0.0} & 3.1{\scriptsize \color{gray} 0.5}    & 28.9{\scriptsize \color{gray} 0.3}    & 3.7{\scriptsize \color{gray} 0.5}      & 32.0{\scriptsize \color{gray} 0.0}      & 5.5{\scriptsize \color{gray} 0.3} & 46.0{\scriptsize \color{gray} 0.0}  & 4.3{\scriptsize \color{gray} 0.2}           & 23.0{\scriptsize \color{gray} 0.0}  & \textbf{2.1}{\scriptsize \color{gray} 0.1}  & 21.0{\scriptsize \color{gray} 0.0} & \textbf{2.1}{\scriptsize \color{gray} 0.1}  \\
		Lab\_3   & 4                                   & 5                                    & 19.0{\scriptsize \color{gray} 0.0} & 8.4{\scriptsize \color{gray} 1.2}    & 34.0{\scriptsize \color{gray} 0.0}    & 18.5{\scriptsize \color{gray} 2.7}     & 26.0{\scriptsize \color{gray} 0.0}      & 5.9{\scriptsize \color{gray} 0.4} & 23.0{\scriptsize \color{gray} 0.0}  & \textbf{3.1}{\scriptsize \color{gray} 0.2}  & 25.0{\scriptsize \color{gray} 0.0}  & 3.3{\scriptsize \color{gray} 0.4}           & 24.0{\scriptsize \color{gray} 0.0} & 3.2{\scriptsize \color{gray} 0.2}           \\
		Lab\_4   & 4                                   & 9                                    & -                                  & -                                    & -                                     & -                                      & -                                       & -                                 & -                                   & -                                           & 70.0{\scriptsize \color{gray} 0.0}  & 6.5{\scriptsize \color{gray} 0.6}           & 60.1{\scriptsize \color{gray} 3.5} & \textbf{6.3}{\scriptsize \color{gray} 0.4}  \\
		Lab\_5   & 12                                  & 17                                   & -                                  & -                                    & -                                     & -                                      & -                                       & -                                 & 87.0{\scriptsize \color{gray} 0.0}  & \textbf{18.7}{\scriptsize \color{gray} 1.8} & 93.0{\scriptsize \color{gray} 0.0}  & 19.1{\scriptsize \color{gray} 2.0}          & 83.0{\scriptsize \color{gray} 0.0} & 19.0{\scriptsize \color{gray} 2.0}          \\
		Field\_1 & 2                                   & 4                                    & 19.0{\scriptsize \color{gray} 0.0} & 7.0{\scriptsize \color{gray} 2.0}    & 90.0{\scriptsize \color{gray} 0.0}    & 15.3{\scriptsize \color{gray} 3.1}     & 19.0{\scriptsize \color{gray} 0.0}      & 5.7{\scriptsize \color{gray} 1.1} & 14.1{\scriptsize \color{gray} 0.3}  & 2.9{\scriptsize \color{gray} 1.4}           & 16.0{\scriptsize \color{gray} 0.0}  & \textbf{2.6}{\scriptsize \color{gray} 0.3}  & 16.0{\scriptsize \color{gray} 0.0} & 3.1{\scriptsize \color{gray} 1.0}           \\
		Field\_2 & 2                                   & 6                                    & -                                  & -                                    & -                                     & -                                      & -                                       & -                                 & 46.0{\scriptsize \color{gray} 0.0}  & \textbf{6.1}{\scriptsize \color{gray} 0.4}  & 53.0{\scriptsize \color{gray} 0.0}  & 6.3{\scriptsize \color{gray} 0.5}           & 52.5{\scriptsize \color{gray} 0.5} & 6.3{\scriptsize \color{gray} 0.5}           \\
		Field\_3 & 4                                   & 8                                    & -                                  & -                                    & -                                     & -                                      & -                                       & -                                 & 75.0{\scriptsize \color{gray} 0.0}  & \textbf{11.7}{\scriptsize \color{gray} 1.2} & 84.0{\scriptsize \color{gray} 0.0}  & 12.3{\scriptsize \color{gray} 1.4}          & 78.6{\scriptsize \color{gray} 0.5} & \textbf{11.7}{\scriptsize \color{gray} 0.7} \\
		Field\_4 & 6                                   & 10                                   & -                                  & -                                    & -                                     & -                                      & -                                       & -                                 & 67.0{\scriptsize \color{gray} 0.0}  & \textbf{13.2}{\scriptsize \color{gray} 1.5} & 77.0{\scriptsize \color{gray} 0.0}  & 13.6{\scriptsize \color{gray} 1.6}          & 76.0{\scriptsize \color{gray} 0.0} & 13.5{\scriptsize \color{gray} 1.3}          \\
		Field\_5 & 6                                   & 11                                   & -                                  & -                                    & -                                     & -                                      & -                                       & -                                 & 282{\scriptsize \color{gray} 0.0}   & 56.5{\scriptsize \color{gray} 6.6}          & 289{\scriptsize \color{gray} 1.0}   & \textbf{50.7}{\scriptsize \color{gray} 5.5} & 262{\scriptsize \color{gray} 0.5}  & 51.4{\scriptsize \color{gray} 6.6}          \\
		\bottomrule
	\end{tabular}
	\label{table:thetable}
\end{sidewaystable}

\subsection{Benchmark}

\paragraph{Algorithms under comparison}
We compare our approach with two different formulations that combine a discrete search with joint nonlinear optimization for solving Task and Motion Planning problems.
\begin{itemize}

	\item[--] \emph{One-way interface between Top-K Planning and a nonlinear optimizer (One-way)}.
		This baseline combines Top-K planning \cite{katz2018novel} to generate a set of different task plans with a nonlinear optimizer to evaluate the plans.
		The planner does not receive any information about the geometric reasons for infeasibility and only blocks the evaluated plans.

	\item[--] \emph{Multi-Bound Tree Search (MBTS)}.
		The MBTS Solver \cite{toussaint2017multi} incrementally builds a tree in a breadth-first order to explore sequences of discrete actions that reach the high-level goal.
		Instead of solving the full continuous optimization problem directly, MBTS first computes relaxed versions (\textit{bounds})
		that consider a subset of variables and constraints (see \cref{sec:multibound}).

	\item[--] \emph{Four Variations of our Factored-NLP Planner (in short, FNPP)}.
		We evaluate our full planner \textit{FNPP\_trng}, and
		three additional versions: \textit{FNPP\_t}, \textit{FNPP\_tr}, and \textit{FNPP\_trn} to conduct an ablation study of the algorithm to extract infeasible subgraphs.
		Suffixes indicate: $t$=time search, $r$=relaxation, $n$=convergence heuristic, and $g$=feasible graph database.

\end{itemize}

\paragraph{Metrics}

Each algorithm is run 10 times with different random seeds and a timeout of \num{100} seconds.
For each method, we report on the number of evaluated NLPs (\textit{NLP}) (lower is better) and the CPU time in seconds (\textit{time}) (lower is better) in \cref{table:thetable}.
``--'' indicates a failure to find a solution within \num{100} seconds with at least a \SI{70}{\percent} success rate.

\textit{Time}\footnote{Experiments are conducted on a Single Core i7-1165G7@2.80GHz} provides an objective way to compare algorithms that use different underlying methods.
The number of solved NLPs is informative but does not capture
the influence of the size and feasibility of NLPs on the running time of the solver.

For each problem, \textit{N} denotes the length of the shortest found task plan that is geometrically feasible, and $N_0$ is the length of the task plan that solves the initial discrete task (that is, without considering the continuous constraints).
\textit{N} and $N_0$ are proxies for the difficulty: the number of candidate plans typically grows exponentially with $N$, and the difference $N-N_0$ indicates the amount of
information about the continuous constraints that should be provided to the discrete planner.
The approximate branching factor is \num{12} in \textit{Lab\_\{1,2,3,4,5\}},
\num{13} in \textit{Field\_\{2,3,4,5\}}, \num{24} in \textit{Work\_\{2,3,4,5\}}, \num{4} in \textit{Work\_1}, and \num{5} in \textit{Field\_1}.

\paragraph{Comparison to baselines}

Concerning the problems solved, \textit{One-way} and \textit{MBTS} can only solve the easier problems in each scenario, while \textit{FNPP\_trn/trng} solves all the problems.
Our algorithm is significantly faster in the problems also solved by \textit{One-way} and \textit{MBTS}, because the more efficient encoding of geometric information reduces the running time.

The success rate of our planners \textit{FNPP\_trn/trng} is \SI{100}{\percent} in all problems except for \textit{Field\_5} (\SI{80}{\percent}), \textit{Work\_4} (\SI{95}{\percent}), and \textit{Work\_5} (\SI{90}{\percent}) where the optimizer fails to solve feasible Factored-NLPs in a few runs.
The performance of \textit{FNPP\_trng} is not affected by the branching factor of the underlying problem and provides good scaling with respect to \(N\) and \(N-N_0\).
The highest computational time corresponds to \textit{Field\_5} and \textit{Work\_5}, that require a long plan and detecting of collisions between movable objects.
In TAMP, the practical size of the Factored-NLPs is \(O(n^2 K)\) (where \(K\) is the length of the action sequence, and \(n\) is the number of objects and robots).
The domains are modelled using a small set (\(<20\)) of different types of nonlinear constraints (e.g., \cref{fig:cg_example2}).

\subsection{Ablation Study}

\begin{itemize}
	\item[--] \emph{Analysis of the relaxations:}
		\textit{FNPP\_t} detects conflicts of the form \seq{s_k}{s_{k+l}}, while \textit{FNPP\_tr} checks relaxations to generate smaller conflicts \seq{p_k}{p_{k+l}}.
		Small conflicts lead to more aggressive pruning of task plans and are essential to solve the harder problems (the number of solved problems is \num{5} versus \num{13} out of \num{15}).
		After an ablation study of each relaxation, we observe that the \textit{Removal of trajectories} and the \textit{Removal of collision constraints} are the most informative relaxations.

	\item[--] \emph{Analysis of the convergence heuristic:}
		The results show that the convergence heuristic is important in problems that require reasoning about the collisions between movable objects, e.g., when the robot must move one object before placing another to avoid a collision.
		In this case, the relaxations are not informative, while the convergence point of the optimizer in these infeasible problems usually indicates which objects are in collision.
		\textit{FNPP\_tr} solves \num{13} out of \num{15}, and  \textit{FNPP\_trn} solves \num{15} out of \num{15}.

	\item[--] \emph{Analysis of the database of feasible graphs:}
		\textit{FNPP\_trng} reduces the number of solved NLPs, from a total average of \num{1673} to \num{1508}, but there is no improvement in computational time.
		We conjecture that the database approach will provide higher benefits in a setting where solving the NLPs requires more time.

\end{itemize}

\begin{figure}[!t]
	\centering
	\includegraphics[width=.32\textwidth]{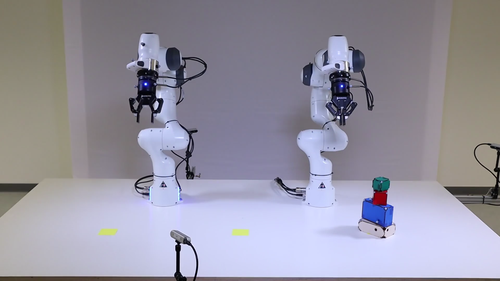}
	\includegraphics[width=.32\textwidth]{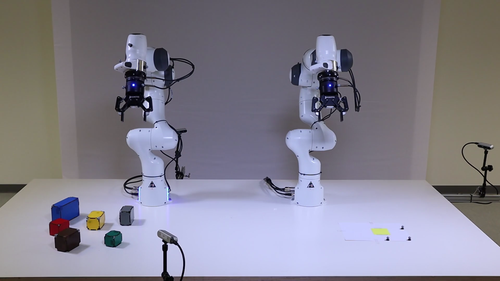}
	\includegraphics[width=.32\textwidth]{pics_nyu_inria/videos/real_obs_mid_fast.png}
	\caption{Three scenarios in the real world evaluation. \textit{Left}: \textit{Hanoi-Tower}. \textit{Center}: \textit{Tower}. \textit{Right}: \textit{Obstacles-Tower}. \vspace{.5cm}}
	\label{fig:planner:real}
\end{figure}

\subsection{Scalability}
We conduct two additional experiments in the \textit{Laboratory} scenario to explicitly evaluate the scalability of the method when increasing the number of blocks to be stacked (from \num{4} to \num{32}) and the number of movable obstacles on a cluttered table (from \num{1} to \num{6}).

Experimentally, the running time of the Factored-NLP Planner scales polynomially with the number of objects and plan length, and the practical bottleneck is the time spent on solving large nonlinear programs (with cubic complexity on the number of objects and linear on the plan length).

The main weakness of our method is that the nonlinear optimizer is not guaranteed to find a solution for a (sub)Factored-NLP even if one exists, given that the nonlinear constraints define a non-convex optimization problem (which could break the assumption in \cref{thr:thetheorem}).
However, the extensive experiments demonstrate that the solver is efficient and reliable in relevant use-cases of TAMP.

\subsection{Real-Time Planning in the Real World}
\label{sec:real}
We demonstrate our solver in a real-world version of the \textit{Laboratory} environment (two 7-DOF manipulators and up to \num{6} movable objects, see \cref{fig:real}.
The solver is integrated into a \textit{Sense}-\textit{Plan}-\textit{Act} pipeline, where we first perceive the scene with an external motion capture system, compute a full discrete and continuous plan, and then execute the plan.

The real-world evaluation consists of three scenarios: \textit{Tower}, \textit{Hanoi-Tower}, and \textit{Obstacles-Tower}, for a total of \num{11} problems (see \cref{fig:planner:real}).
The high-level goal is to build a tower of cubes at different locations: \textit{Hanoi-Tower} introduces the classical Hanoi logic constraints,
and \textit{Obstacles-Tower} requires plans that first move obstacles away.
The planning time differs across problems: \SI{2.8}{s} to build a tower of \num{6} blocks in the center of the table (\num{12} discrete actions),
\SI{8.8}{s} to remove two obstructing blocks and stack \num{4} blocks (\num{12} actions), \SI{9.4}{s} to build a Hanoi Tower (\num{12} actions), and \SI{27.2}{s} to solve a problem that requires removing obstructing blocks and transferring blocks from the left to the right side (\num{16} actions).
Recordings of planning and execution are shown on the project webpage\footnote{\url{https://quimortiz.github.io/graphnlp/}}.

\begin{figure}[!t]
	\begingroup
	\setlength{\tabcolsep}{2pt} %

	\begin{tabular}{ccc}

		\includegraphics[width=.32\linewidth]{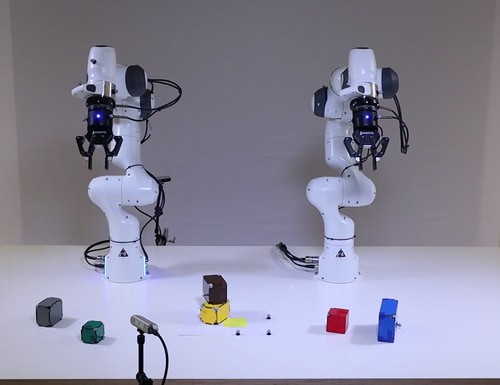} &
		\includegraphics[width=.32\linewidth]{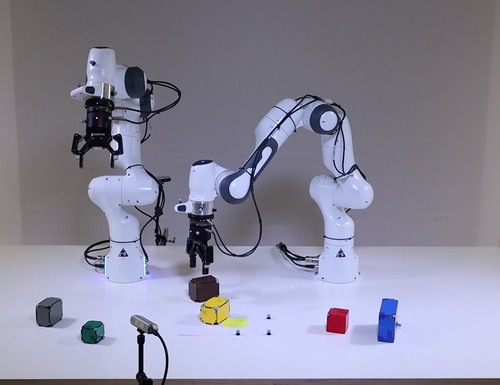} &
		\includegraphics[width=.32\linewidth]{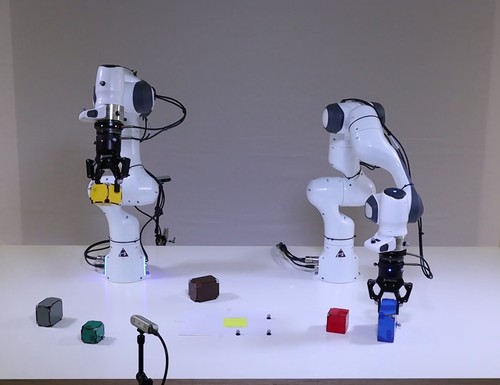}   \\

		\includegraphics[width=.32\linewidth]{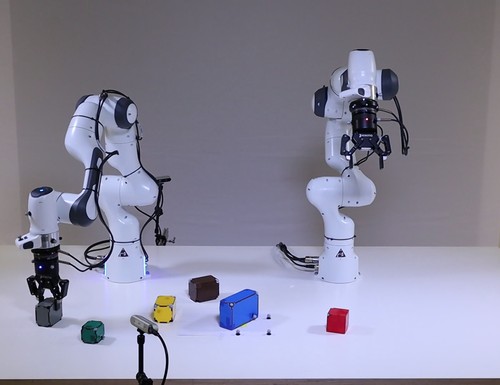} &
		\includegraphics[width=.32\linewidth]{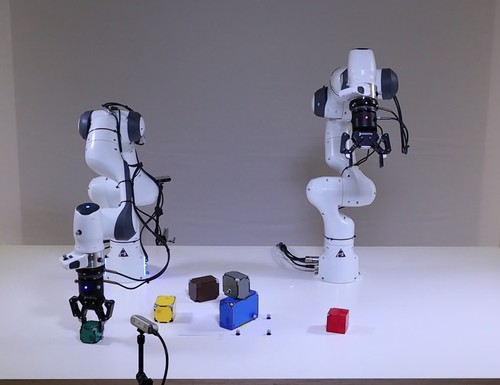} &
		\includegraphics[width=.32\linewidth]{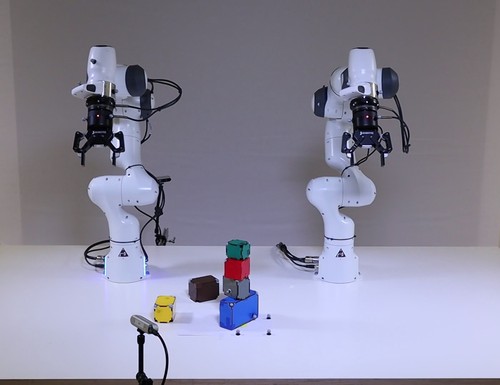}

	\end{tabular}
	\endgroup
	\caption{The goal in this real-world experiment is to build the tower \textit{blue-gray-red-green} in the central spot (highlighted in yellow).
    The solution, computed in only \SI{8.88}{s} with our planner, requires
    a task plan with \num{12} discrete actions, moving first the \textit{brown} and \textit{yellow} blocks to avoid collisions.
	}
	\label{fig:real}
\end{figure}

\section{Limitations}
\label{sec:graph-nlp:limitations}

Similarly to our first contribution, \textit{Diverse Planning for LGP}, described in \cref{ch:diverse_planning}, our TAMP solver shares the limitations of the underlying LGP formulation, namely the local convergence of optimization methods, which might prevent finding a solution even if a problem is feasible.

To address this issue, a possible solution is to store how many times a subgraph has been found to be infeasible and use this information in a soft conflict formulation (where infeasible subgraphs only penalize task plans) or in a probabilistic conflict formulation (see also the discussion in \cref{sec:diverse:limitations}).

From an implementation perspective, our method requires a factored-NLP formulation where any subset of constraints can be evaluated for feasibility.
This prevents the use of off-the-shelf trajectory optimization frameworks, which often assume a non-factored, unstructured trajectory optimization problem.

Solving large factored nonlinear programs with joint optimization can be inefficient and is more prone to converging to infeasible local minima.
A natural way to address this issue is to combine both joint optimization and conditional constrained sampling within a single solver, bridging the gap between sample-based and optimization-based approaches to TAMP.
A foundational first step in this direction is presented in the second part of the thesis (\cref{ch:mcts,ch:meta-solver}), where we analyze how to design TAMP solvers that automatically select between sampling and optimization operations.

\section{Conclusion}

We present a solver that combines nonlinear optimization and PDDL planning for the joint optimization of discrete and continuous variables in robotic planning.
The key contribution is the novel bidirectional interface between the task plan and the continuous constraints, realized through the detection of infeasible subgraphs and a reformulation to inform the task planner about subgraph infeasibility.
The problem formulation is formalized as \textit{PNTC}, which extends classical planning with nonlinear transition constraints.

Our experiments in Task and Motion Planning show that the algorithm is faster and more scalable than the Multi-Bound Tree search for LGP, while maintaining generality and using the same input information.
These results are further validated in real-world experiments, where our solver generates plans for two 7-DOF robots and six objects in a few seconds.

In this chapter, we formalize the factored structure of trajectory optimization problems first presented in \cref{sec:bg:structure} in an illustrative manner.
Besides planning, this structure is beneficial for reasoning about computational operations (see \cref{ch:mcts}) and learning (see \cref{ch:gans,ch:learn-feas}).

The structure of the Factored-NLPs in \cref{ch:mcts,ch:learn-feas} is slightly different, as these methods use minimal representations, compressing several variables that are constant into a single variable.
This results in a more compact representation but obscures the clear temporal structure presented here.
On the other hand, the Factored-NLPs in \cref{ch:learn-feas} are built using the same formalization of Planning with Nonlinear Transition Constraints but with a redundant representation of the continuous state, which allows for better generalization across different task plans when using graph neural networks.

%% file: mcts_keyframes.tex
\part{\namePartMetaSolver}

\chapter{\nameChapterThree}
\label{ch:mcts}

\section{Introduction}

A core component of Task and Motion Planning (TAMP) is generating the keyframe configurations that fulfil the geometric and physical constraints of a fixed task plan.
The sequence of keyframe configurations is often computed through either joint optimization or a predetermined series of sampling operations.
However, both naïve sequential conditional sampling of individual variables and full joint optimization of all variables at once can be highly inefficient and non-robust, depending on the geometric environment.
As an example, consider the keyframes for the task plan pick-handover-place as shown in \cref{fig:keyframes-handover}, where obstacles, kinematic, and grasp constraints pose challenges for both sampling and optimization methods.

In this chapter\footnote{
	This chapter is based on the publication:
	Ortiz-Haro, J., Hartmann, V.
	N., Oguz, O.
	S., and Toussaint, M.
	(2021).
	Learning Efficient Constraint Graph Sampling for Robotic Sequential Manipulation.
	IEEE International Conference on Robotics and Automation (ICRA) (pp.
	4606-4612).
}, we present a novel algorithm that learns how to break a factored nonlinear program into smaller problems to generate solutions more efficiently.

Our method relies on a factored representation of the feasibility problems and utilizes Monte-Carlo Tree Search to learn assignment orders for the variables, with the aim of minimizing the computation time to generate feasible full samples.
As an online learning algorithm, Monte-Carlo Tree Search is most valuable when multiple diverse solutions must be generated within a fixed computational budget, rather than seeking a single solution.
The sparse factored representation of the nonlinear program, as presented in \cref{sec:bg:structure}, together with local information about the dimensionality of individual variables and constraints, is used to reduce the space of possible computations.

We demonstrate that our learning method quickly converges to the best sampling strategy for a given problem, outperforming user-defined orderings or fully joint optimization, while also providing higher sample diversity.

This work is a first step towards meta-solvers for TAMP, where the ultimate goal is to develop a flexible TAMP solver that automatically selects between sequential sampling and joint optimization to compute robot motions, while also balancing the search in the discrete and continuous levels of the TAMP problem.

To this end, in this chapter we first consider the subproblem of generating keyframes for a fixed high-level manipulation plan.
Though not encompassing the entire TAMP problem, this is a vital component, as keyframes are later used as waypoints for subsequent trajectory optimization or sampling-based motion planning algorithms.
Especially, this setting already reveals the complex decision-making and infrastructure required to choose between sequential sampling or joint optimization operations.

In the following \cref{ch:meta-solver}, we extend some of the ideas and techniques presented here to address the full TAMP problem, where the high-level task plan is not fixed, but also optimized.

\begin{figure}
	\centering
	\begin{subfigure}[t]{.23\textwidth}
		\centering
		\includegraphics[width=.97\linewidth]{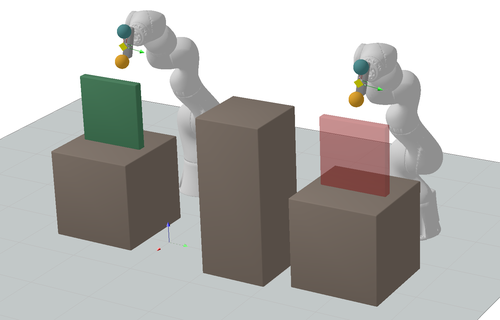}
	\end{subfigure}%
	\begin{subfigure}[t]{.23\textwidth}
		\centering
		\includegraphics[width=.97\linewidth]{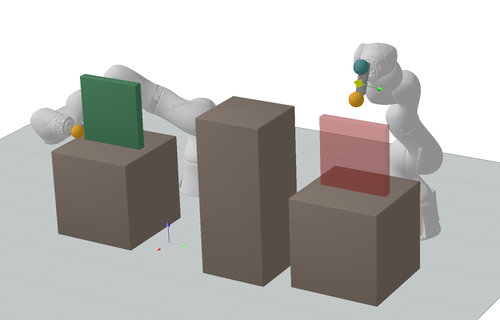}
	\end{subfigure}\\[1mm]

	\begin{subfigure}[t]{.23\textwidth}
		\centering
		\includegraphics[width=.97\linewidth]{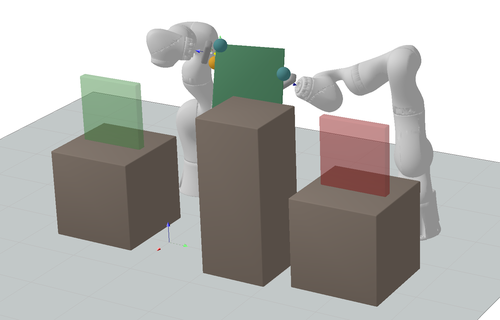}
	\end{subfigure}%
	\begin{subfigure}[t]{.23\textwidth}
		\centering
		\includegraphics[width=.97\linewidth]{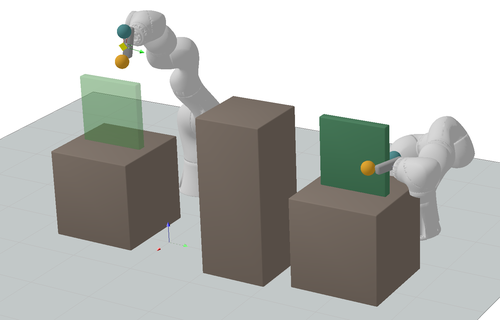}
	\end{subfigure}
	\caption{A sequence of keyframes in the \textit{Handover} problem.
		The fixed task plan is to move the green box to the target location (red) by performing an intermediate handover.
		Robot $Q$ picks the box (top-right) and hands it to robot $W$ (bottom-left) to place it in the goal position (bottom-right).
	}
	\label{fig:keyframes-handover}
\end{figure}

\section{Related Work}

\paragraph{Problem decomposition and constraint satisfaction}
In continuous optimization, decompositions utilizing the underlying problem structure can enhance the performance of algorithms both theoretically and practically~\cite{dantzig1960decomposition,bnnobrs1962partitioning, boyd2011distributed,bertsekas1979convexification}.

A constraint satisfaction problem (CSP) \cite{rossi2006handbook} reveals a graph dependency structure between variables and constraints and often presumes finite discrete domains for each variable.
Algorithms solving CSPs can exploit the graph structure, for example, identifying connected components or trees, to efficiently assign variables \cite{mouhoub2011heuristic}.
There is growing interest in sampling solutions that satisfy CSPs \cite{dechter2002generating, gogate2006new, ermon2012uniform}.
Generating multiple solutions to the problem at hand enhances the applicability of these methods in scenarios where some constraints cannot be modeled or are evaluated only after the fact \cite{danna2007generating}.

In robotics applications, such as
~\cite{19-driess-RSSws,toussaint2017multi, tonneau2018efficient, orthey2018quotient,hartmann2020robust} among many others, complex nonlinear problems are frequently decomposed into a sequence of simpler subproblems.
The solutions to the simpler problems are then used to guide optimization and sampling methods toward the solution of the comprehensive problem.

\paragraph{Meta-decision processes}
The optimization of computational operations is directly related to optimal metareasoning \cite{Russell91Principles}.
Metareasoning has varied applications in heuristic search~\cite{o2015metareasoning, lieder2014algorithm, zilberstein2011metareasoning, karpas2018rational}, as well as in other resource-limited planning processes~\cite{bratman1988plans, boddy1989solving}.
More recently, metareasoning has been applied to reinforcement learning~\cite{pascanu2017learning}, temporal planning~\cite{cashmore2018temporal}, and path planning \cite{mandalika2019generalized}.

In our work, we employ bandit algorithms~\cite{auer2002finite,kocsis2006bandit} to deliberate on the sequence of computational operations for generating solutions to a nonlinear program, where the decisions we make concern the selection of which variables to compute next.

\section{Sampling Sequences in the Pick and Place Task Plan}

\begin{figure}[h]
	\centering
	\begin{subfigure}[b]{.9\textwidth}
		\centering
		\includegraphics[width=.7\linewidth]{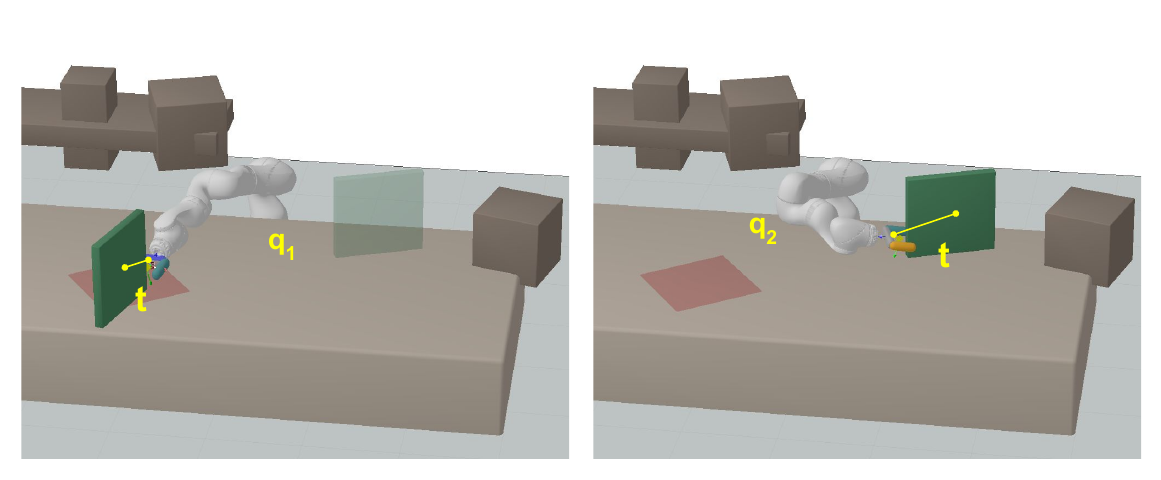}
		\caption{Visual representation of the Pick and Place problem.
			Left image shows the pick keyframe and right image shows the place keyframe.
			Variables are the robot configuration when picking $q_1$, the robot configuration when placing $q_2$, and the relative transformation between the gripper and the object $t$, which is kept constant in both keyframes because the grasp is stable.
		}
		\label{fig:mcts:pp}
	\end{subfigure}\vspace{1cm}

	\begin{subfigure}[b]{.9\textwidth}
		\centering
		\begin{tikzpicture}[scale=0.7,every node/.style={transform shape}]
			\node[latent] (zero) {} ; %

			\node[latent, below=of zero, xshift=5cm] (tq1q2) {$t,q_1,q_2$} ; %

			\node[latent, below=of zero, xshift=-2cm] (t) {$t$} ; %

			\node[latent, below=of zero] (tq1) {$t,q_1$} ; %
			\node[latent, below=of zero, xshift=2cm] (tq2) {$t,q_2$} ; %
			\node[latent, below=of tq2] (tq2-tq1q2) {$t,q_1,q_2$} ; %

			\node[latent, below=of t, xshift=-.5cm] (t-tq1) {$t,q_1$} ; %
			\node[latent, below=of t, xshift=.5cm] (t-tq2) {$t,q_2$} ; %
			\node[latent, below=of t-tq2, xshift=.5cm] (t-tq2-tq1q2) {$t,q_1,q_2$} ; %
			\node[latent, below=of t-tq1, xshift=-.5cm] (t-tq1-tq1q2) {$t,q_1,q_2$} ; %
			\node[latent, below=of tq1, xshift=0cm] (tq1-tq1q2) {$t,q_1,q_2$} ; %

			\factoredge{}{t-tq1}{t-tq1-tq1q2}
			\factoredge{}{zero}{tq2}
			\factoredge{}{tq2}{tq2-tq1q2}
			\factoredge{}{tq1}{tq1-tq1q2}
			\factoredge{}{zero}{tq1q2}
			\factoredge{}{zero}{t}
			\factoredge{}{zero}{tq1}
			\factoredge{}{t}{t-tq2}
			\factoredge{}{t-tq2}{t-tq2-tq1q2}
			\factoredge{}{t}{t-tq1}

		\end{tikzpicture}
		\caption{Possible sampling sequences for solving the Pick and Place problem.}
		\label{fig:sequence-pick-place}
	\end{subfigure}
	\caption{The Pick and Place task plan.}
	\label{fig:sequence-pick-place-all}
\end{figure}
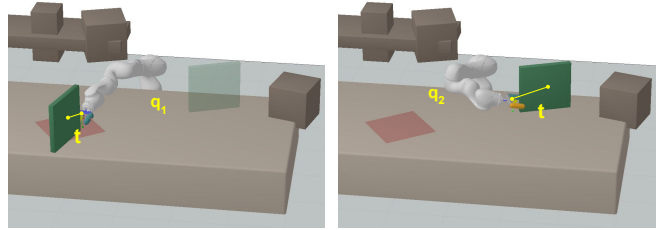
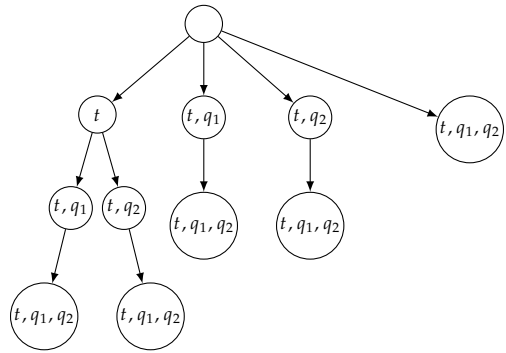

Consider the problem of computing the keyframe configurations for a Pick and Place task plan, defined by three vector variables: the robot configuration for picking and placing the object, and the relative transformation between the gripper and the object, \cref{fig:mcts:pp}.

Some potential sampling sequences to generate a full solution are:
(i) optimize all variables jointly, (ii) sample the relative transformation first, then compute the robot configurations, or (iii) compute the pick configuration and the relative transformation jointly first, followed by the place configuration.

All possible sequences of sampling operations can be represented as a tree, where each node indicates the subset of variables that have been computed, and each edge denotes a sampling operation.
We illustrate the sampling tree for the Pick and Place task plan in \cref{fig:sequence-pick-place}.
For clarity, we omit sampling operations that have zero probability of success, a process detailed later in \cref{sec:mct:pruning}, resulting in five possible sampling sequences.

The efficacy of each strategy depends on the computational time of the individual operations, their success rates, and the probability of feasibility of subsequent sampling steps, since often partial assignments are not viable for subsequent steps.
For example, a valid relative transformation between the gripper and the object might not permit an inverse kinematics solution due to collisions or the robot's reachability limits.

\section{Sequential Sampling in Factored-NLPs as a Markov Decision Process}
\label{sec:framework}

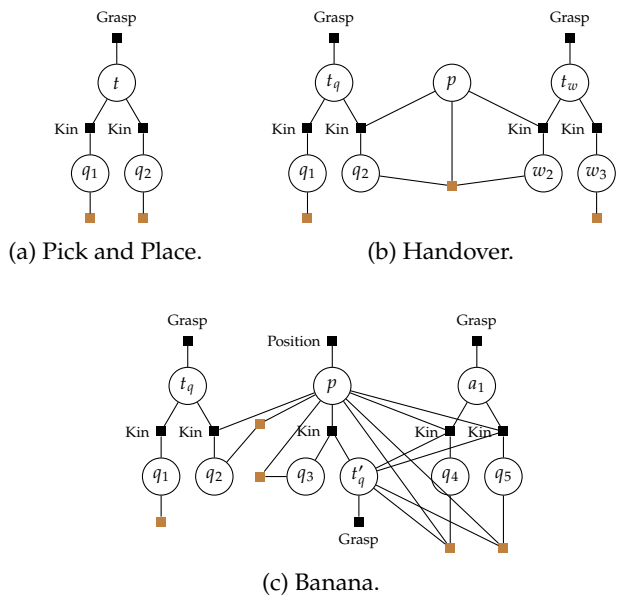
\begin{figure*}[t]
	\centering

	\begin{subfigure}[b]{.2\textwidth}
		\centering
		\begin{tikzpicture}[scale=0.7,every node/.style={transform shape}]
			\node[latent] (t) {$t$} ; %
			\node[latent, below=of t, xshift=-.5cm] (q1) {$q_1$} ; %
			\node[latent, below=of t, xshift=.5cm] (q2) {$q_2$} ; %
			\factor[above=of q2] {Kina2} {left:Kin} {t,q2} {};
			\factor[above=of q1] {Kina1} {left:Kin} {t,q1} {};
			\factor[below=of q1,color=brown] {Collision} {} {q1} {};
			\factor[below=of q2,color=brown] {Collision} {} {q2} {};
			\factor[above=of t] {Grasp} {Grasp} {t} {};
		\end{tikzpicture}
		\caption{\textup{Pick and Place.}}
		\label{fig:pick_place}
	\end{subfigure}%
	\hspace{1em}%
	\begin{subfigure}[b]{.38\textwidth}
		\centering
		\begin{tikzpicture}[scale=0.7,every node/.style={transform shape}]
			\node[latent ] (p) {$p$} ; %
			\node[latent, left=1.5 of p , yshift=.0cm] (ta) {$t_q$} ; %
			\node[latent, right=1.5 of p, yshift=.0cm ] (tb) {$t_w$} ; %
			\node[latent, below=of ta, xshift=-.5cm] (qa1) {$q_1$} ; %
			\node[latent, below=of ta, xshift=.5cm] (qa2) {$q_2$} ; %
			\node[latent, below=of tb, xshift=-.5cm] (qb2) {$w_2$} ; %
			\node[latent, below=of tb, xshift=.5cm] (qb1) {$w_3$} ; %

			\factor[above=of qa2] {Kina2} {left:Kin} {ta,qa2,p} {};
			\factor[above=of qb2] {Kinb2} {left:Kin} {tb,qb2,p} {};

			\factor[above=of qa1] {Kina1} {left:Kin} {ta,qa1} {};
			\factor[above=of qb1] {Kinb1} {left:Kin} {tb,qb1} {};

			\factor[below=1.5 of p,color=brown] {Collision} {} {qa2,qb2,p} {};
			\factor[below=of qa1,color=brown] {Collision} {} {qa1} {};
			\factor[below=of qb1,color=brown] {Collision} {} {qb1} {};
			\factor[above=of ta] {Grasp} {Grasp} {ta} {};
			\factor[above=of tb] {Grasp} {Grasp} {tb} {};
		\end{tikzpicture}
		\caption{
			\textup{Handover.}
		}
		\label{fig:factor-handover}
	\end{subfigure}%
	\vspace{.5cm}%

	\begin{subfigure}[b]{.38\textwidth}
		\centering
		\begin{tikzpicture}[scale=0.7,every node/.style={transform shape}]
			\node[latent ] (p) {$p$} ; %
			\node[latent, left=2 of p , yshift=.0cm] (ta) {$t_q$} ; %
			\node[latent, right=2 of p, yshift=.0cm ] (tb) {$a_1$} ; %
			\node[latent, below=of ta, xshift=-.5cm] (qa1) {$q_1$} ; %
			\node[latent, below=of ta, xshift=.5cm] (qa2) {$q_2$} ; %

			\node[latent, below=of p, xshift=-.5cm] (qx) {$q_3$} ; %
			\node[latent, below=of tb, xshift=-.5cm] (qb1) {$q_4$} ; %
			\node[latent, below=of tb, xshift=.5cm] (qb2) {$q_5$} ; %
			\node[latent, below=of p, xshift=.5cm] (ta2) {$t_q'$} ; %

			\factor[above=of qa2] {Kina2} {left:Kin} {ta,qa2,p} {};
			\factor[above=of qb1] {Kinb2} {left:Kin} {tb,qb1,p,ta2} {};

			\factor[above=of qa1] {Kina1} {left:Kin} {ta,qa1} {};
			\factor[above=of qb2] {Kinb1} {left:Kin} {tb,qb2,ta2,p} {};

			\factor[below=of p] {KinX} {left:Kin} {ta2,qx,p} {};

			\factor[below=of qa1,color=brown] {Collision} {} {qa1} {};
			\factor[right=of qa2,yshift=1cm,color=brown] {Collision} {} {qa2,p} {};
			\factor[below=of qb2,color=brown,yshift=-.5cm] {Collision} {} {qb2,p,ta2} {};
			\factor[below=of qb1,color=brown,yshift=-.5cm] {Collision} {} {qb1,p,ta2} {};
			\factor[left=of qx,color=brown] {Collision} {} {qx,p} {};
			\factor[above=of p] {Position} {left:Position} {p} {};
			\factor[above=of ta] {Grasp} {Grasp} {ta} {};
			\factor[below=of ta2] {Grasp} {below:Grasp} {ta2} {};
			\factor[above=of tb] {Grasp} {Grasp} {tb} {};
		\end{tikzpicture}
		\caption{\textup{Banana.}}
		\label{fig:banana-with-regrasp}
	\end{subfigure}%
	\caption{
		Factored-NLPs of the three task plans considered in this chapter.
		Circles are variables and squares are constraints.
		Brown squares represent collision constraints.
		See the main text and \cref{sec:bg:structure} for a description of the variables and constraints.
		\textit{Pick and Place} is used to illustrate and analyze different sampling sequences.
		The more complex \textit{Handover} and \textit{Banana} tasks are used to evaluate our method.
	}
	\label{fig:factor_graphs}
\end{figure*}

In this chapter, we present a framework to learn the optimal strategy for generating multiple diverse solutions for a Factored-NLP (\cref{sec:bg:factored-nlp}), where optimal is defined as maximizing the number of diverse solutions found within a fixed computational time.

The Factored-NLPs of the three manipulation tasks considered in this chapter are shown in \cref{fig:factor_graphs}.
Variables $q$ denote the configuration of the robot, with the subscript indicating different time-steps.
There are two types of variables for the configuration of objects: variables $p$ denote the absolute positions of objects, and $t$ represents relative transformations with respect to the robot gripper.
In \textit{Handover} (\cref{fig:keyframes-handover}), the configurations for the second robot are denoted with $w$.
In \textit{Banana} (\cref{fig:keyframes-banana}), $t$ represents the grasp of the box, and $a$ represents the grasp of the banana.
The meaning of each constraint is explained in detail in \cref{sec:bg:structure}.

Note that the Factored-NLPs used in this chapter have some differences compared to the ones presented in \cref{sec:bg:structure,ch:bid}.
First, we only consider the keyframes of a task plan and, therefore, do not include variables to represent trajectories.
Second, we use here a more compact representation of the optimization problem, where all variables that are constrained to be fixed are removed, and consecutive free variables that are constrained to be equal are condensed into a single variable -- 
resulting in a less structured and more problem-specific representation.

Given a Factored-NLP with a set of $N$ variables \(X = \{x_1,\ldots,x_N\}\), \(x_i \in \mathbb{R}^{n_i}\), and constraints \(\Phi = \{\phi_1,\ldots,\phi_L\}\), we first observe that a complete value assignment to all the variables in \(X\) can be computed in different ways: variables could be assigned jointly, one by one, or following any user-defined order.

To formalize this, we define an \textit{assignment state} \(w \in W\) to indicate the set of variables of \(X\) that have been assigned.
The set of all possible assignment states \(W\) is the power set \(P(\{1,\ldots,N\})\), i.e., the set of all subsets of \(\{1,\ldots,N\}\).
For instance, an assignment state \(w' = \{1,2,3\}\) means that values have been assigned to the variables \(x_1, x_2,\) and \(x_3\).
Conversely, \(x_{w}\) denotes the subset of variables in \(X\) indicated by \(w\), e.g., \(x_{w'} = \{x_1, x_2, x_3\}\).

Hence, any full assignment for all \(x_i\) can be computed following a sequence of assignment states \(H = \left(w_0, w_1, \ldots, w_g\right)\), where \(w_i \subset w_{i+1}\), and \(w_0 = \{\}, w_g = W\).
For any given step \(i\), the sequence up to that point is represented as \(H_i = (w_0, \ldots, w_i)\).

A transition \(w_i \to w_j\) in a sequence implies a computational operation
\begin{equation}
	x_{w_j} = o_{w_i, w_j}(x_{w_i})\,,
\end{equation}
that assigns numerical values to the variables indexed in \(w_j \setminus w_i\), conditioned on the given values in \(x_{w_i}\), which are kept fixed.
If the compute operation is successful, the new assignment state is \(w_j\), and the value for all computed variables so far is \(x_{w_j}\) (note that
\(x_{w_j}\) also includes the values of \(x_{w_i}\) for the variables in \(w_i\)).

The set \(\mathcal{Y}(w_i) = \{w_{j} \in W \mid w_i \subset w_{j}\}\) denotes the assignment states that can be reached with a valid transition from \(w_i\).
The computational operations are implemented either with direct conditional sampling or using optimization methods initialized with a randomized guess.
For instance, in the Pick and Place problem presented in the introduction, examples of such transitions include generating 6-DOF grasps or solving inverse kinematics (i.e., computing the robot joint values for a given end-effector position).

\cref{alg:generate-solutions} shows how to generate solutions from a Factored-NLP by choosing a valid transition in \textit{assignment space} and executing the corresponding compute operation at each step.
The performance of the algorithm depends heavily on how transitions are selected (our method described in \cref{sec:UCT} is in orange).

Each transition \(w_i \to w_j\) is a conditional sampling operation, which generates samples with an unknown but defined conditional probability density
\(p_{w_j | w_i}(x_{w_j} | x_{w_i})\), defined on the feasible manifold of \(x_{w_j}\).
There are no assumptions about the shape of the distribution.
We assume that the probability density is zero everywhere in case the transition conditioned on \(x_{w_i}\) is infeasible.
The joint probability density is
\begin{equation}
	p_{H_j}(x_{w_j}) = p_{w_j | w_i}(x_{w_j} | x_{w_i}) \cdot p_{H_i}(x_{w_i})\,,
\end{equation}
and it depends on the history \(H_j = (w_0, \ldots, w_i, w_j)\) of transitions to reach \(w_j\).

A valid transition \(w_i \to w_j\) between assignment states does not always produce an assignment that satisfies the associated constraints for \(x_{w_j}\).
There are two sources of infeasibility: \textit{(i)} the problem conditioned on the previously assigned variable \(x_{w_i}\) could be infeasible, or \textit{(ii)} the sampling operation could fail to generate a solution even if one exists.

We can define a transition probability for transitions between the current sequence \(H_i\) and the next assignment state \(w_{i+1}\) (equivalently, between \(H_i\) and \(H_{i+1} = H_{i} \cup \{w_{i+1}\}\)) as follows:
\begin{align}\label{eq:successRate}
	Pr \left( H_i \to H_{i+1} \right) =
	\hat p_{H_i, H_{i+1}} \cdot \frac{ \int_{ \Omega_{i,i+1} } p_{H_i}(x_{w_{i}}) \, d\Omega }{ \int_{ \Omega_i } p_{H_i}(x_{w_i}) \, d\Omega },
\end{align}
where \(\Omega_i\) is the feasible space for \(x_{w_i}\) and \(\Omega_{i,i+1} \subseteq \Omega_{i}\) is the feasible space for \(x_{w_i}\) for which a joint sample \(x_{w_{i+1}}\) exists.
In this estimate, the success rate \(\hat p_{H_i, H_{i+1}} \in (0,1)\) represents the probability that a solution satisfying the constraints will be found if one exists, which we estimate as constant for a given transition, rather than depending on the previous assignment \(x_{w_i}\).
Note that the success rate could be low for certain transitions, while it could be close to \num{1} for others.

\paragraph{Assignment state transition as a Markov Decision Process}
\label{sec:markov}

We can now define a Markov Decision Process (MDP):
The state space is \(\{ H_i \} \cup \{\neg\}\), i.e., the sets of partial sequences starting from \(w_0\) plus an infeasible state \(\neg\).
The action space contains the possible transitions between sequences, i.e., adding a new assignment state.
Each action has a success rate \(p_a = Pr(H \to H')\) of reaching the new sequence as defined in \eqref{eq:successRate}, (corresponding to the probability of generating an assignment satisfying the constraints) and a probability \(1-p_a\) of going to the infeasible state \(\neg\).
We note that the MDP has a maximum horizon length of \(N\) transitions, which corresponds to assigning all the variables one by one.

We use the reward structure to obtain a concept of optimality that directly relates to our objective of maximizing the number of samples:
We introduce a reward \(r_g=1\) for reaching the goal \(w_g = W\) (namely, all variables have been correctly computed), and a stochastic cost, i.e., a negative reward, \(r_t\) on each transition, which is the compute time the transition takes.
We weigh the two costs linearly, obtaining \(\lambda r_t + (1-\lambda) r_g\) with \(\lambda \in (0,1)\).

\section{Choosing Computational Operations with Monte-Carlo Tree Search}

We use Monte-Carlo Tree Search (MCTS, \cite{browne2012survey}) on the previously defined MDP to find the optimal sequence of computational operations.

MCTS incrementally builds a tree of possible transition sequences to find the most promising one.
This is achieved by randomly selecting transitions starting from the initial state \(w_0\) of the MDP,
executing their corresponding conditional sampling operations,
and using the reward structure of the MDP to guide the search.
Thus, in our setting, the nodes of the tree built by the MCTS are sequences of assignment states.

The reward structure of the MDP guides the tree search algorithm to choose transitions that have a low computational cost and a high probability of producing a full solution to the Factored-NLP.

An estimate of the average computational cost \(\hat{c} = -\sum {r_t}\) can be used as guidance for choosing a suitable \(\lambda = \frac{1}{\hat{c}+1}\).
In this case, the optimal solution in the MDP corresponds to the sequence of computational operations that maximizes the number of generated samples for the original factored nonlinear program.

\subsection{Upper Confidence Tree (UCT)}\label{sec:UCT}

In particular, we use the UCT algorithm \cite{kocsis2006bandit} to balance the exploitation of known sequences with the exploration of new sequences.
UCT expands the child node \(z\) that maximizes
\begin{equation}\label{eq:uct}
	Q_z + c\sqrt{\frac{\ln M_z}{m_z}},
\end{equation}
where \(Q_z\) is the current estimate of the expected reward of node \(z\), \(M_z\) is the number of simulations that have evaluated the parent node, \(m_z\) is the number of rollouts that have evaluated node \(z\), and \(c\) is a constant chosen by the user.
The estimate of the \(Q\) function for the nodes of the tree converges as the number of rollouts increases.
Hence, the UCT algorithm incrementally builds a search tree, estimates the \(Q\) values of each transition, and chooses actions using the upper confidence bound~\eqref{eq:uct}, leading to \cref{alg:generate-solutions}.

We note that, if at least one of the sampling sequences assigns a non-zero probability density to the entire feasible manifold, so does our algorithm, since UCT never stops exploring all possible sequences.
\begin{figure}[!t]
	\centering
	\begin{minipage}{.8\linewidth}
		\centering
		\begin{algorithm}[H] %
			\caption{Framework to generate solutions from a Factored-NLP, with our contributions in {\color{orange}orange}.}
			\label{alg:generate-solutions} %
			\begin{algorithmic}[1] %
				\State \textbf{Input:}
				Factored-NLP, compute time budget \(T\)
				\State { \(L \gets \{ \} \) } \Comment{{\color{gray} \small Empty list of samples}}
				\While{$\text{accumulated compute time} \ t \le T$}
				\State{feasible $\gets$ True}
				\State{ \(w \gets w_0\) } \Comment{{\color{gray} \small Initial empty assignment state ($w_0=\emptyset\,$)}}
				\State{ \( x_w \gets x_{w_0} \) } \Comment{{\color{gray} \small Initial empty sample ($x_{w_0}=[]\,$) }}
				\While{\(w \neq W\)}
				\State{ \(w' \gets\) \textit{Choose next state from} \(\mathcal{Y}(w)\) \color{orange}{using UCT}}
				\State{\(x_{w'}\), feasible, {\color{orange}{reward}} \(\gets o_{w, w'}(x_{w})\) } \Comment{{\color{gray} \small Execute computational operation}}
				\State{\color{orange}\textit{Update UCT tree with} reward}
				\If{feasible}
				\State{ \(w \gets w'\)}
				\State{ \(x_w \gets x_w'\)}
				\Else
				\State{\textbf{break}}
				\EndIf
				\EndWhile
				\If{feasible}
				\State{append \(x_W\) to \(L\)}
				\EndIf
				\EndWhile
				\State \textbf{Output:}
				List of valid samples \(L\)
			\end{algorithmic}
		\end{algorithm}
	\end{minipage}
\end{figure}

\subsection{Pruning the Sampling Tree Using the Factored-NLP}
\label{sec:mct:pruning}

Since the number of sequences and transitions grows exponentially with the dimension of the problem (see \cref{tab:number-edges}), it is not possible to naively apply MCTS to the previously defined MDP.
However, most real-world problems have a sparse structure that can be leveraged to significantly reduce the set of available sequences and transitions.

We provide two examples of transitions that can be pruned.
Consider a Pick and Place problem represented with the variables $\{ q_1 , q_2 , t \}$, where $q_1, q_2$ are the robot configurations for picking and placing, and $t$ is the relative transformation between the end-effector and the object (\cref{fig:sequence-pick-place-all}).

\begin{itemize}
	\item We prune $q_1 \to (q_1 , t)$ and $q_2 \to (q_2 , t)$.
	      If $q_1$ or $q_2$ is sampled independently at random, the probability of correctly grasping the object is zero.
	\item We prune $t \to (q_1 , q_2)$ because $q_1$ and $q_2$ are independent if $t$ is fixed (and it does not make sense to generate $q_1$ and $q_2$ jointly).
\end{itemize}

More generally, we combine knowledge of the Factored-NLP structure with the local information in each variable and constraint, i.e., the number of available degrees of freedom and the number of equality constraints, to prune a transition $w_i \to w_j$ if it fulfills any of the two criteria:

\paragraph{Zero probability of success}
In general, the probability of sampling variables from the ambient space such that the equality constraints are fulfilled is zero.
Hence, a transition can be pruned if the number of new, linearly independent, equality constraints exceeds the number of new degrees of freedom that are added in that transition.

\paragraph{Equivalence under conditional independence}
Given an assignment state $w_i$, two variables in the Factored-NLP are conditionally independent if all paths that connect them contain already assigned variables.
Therefore, we delete a transition that jointly samples variables that are conditionally independent with respect to those assigned in $w_i$.
In such cases, joint sampling means choosing an ordering for two conditionally independent processes (which is already available as alternative valid transitions).

In \cref{tab:number-edges}, we show how the number of transitions is reduced in the three scenarios evaluated in this chapter.

\begin{table}[t]
	\centering
	\small
	\caption{Number of transitions between assignment states in the three evaluated task plans, before and after transition pruning.}
	\label{tab:number-edges}
	\begin{tabular}{@{}lccccc@{}}
		\toprule
		               & Variables & \begin{tabular}[c]{@{}c@{}}
			                             Transitions \\ Complete\end{tabular} & \begin{tabular}[c]{@{}c@{}}Transitions \\ after Pruning\end{tabular} & \begin{tabular}[c]{@{}c@{}} Pruning \\ Ratio [\%]\end{tabular} \\ \midrule
		Pick and Place & 3         & 19                                   & 8                                                                    & 57.89                                                          \\
		Handover       & 7         & 2059                                 & 163                                                                  & 92.09                                                          \\
		Banana         & 9         & 19171                                & 534                                                                  & 97.21                                                          \\
		\bottomrule
	\end{tabular}
\end{table}

\subsection{Family of Problems and Tree Warm Start}
\label{sec:mct:family}
Obstacle and object configurations impact the computational time and the success rate of the conditional sampling operations, potentially changing the optimal sampling sequence.
Thus, the best sampling order depends on both the Factored-NLP and the specific problem instance.

MCTS provides a good framework to incorporate the information gathered from solutions to similar problems as a warm start.
We propose to warm-start the tree search by initializing $Q_z$ of each node with the average of the $Q$ values from previous problems, and set $m_z$ to an equivalent count visit $m_{\textrm{equiv}}$, which models how confident we are with the warm start, following the approach presented in \cite{gelly2007combining}.

\section{Experimental Results}

\begin{figure}
	\centering
	\begin{subfigure}[t]{.16\textwidth}
		\centering
		\includegraphics[width=.9\linewidth]{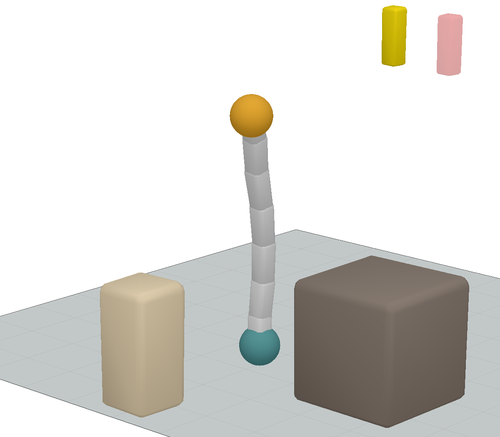}
	\end{subfigure}%
	\begin{subfigure}[t]{.16\textwidth}
		\centering
		\includegraphics[width=.9\linewidth]{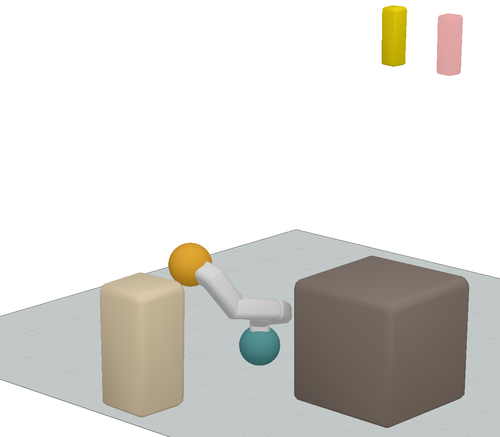}
	\end{subfigure}%
	\begin{subfigure}[t]{.16\textwidth}
		\centering
		\includegraphics[width=.9\linewidth]{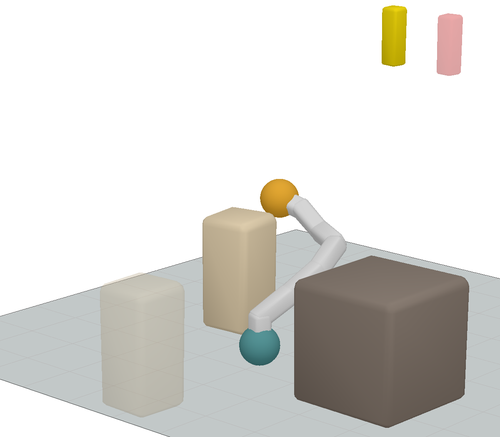}
	\end{subfigure}\\[1mm]

	\begin{subfigure}[t]{.16\textwidth}
		\centering
		\includegraphics[width=.9\linewidth]{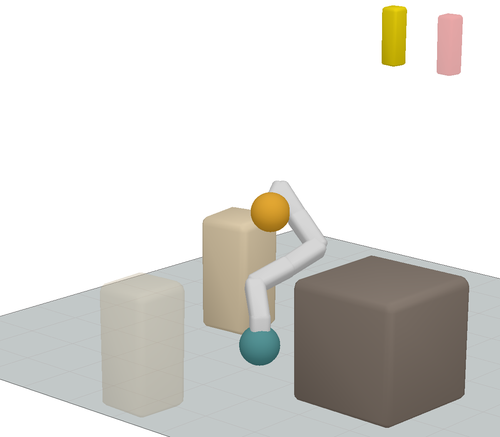}
	\end{subfigure}%
	\begin{subfigure}[t]{.16\textwidth}
		\centering
		\includegraphics[width=.9\linewidth]{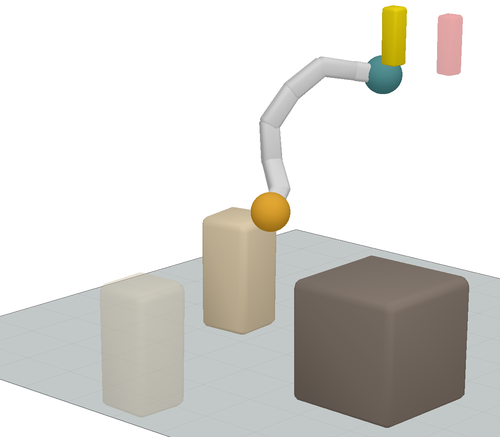}
	\end{subfigure}%
	\begin{subfigure}[t]{.16\textwidth}
		\centering
		\includegraphics[width=.9\linewidth]{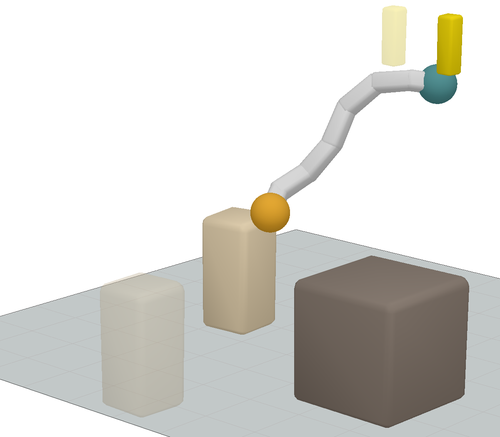}
	\end{subfigure}
	\caption{A sequence of keyframes in one instance of the \textit{Banana} problem.
		The task plan is to move the light brown box, climb on top of it, reach the banana, and place it on the red configuration.
		The dark brown object is an obstacle.
	}
	\label{fig:keyframes-banana}
\end{figure}

\subsection{Scenarios}

\begin{itemize}
	\item
	      \textit{Handover}:
	      Two robots (7D) have to collaborate to place a box in the goal configuration.
	      Grasping is modeled with a two-finger gripper.
	      \Cref{fig:keyframes-handover} shows a possible solution, and \cref{fig:factor-handover} shows the Factored-NLP.

	\item
	      \textit{Banana:}
	      The hanging `banana' has to be moved to a goal position.
	      The robot (8D) can interact with the world using both sides of the kinematic chain as a gripper.
	      The box has to be moved and used as a tool to stand on to be able to reach the goal.
	      Grasping is modeled as `grasp by touch'.
	      \cref{fig:keyframes-banana} shows snapshots from a solution sequence, and \cref{fig:banana-with-regrasp} shows the Factored-NLP.
\end{itemize}

For both problems, \num{8} different instances (see \cref{fig:instances:handover} and \cref{fig:instances:banana})
(i.e., varying obstacle and object configurations) are evaluated\footnote{Experiments are run using an Intel(R) Core(TM) i5-4200U 1.60GHz CPU.
} multiple times, and unless stated otherwise, the average results are reported in the following analyses.
The different instances can significantly impact the computation times and, thus, the optimal sampling sequence.

\begin{figure}[ht]
	\centering
	\includegraphics[width=0.22\textwidth]{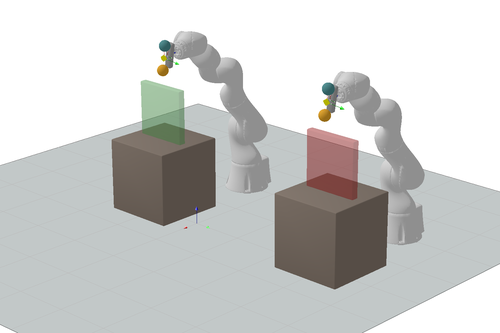}\quad
	\includegraphics[width=0.22\textwidth]{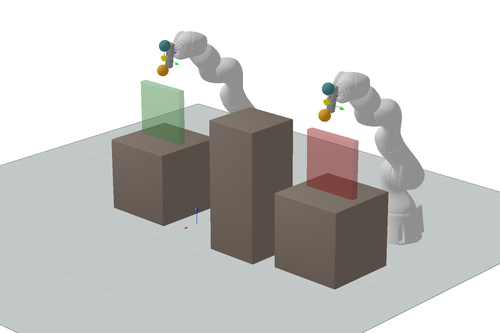}\quad
	\includegraphics[width=0.22\textwidth]{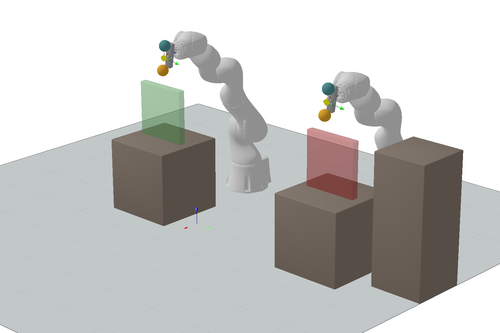}\quad
	\includegraphics[width=.22\textwidth]{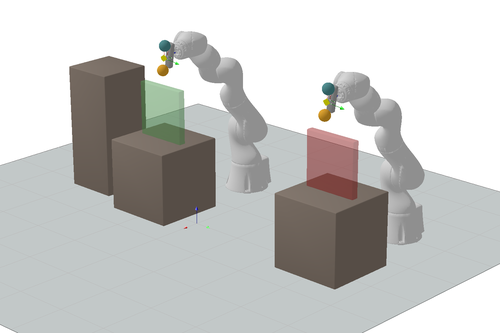} \\
	\includegraphics[width=.22\textwidth]{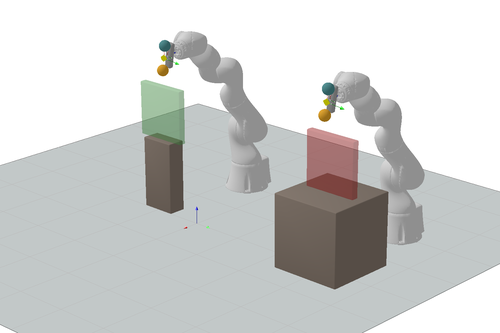} \quad
	\includegraphics[width=.22\textwidth]{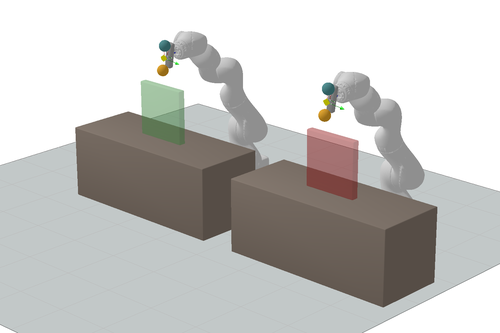} \quad
	\includegraphics[width=.22\textwidth]{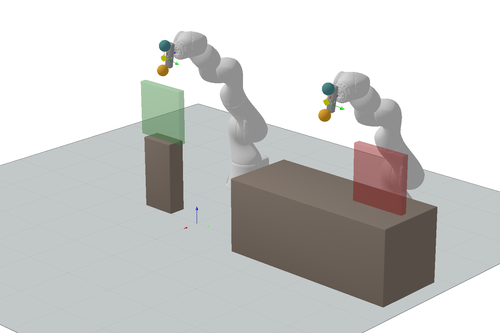} \quad
	\includegraphics[width=.22\textwidth]{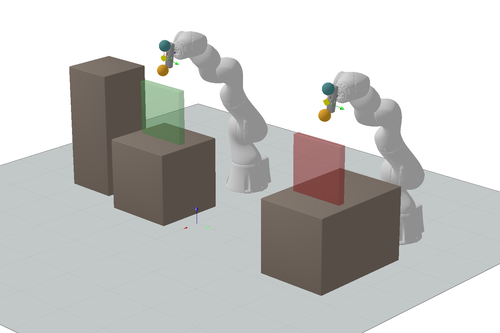}
	\caption{\textit{Handover} scenario.
		The obstacles (dark brown) constrain the possible grasps and handover positions of the robots.
  The optimal sampling sequence might vary in different scenarios, based on the number, size, and position of the obstacles.}
	\label{fig:instances:handover}
\end{figure}

\newpage
\subsection{Computational Operations}

One motivation for this work is that both the computation time and the success rate of conditional sampling operations must be considered to design efficient algorithms.
As an illustrative example, we report the average computational time and success rate of some sampling operations in the first instance of the \textit{Handover} scenario:
relative transformation (\SI{0.68}{ms}, \SI{100}{\percent}), inverse kinematics (\SI{1.55}{ms}, \SI{87}{\percent}), grasp of a fixed object (\SI{5.8}{ms}, \SI{55}{\percent}), pick and place (\SI{54}{ms}, \SI{46}{\percent}), and handover keyframe (\SI{68}{ms}, \SI{44}{\percent}).

In practice, our algorithm finds the optimal balance between \textit{(i)} choosing simple operations which are fast and have a high success rate but can potentially induce infeasibility in future assignments, and \textit{(ii)} optimizing variables jointly, which considers joint feasibility but is slower and often has a lower success rate.

The conditional sampling operations are implemented by randomizing the initial guess of a nonlinear solver without a cost term and are solved using the Augmented Lagrangian algorithm.
The initial guess covers the entire ambient space of the current partial assignment, ensuring that all possible partial solutions have a non-zero probability (see \cref{sec:framework}).

\subsection{Number of Samples and Approximate Coverage}

\begin{figure}[t]
	\centering
	\begin{subfigure}[b]{.4\textwidth}
		\centering
		\includegraphics[height=40mm]{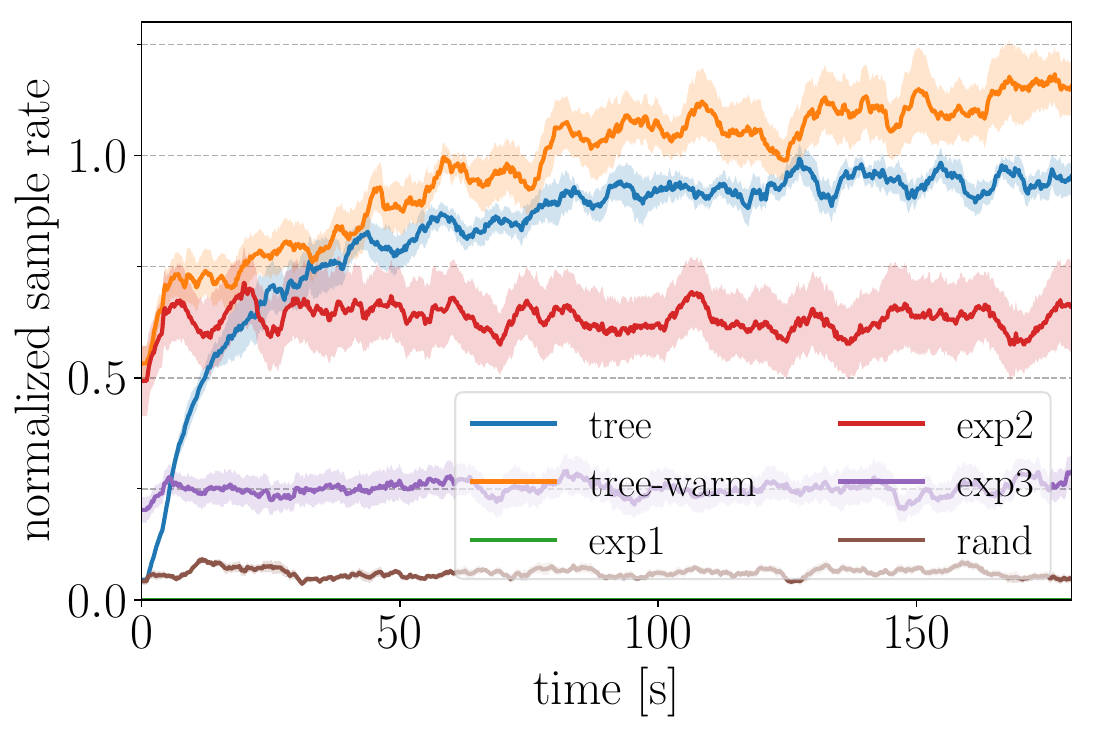}
		\caption{\textit{Handover}: Sample rate. }
		\label{fig:handover_rate}
	\end{subfigure}%
	\begin{subfigure}[b]{.4\textwidth}
		\centering
		\includegraphics[height=40mm]{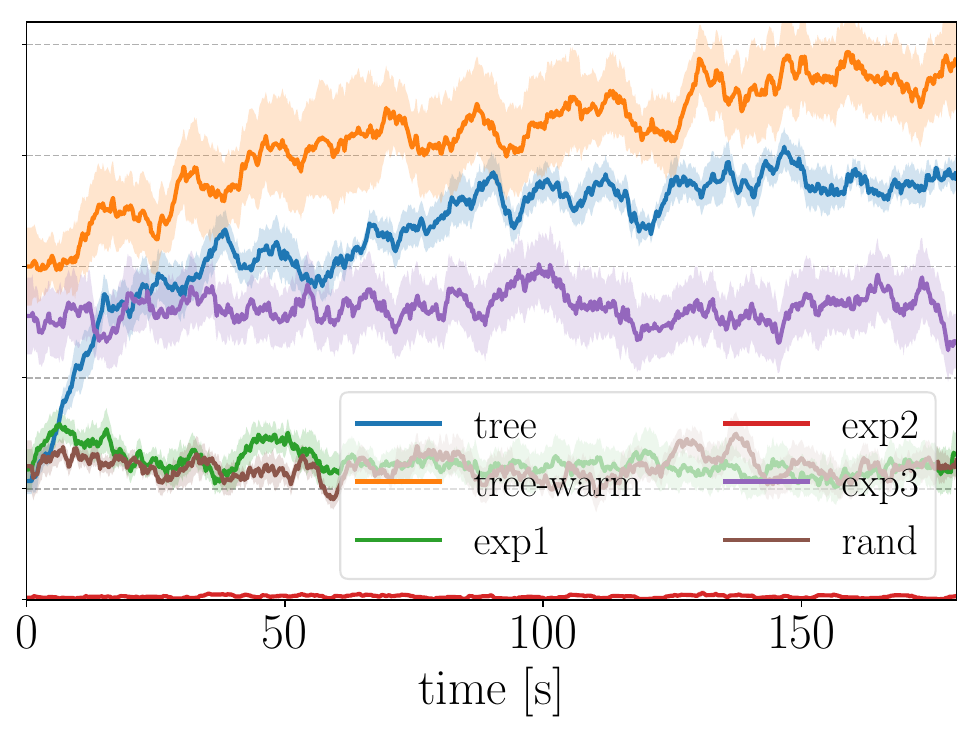}
		\caption{\textit{Banana}: Sample rate.}
		\label{fig:banana_rate}
	\end{subfigure} \\[10pt]
	\begin{subfigure}[b]{.8\textwidth}
		\includegraphics[width=\linewidth]{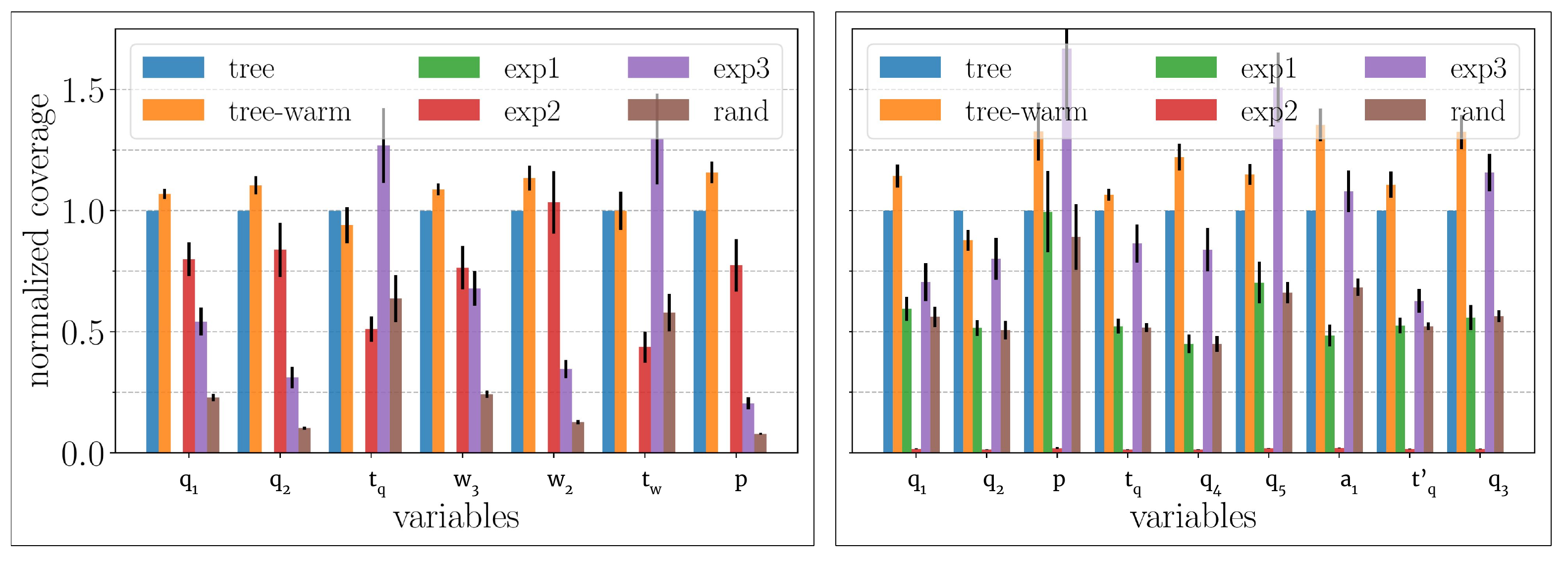}
		\caption{Coverage.
			Left: \textit{Handover}.
			Right: \textit{Banana}.
		}
		\label{fig:all_coverage}
	\end{subfigure}%
	\caption{Sampling rates and approximated projected coverage (normalized by \textit{Tree}), averaged over all problems instances.}
	\label{fig:exp}
\end{figure}

\begin{table}[ht]
	\small
	\setlength{\tabcolsep}{5pt}
	\centering
	\caption{Sampling sequences in \textit{Banana} and \textit{Handover} scenarios.
		Sampling order is left to right, each tuple $(.)
		$ denotes joint sampling of the subset. }
	\label{tab:expert-sequences}
	\begin{tabular}{@{}llr@{}}
		\toprule
		\multirow{2}{*}{\shortstack[l]{\textit{Expert 1}               \\ Joint Opt.}} & Handover & $(q_1,
		q_2, t_q , w_3,  w_2, t_w , p )   $                            \\
		 & Banana &
		$( q_1, q_2,   t_q, p ,q_3,  t_q',  q_4, q_5, a_1)$
		\\ \midrule
		\multirow{2}{*}{\shortstack[l]{\textit{Expert 2}               \\ One-by-one}}& Handover  &  $p,t_q,t_w,q_1,q_2,w_2,w_3$ \\
		 & Banana & $p,t_q,a_1,t_q', q_1,q_2,q_3,q_4,q_5$              \\ \midrule
		\multirow{2}{*}{\shortstack[l]{\textit{Expert 3}               \\ Sequential Keyframes}} & Handover &
		$( q_1 , t_q), ( q_2 , p, w_2 , t_w ), (w_3)$                  \\
		 & Banana &
		$(q_1, t_q),( q_2, p),( q_3, t_q'),(q_4, a_1),q_5$             \\  \midrule
		\multirow{3}{*}{\shortstack[l]{\textit{Tree}                   \\Best found Sequence \\(Examples)}}& Handover  &  $t_w, t_q, q_1, w_3, (q_2,p),w_2$ \\
		 &        & $(w_3,t_w),(w_2,p),(q_2,t_q),q_1$                  \\
		 & Banana & $(q_1,t_q)$, $(q_2,p),(q_4,q_5,a_1,t_q'), q_3$     \\
		 &        & $(p,q_5,a_1,t_q')$,$q_3$,$(q_1,t_q),$ $q_2,$ $q_4$ \\ %
		\bottomrule
	\end{tabular}
\end{table}

We compare the number of samples generated by our algorithms, \textit{Tree} and \textit{Tree-warm}, against \textit{Expert}-sequences and \textit{Rand}.
\textit{Tree} is our MCTS-based meta-solver (see \cref{alg:generate-solutions}), including the pruning (see \cref{sec:mct:pruning}) and starting from an empty tree. \textit{Tree-warm} is the same algorithm (also with pruning) but with a warm start using the average reward from previous problems (see \cref{sec:mct:family}).
\textit{Expert}-sequences are fixed sampling orders that represent different user-defined strategies (see \cref{tab:expert-sequences}).
\textit{Rand} selects computational operations at random, without learning.

\newpage
\paragraph{Number of samples}
We plot the evolution of the sampling rate (including the MCTS overhead for our algorithms) in our two scenarios in \cref{fig:handover_rate,fig:banana_rate}.
Asymptotically, \textit{Tree} outperforms all \textit{Experts}, which exhibit disparate performance across the two problems.
Due to the exploration of possible sampling sequences, initially, the sample rate of \textit{Tree} is lower compared to some of the \textit{Expert}-sequences.
However, \textit{Tree-warm} mitigates this issue due to the warm start and improves its sample rate over time.
We also note that the \textit{Expert}-sequences are more sensitive to the different problem instances.

The final sample rates achieved by \textit{Tree-warm} in the different instances range between \(7-15\) \si{samples/s} in the \textit{Handover} scenario and \(5-10\) \si{samples/s} in the \textit{Banana} scenario.
Examples of some of the sequences discovered by \textit{Tree} are shown in \cref{tab:expert-sequences}.

\newpage
\paragraph{Approximate coverage}
We evaluate the coverage of the solution manifold achieved by our algorithms \textit{Tree} and \textit{Tree-warm}, which should achieve good coverage (i.e., a diverse set of samples) by design.
Since the coverage of a nonlinear manifold embedded in a high-dimensional space cannot be evaluated reliably with a relatively low number of samples, we evaluate the \emph{projected coverage} as a proxy measure:
We project the full samples onto each variable of the sequence
and compare the coverage in each of these subspaces.
We discretize these subspaces and count the number of occupied cells, i.e., cells that are occupied by multiple samples are only counted once.
Although an approximation, it provides useful and interpretable information about which parts of the solution-manifold are covered.

The results shown in \cref{fig:all_coverage} are normalized by the number of occupied cells in \textit{Tree}.
While some \textit{Experts} achieve better coverage on a subset of variables, coverage by \textit{Tree} and \textit{Tree-warm} mostly outperform the others.
This confirms our hypothesis that, with the correct design and implementation of the sampling operations, maximizing the number of samples is a good heuristic to maximize the coverage.

\section{Limitations}

A limitation of our framework is that the algorithm must explore inefficient computations and needs some time to converge to the best sampling order, which could prevent usage when the computational budget is limited.
As demonstrated, this limitation is alleviated by warm-starting the algorithm using information from previous runs.
The chosen problem setting of learning to generate solutions from online experience might not be relevant when the goal is to generate a single or a few valid solutions as quickly as possible.
In this case, supervised learning methods that learn from a dataset of problems solved offline could be a better fit (see Part III of this thesis).

In this chapter, we have addressed a subproblem of TAMP, in which the task plan is fixed.
Integrating our ideas to combine sampling and optimization into a complete TAMP solver is challenging, as both the number of possible sampling sequences and the number of candidate task plans grow exponentially with the number of variables.
The full TAMP problem requires an informed balance between the compute effort spent across both different task plans and different compute operations.
In \cref{ch:meta-solver}, we present a meta-solver to address the complete TAMP problem but, limit the number of allowed sampling operations to keep the search space (and the implementation) manageable.

\begin{figure}[t]
	\centering
	\includegraphics[width=0.22\textwidth]{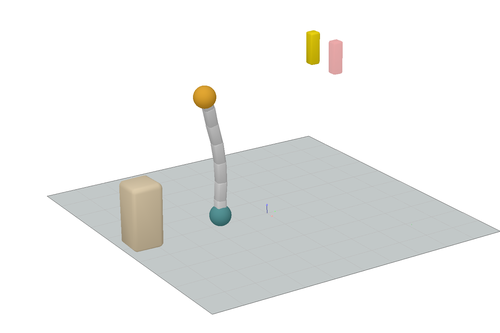}\quad
	\includegraphics[width=0.22\textwidth]{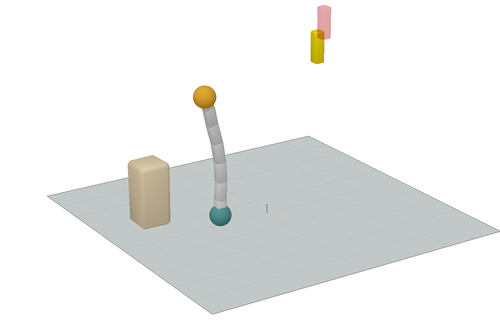}\quad
	\includegraphics[width=0.22\textwidth]{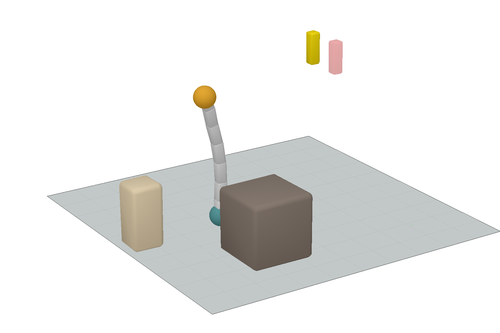}\quad
	\includegraphics[width=.22\textwidth]{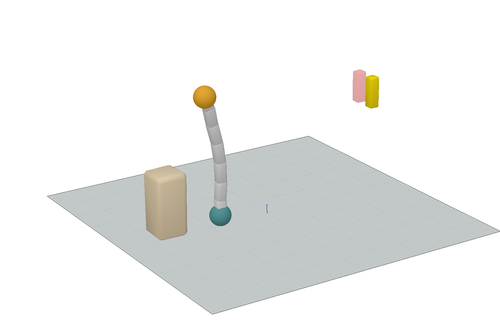} \\
	\includegraphics[width=.22\textwidth]{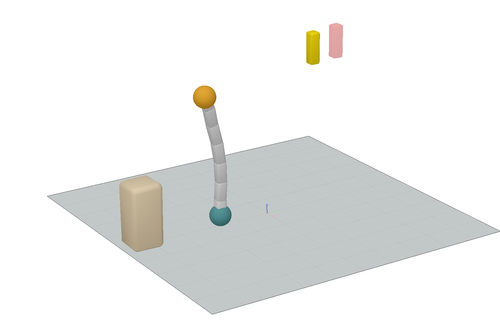} \quad
	\includegraphics[width=.22\textwidth]{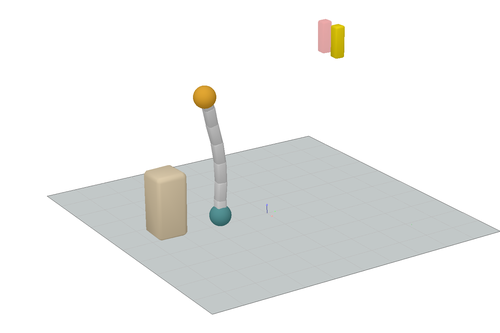} \quad
	\includegraphics[width=.22\textwidth]{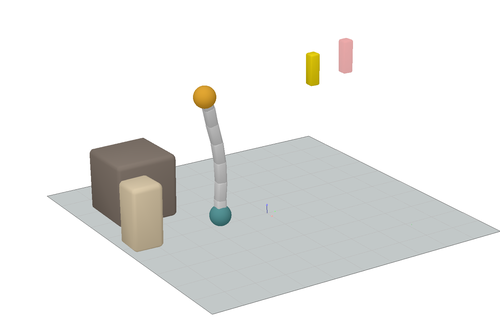} \quad
	\includegraphics[width=.22\textwidth]{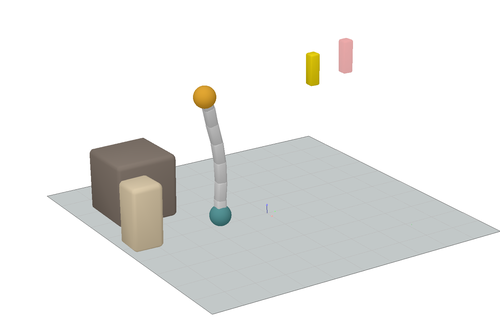}
	\caption{
		\textit{Banana}
		Scenario: Initial and goal configurations of the banana (shown in yellow and red) constrain the placement of the box (light brown).
		The dark brown box is a fixed obstacle.
  }
  \vspace{.2cm}
	\label{fig:instances:banana}
\end{figure}

\section{Conclusion}

We have proposed a meta-algorithm to reason about optimal decompositions of factored nonlinear programs in robotic manipulation planning.
Our algorithm chooses the computational decisions, i.e., which subset of variables to conditionally sample next, to maximize the number of generated samples in a fixed computational time.

We use the method to efficiently generate a diverse set of samples for keyframes in robotic manipulation, which is an essential component in any solver for TAMP problems or multimodal motion planning.

Our framework naturally allows us to also include cost factors in the Factored-NLP.
However, we neglect cost terms because our approach is tailored to provide a diverse set of feasible samples that can be used in higher-level optimization or motion planning, where the diversity and uniform coverage of samples is an essential ingredient for ensuring completeness.

The same problem setting used in this chapter, generating diverse samples from constraint manifolds, will be revisited in \cref{ch:gans} through the perspective of deep learning.
In \cref{ch:gans}, deep generative models are trained with a dataset of solutions to similar problems and are used at runtime to compute solutions for new problems faster.
In contrast to the meta-solver presented in this chapter, the neural models require that a user defines a fixed sequence of sampling operations beforehand.

%% file: optimization_over_computation.tex
\chapter{\nameChapterFour}
\label{ch:meta-solver}

\section{Introduction}

In this chapter\footnote{
	We plan to extend and submit the content of this chapter to a robotics or planning conference, for instance, IROS, ICRA, or ICAPS.
	This research has been conducted in collaboration with Erez Karpas and Marc Toussaint.
},
we present a meta-solver for the comprehensive TAMP problem.
Our meta-solver is a search algorithm that combines optimization and sampling computations to solve TAMP problems, guided by heuristics from the discrete task abstraction.

While most TAMP planners use \textit{roughly equivalent} problem formulations, they differ significantly in A) how they interleave and combine search across task and motion, and B) the computational methods used to compute the continuous variables.

In particular, all TAMP planners use a predefined set of fixed computation operations to compute low-level motion, leading to disparate performance depending on the problem at hand.
Solvers that rely on sampling partial solutions are inefficient when there are long-term dependencies in the low-level motion.
Conversely, solvers that use joint nonlinear optimization are inefficient if the problem is highly decomposable and fail when the optimization problem has infeasible local optima.

In contrast, our meta-solver not only combines search at both the continuous and discrete levels but also reasons about \textit{the best way to solve the continuous level}, deciding which computation operations to perform and in what order.
These decisions are crucial to solving problems efficiently because the time spent on motion planning is often the biggest bottleneck for all TAMP solvers.
Depending on the methods used to compute the motion, running time can vary from milliseconds to several minutes.

This work is primarily inspired by our previous studies on meta-solvers for computing the keyframes for a fixed manipulation plan (\cref{ch:mcts}), our novel factored TAMP formulation and solvers (\cref{ch:bid}), and the sampling-based TAMP solver PDDLStream \cite{garrett2020pddlstream}.
In addition to the formulation and algorithmic tools from \cref{ch:bid}, we now include conditional sampling of keyframes and use a similar mechanism to that of PDDLStream to order and prioritize computations.

Our approach represents a first step toward bridging the gap between optimization and sample-based approaches to TAMP, blending and converging towards the best strategy based on the problem at hand.

\section{Related Work}

An extensive discussion of related work in Task and Motion Planning is provided in \cref{sec:bg:related-work}.
In contrast to state-of-the-art solvers for TAMP that use either only sampling-based or only optimization-based methods by design, our TAMP meta-solver can adaptively choose between sampling or optimization computations.

PDDLStream has been extended to include optimization operations in the Ph.D.
thesis~\cite{garrett2021thesis} by merging some sequences of samplers into a larger optimization problem, resulting in a hybrid sampling and optimization algorithm.
However, the algorithm does not explicitly decide when to use sampling or optimization to compute the same set of continuous variables.
In contrast, our meta-solver explicitly reasons about which method is more efficient for generating the same set of continuous variables.

To merge sampling and optimization, we define a computational state and reason directly about which computational decisions to take next.
Our approach is inspired by classical work on meta-reasoning and decision-making \cite{Russell91Principles}, and its successful application in search, planning, and scheduling, for example, \cite{seipp2012learning,shperberg2019allocating,o2015metareasoning,zilberstein2011metareasoning,lieder2014algorithm}.

Allowing both sampling and optimization operations in our algorithm defines a search problem with an infinite branching factor to solve the original TAMP problem.
Such problems can be addressed by Partial Expansion A* \cite{felner2012partial}, Iterative Broadening \cite{ginsberg1992iterative}, or Iterative Deepening \cite{korf1985depth}.
Instead, we choose to explore a computational space with a simple A*-like algorithm that combines heuristic search with a computational level to widen the tree incrementally, similarly to PDDLStream \cite{garrett2020pddlstream}.
Interestingly, different notions of cost, heuristic, and search algorithms in the computational space result in different TAMP meta-solvers.

Recently, \emph{Effort Level Search in Infinite Completion Trees} analyzed a computational decision process with infinite branching in the context of TAMP \cite{toussaint2023effort}.
The algorithm decides where to allocate compute time given a fixed set of possible computational decisions to address each component of the TAMP problem (enumerating task plans, computing keyframes with nonlinear optimization, computing paths with sample-based motion planning, and optimization).
However, it does not address the trade-off between choosing conditional sampling and optimization when solving TAMP, as we do in our meta-solver.

\section{The Gap Between Sampling and Optimization Approaches}

Within the TAMP research community, a central concern is the large number of \textit{slightly} different problem formulations, as every TAMP solver often introduces its own unique formulation.
This has made it difficult to create hybrid methods or combine ideas from different TAMP solvers.

However, when comparing our recent optimization-based TAMP solver (Factored-NLP Planner, see \cref{ch:bid}) to PDDLStream, a state-of-the-art sample-based solver, one realizes that both formulations use very similar discrete abstractions and task planning, and represent the motion with equivalent variables and constraints, using constraint networks (in PDDLStream) or Factored-NLPs (in our work).
Additionally, they both use discrete task planning to guide the solver towards the goal (when using the \textit{optimistic} algorithms in PDDLStream).

The fundamental difference\footnote{Additional differences, though not analyzed in this chapter, include whether solvers use a conflict-based approach and whether they reuse partial solutions to compute the final solution.
} is in how the values for the constraint networks or factored nonlinear programs are assigned.
In PDDLStream, they are computed using a sequence of predefined sampling iterations, while the \textit{Factored-NLP Planner} uses joint optimization.
Both methods have advantages and disadvantages, as illustrated in \cref{ch:mcts}.

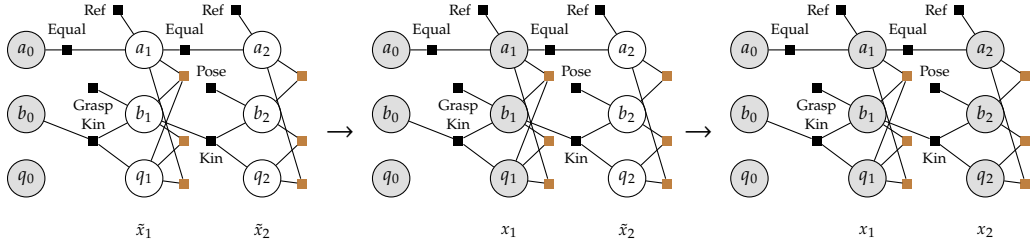
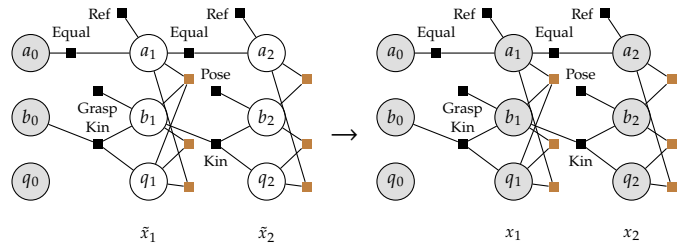
\begin{figure}[h]
	\centering
	\begin{subfigure}[t]{.99\textwidth}
		\setlength{\tabcolsep}{1pt}
		\centering
		\begin{tabular}{ccccc}
			\centered{
				\begin{tikzpicture}[scale=0.7,every node/.style={transform shape}]
					\node[obs] (a0) {$a_0$} ;
					\node[obs,below=.5 of a0  ] (b0) {$b_0$} ;
					\node[obs,below=.5  of b0] (q0) {$q_0$} ;
					\node[latent,right=1.5 of a0  ] (a1) {$a_1$} ;
					\node[latent,below=.5 of a1  ] (b1) {$b_1$} ;
					\node[latent,below=.5 of b1] (q1) {$q_1$} ;
					\node[latent,right=1.5 of a1  ] (a2) {$a_2$} ;
					\node[latent,below=.5 of a2  ] (b2) {$b_2$} ;
					\node[latent,below=.5 of b2] (q2) {$q_2$} ;

					\node (pro1) [sec, below of=q1] {\textcolor{black}{$\tilde{x}_1$}};
					\node (pro2) [sec, below of=q2] {\textcolor{black}{$\tilde{x}_2$}};

					\factor[left=.5 of b1, yshift=0.5cm] {trajp0} { below:Grasp } {b1} {};
					\factor[left=.5 of b2, yshift=0.5cm] {trajp0} { Pose } {b2} {};
					\factor[left=.5 of b1, yshift=-.5cm] {trajp0} {above:Kin} {b0, q1,b1} {};
					\factor[left=.5 of b2, yshift=-.5cm] {} {below:Kin} {b1, q2,b2} {};
					\factor[left=1 of a1] {trajp0} {Equal} { a0, a1} {};
					\factor[left=1 of a2] {trajp0} {Equal} { a1, a2} {};
					\factor[right=.3 of a1, yshift=-0.5cm,color=brown] {} {} {a1,b1,q1} {};
					\factor[right=.3 of b1, yshift=-0.5cm,color=brown] {} {} {b1,q1} {};
					\factor[right=.3 of q1, yshift=-0.1cm,color=brown] {} {} {a1,q1} {};
					\factor[right=.3 of a2, yshift=-0.5cm,color=brown] {} {} {a2,b2} {};
					\factor[right=.3 of b2, yshift=-0.5cm,color=brown] {} {} {b2,q2} {};
					\factor[right=.3 of q2, yshift=-0.1cm,color=brown] {} {} {a2,q2} {};

					\factor[above=.3 of a1,xshift=-.5cm] {trajp0} { left:Ref } {a1} {};
					\factor[above=.3 of a2,xshift=-.5cm] {} { left:Ref } {a2} {};

				\end{tikzpicture}
			}

			 &

			\centered{
				$\to$
			}

			 &

			\centered{
				\begin{tikzpicture}[scale=0.7,every node/.style={transform shape}]
					\node[obs] (a0) {$a_0$} ;
					\node[obs,below=.5 of a0  ] (b0) {$b_0$} ;
					\node[obs,below=.5  of b0] (q0) {$q_0$} ;
					\node[obs,right=1.5 of a0  ] (a1) {$a_1$} ;
					\node[obs,below=.5 of a1  ] (b1) {$b_1$} ;
					\node[obs,below=.5 of b1] (q1) {$q_1$} ;
					\node[latent,right=1.5 of a1  ] (a2) {$a_2$} ;
					\node[latent,below=.5 of a2  ] (b2) {$b_2$} ;
					\node[latent,below=.5 of b2] (q2) {$q_2$} ;

					\node (pro1) [sec, below of=q1] {\textcolor{black}{$x_1$}};
					\node (pro2) [sec, below of=q2] {\textcolor{black}{$\tilde{x}_2$}};

					\factor[left=.5 of b1, yshift=0.5cm] {trajp0} { below:Grasp } {b1} {};
					\factor[left=.5 of b2, yshift=0.5cm] {trajp0} { Pose } {b2} {};
					\factor[left=.5 of b1, yshift=-.5cm] {trajp0} {above:Kin} {b0, q1,b1} {};
					\factor[left=.5 of b2, yshift=-.5cm] {} {below:Kin} {b1, q2,b2} {};
					\factor[left=1 of a1] {trajp0} {Equal} { a0, a1} {};
					\factor[left=1 of a2] {trajp0} {Equal} { a1, a2} {};
					\factor[right=.3 of a1, yshift=-0.5cm,color=brown] {} {} {a1,b1,q1} {};
					\factor[right=.3 of b1, yshift=-0.5cm,color=brown] {} {} {b1,q1} {};
					\factor[right=.3 of q1, yshift=-0.1cm,color=brown] {} {} {a1,q1} {};
					\factor[right=.3 of a2, yshift=-0.5cm,color=brown] {} {} {a2,b2} {};
					\factor[right=.3 of b2, yshift=-0.5cm,color=brown] {} {} {b2,q2} {};
					\factor[right=.3 of q2, yshift=-0.1cm,color=brown] {} {} {a2,q2} {};

					\factor[above=.3 of a1,xshift=-.5cm] {trajp0} { left:Ref } {a1} {};
					\factor[above=.3 of a2,xshift=-.5cm] {} { left:Ref } {a2} {};
				\end{tikzpicture}}
			 &

			\centered{
				$\to$
			}

			 &

			\centered{
				\begin{tikzpicture}[scale=0.7,every node/.style={transform shape}]
					\node[obs] (a0) {$a_0$} ;
					\node[obs,below=.5 of a0  ] (b0) {$b_0$} ;
					\node[obs,below=.5  of b0] (q0) {$q_0$} ;
					\node[obs,right=1.5 of a0  ] (a1) {$a_1$} ;
					\node[obs,below=.5 of a1  ] (b1) {$b_1$} ;
					\node[obs,below=.5 of b1] (q1) {$q_1$} ;
					\node[obs,right=1.5 of a1  ] (a2) {$a_2$} ;
					\node[obs,below=.5 of a2  ] (b2) {$b_2$} ;
					\node[obs,below=.5 of b2] (q2) {$q_2$} ;

					\node (pro1) [sec, below of=q1] {\textcolor{black}{$x_1$}};
					\node (pro2) [sec, below of=q2] {\textcolor{black}{$x_2$}};

					\factor[left=.5 of b1, yshift=0.5cm] {trajp0} { below:Grasp } {b1} {};
					\factor[left=.5 of b2, yshift=0.5cm] {trajp0} { Pose } {b2} {};
					\factor[left=.5 of b1, yshift=-.5cm] {trajp0} {above:Kin} {b0, q1,b1} {};
					\factor[left=.5 of b2, yshift=-.5cm] {} {below:Kin} {b1, q2,b2} {};
					\factor[left=1 of a1] {trajp0} {Equal} { a0, a1} {};
					\factor[left=1 of a2] {trajp0} {Equal} { a1, a2} {};
					\factor[right=.3 of a1, yshift=-0.5cm,color=brown] {} {} {a1,b1,q1} {};
					\factor[right=.3 of b1, yshift=-0.5cm,color=brown] {} {} {b1,q1} {};
					\factor[right=.3 of q1, yshift=-0.1cm,color=brown] {} {} {a1,q1} {};
					\factor[right=.3 of a2, yshift=-0.5cm,color=brown] {} {} {a2,b2} {};
					\factor[right=.3 of b2, yshift=-0.5cm,color=brown] {} {} {b2,q2} {};
					\factor[right=.3 of q2, yshift=-0.1cm,color=brown] {} {} {a2,q2} {};

					\factor[above=.3 of a1,xshift=-.5cm] {trajp0} { left:Ref } {a1} {};
					\factor[above=.3 of a2,xshift=-.5cm] {} { left:Ref } {a2} {};
				\end{tikzpicture}
			}
		\end{tabular}
		\caption{Pick and Place -- Sampling.}
	\end{subfigure} \vspace{21pt}

	\begin{subfigure}[t]{.99\textwidth}
		\centering
		\setlength{\tabcolsep}{1pt}
		\begin{tabular}{ccc}
			\centered{
				\begin{tikzpicture}[scale=0.7,every node/.style={transform shape}]
					\node[obs] (a0) {$a_0$} ;
					\node[obs,below=.5 of a0  ] (b0) {$b_0$} ;
					\node[obs,below=.5  of b0] (q0) {$q_0$} ;
					\node[latent,right=1.5 of a0  ] (a1) {$a_1$} ;
					\node[latent,below=.5 of a1  ] (b1) {$b_1$} ;
					\node[latent,below=.5 of b1] (q1) {$q_1$} ;
					\node[latent,right=1.5 of a1  ] (a2) {$a_2$} ;
					\node[latent,below=.5 of a2  ] (b2) {$b_2$} ;
					\node[latent,below=.5 of b2] (q2) {$q_2$} ;

					\node (pro1) [sec, below of=q1] {\textcolor{black}{$\tilde{x}_1$}};
					\node (pro2) [sec, below of=q2] {\textcolor{black}{$\tilde{x}_2$}};

					\factor[left=.5 of b1, yshift=0.5cm] {trajp0} { below:Grasp } {b1} {};
					\factor[left=.5 of b2, yshift=0.5cm] {trajp0} { Pose } {b2} {};
					\factor[left=.5 of b1, yshift=-.5cm] {trajp0} {above:Kin} {b0, q1,b1} {};
					\factor[left=.5 of b2, yshift=-.5cm] {} {below:Kin} {b1, q2,b2} {};
					\factor[left=1 of a1] {trajp0} {Equal} { a0, a1} {};
					\factor[left=1 of a2] {trajp0} {Equal} { a1, a2} {};
					\factor[right=.3 of a1, yshift=-0.5cm,color=brown] {} {} {a1,b1,q1} {};
					\factor[right=.3 of b1, yshift=-0.5cm,color=brown] {} {} {b1,q1} {};
					\factor[right=.3 of q1, yshift=-0.1cm,color=brown] {} {} {a1,q1} {};
					\factor[right=.3 of a2, yshift=-0.5cm,color=brown] {} {} {a2,b2} {};
					\factor[right=.3 of b2, yshift=-0.5cm,color=brown] {} {} {b2,q2} {};
					\factor[right=.3 of q2, yshift=-0.1cm,color=brown] {} {} {a2,q2} {};

					\factor[above=.3 of a1,xshift=-.5cm] {trajp0} { left:Ref } {a1} {};
					\factor[above=.3 of a2,xshift=-.5cm] {} { left:Ref } {a2} {};
				\end{tikzpicture}
			}

			 &

			\centered
			{
				$\to$
			}
			 &

			\centered{
				\begin{tikzpicture}[scale=0.7,every node/.style={transform shape}]
					\node[obs] (a0) {$a_0$} ;
					\node[obs,below=.5 of a0  ] (b0) {$b_0$} ;
					\node[obs,below=.5  of b0] (q0) {$q_0$} ;
					\node[obs,right=1.5 of a0  ] (a1) {$a_1$} ;
					\node[obs,below=.5 of a1  ] (b1) {$b_1$} ;
					\node[obs,below=.5 of b1] (q1) {$q_1$} ;
					\node[obs,right=1.5 of a1  ] (a2) {$a_2$} ;
					\node[obs,below=.5 of a2  ] (b2) {$b_2$} ;
					\node[obs,below=.5 of b2] (q2) {$q_2$} ;

					\node (pro1) [sec, below of=q1] {\textcolor{black}{$x_1$}};
					\node (pro2) [sec, below of=q2] {\textcolor{black}{$x_2$}};

					\factor[left=.5 of b1, yshift=0.5cm] {trajp0} { below:Grasp } {b1} {};
					\factor[left=.5 of b2, yshift=0.5cm] {trajp0} { Pose } {b2} {};
					\factor[left=.5 of b1, yshift=-.5cm] {trajp0} {above:Kin} {b0, q1,b1} {};
					\factor[left=.5 of b2, yshift=-.5cm] {} {below:Kin} {b1, q2,b2} {};
					\factor[left=1 of a1] {trajp0} {Equal} { a0, a1} {};
					\factor[left=1 of a2] {trajp0} {Equal} { a1, a2} {};
					\factor[right=.3 of a1, yshift=-0.5cm,color=brown] {} {} {a1,b1,q1} {};
					\factor[right=.3 of b1, yshift=-0.5cm,color=brown] {} {} {b1,q1} {};
					\factor[right=.3 of q1, yshift=-0.1cm,color=brown] {} {} {a1,q1} {};
					\factor[right=.3 of a2, yshift=-0.5cm,color=brown] {} {} {a2,b2} {};
					\factor[right=.3 of b2, yshift=-0.5cm,color=brown] {} {} {b2,q2} {};
					\factor[right=.3 of q2, yshift=-0.1cm,color=brown] {} {} {a2,q2} {};

					\factor[above=.3 of a1,xshift=-.5cm] {trajp0} { left:Ref } {a1} {};
					\factor[above=.3 of a2,xshift=-.5cm] {} { left:Ref } {a2} {};
				\end{tikzpicture}
			}
		\end{tabular}
		\caption{Pick and Place -- Optimization.}
	\end{subfigure}
	\caption{Sampling and optimization approaches use different computations to generate keyframes for the Pick and Place task plan.
		From left to right, transitions from white to gray circles indicate the order in which the value of the continuous states has been computed.
	}
	\label{fig:compu:sampling-vs-optimization-factored-nlp}
\end{figure}

\begin{figure}
	\centering
	\begin{subfigure}[t]{0.24\textwidth}
		\includegraphics[width=.9\textwidth]{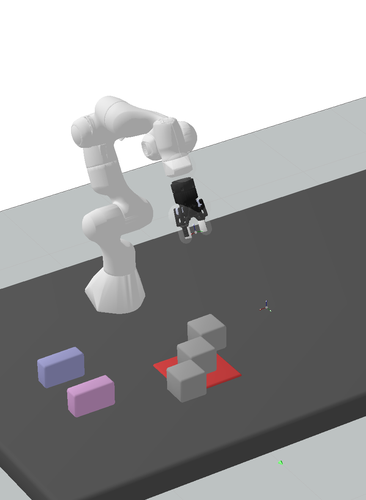}
		\caption{}\label{fig:meta:blocks-cluttered-table-env-X}
	\end{subfigure}
	\begin{subfigure}[t]{0.24\textwidth}
		\includegraphics[width=.9\textwidth]{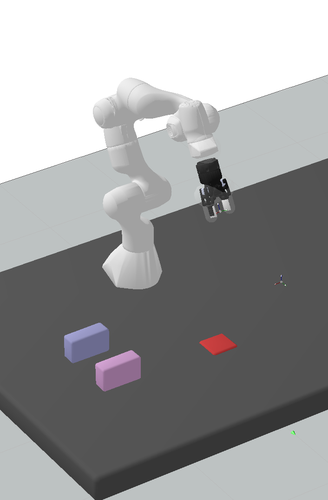}
		\caption{}\label{fig:meta:blocks-small-table-env-X}
	\end{subfigure}
	\caption{Two representative TAMP problems, where the goal is to place the two colored blocks on the red table.
		Note that the two problems are identical at the discrete task level, and only the geometric scene is different (because of the size of the table and the presence of unmovable obstacles shown in color gray).
		Interestingly, (a) is solved more efficiently with sampling-methods, while (b) is solved better with optimization methods.
	}
	\label{fig:meta:representative}
\end{figure}

\paragraph{Space of possible computations}
The trade-off between sampling and optimization has been explored in \cref{ch:mcts} for a fixed task plan.
However, in this chapter, we consider the full TAMP problem, which also entails finding potential task plans and balancing computational time and search effort among different candidate task plans.
The focus here is on generating an initial solution to the TAMP problem rapidly, rather than producing multiple solutions for the keyframes of a fixed task plan.

Addressing the full TAMP problem is considerably more challenging than solving for a fixed task plan.
Therefore, in this chapter, we confine the space of possible computations.
While in \cref{ch:mcts} we considered any conceivable sequence of conditional sampling or optimization operations, here we impose the following limitations:

\begin{itemize}
	\item The maximum allowed fine factorization for sampling is at the level of continuous states, instead of individual variables.
	      Hence, the algorithm cannot opt to compute only a relative transformation for the grasp or placement.
	      Instead, it must choose between computing a single full state, e.g., the robot configuration and the grasp, or a sequence of states, e.g., the pick and place configurations.
	\item Computing continuous states backward in time is prohibited.
	      For example, in a candidate task plan of Pick and Place, the continuous state for the place cannot be computed before the state for the pick (note that this was permitted in \cref{ch:mcts}).
	      Therefore, the continuous states must be computed either jointly (using joint optimization) or sequentially in a forward manner (first pick, then place).
\end{itemize}

With these two constraints, the available sequences of computations for a fixed Pick and Place task plan in an environment with one robot \(Q\) and two blocks \(A, B\) (refer to \cref{fig:meta:blocks-cluttered-table-env-X}) are illustrated in \cref{fig:compu:sampling-vs-optimization-factored-nlp}.
In this scenario, only two possible sequences exist because the task plan of Pick and Place comprises just two steps.
For a sequence of length three, the number of possible computations is four (i.e., the number of ordered partitions).

\section{The TAMP Computation Tree}

We first define a \emph{computational state} and a \emph{computation tree} (i.e., a tree of \emph{computational states}) to model the various computations that can be performed while resolving TAMP problems with optimization or sampling-based approaches.
Utilizing this computation tree, we will subsequently define a TAMP meta-solver as a particular search algorithm on the computation tree, and we will demonstrate that different TAMP solvers navigate the tree in very diverse and ingenious ways.

\paragraph{Discrete-continuous states}

Essentially, most TAMP solvers conceptualize the TAMP problem as a hybrid planning problem in the space of discrete-continuous states \((s, x)\).
Here, we use our recent formulation, Planning with Nonlinear Transition Constraints (PNTC) (see \cref{ch:bid}), as a point of reference.
Remember that in PNTC, TAMP is framed as a hybrid planning problem with discrete-continuous states \((s, x)\), where the initial state is \((s_0, x_0)\),
and the goals are the states \((s_g, x_g)\) with \(g \subseteq s_g\), and where the continuous state \(x_g\) satisfies the nonlinear constraints \(\phi(x_g; s_g)\).

A node \((s_{k+1}, x_{k+1})\) is a successor of \((s_k, x_k)\) if a discrete action \(a_{k+1} \in \mathcal{A}(s_k)\) exists such that \(s_{k+1} = \text{succ}(s_k, a_{k+1})\), and the pair of continuous states meet the nonlinear constraints \(\phi(x_k, x_{k+1}; s_k, s_{k+1})\) set by the discrete transition \(s_k \to s_{k+1}\).
The vector-valued nonlinear constraints \(\phi(x_k, x_{k+1};s_k, s_{k+1})\) amalgamate all the nonlinear constraints defined by the PNTC formulation
\(
\{ \phi_b(x_{k}^{b_0}, x_{k+1}^{b_1}) | ~ \phi_b \equiv \Pi(p, \tilde{p}), \forall p \subseteq s_k, \tilde{p} \subseteq s_{k+1} \}
\).

\cref{ch:bid} provides an in-depth discussion on the significance of discrete states, discrete actions, continuous states, and nonlinear constraints in the TAMP context.
We recall that in PNTC, as one continuous state corresponds to each discrete state, the continuous state now represents the keyframe configuration and the trajectory from the preceding keyframe (instead of a single configuration).

\paragraph{Computational state}

However, the notion of discrete-continuous states \((s, x)\) is not sufficient to represent the different computational operations that can be performed while solving TAMP problems.
From an optimization perspective, the continuous state is a \textit{free} variable to be optimized later.
From a sampling perspective, the continuous state is \textit{fixed}, computed with a sampling operation.
Thus, it cannot model the behavior of hybrid approaches that may combine both \emph{free} continuous states to be optimized later and \emph{fixed} continuous states that have been computed.

For designing a meta-solver, we need to define a \textit{computational state}, a more flexible notion of \textit{state} that represents that the continuous state can be either fixed or a free variable to be optimized later.
The computational state has this designation because it models the state of computations: which parts of the problem have been computed and which have not, instead of the state of the world.

Taking the PNTC formulation and notation as a reference (\cref{sec:planner:formulation}), we define a \textit{computational state} \(N\) as a 4-tuple \((s,x, \fX, \Phi)\) where,

\begin{itemize}
	\item \(s \in \mathcal{S}\) is a discrete state.
	      It represents the current discrete state of the world.
	\item \(x \in \mathcal{X}\) is a fixed continuous state.
	\item \(\fX \) is a set of free continuous states \( \fx_k \) that have not been computed yet.
	\item \(\Phi \) is a set of nonlinear constraints \( \phi_b \) on the free states, which should be satisfied when computing valid values for the free states.
\end{itemize}

Importantly, in a node \( N \), the fixed continuous state \(x\) can correspond to the discrete state of a previous step, instead of the discrete state \(s\) of the current one (see also \cref{fig:compu_example} that appears later).

Note that here we do not consider the fine factorization of the continuous states \(x \in \mathcal{X}\) into a set of variables \(\{x^i \in \mathcal{X}^i\}\), and we introduce the notation \(\tilde{x}\) to make an explicit distinction between states that have been already assigned, and those which have not.

The initial computational state is \( N_0 = (s_0,x_{0},\emptyset, \emptyset) \), i.e., a fixed continuous state \( x_{0} \) and a fixed discrete state \( s_0 \) that represents the initial continuous-discrete state of the world, and no free states or nonlinear constraints.

\paragraph{Expansions in a computation tree}

A computational state \(N = (s,x, \fX, \Phi)\) can be expanded in two different ways:

\begin{enumerate}
	\item A \textbf{discrete expansion} with a discrete action \(a\) that is applicable to the current discrete state \(s\),
	      \begin{equation}
          \label{eq:discrete_expansion}
		      \texttt{Discrete\_Expansion}(N,a) \rightarrow N' = (s',x, \fX', \Phi'),
	      \end{equation}
	      where \(s' = \text{succ}(s,a)\), \(\fX' = \fX \cup \fx'\), and \(\Phi' = \Phi \cup \phi'\).

	      Here, \(\fx'\) and \(\phi'\) are the new free continuous states and the set of nonlinear constraints that represent the motion corresponding to applying the discrete action \(a\) to the discrete state \(s\).
	      The constraints can depend on the state from the last time step, which can be either fixed or free.

	      Therefore, the discrete expansion changes the discrete state, does not change the fixed continuous state, and introduces new free states or constraints.

	      The discrete expansion is deterministic, and there is a finite number of possible expansions for each node.

	\item A \textbf{numeric expansion} that assigns values to the set of free states \(\fX\), subject to the constraints \(\Phi\).
	      \begin{equation}
          \label{eq:numeric_expansion}
		      \texttt{Numeric\_Expansion}(N) \rightarrow N' = (s,x', \emptyset, \emptyset) ~ \text{or} ~ \texttt{FAIL},
	      \end{equation}
	      where \(x'\) is the value assignment for the last free state in \(\fX\).
	      The values for all the continuous states in the sequence are stored, but only the last fixed state is required for future expansions.

	      The numeric expansion is stochastic,
	      can be executed infinitely, and there could be zero, a finite, or an infinite number of valid possible expansions for each node.
	      If the constraints are not satisfiable, as is often the case in task and motion planning, the expansion fails, resulting in a dead state \texttt{FAIL}.

	      Usually, nonlinear constraints define a manifold of possible solutions, resulting in an infinite branching factor for the numeric expansion.

	      Even if a feasible expansion exists, we cannot guarantee that one attempt to compute values will succeed because the optimization solver might fail to find a solution.
	      However,
	      it could be that by repeating the operation, we can expand the node correctly.
	      This stochastic behavior is implemented using randomized initialization and cost functions for nonlinear optimization, randomized constraint sampling, or sample-based motion planning.

\end{enumerate}

Why does it make sense to have a notion of free states?
First, free states in a computational state will be optimized later, allowing us to consider the joint nonlinear constraints that depend on the variables of the next step (e.g., \(\phi(x_1,x_2;s_1,s_2)\)).
As demonstrated by the success of optimization-based TAMP solvers, it is often beneficial to optimize the continuous variables jointly, rather than sequentially.
Moreover, numeric expansions are computationally expensive, and it may be advantageous to delay them as much as possible.

A \textit{computation tree} is a tree whose nodes are computational states, and whose edges are discrete or numeric expansions.
In the context of TAMP, we denote the computation tree as the \textbf{TAMP Computation Tree}.
To solve the original TAMP problem using the TAMP Computation Tree, the goal is to find a valid sequence of computational states \( N_0, N_1, \ldots, N_K \) such that \( N_K = (s_K,x_{K},\emptyset, \emptyset) \) is a goal state, which means reaching a discrete goal state \( s_K \supseteq g\), with a continuous state \( x_{K} \) that satisfies the nonlinear constraints \( \phi(x_{K};s_K) =
0 \), and without any free states or remaining constraints.

\section{An Example of a TAMP Computation Tree and Computational States}

In \cref{fig:compu_example}, we show an instance of a TAMP Computation Tree for a TAMP problem with one robot, \(Q\), two movable objects, \(A\) and \(B\), and a red table.

The initial state \(x_0\) is shown in \cref{fig:meta-tree-init}, and the initial discrete state is \\
\(s_0 = [ \texttt{parent\_A = A\_init}, \texttt{parent\_B = B\_init}, \texttt{gripper\_Q = free} ]\).
Given the initial state \(N_0=(s_0,x_0, \{\}, \{\})\), we can apply a discrete action \(a_1 = \texttt{pick A with Q from A\_init }\) to obtain a new node \(N_1 = (s_1,x_0, \fx_1, \Phi)\), with \(\Phi = \{\phi(x_0,\fx_1;s_0,s_1)\}\) and \(s_1 = \textup{succ}(s_0,a_1)\).

Thus, we now have a different discrete state that assumes the discrete action will be applied successfully.
However, the continuous state has not been assigned and remains as a free state subject to constraints.
Similarly, we can extend the initial state with another valid action \(a_1' = \texttt{pick B with Q from B\_init}\), resulting in another state \(N_2\) with a set of different nonlinear constraints \(\Phi'\).

In the computational state \(N_1=(s_1,x_0, \{\fx_1\}, \{\phi(x_0,\fx_1;s_0,s_1)\})\), we can decide to compute the free variables.
Specifically, we want to assign values to \(\fx_1\) subject to the constraints \(\phi(x_0,\fx_1;s_0,s_1)\).
If this operation succeeds, it results in the state \(N_3=(s_1,x_1, \{\}, \{\})\).
Now, the last continuous state is \(x_1\), and both the set of free states and constraints are empty.
If the assignment process fails, a dead child node is created.

Alternatively, we can perform a discrete expansion of the node \(N_1\) with new discrete actions \(a_2=\texttt{place A with Q on red table}\), resulting in \(N_4=(s_2,x_0, \{\fx_1, \fx_2\}, \Phi'')\), where \(\Phi'' = \{\phi(x_0,\fx_1;s_0,s_1), \phi(\fx_1,\fx_2;s_1,s_2)\}\).
The discrete state has changed to
\(s_2\) = [ \texttt{parent\_A = red table}, \texttt{parent\_B = B\_init}, \texttt{gripper\_Q = free}],
and now we have two free states and additional constraints (\cref{fig:graphnlp_representative2}).

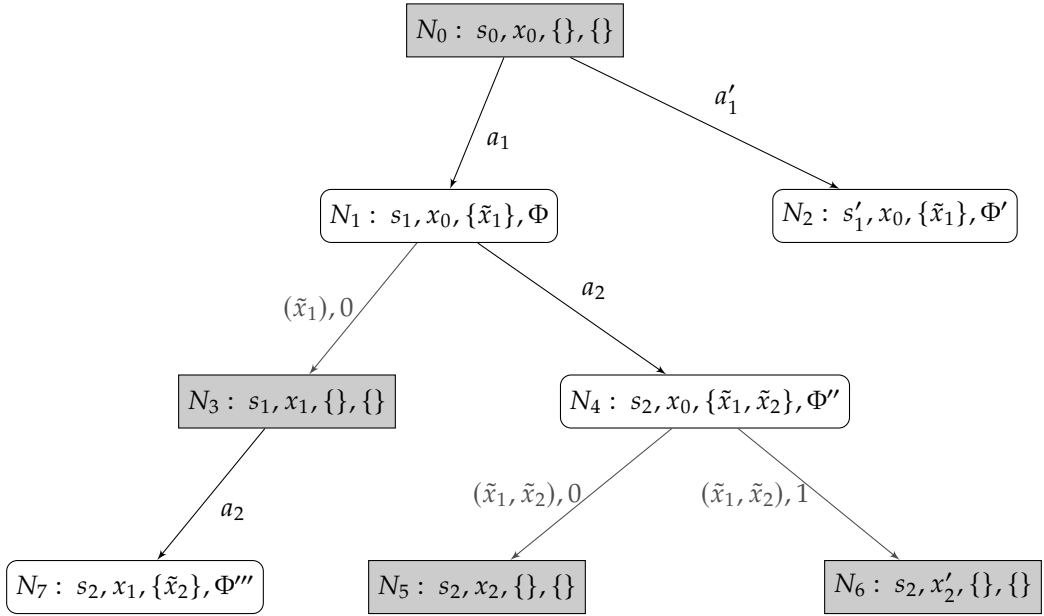
\begin{figure}
	\centering
	\begin{tikzpicture}[node distance = 7em, auto]
		\node [block] (n0) {$N_0: ~ s_0, x_0, \{\},\{\}$};

		\node [block2, below of=n0, xshift=-1cm] (n1) {$N_1 : ~ s_1, x_0, \{\fx_1\} , \Phi$};

		\node [block2, below of=n0,xshift=5cm] (n3) {$N_2 : ~ s_1', x_0, \{\fx_1\}, \Phi'$};

		\node [block,below of=n1,xshift=-2cm] (n5) {$N_3: ~ s_1,x_1, \{\}, \{\}$};
		\node [block2,below of=n1,xshift=3.5cm] (n7) {$N_4: ~ s_2, x_0, \{\fx_1, \fx_2\}, \Phi''$};

		\node [block,below of=n7,xshift=-3cm] (n8) {$N_5: ~ s_2,x_2, \{\}, \{\}$ };

		\node [block,below of=n7,xshift=3cm] (n8bis) {$N_6: ~s_2,x_2', \{\}, \{\}$ };

		\node [block2,below of=n5,xshift=-2cm] (n10) {$N_7: ~ s_2,x_1,\{\fx_2\},\Phi'''$};

		\path [line, color=black!70] (n1) -- node [left]{$(\fx_1),0$} (n5);

		\path [line, color=black!70] (n7) -- node [left]{$(\fx_1, \fx_2),0$} (n8);
		\path [line, color=black!70] (n7) -- node [left]{$(\fx_1, \fx_2),1$} (n8bis);

		\path [line] (n0) -- node {$a_1$ } (n1);

		\path [line] (n0) -- node {$a_1'$ } (n3);
		\path [line] (n1) -- node {$a_2$ } (n7);
		\path [line] (n5) -- node { $a_2$ } (n10);

	\end{tikzpicture}
	\caption{The TAMP Computation Tree.
		White nodes contain both fixed continuous states and free continuous states subject to constraints .
		Gray nodes do not contain any free states.
		The color of the edges indicates the two types of expansions: numeric expansion (in gray) and discrete extension (in black).
		See the main text for details.
		\vspace{.5cm}}
	\label{fig:compu_example}
\end{figure}

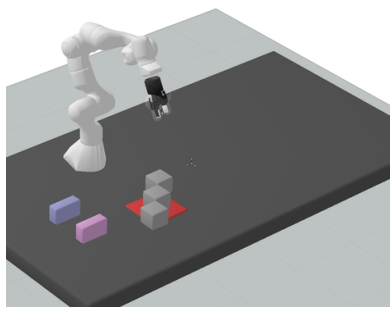
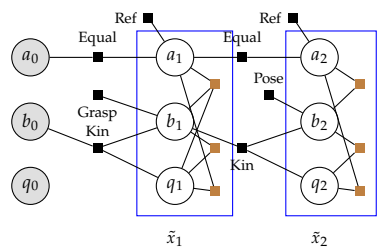
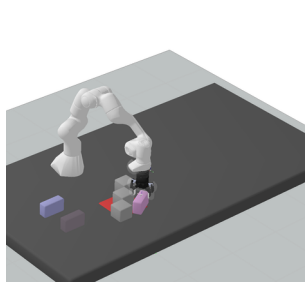
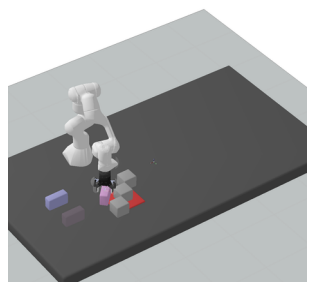
\begin{figure}
	\centering
	\begin{subfigure}[t]{0.6\textwidth}
		\centering
		\includegraphics[width=.6\textwidth]{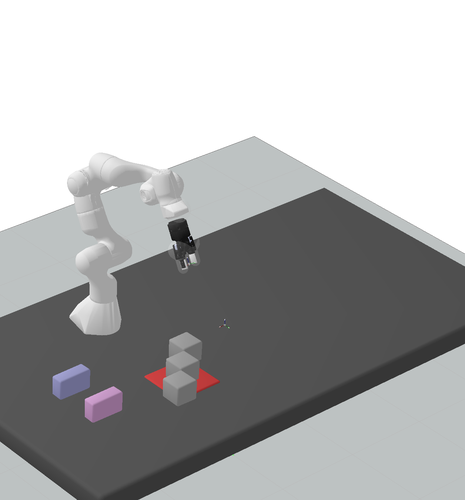}
		\caption{Initial continuous state $x_0$.}
		\label{fig:meta-tree-init}
	\end{subfigure} \\
	\vspace{20pt}

	\begin{subfigure}[t]{0.9\textwidth}
		\centering
		\begin{tikzpicture}[scale=0.7,every node/.style={transform shape}]
			\node[obs] (a0) {$a_0$} ;
			\node[obs,below=.5 of a0  ] (b0) {$b_0$} ;
			\node[obs,below=.5  of b0] (q0) {$q_0$} ;
			\node[latent,right=2 of a0  ] (a1) {$a_1$} ;
			\node[latent,below=.5 of a1  ] (b1) {$b_1$} ;
			\node[latent,below=.5 of b1] (q1) {$q_1$} ;
			\node[latent,right=2 of a1  ] (a2) {$a_2$} ;
			\node[latent,below=.5 of a2  ] (b2) {$b_2$} ;
			\node[latent,below=.5 of b2] (q2) {$q_2$} ;

			\node (pro1) [sec, below of=q1] {\textcolor{black}{$\tilde{x}_1$}};
			\node (pro2) [sec, below of=q2] {\textcolor{black}{$\tilde{x}_2$}};

			\draw[color=blue] (2,-3) rectangle (3.8,.5);
			\draw[color=blue] (4.8,-3) rectangle (6.5,.5);

			\factor[above=.3 of a1,xshift=-.5cm] {trajp0} { left:Ref } {a1} {};
			\factor[above=.3 of a2,xshift=-.5cm] {trajp0} { left:Ref } {a2} {};
			\factor[left=1 of b1, yshift=0.5cm] {trajp0} { below:Grasp } {b1} {};
			\factor[left=.5 of b2, yshift=0.5cm] {trajp0} { Pose } {b2} {};
			\factor[left=1 of b1, yshift=-.5cm] {trajp0} {above:Kin} {b0, q1,b1} {};
			\factor[left=1 of b2, yshift=-.5cm] {} {below:Kin} {b1, q2,b2} {};
			\factor[left=1 of a1] {trajp0} {Equal} { a0, a1} {};
			\factor[left=1 of a2] {trajp0} {Equal} { a1, a2} {};
			\factor[right=.3 of a1, yshift=-0.5cm,color=brown] {} {} {a1,b1,q1} {};
			\factor[right=.3 of b1, yshift=-0.5cm,color=brown] {} {} {b1,q1} {};
			\factor[right=.3 of q1, yshift=-0.1cm,color=brown] {} {} {a1,q1} {};
			\factor[right=.3 of a2, yshift=-0.5cm,color=brown] {} {} {a2,b2} {};
			\factor[right=.3 of b2, yshift=-0.5cm,color=brown] {} {} {b2,q2} {};
			\factor[right=.3 of q2, yshift=-0.1cm,color=brown] {} {} {a2,q2} {};
		\end{tikzpicture}
		\caption{Free continuous states ($\tilde{x}_1$, $\tilde{x}_2$) subject to constraints.
			Each state, represented by a blue rectangle, is factorized into variables $[a,b,q]$.
			Nonlinear constraints are shown as squares.
			Gray circles represent fixed variables (in this case, the initial state).
		}
		\label{fig:graphnlp_representative2}
	\end{subfigure} \\
	\vspace{20pt}

	\begin{subfigure}[t]{0.35\textwidth}
		\centering
		\includegraphics[width=.8\textwidth]{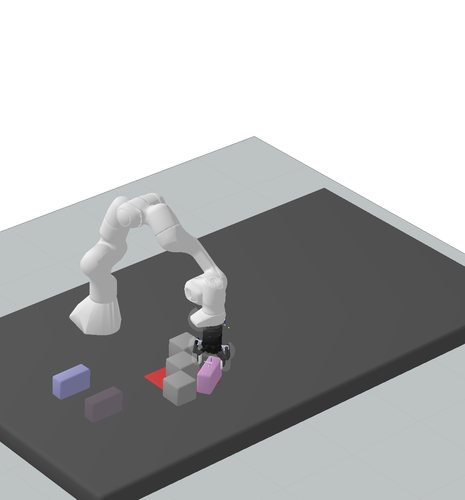}
		\caption{Continuous state $x_2$ for discrete state $s_2$.}
		\label{fig:states1}
	\end{subfigure}
	\hspace{1cm}
	\begin{subfigure}[t]{0.35\textwidth}
		\centering
		\includegraphics[width=.8\textwidth]{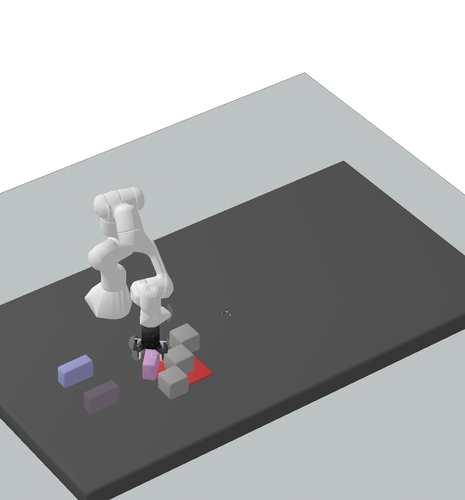}
		\caption{A different continuous state $x_2'$ for discrete state $s_2$.}
		\label{fig:states2}
	\end{subfigure}
	\vspace{10pt}
	\caption{Components of computational states in the TAMP Computation Tree of \cref{fig:compu_example} in the environment shown in \cref{fig:meta:blocks-cluttered-table-env-X}.}
\end{figure}

We can decide to expand \(N_3\) numerically by computing the two free states jointly, resulting in the state \(N_5 = (s_2, x_2, \{\}, \{\})\), with a possible \(x_2\) shown in \cref{fig:states1}.
This operation can be repeated multiple times, resulting in different (if successful) compute states \(N_6 = (s_2, x_2', \{\}, \{\})\) with \(x_2'\) shown in \cref{fig:states2}.

\clearpage

\section{A Practical Meta-Solver for TAMP}

The TAMP Computation Tree provides a framework for designing algorithms for TAMP that automatically select between joint optimization and sampling operations.
With this formulation, designing a TAMP meta-solver corresponds to defining a search strategy on the tree that chooses which node to expand next.

Exploration in the TAMP Computation Tree is a daunting task.
The state space of the computation tree has an explicitly infinite branching factor in the compute values operation, and the number of possible unique discrete states grows exponentially with the number of objects.
In addition to these two inherent TAMP challenges, the state space in the TAMP Computation Tree is even larger, as it includes the state of computation rather than just the original discrete-continuous state space.

In this section, we propose a simple search algorithm that explores the TAMP Computation Tree in an effective manner.
The goal of this algorithm is not to outperform state-of-the-art TAMP solvers, but to provide a foundational understanding and intuition of TAMP meta-solvers, which could be revolutionary for designing efficient TAMP solvers in the future.

\subsection{Algorithm}

Our TAMP meta-solver is a heuristic search algorithm on the space of computational states.
The objective of the search algorithm is to find an optimal solution to the TAMP problem, i.e., the shortest task plan that has a valid motion, using the least compute effort.

\paragraph{The computational level}

To manage the infinite branching factor of the numeric expansion, we associate a computational level with each node, \(\text{l}(N) \in \mathbb{N}\), to incrementally enumerate the infinite branching factor, analogous to the level used in PDDLStream.
Each time a node with free variables undergoes numeric expansion, we increment its level by one and reintroduce it into the open list.
When we create a new node, it inherits the level of its parent node.

Incorporating the level into the node score provides an iterative-widening search algorithm that can repeat numeric expansion operations, each time at a higher cost.
This level is the mechanism that allows us to explore new branches of the tree while conducting new expansions in previous nodes, resulting in different continuous states that could potentially lead to a solution.

\paragraph{TAMP cost and compute cost}

We distinguish between two types of costs: compute cost and TAMP cost.
The TAMP cost is the cost in the original TAMP problem, assuming each discrete action has a unit cost.
The compute cost refers to the cost of the underlying compute operations.

For the compute cost, we use a straightforward cost model for the operations in the TAMP Computation Tree.
\begin{itemize}
	\item The compute cost for a discrete expansion is zero.
	\item The compute cost for a numeric expansion of a node \(N = (s, x, \fX, \Phi)\) equals the number of free states \(|\fX|\) we attempt to compute.
\end{itemize}

Thus, the compute cost-to-go \(c_c(N)\) for node \(N = (s, x, \fX, \Phi)\) is the optimal compute cost required to reach a goal state, i.e., solving the original TAMP problem from \(N\).
It can be expressed as the sum of two components:
\begin{equation}
	c_c(N) = |\fX| + c_{\text{TAMP}}(s)\,,
\end{equation}
where \(|\fX|\) is the number of free states, and \(c_{\text{TAMP}}(s)\), the TAMP cost-to-go, is the length of the shortest task plan from \(s\) to a goal state that is also valid for continuous motion.

We can also define a compute cost-to-come \(g_c(N)\) as the number of continuous states that have been assigned, starting from the root node.
The TAMP cost-to-come \(g_{\text{TAMP}}(s)\) is the length of the task plan from the root to the current node.

Evaluating the function \(c_{\text{TAMP}}(s)\) would require solving a TAMP problem in itself.
However, we can define a lower bound \(c_{\text{TASK}}(s) \leq c_{\text{TAMP}}(s)\) that disregards the geometric information and can be efficiently calculated by calling a discrete task planner.

\paragraph{Score function}

To balance the TAMP cost with the compute cost and progressively explore the infinite branching factor of the numeric expansions, we define a score function.
The score function \(f(N)\) for a node \(N = (s, x, \fX, \Phi)\) is a tuple of two values:
\begin{equation}
	f(N) \rightarrow [ c_{\text{TASK}}(s) + g_{\text{TAMP}}(s) + l(N), -g_c(N) ]\,.
	\label{eq:f-score}
\end{equation}

The best node is selected based on the first value (lower is better), and the second value is used for breaking ties (lower is better, i.e., we prefer a node where we have already invested compute effort).
If ties persist, they are broken randomly.
These two components prioritize expanding nodes in task plans with potentially few discrete actions, which have not been previously attempted, and that will require less compute effort to achieve (since some states have already been computed).

\begin{figure}[!t]
	\centering
	\begin{minipage}{0.8\textwidth}
		\begin{algorithm}[H]
			\centering
			\caption{The TAMP meta-solver.}\label{alg:code-meta}
			\begin{algorithmic}[1]
				\State \textbf{Input:}
				$N_0 = (s_0,x_0, \{\}, \{\}, l=0)$  \\
				\Comment{{\color{gray} \small Initial computational state (with computational level $l$)}}
				\State $L \gets \{ N_0 \}$ \Comment{{\color{gray} \small Open list}}
				\While{ $| L | > 0$ }
				\State $N \gets \texttt{Choose Best}(L)$ \Comment{{\color{gray} \small Node with best score $f(N)$ (\cref{eq:f-score})}}
				\State $\texttt{Remove} ~ N ~ \texttt{from} ~ L$
				\If{$|N.\fX| = 0 ~ \And ~ g \subseteq N.s $}
				\Return $N$ \Comment{{\color{gray} \small We have reached the discrete goal, and all continuous states are assigned}}
				\EndIf
				\If{$|N.\fX| > 0$}
				\State $N' \gets \texttt{Numeric\_Expansion}(N) \quad $ \Comment{{\color{gray} \small Compute values for free variables (\cref{eq:numeric_expansion})}}
				\If{$N' \neq \texttt{FAIL}$}
				\State $N'.l \gets N.l$ \Comment{{\color{gray} \small Inherit level}}
				\State $L \gets L \cup \{ N' \}$ \Comment{{\color{gray} \small Add new node to open list}}
				\State $N.l \gets N.l + 1$ \Comment{{\color{gray} \small Increase level}}
				\State $L \gets L \cup \{ N \}$ \Comment{{\color{gray} \small Add old node to open list}}
				\Else
				\State $N.l \gets N.l + 1$ \Comment{{\color{gray} \small Increase level}}
				\State $L \gets L \cup \{ N \}$ \Comment{{\color{gray} \small Add old node to open list}}
				\State \textbf{continue}
				\EndIf
				\EndIf
				\For{$a \in \mathcal{A}(N.s)$}
				\State $N' \gets \texttt{Discrete\_Expansion}(N,a)\quad$ \Comment{{\color{gray} \small Expand with discrete action (\cref{eq:discrete_expansion})}}
				\State $N'.l \gets N.l$ \Comment{{\color{gray} \small Inherit level}}
				\State $L \gets L \cup \{ N' \}$ \Comment{{\color{gray} \small Add node to open list}}
				\EndFor
				\EndWhile
			\end{algorithmic}
		\end{algorithm}
	\end{minipage}
\end{figure}

\paragraph{What is missing to outperform state-of-the-art TAMP solvers?}

In the experiments section, we demonstrate how our simple meta-algorithm can outperform
sample-based and optimization-based TAMP solvers in short-horizon planning problems involving two objects, and one or two robots, but which require finding the right balance between individually sampling states or optimizing them jointly.

To match and surpass the performance of state-of-the-art TAMP solvers on larger TAMP problems, our meta-solver needs two additional ingredients to share information between the different computational states in the tree:

\begin{itemize}

	\item \emph{Detection and encoding of geometric conflicts} – A fundamental challenge in TAMP is that the PDDL heuristic distance does not include information about geometry and can be rather uninformative in some contexts.
	      To address this issue, we could incrementally integrate information about geometry back into the discrete task description.
	      In our previous TAMP solvers, we incorporated negative information in the form of conflicts (either prefixes or subsets of infeasible nonlinear constraints) to modify the original discrete planning task.
	      Alternatives include employing learned heuristics based on the solutions of previous similar problems or designing heuristics that, while fast to compute, can incorporate some knowledge about geometry.

	\item \emph{Reusing computations} – Solving a TAMP problem involves multiple numeric expansions, which can be computationally expensive.
	      To alleviate this, we could reuse computations from previous nodes in new expansions.
	      For instance, if we have already computed a trajectory between two configurations, that computation could be reused in other expansions requiring the same calculation.

\end{itemize}

Incorporating these two ideas into the meta-solver (\cref{alg:code-meta}) using the tools we developed in \cref{ch:bid}, namely conflict detection and creating a database of feasible partial solutions, is a direction for future work.
Another open question for future research is how to better utilize the history of all computations, reusing information between sibling nodes, different nodes with the same discrete state, and similar task plans.

\section{Analyzing and Designing TAMP Solvers with the TAMP Computation Tree}

The TAMP Computation Tree is a framework that models the different computational operations performed while solving TAMP problems.
Crucially, it provides a unified framework encompassing both conditional sampling and joint optimization.
In the TAMP Computation Tree, conditional sampling operations correspond to sampling the continuous states one at a time, while joint nonlinear optimization involves assigning values to a sequence of free states collectively.

We propose that TAMP solvers can be conceptualized as sophisticated search strategies for navigating the TAMP Computation Tree.
These strategies can share information between different computational states and store the history of all prior decisions and computations.

\begin{theorem}[Existence of a Search Algorithm in the TAMP Computation Tree]
	\label{th:tree}
	If the set of computational operations a TAMP solver utilizes is encompassed by the possible expansions in the TAMP Computation Tree, then a search algorithm exists within the TAMP Computation Tree that can emulate the solver's behavior.
\end{theorem}

This theorem presents an existence argument; we cannot offer a constructive proof, meaning we cannot articulate a concise analytical function for selecting which node to expand next.
The search algorithm need not be a classical one like depth-first or breadth-first search; it can be a complex system that decides on operations to perform based on all prior actions.
This could include identifying geometric conflicts or assessing bounds of infeasibility prior to a numeric expansion.

\begin{figure}[!t]
	\centering
	\begin{minipage}{0.8\textwidth}
		\begin{algorithm}[H]
			\centering
			\caption{Optimization-based TAMP solver in the TAMP Computation Tree.
				Changes with respect to the meta-solver highlighted in \textcolor{orange}{orange}.
			}\label{alg:code-meta-opt}
			\begin{algorithmic}[1]
				\State \textbf{Input:}
				$N_0 = (s_0,x_0, \{\}, \{\}, l=0)$  \\
				\Comment{{\color{gray} \small Initial computational state (with computational level $l$)}}
				\State $L \gets \{ N_0 \}$ \Comment{{\color{gray} \small Open list}}
				\While{ $| L | > 0$ }
				\State $N \gets \texttt{Choose Best}(L)$ \Comment{{\color{gray} \small Node with best score $f(N)$ (\cref{eq:f-score})}}
				\State $\texttt{Remove} ~ N ~ \texttt{from} ~ L$
				\If{$|N.\fX| = 0 ~ \And ~  g \subseteq N.s$}
				\Return $N$ \Comment{{\color{gray} \small We have reached the discrete goal, and all continuous states are assigned}}
				\EndIf
				\If{$|N.\fX| > 0 \And \textcolor{orange}{g \subseteq N.s}$} \Comment{\small \textcolor{orange}{\textbf{Optimization-based:}
						Only compute values for a full candidate task plan}}

				\State $N' \gets \texttt{Numeric\_Expansion}(N)\quad$ \Comment{{\color{gray} \small Compute values for free variables (\cref{eq:numeric_expansion})}}
				\If{$N' \neq \texttt{FAIL}$}
				\State $N'.l \gets N.l$ \Comment{{\color{gray} \small Inherit level}}
				\State $L \gets L \cup \{ N' \}$ \Comment{{\color{gray} \small Add new node to open list}}
				\State $N.l \gets N.l + 1$ \Comment{{\color{gray} \small Increase level}}
				\State $L \gets L \cup \{ N \}$ \Comment{{\color{gray} \small Add old node to open list}}
				\Else
				\State $N.l \gets N.l + 1$ \Comment{{\color{gray} \small Increase level}}
				\State $L \gets L \cup \{ N \}$ \Comment{{\color{gray} \small Add old node to open list}}
				\State \textbf{continue}
				\EndIf
				\EndIf
				\For{$a \in \mathcal{A}(N.s)$}
				\State $N' \gets \texttt{Discrete\_Expansion}(N,a)\quad$ \Comment{{\color{gray} \small Expand with discrete action (\cref{eq:discrete_expansion})}}
				\State $N'.l \gets N.l$ \Comment{{\color{gray} \small Inherit level}}
				\State $L \gets L \cup \{ N' \}$ \Comment{{\color{gray} \small Add node to open list}}
				\EndFor
				\EndWhile
			\end{algorithmic}
		\end{algorithm}
	\end{minipage}
\end{figure}

Some advanced TAMP solvers factorize TAMP problems (i.e., the full discrete and continuous state is decomposed).
For clarity in this presentation, we have depicted a TAMP Computation Tree in the full configuration space, without explicitly modeling factorization.
Incorporating factorization would simply require breaking down the full continuous state into variables, as demonstrated in the Planning with Nonlinear Transition Constraints formulation.

We now illustrate how traditional sample-based or optimization-based TAMP solvers are represented as search algorithms in the TAMP Computation Tree that utilize only a subset of the available operations.

\begin{figure}[!t]
	\centering
	\begin{minipage}{0.8\textwidth}
		\begin{algorithm}[H]
			\caption{Sampling-based TAMP solver in the TAMP Computation Tree.
					{Changes with respect to the meta-solver highlighted in \textcolor{orange}{orange}.}}
			\label{alg:code-meta-sampling}
			\begin{algorithmic}[1]
				\State \textbf{Input:}
				$N_0 = (s_0,x_0, \{\}, \{\}, l=0)$  \\
				\Comment{{\color{gray} \small Initial computational state (with computational level $l$)}}
				\State $L \gets \{ N_0 \}$ \Comment{{\color{gray} \small Open list}}
				\While{ $| L | > 0$ }
				\State $N \gets \texttt{Choose Best}(L)$ \Comment{{\color{gray} \small Node with best score $f(N)$ (\cref{eq:f-score})}}
				\State $\texttt{Remove} ~ N ~ \texttt{from} ~ L$
				\If{$|N.\fX| = 0 ~ \And ~ g \subseteq N.s $}
				\Return $N$ \Comment{{\color{gray} \small We have reached the discrete goal, and all continuous states are assigned}}
				\EndIf
				\If{$|N.\fX| > 0$}
				\State $N' \gets \texttt{Numeric\_Expansion}(N)\quad$ \Comment{{\color{gray} \small Compute values for free variables (\cref{eq:numeric_expansion})}}
				\If{$N' \neq \texttt{FAIL}$}
				\State $N'.l \gets N.l$ \Comment{{\color{gray} \small Inherit level}}
				\State $L \gets L \cup \{ N' \}$ \Comment{{\color{gray} \small Add new node to open list}}
				\State $N.l \gets N.l + 1$ \Comment{{\color{gray} \small Increase level}}
				\State $L \gets L \cup \{ N \}$ \Comment{{\color{gray} \small Add old node to open list}}
				\Else
				\State $N.l \gets N.l + 1$ \Comment{{\color{gray} \small Increase level}}
				\State $L \gets L \cup \{ N \}$ \Comment{{\color{gray} \small Add old node to open list}}
				\State \textbf{continue}
				\EndIf
				\EndIf
				\If{\textcolor{orange}{$|N.\fX| = 0$}}
				\Comment{\small \textcolor{orange}{\textbf{Sampling-Based:}
						Discrete expansion only in nodes without free variables}}
				\For{$a \in \mathcal{A}(N.s)$}
				\State $N' \gets \texttt{Discrete\_Expansion}(N,a)\quad$ \Comment{{\color{gray} \small Expand with discrete action (\cref{eq:discrete_expansion})}}
				\State $N'.l \gets N.l$ \Comment{{\color{gray} \small Inherit level}}
				\State $L \gets L \cup \{ N' \}$ \Comment{{\color{gray} \small Add node to open list}}
				\EndFor
				\EndIf
				\EndWhile
			\end{algorithmic}
		\end{algorithm}
	\end{minipage}
\end{figure}

\paragraph{Optimization-based TAMP solver}

Within the TAMP Computation Tree framework, an optimization-based solver applies numeric expansion exclusively to computational states that satisfy the discrete goal—thereby optimizing all continuous states for a candidate plan collectively at the end.

\cref{alg:code-meta-opt} demonstrates this approach by altering a single line in the TAMP meta-solver, highlighted in \textcolor{orange}{orange}.

\paragraph{Sampling-based TAMP solver}
In the TAMP Computation Tree framework, a sampling-based solver permits only the discrete expansion of nodes without free variables.
This approach limits the number of free states in a compute state to one, ensuring that motion is computed sequentially, step by step.
\cref{alg:code-meta-sampling} implements this method by changing one line in the TAMP meta-solver, also highlighted in \textcolor{orange}{orange}.

\section{Experimental Results}

\begin{figure}

	\centering
	\begin{subfigure}[t]{0.2\textwidth}
		\includegraphics[width=.9\textwidth]{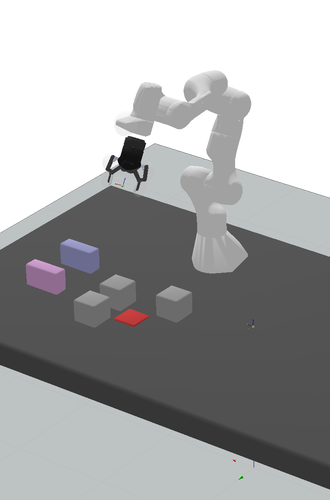}
		\caption{}\label{fig:meta:tower-env}
	\end{subfigure}
	\begin{subfigure}[t]{0.2\textwidth}
		\includegraphics[width=.9\textwidth]{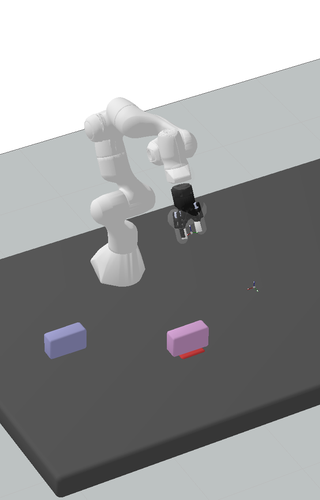}
		\caption{}\label{fig:meta:move_first-env}
	\end{subfigure}
	\begin{subfigure}[t]{0.2\textwidth}
		\includegraphics[width=.9\textwidth]{figs_old_computer/opti_computation/Pictures/r1_2b_3obs_put_on_table.png}
		\caption{}\label{fig:meta:blocks-cluttered-table-env}
	\end{subfigure}
	\begin{subfigure}[t]{0.2\textwidth}
		\includegraphics[width=.9\textwidth]{figs_old_computer/opti_computation/Pictures/r1_2b_small_table_0_obs.png}
		\caption{}\label{fig:meta:blocks-small-table-env}
	\end{subfigure} \\

	\begin{subfigure}[t]{0.25\textwidth}
		\includegraphics[width=.9\textwidth]{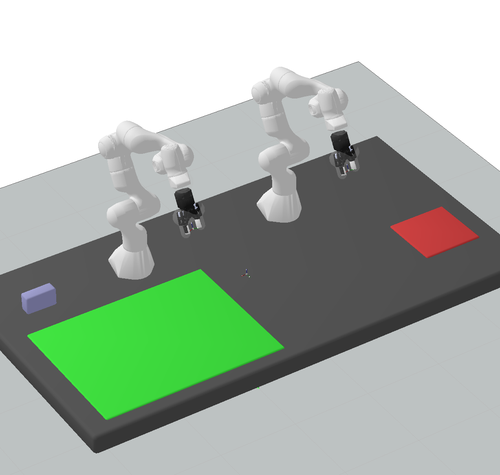}
		\caption{}\label{fig:meta:transfer-big-table-env}
	\end{subfigure}
	\begin{subfigure}[t]{0.25\textwidth}
		\includegraphics[width=.9\textwidth]{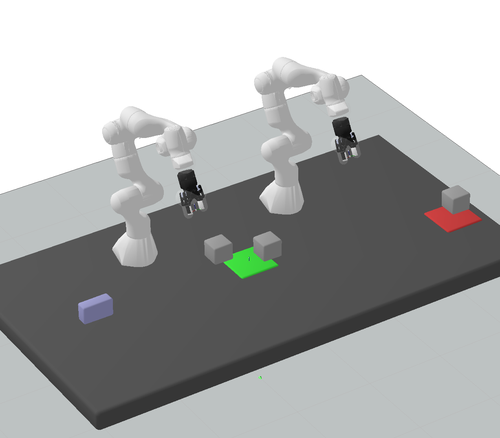}
		\caption{}\label{fig:meta:transfer-with-obs-env}
	\end{subfigure}
	\begin{subfigure}[t]{0.25\textwidth}
		\includegraphics[width=.9\textwidth]{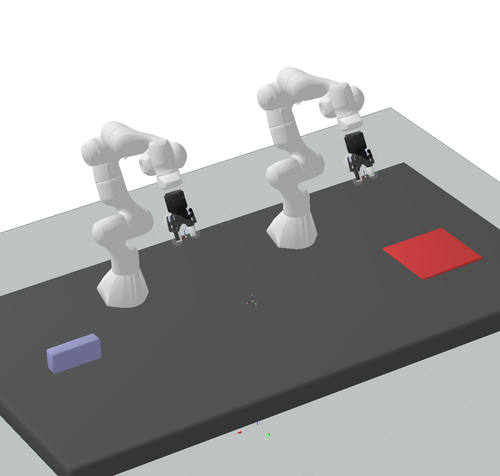}
		\caption{}\label{fig:meta:handover-env}
	\end{subfigure}
	\caption{Benchmark TAMP problems.
	}
	\label{fig:meta:tamp-problems}
\end{figure}

\begin{figure}[!t]
	\centering
	\begin{subfigure}[t]{.8\textwidth}
		\centering
		\includegraphics[width=.18\textwidth]{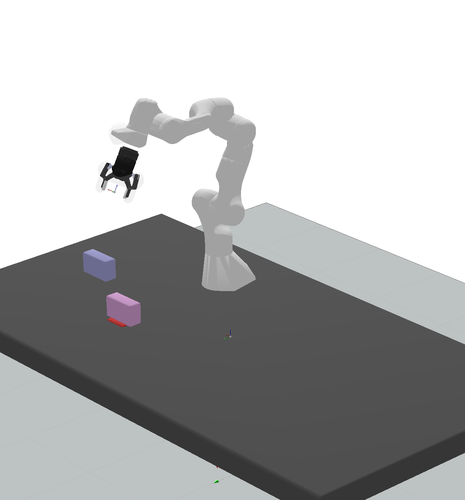}
		\includegraphics[width=.18\textwidth]{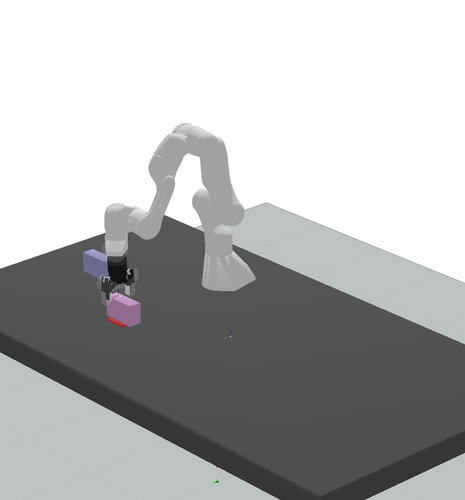}
		\includegraphics[width=.18\textwidth]{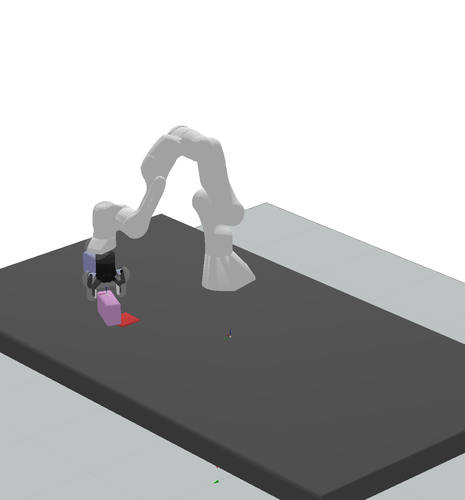}
		\includegraphics[width=.18\textwidth]{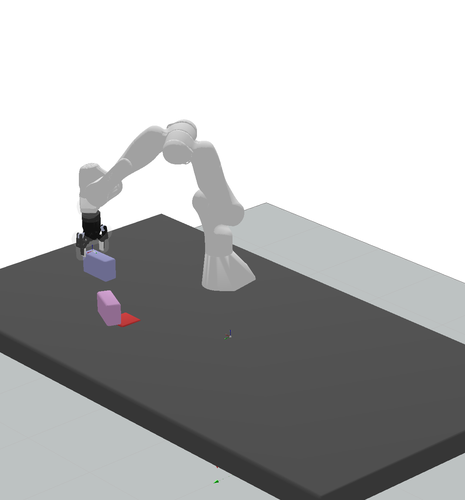}
		\includegraphics[width=.18\textwidth]{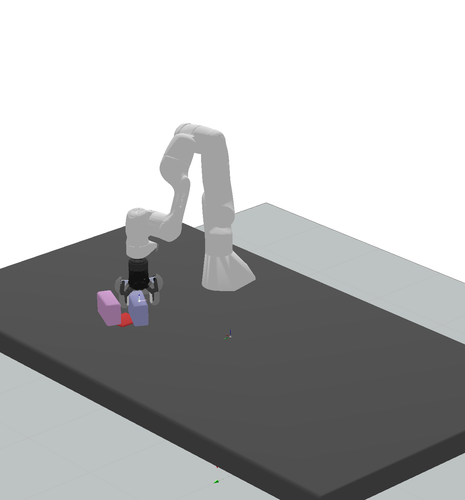}
		\caption{}
		\label{fig:meta:move_first-sol}
	\end{subfigure}

	\begin{subfigure}[t]{.8\textwidth}
		\centering
		\includegraphics[width=.18\textwidth]{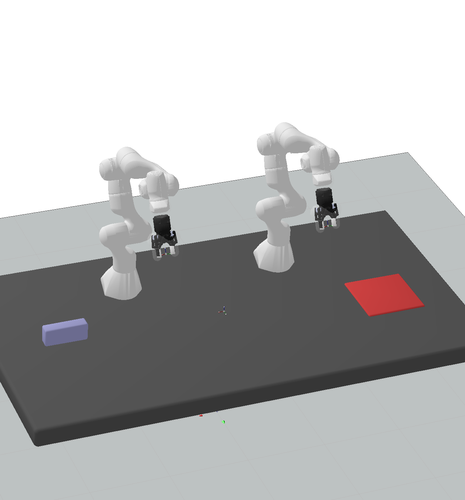}
		\includegraphics[width=.18\textwidth]{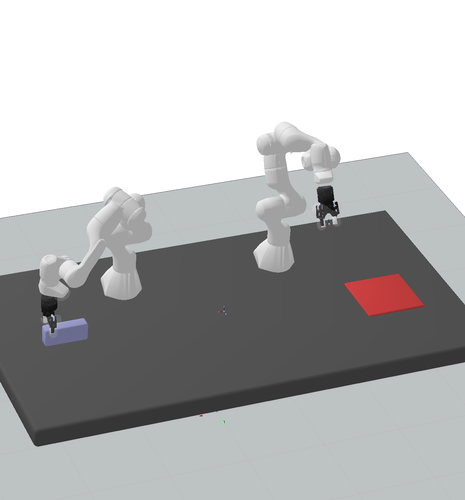}
		\includegraphics[width=.18\textwidth]{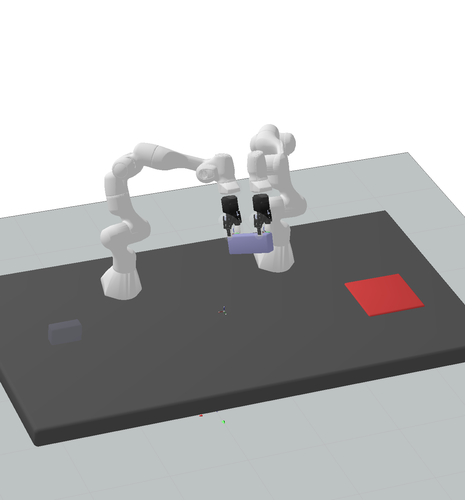}
		\includegraphics[width=.18\textwidth]{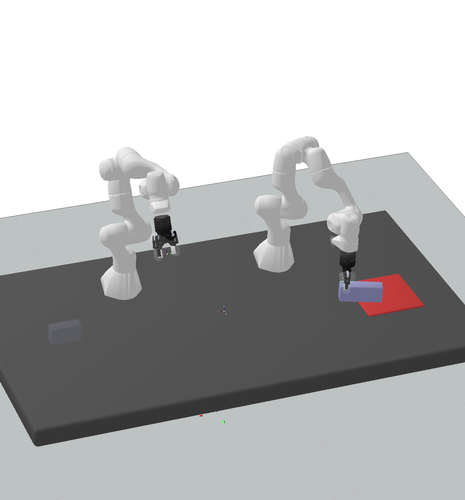}
		\caption{}
		\label{fig:meta:handover-sol}
	\end{subfigure}

	\begin{subfigure}[t]{.8\textwidth}
		\centering
		\includegraphics[width=.18\textwidth]{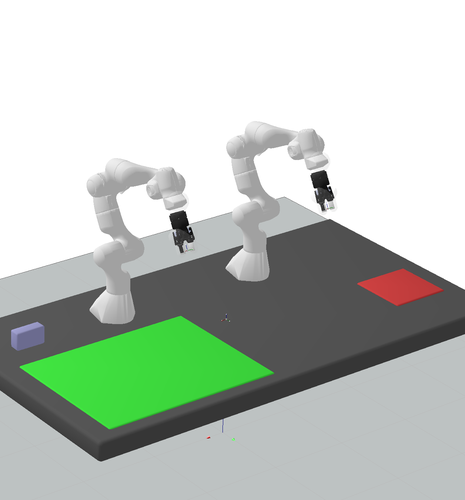}
		\includegraphics[width=.18\textwidth]{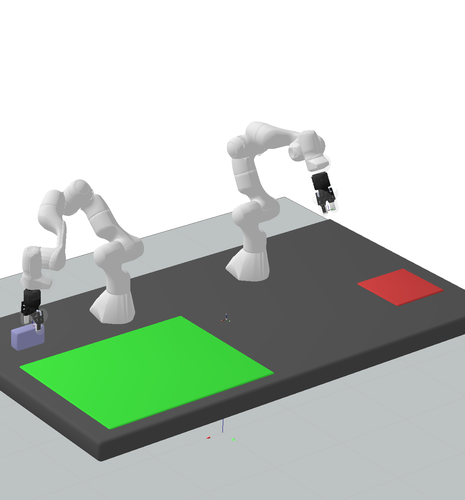}
		\includegraphics[width=.18\textwidth]{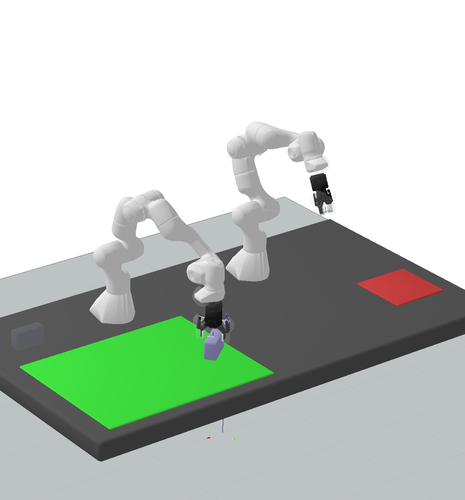}
		\includegraphics[width=.18\textwidth]{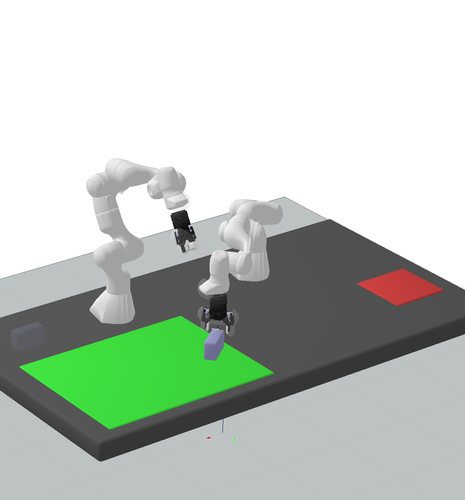}
		\includegraphics[width=.18\textwidth]{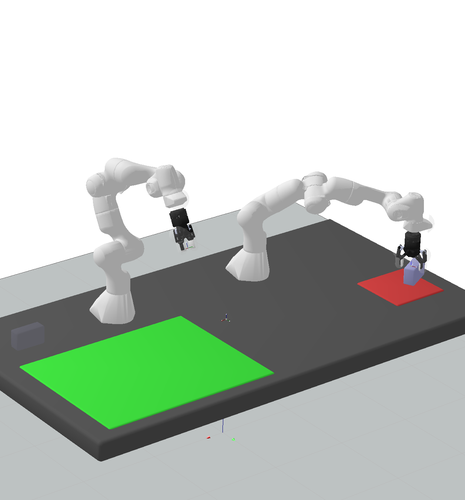}
		\caption{}
		\label{fig:meta:transfer-big-table-sol}
	\end{subfigure}
	\caption{Example solutions to the TAMP problems in \cref{fig:meta:tamp-problems} (Part 1).
	}
	\label{fig:meta:sol-all1}
\end{figure}

\begin{figure}[!t]
  \centering
	\begin{subfigure}[t]{.8\textwidth}
		\centering
		\includegraphics[width=.18\textwidth]{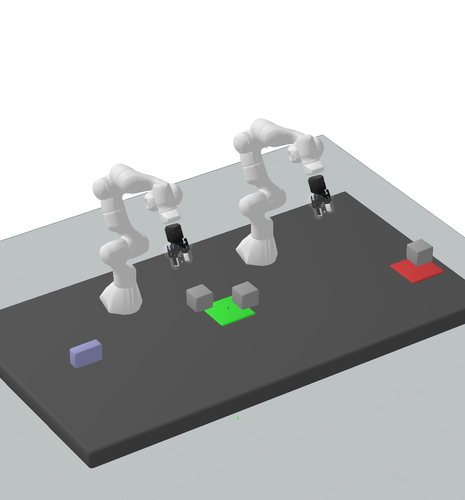}
		\includegraphics[width=.18\textwidth]{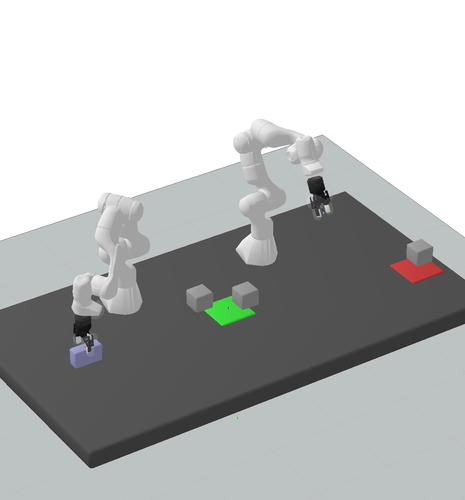}
		\includegraphics[width=.18\textwidth]{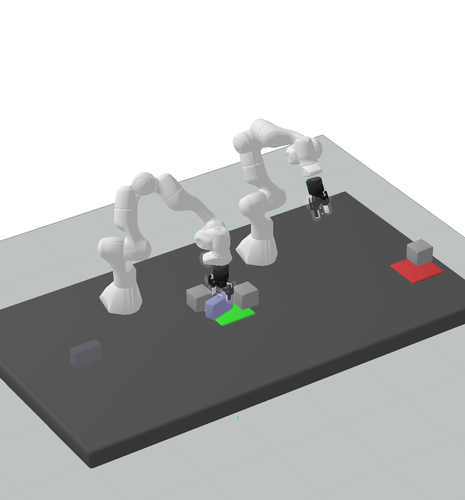}
		\includegraphics[width=.18\textwidth]{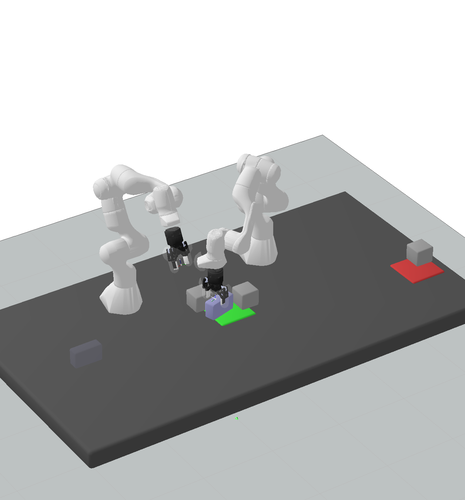}
		\includegraphics[width=.18\textwidth]{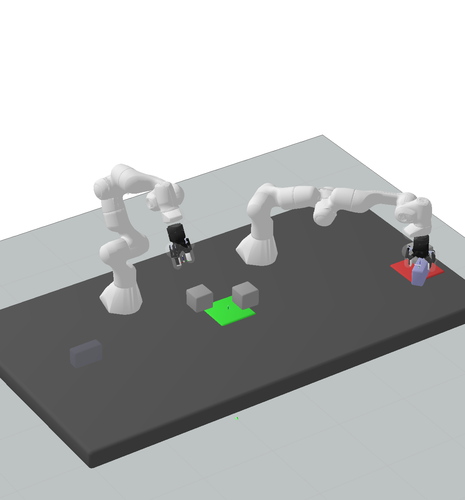}
		\caption{}
		\label{fig:meta:transfer-with-obs-sol}
	\end{subfigure}

	\begin{subfigure}[t]{.8\textwidth}
		\centering
		\includegraphics[width=.18\textwidth]{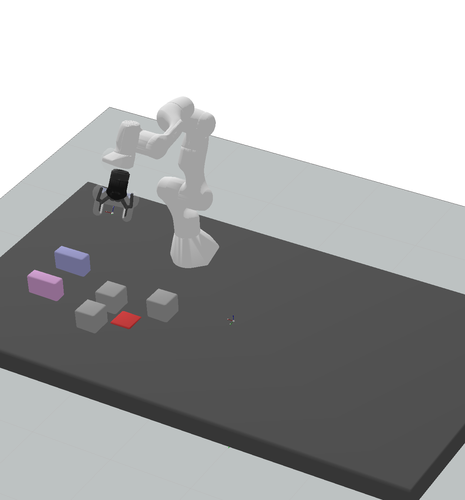}
		\includegraphics[width=.18\textwidth]{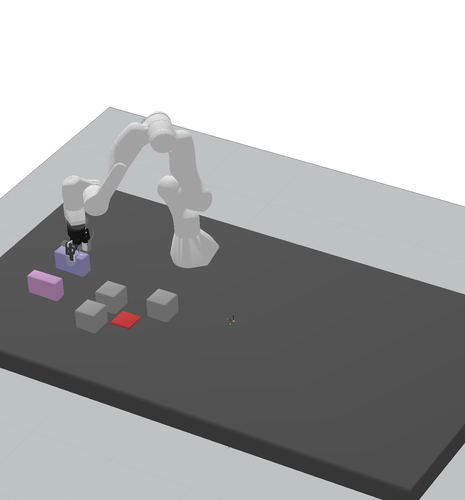}
		\includegraphics[width=.18\textwidth]{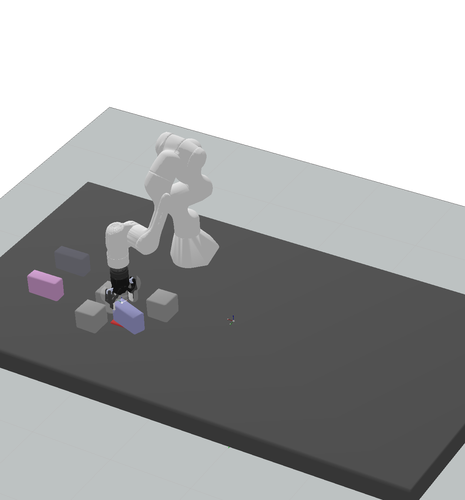}
		\includegraphics[width=.18\textwidth]{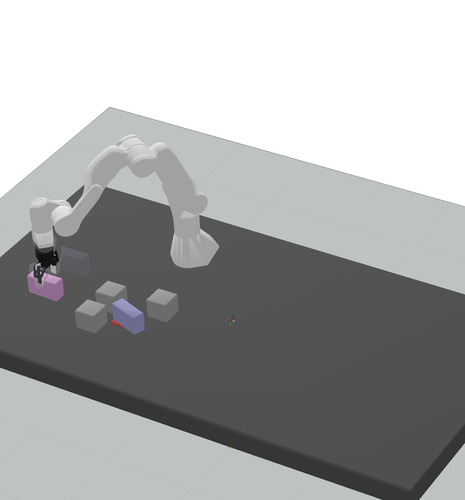}
		\includegraphics[width=.18\textwidth]{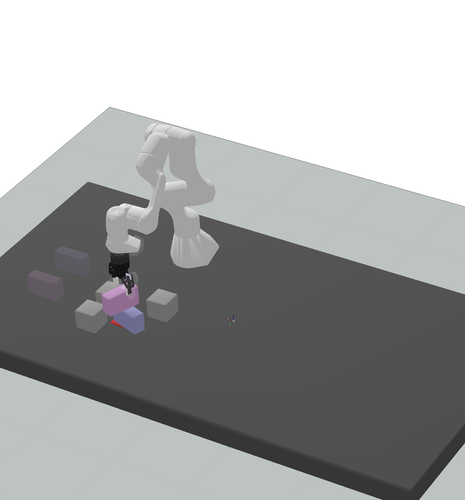}
		\caption{}
		\label{fig:meta:tower-sol}
	\end{subfigure}
	\begin{subfigure}[t]{.8\textwidth}
		\centering
		\includegraphics[width=.18\textwidth]{figs_old_computer/opti_computation/Pictures/r1_place_two_in_cluttered/crop/0000.png}
		\includegraphics[width=.18\textwidth]{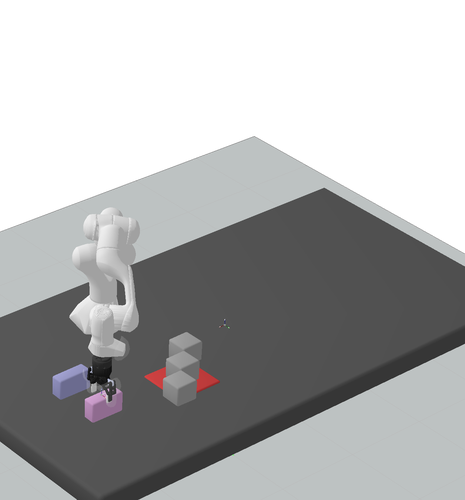}
		\includegraphics[width=.18\textwidth]{figs_old_computer/opti_computation/Pictures/r1_place_two_in_cluttered/crop/0003.png}
		\includegraphics[width=.18\textwidth]{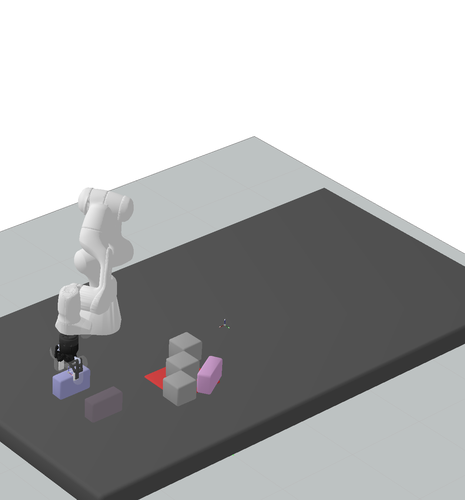}
		\includegraphics[width=.18\textwidth]{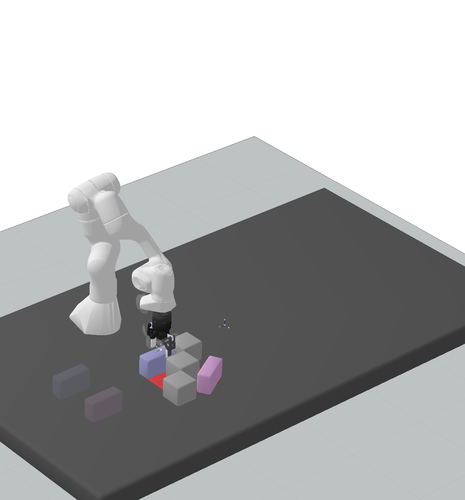}
		\caption{}
		\label{fig:meta:blocks-cluttered-table-sol}
	\end{subfigure}
	\begin{subfigure}[t]{.8\textwidth}
		\centering
		\includegraphics[width=.18\textwidth]{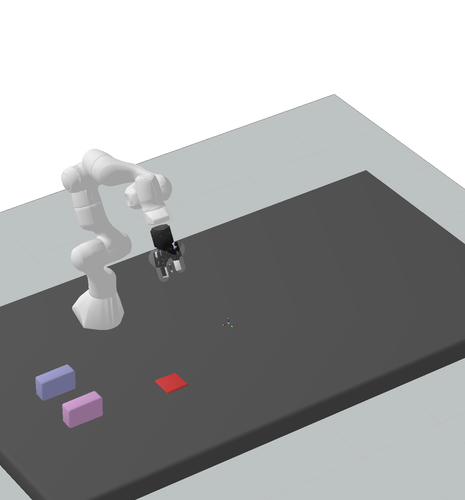}
		\includegraphics[width=.18\textwidth]{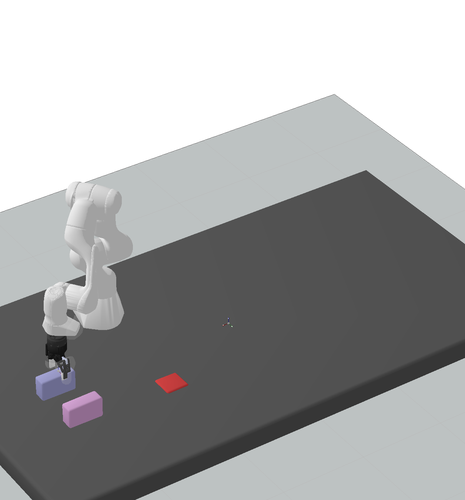}
		\includegraphics[width=.18\textwidth]{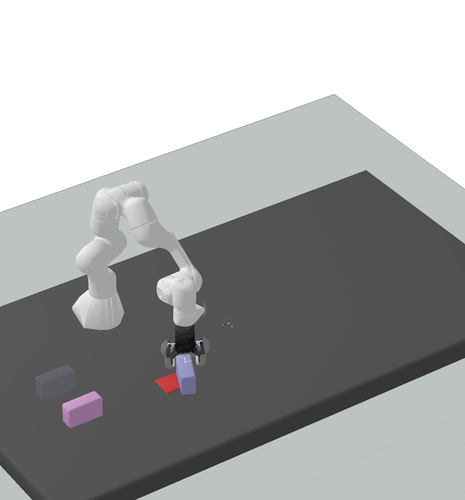}
		\includegraphics[width=.18\textwidth]{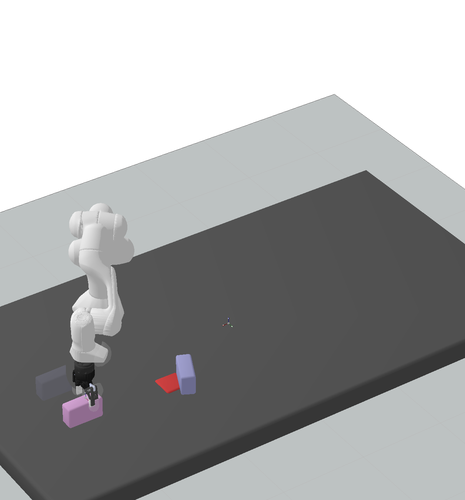}
		\includegraphics[width=.18\textwidth]{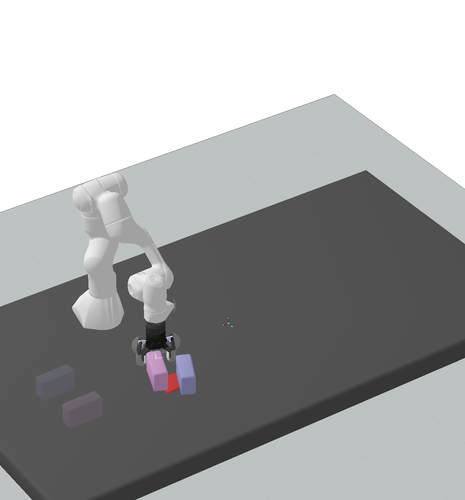}
		\caption{}
		\label{fig:meta:blocks-small-table-sol}
	\end{subfigure}
	\caption{Example solutions to the TAMP problems in \cref{fig:meta:tamp-problems} (Part 2).
	}
	\label{fig:meta:sol-all2}
\end{figure}

The goal of the experimental evaluation is twofold.
First, we aim to present a set of small and illustrative TAMP problems that expose the balance and trade-off between optimization and sampling.
Second, we demonstrate that our TAMP meta-solver bridges the gap between optimization and sample-based solvers for TAMP, outperforming them in terms of average computational time across different problems.
Importantly, we do not expect our TAMP solver to outperform state-of-the-art TAMP solvers in problems that involve interacting with a large number of objects.
Our TAMP meta-solver, implemented as a simple yet effective A*-like search on the tree, cannot yet scale to large settings, which require more effective ways to explore the space of possible computations.

\paragraph{Benchmark}
Our benchmark comprises 7 problems that involve at most two robots and two objects.
The maximum length of the task plan required to solve all problems is four.
The main challenges are the interdependencies between the different steps of the motion and the constraints imposed by the geometry and obstacles of the environment.

\begin{enumerate}
	\item \emph{Blocks on a small table} (in short, small) -- The task is to move the two blocks to the red table, which requires very precise placement to ensure that the two blocks fit on the table.
	      See the environment in \cref{fig:meta:blocks-small-table-env} and
	      the example solution in \cref{fig:meta:blocks-small-table-sol}.
	\item \emph{Blocks on a cluttered table} (cluttered) -- The task is to move the two blocks to the red table.
	      The table is now larger, but there are three obstacles in the middle.
	      See the environment in \cref{fig:meta:blocks-cluttered-table-env} and the example solution in \cref{fig:meta:blocks-cluttered-table-sol}.
	\item \emph{Handover of a small block} (handover) -- The goal is to move the object to the red table.
	      Due to reachability constraints, the optimal solution requires a handover, which demands very precise grasping to avoid collisions between the end-effectors.
	      See the environment in \cref{fig:meta:handover-env} and the example solution in \cref{fig:meta:handover-sol}.
	\item \emph{Transfer of a block} (transfer) -- The task is to move the block to the red table.
	      Now, the block is very small, so the robots cannot perform a handover and should use the large auxiliary table instead.
	      The table is large, and both robots can only reach a small subset of the table.
	      See the environment in \cref{fig:meta:transfer-big-table-env} and the example solution in \cref{fig:meta:transfer-big-table-sol}.
	\item \emph{Transfer of a block with obstacles} (transfer-obs) -- The task is to move the block to the red table.
	      Now, the table is cluttered with obstacles, and the auxiliary table is situated in the middle.
	      See the environment in \cref{fig:meta:transfer-with-obs-env} and the example solution in \cref{fig:meta:transfer-with-obs-sol}.
	\item \emph{Stacking blocks} (stack) -- The task is to build a tower on a table with obstacles.
	      See the environment in \cref{fig:meta:tower-env} and the example solution in \cref{fig:meta:tower-sol}.
	\item \emph{Block on an occupied table} (occupied) -- The task is to put a block on a table that is already filled with one block.
	      Therefore, the optimal solution requires moving the obstructing block first.
	      See the environment in \cref{fig:meta:move_first-env} and the example solution in \cref{fig:meta:move_first-sol}.
\end{enumerate}

Note that in the current implementation, we do not compute the continuous trajectory between keyframes; instead, we consider only the keyframe configurations as the continuous state.
In tabletop environments, once the sequence of keyframes is computed, the continuous trajectory can be efficiently computed using sample-based motion planning (e.g., RRT) and/or trajectory optimization.

\paragraph{Algorithms}
We evaluate the TAMP meta-solver (\cref{alg:code-meta}) against a representative optimization-based TAMP solver (\cref{alg:code-meta-opt}) and a sampling-based TAMP solver (\cref{alg:code-meta-sampling}), implemented as different search strategies in the TAMP Computation Tree.
The solvers will be denoted, respectively, by \emph{Meta}, \emph{Opt (Optimization)}, and \emph{Sampling}.

\paragraph{Evaluation metrics}

As a metric, we use the compute time to find a solution.
In particular, we consider only the compute time spent in the \texttt{Numeric\_Expansion} operations, while neglecting the time spent computing the PDDL heuristic and managing the search queue in the TAMP Computation Tree, which is usually one or two orders of magnitude smaller.

Experiments are run 20 times, and we report the first quartile, median, and third quartile.
We normalize the compute time by the median of the best-performing algorithm in each problem.
In each run, we use a different random seed, which influences the initialization of nonlinear optimization and sampling operations, often resulting in very disparate compute times because numeric expansions fail or compute values that are not valid later to reach the goal.

\subsection{Example Execution of the Three Algorithms}

We first show how \emph{Sampling}, \emph{Opt (Optimization)}, and \emph{Meta} explore the TAMP Computation Tree in different ways when solving different problems.

\cref{fig:meta:tree-small-table} shows the TAMP Computation Tree for the problem \textit{Blocks on a small table}, and
\cref{fig:meta:tree-big-table-obs} shows the TAMP Computation Tree for the problem \textit{Blocks on a cluttered table}, for the three algorithms.
The trees correspond to a reference run in our benchmark; because all algorithms are stochastic, the tree might vary between different runs.

\paragraph{Legend:}
We represent computation trees following the example in \cref{fig:compu_example}.
Gray squares represent compute states without free continuous states.
The green square is a fixed continuous state that reaches the goal, i.e., a solution to the TAMP problem.
A red square indicates that the numeric expansion has failed.
The top gray square is the initial computational state.
Circles are compute states with one or more free continuous states.
Blue circles are compute states that reach the discrete goal but contain free continuous states.
An arrow from a circle to a square indicates a numeric expansion.
The number of continuous states computed in a numeric expansion is equal to the number of circular nodes from the previous last fixed state (previous gray square).
An arrow from a square or circle node to a circle node indicates a discrete expansion.

\paragraph{Observations:}
The \textit{optimization} solver only performs numeric expansion on nodes that reach the discrete goal.
This is reflected in the tree because only the blue circles (goal candidates) undergo numerical expansion (i.e., have children that are squares).
The numeric expansion computes all the free states from the root and the circles jointly.

The \textit{sampling} solver performs numeric expansions on all nodes that have free states and does not allow for consecutive free states.
This is illustrated in the tree because a circle node is never a descendant of another circle node.

The \textit{meta} solver combines sampling and optimization, allowing for multiple free states and numeric expansion of intermediate states, resulting in an adaptive combination of circle and square nodes in the tree.

\paragraph{Performance in these problems:}
The two selected problems are representative of the performance of the three algorithms in the benchmark.
The initial state and the goal are the same, but the size of the table and the presence of three obstacles significantly impact performance.

\textit{Blocks on a small table} (\cref{fig:meta:blocks-small-table-env})
-- requires joint reasoning to place the blocks precisely on the small table without collisions.
The optimization solver only requires two numeric expansions because the placement of the two blocks is optimized jointly (the first expansion failed due to convergence to a poor local optimum).
Conversely, the sampling solver requires multiple sampling attempts because most of the partial solutions for the placement of the first block are unsuitable for the placement of the second block.
This is illustrated by the number of failed numerical expansions in the last expansion.
Eventually, with enough attempts, we sample continuous states that are also compatible with the subsequent steps to reach the goal.
The meta-solver initially attempts to generate a solution using sampling, but after a few failed attempts, the search algorithm automatically switches to optimization.
The solution is generated by computing the last three states jointly, conditioned on the pick configuration of the first block -- which allows for taking into account the tight constraints of the placement of the two blocks.

\textit{Blocks on a cluttered table} (\cref{fig:meta:blocks-cluttered-table-env}) -- does not require joint reasoning, and sampling is more effective because it avoids solving a larger optimization problem with multiple infeasible stationary points due to the three obstacles.
The optimization solver requires multiple attempts to find a joint solution to the candidate plans (several red squares).
Conversely, the sampling solver finds the solution with only two attempts, using less expensive computations because each step is computed individually.
Here, the meta-solver behaves similarly to sampling, as the search algorithm always favors easier sampling operations at the beginning.

\begin{figure}
	\centering
	\begin{subfigure}[b]{.9\textwidth}
		\centering
		\includegraphics[height=.2\textheight]{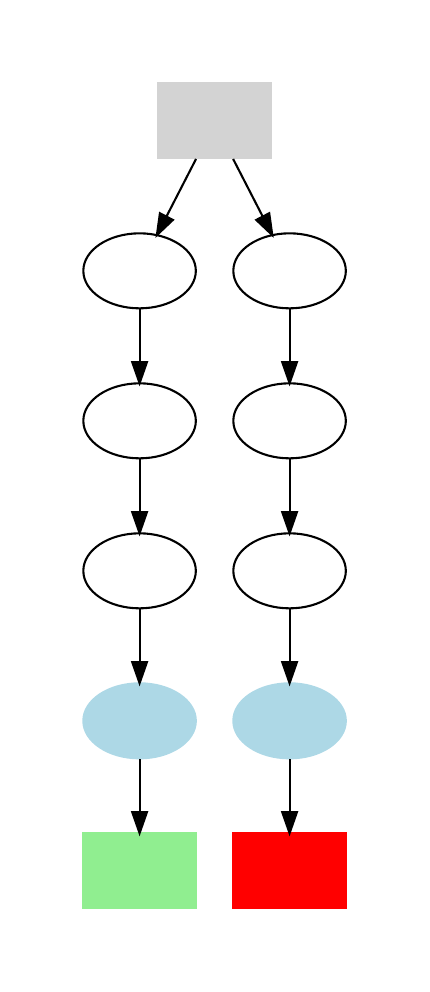}
		\caption{Optimization-based solver (\emph{Opt}).}
	\end{subfigure}\vspace{.5cm}

	\begin{subfigure}[b]{.99\textwidth}
		\centering
		\includegraphics[width=.99\textwidth]{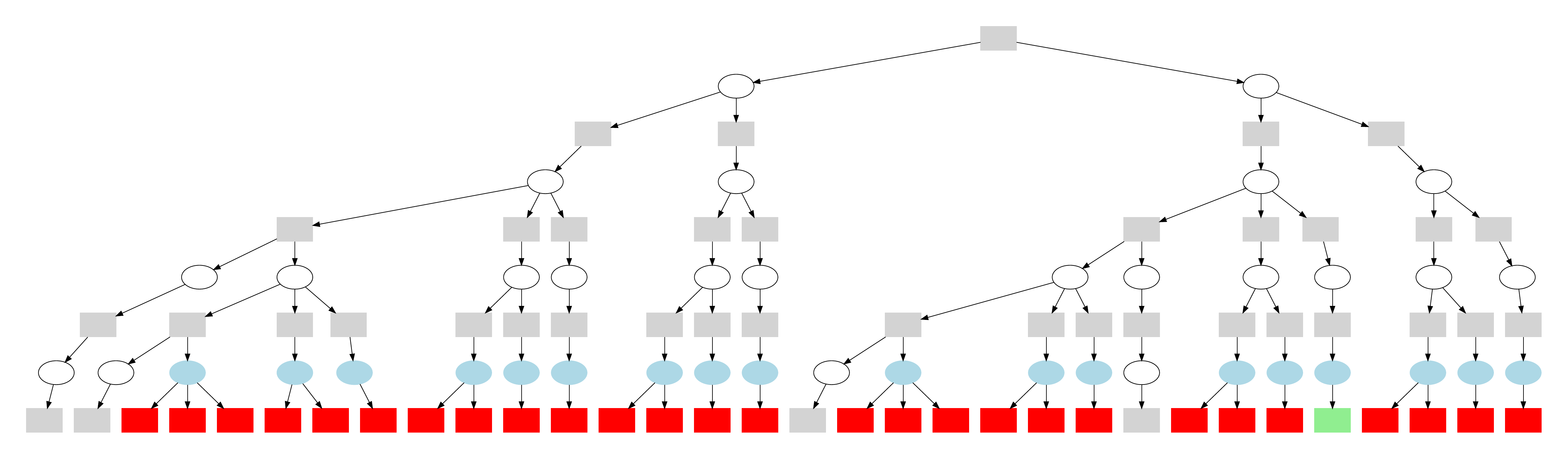}
		\caption{Sampling-based solver (\emph{Sampling}).}
		\label{fig:meta:computetree:small-table:sampling-all}
	\end{subfigure}\vspace{.5cm}

	\begin{subfigure}[b]{.9\textwidth}
		\centering
		\includegraphics[height=.3\textheight]{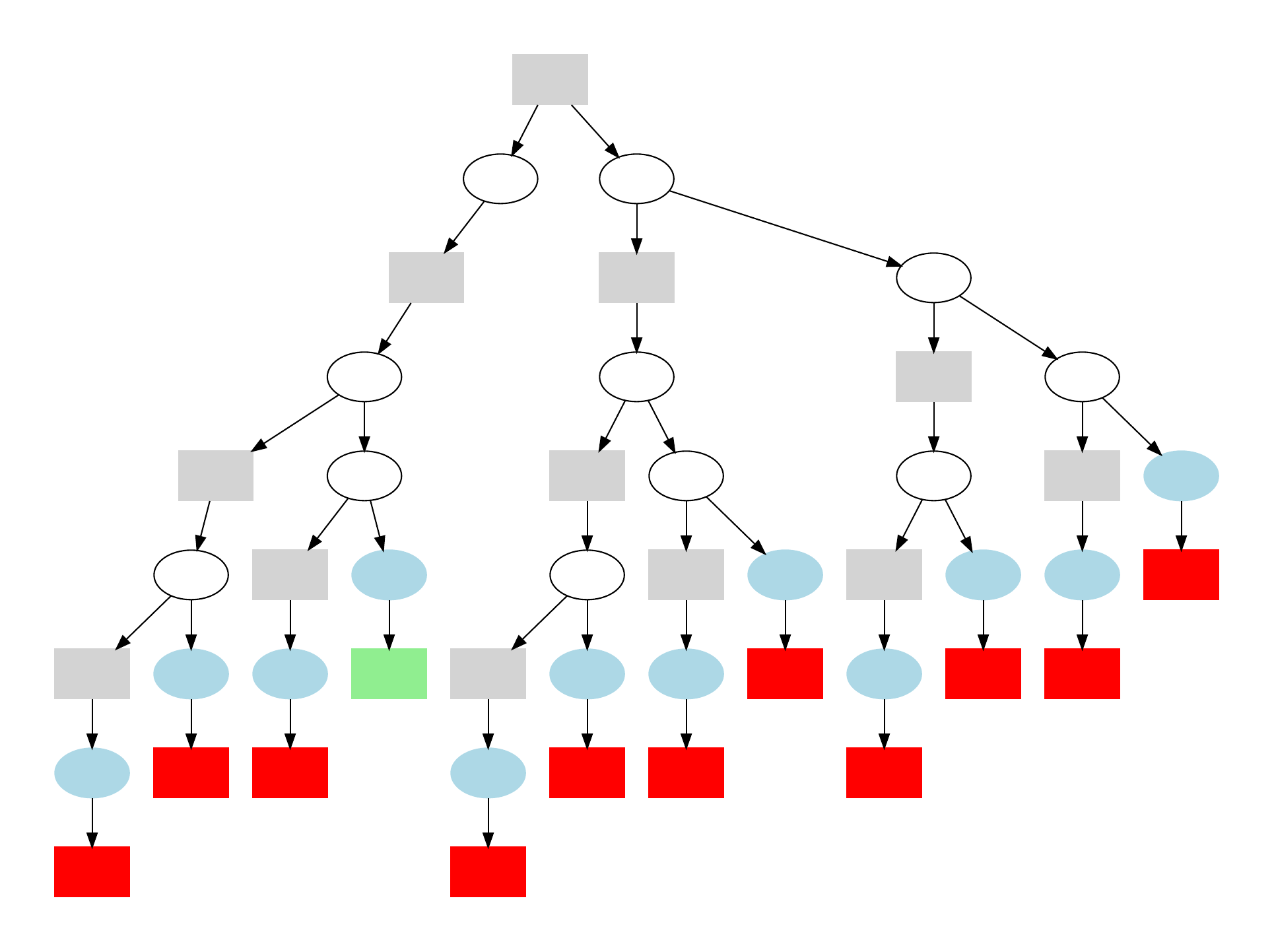}
		\caption{Meta-solver (\emph{Meta}).}
		\label{fig:meta:computetree:small-table:meta-all}
	\end{subfigure}
	\caption{Computation Tree in the problem \textit{Blocks on a small table}.}
	\label{fig:meta:tree-small-table}
\end{figure}

\begin{figure}
	\centering
	\begin{subfigure}[b]{.5\textwidth}
		\centering
		\includegraphics[height=.2\textheight]{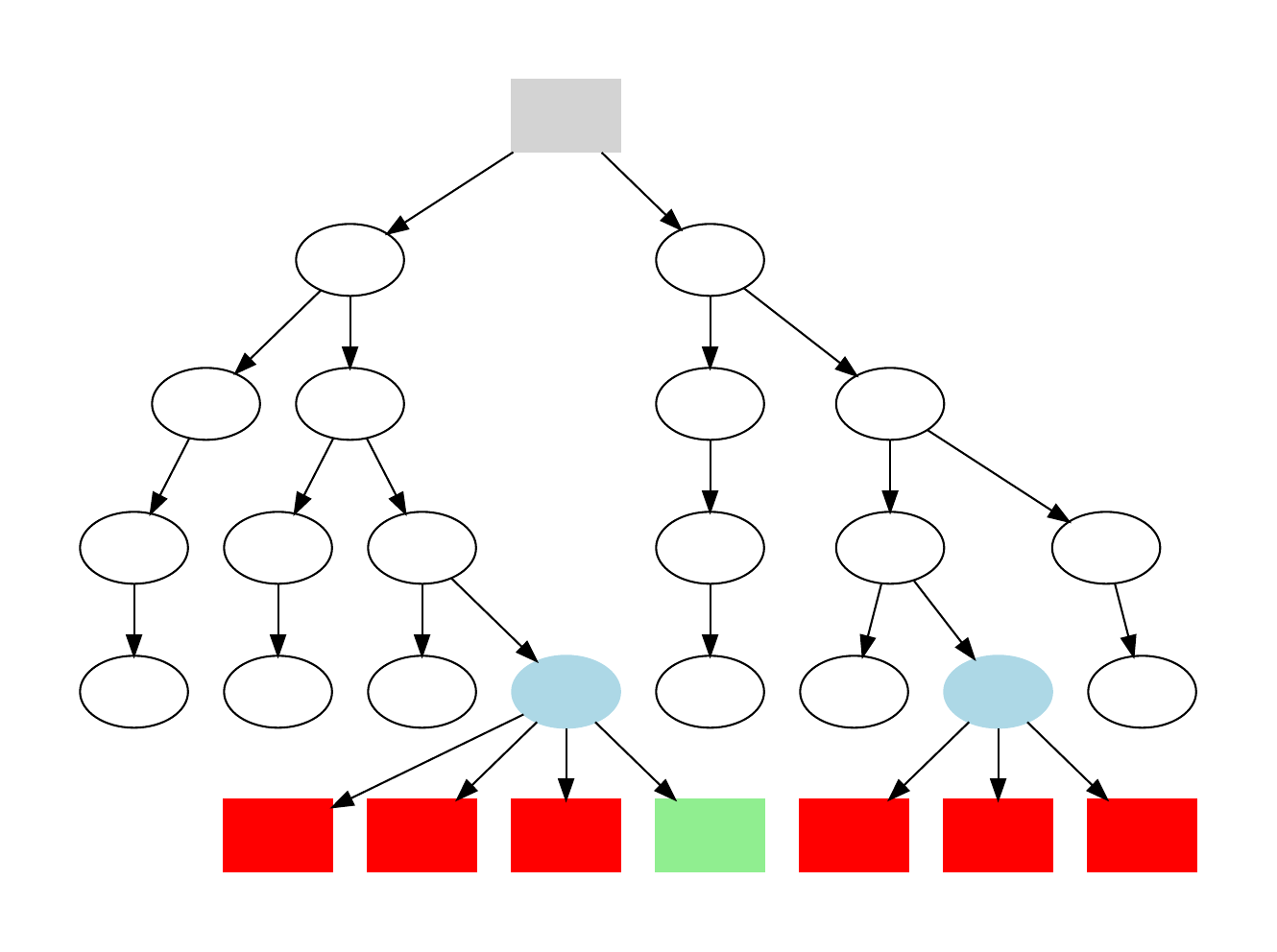}
		\caption{Optimization-based solver (\emph{Opt}).}
		\label{fig:meta:computetree:big-table-obs:opt}
	\end{subfigure}\vspace{.5cm}

	\begin{subfigure}[b]{.55\textwidth}
		\centering
		\includegraphics[height=.3\textheight]{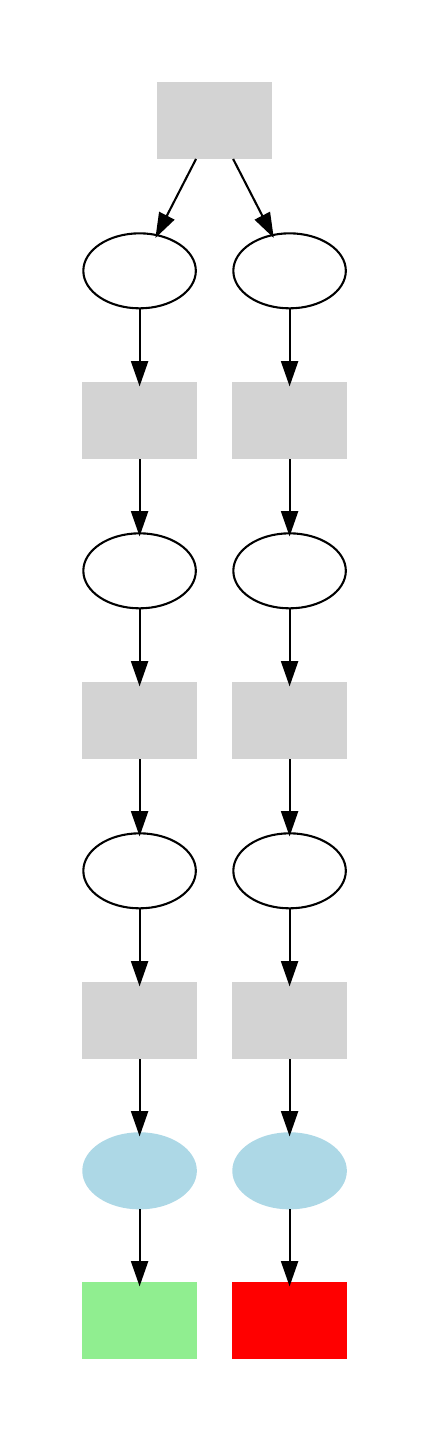}
		\caption{Sampling-based solver (\emph{Sampling}).}
		\label{fig:meta:computetree:big-table-obs:sampling}
	\end{subfigure}\vspace{.5cm}

	\begin{subfigure}[b]{.55\textwidth}
		\centering
		\includegraphics[height=.3\textheight]{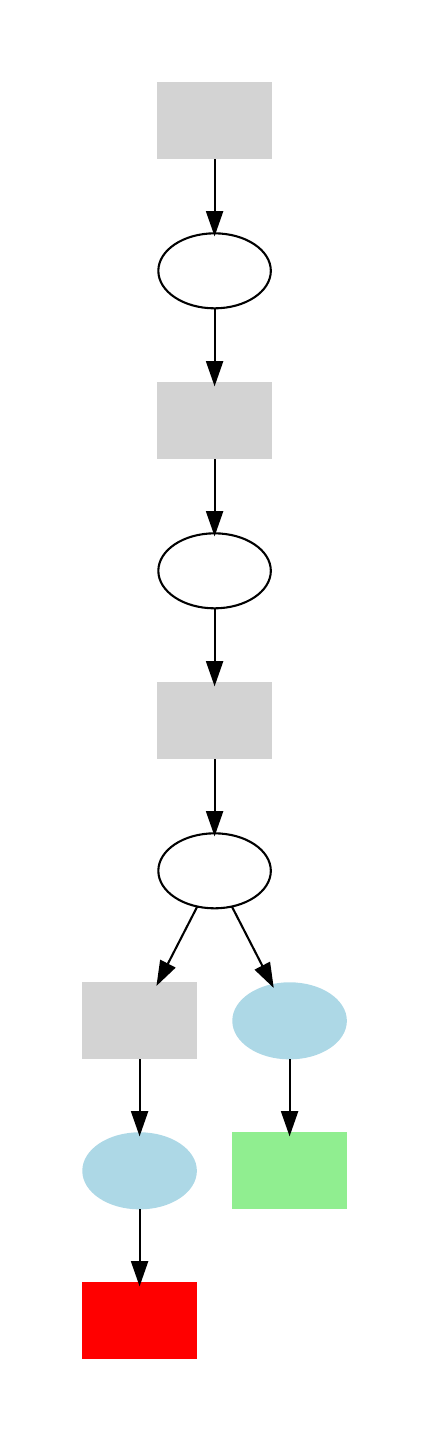}
		\caption{Meta-solver (\emph{Meta}).}
		\label{fig:meta:computetree:big-table-obs:meta}
	\end{subfigure}
	\caption{Computation Tree in the problem \textit{Blocks on a cluttered table}.}
	\label{fig:meta:tree-big-table-obs}
\end{figure}

\subsection{Comparison}

\begin{figure}[t]
	\centering
	\includegraphics[width=.8\textwidth]{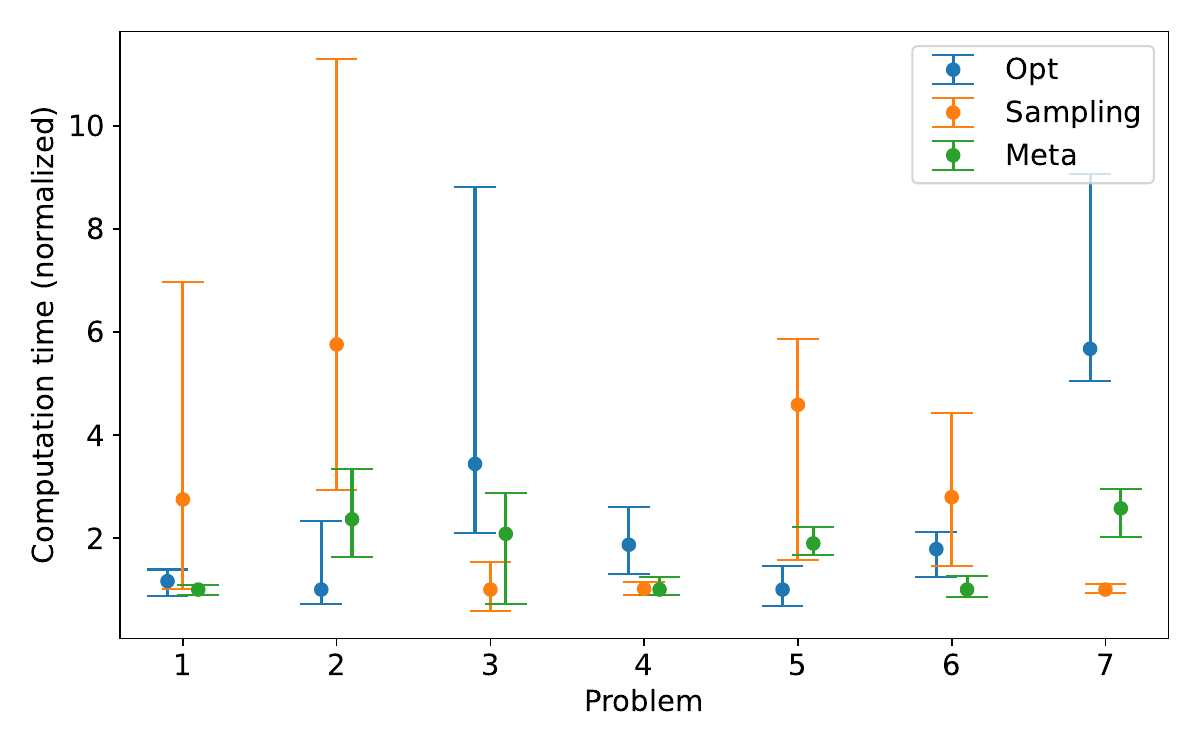}
  \caption{Computation time to find a solution in the \num{7} problems of the benchmark, normalized by the median of the best-performing algorithm in each problem.
    The dot is the median, and the brackets show the interval between the first and third quartile (i.e., the \num{25}th and \num{75}th percentiles).
		The numeric results and the problem labels are shown in \cref{tab:meta:results}.
	}
	\label{fig:meta:results}
\end{figure}

\begin{table}[t]
	\centering
	{
		\small
		\input{texs/tex_2023-09-22_19-52-10.tex}
	}
	\caption{
    Computation time to find a solution in the \num{7} problems of the benchmark, normalized by the median of the best-performing algorithm in each problem. $Q_1, Q_2$, and $Q_3$ are the first quartile, the median, and the third quartile, respectively.
		Colors red and green highlight the worst and best algorithms for each statistic ($Q_1, Q_2$, and $Q_3$) in each problem.
		The same results are shown graphically in \cref{fig:meta:results}.
	}
	\label{tab:meta:results}
\end{table}

Results are shown in \cref{fig:meta:results} and \cref{tab:meta:results}.
From the reported statistics, we consider \( Q_3 \) as the most important one, as it models the worst-case performance.
Minimizing the worst-case performance is critical for real applications, where the goal is to guarantee that an algorithm solves the problem quickly with a very high probability.

\textit{Is sampling or optimization better?}
When comparing only the sampling and optimization-based solvers, optimization performs better in 4 out of 7 problems.
\textit{Sampling} is better in problems that do not require reasoning about joint constraints and have obstacles (which have a more negative impact on the optimization-based solver): \textit{Blocks on a cluttered table}, \textit{Stacking blocks}, and \textit{Transfer of a block with obstacles}.
Optimization is better in problems that require reasoning about joint constraints but do not contain obstacles: \textit{Blocks on a small table}, \textit{Handover of a small block}, \textit{Transfer of a block}, and \textit{Block on an occupied table}.

\textit{Analysis of the performance of the meta-solver:}
If we analyze the results problem by problem, the meta-solver is never the worst in any problem, confirming that providing the flexibility to choose computational operations is required if we want to solve each problem as efficiently as possible.
This is indeed a very strong result, as it shows that our meta-solver can automatically balance between optimization and sampling and can choose the best strategy for each problem online, simply by trying different operations in a clever manner.
Furthermore, the meta-solver outperforms both alternatives in 3 out of 7 problems.

When we summarize the results across problems using the arithmetic mean of \( Q_1 \), \( Q_2 \), and \( Q_3 \), \textit{Meta} outperforms both \textit{Sampling} and \textit{Optimization}.
The most significant difference is in the third quartile, i.e., the worst-case performance, where \textit{Meta} is, on average, twice as fast as the baselines.

\subsection{Discussion of Scalability}

Preliminary tests on larger problems indicate that the \textit{Meta} algorithm scales less effectively than \textit{Opt} and \textit{Sampling} in TAMP problems with more objects.
The branching factor in \textit{Meta} is higher than in \textit{Opt} and \textit{Sampling} because it allows for a more flexible set of operations, and the number of compute states to expand grows more rapidly.
In contrast, \textit{Sampling} and \textit{Opt}, by design, restrict the space of possible computations, enabling them to evaluate more candidate task plans and resulting in shorter times to generate solutions in larger problems.

In our meta-solver, the current A*-like search neither shares information between nodes of different subbranches nor similar plans, nor does it utilize the history of computations, resulting in a suboptimal search strategy when the number of objects increases.

Overall, we believe that the TAMP Computation Tree is a promising framework for analyzing and designing TAMP solvers, but it requires a better scoring function for scalability.
It remains unclear whether achieving scalability will necessitate additional complex reasoning to select the next node for expansion, or if a simple scoring function will suffice.

\section{Limitations}

The framework presented here represents more a foundational study than a practical algorithm superior to existing TAMP solvers.
Indeed, our experimental study focuses primarily on TAMP problems requiring short-horizon planning.

A potential extension for scaling to larger problems involves incorporating the conflict extraction algorithm presented in \cref{ch:bid}.
Upon detecting a conflict, the PDDL domain used to estimate the distance to the goal can be updated using the same reformulation as in \cref{ch:bid}.
Ensuring that the detected conflict is genuine and not the result of an unlucky solver initialization, as previously discussed in the Limitations section of \cref{ch:bid,ch:diverse_planning}, is a significant challenge.
Beyond conflicts, sharing more information between nodes belonging to the same high-level task plan, even if they are not in the same branch, is a promising idea.
Also, sharing positive information between nodes, such as reusing the results of a numeric expansion in a different node, could be a potential direction for scaling the algorithm.

Moreover, we have limited the space of sampling operations to the level of a continuous state, instead of using a finer factorization in terms of single variables or backward-in-time conditional sampling, as in \cref{ch:mcts}.
Thus, our meta-solver is not yet exploiting the factored structure we have explored throughout this thesis.
As future work, we plan to investigate the benefits of allowing this fine discretization in the numeric expansions while balancing the inherent trade-off between implementation complexity and algorithmic performance in our TAMP meta-solver.

\section{Conclusion}

In this chapter, we have introduced the TAMP Computation Tree, an abstract TAMP framework that reasons at the level of computational states, instead of adhering to the traditional notion of discrete-continuous states commonly used in prior TAMP solvers.
A computational state encompasses both fixed and free variable states subject to nonlinear constraints, enabling the modeling of the behavior of both optimization-based and sample-based TAMP solvers.

Leveraging this innovative framework, we have proposed a TAMP meta-solver, where the term \textit{meta} implies that this solver can capitalize on both sampling and optimization techniques to tackle TAMP challenges, acting as a \textit{solver of solvers}.
The meta-solver is realized as a simple yet efficient heuristic search algorithm in the space of computational states, striking a natural balance between optimization and sampling-based computations.

Through a series of illustrative problems, we have demonstrated that the ideal TAMP solver varies with the specific problem, as no single numerical computation method excels across all scenarios.
On a benchmark limited to small problems,
our meta-solver adaptively selects the most suitable strategy for each problem, intelligently experimenting with various operations, and, on average, outperforms both optimization and sampling-based TAMP solvers across a broad spectrum of diverse problems.

%% file: texs/tex_2023-09-22_19-52-10.tex
\begin{tabular}{||l|l||c|c|c||c|c|c||c|c|c||c|c|c||c|c|c||c|c|c||c|c|c||c|c|c||c|c|c||}
\hline
\# & Problem & \multicolumn{3}{c||}{Opt} & \multicolumn{3}{c||}{Sampling} & \multicolumn{3}{c||}{Meta}\\
\hline
 &  & $Q_1$ & $Q_2$ & $Q_3$ & $Q_1$ & $Q_2$ & $Q_3$ & $Q_1$ & $Q_2$ & $Q_3$\\
\hline
1 & handover & \textcolor{ForestGreen}{0.9} & 1.2 & 1.4 & \textcolor{Red}{1.0} & \textcolor{Red}{2.7} & \textcolor{Red}{7.0} & 0.9 & \textcolor{ForestGreen}{1.0} & \textcolor{ForestGreen}{1.1}\\
\hline
2 & small & \textcolor{ForestGreen}{0.7} & \textcolor{ForestGreen}{1.0} & \textcolor{ForestGreen}{2.3} & \textcolor{Red}{2.9} & \textcolor{Red}{5.8} & \textcolor{Red}{11.3} & 1.6 & 2.4 & 3.3\\
\hline
3 & cluttered & \textcolor{Red}{2.1} & \textcolor{Red}{3.4} & \textcolor{Red}{8.8} & \textcolor{ForestGreen}{0.6} & \textcolor{ForestGreen}{1.0} & \textcolor{ForestGreen}{1.5} & 0.7 & 2.1 & 2.9\\
\hline
4 & stack & \textcolor{Red}{1.3} & \textcolor{Red}{1.9} & \textcolor{Red}{2.6} & 0.9 & 1.0 & \textcolor{ForestGreen}{1.1} & \textcolor{ForestGreen}{0.9} & \textcolor{ForestGreen}{1.0} & 1.2\\
\hline
5 & occupied & \textcolor{ForestGreen}{0.7} & \textcolor{ForestGreen}{1.0} & \textcolor{ForestGreen}{1.5} & 1.6 & \textcolor{Red}{4.6} & \textcolor{Red}{5.9} & \textcolor{Red}{1.7} & 1.9 & 2.2\\
\hline
6 & transfer & 1.2 & 1.8 & 2.1 & \textcolor{Red}{1.5} & \textcolor{Red}{2.8} & \textcolor{Red}{4.4} & \textcolor{ForestGreen}{0.8} & \textcolor{ForestGreen}{1.0} & \textcolor{ForestGreen}{1.3}\\
\hline
7 & transfer-obs & \textcolor{Red}{5.1} & \textcolor{Red}{5.7} & \textcolor{Red}{9.1} & \textcolor{ForestGreen}{0.9} & \textcolor{ForestGreen}{1.0} & \textcolor{ForestGreen}{1.1} & 2.0 & 2.6 & 2.9\\
\hline
\hline
 & Summary & 1.7 & 2.3 & 4.0 & 1.3 & 2.7 & 4.6 & 1.2 & 1.7 & 2.1\\
\hline
\end{tabular}

%% file: deep_gans.tex
\part{\namePartLearning}

\chapter{\nameChapterFive}
\label{ch:gans}

\section{Introduction}

Computing keyframe configurations of a manipulation sequence is a core problem in robotics and represents one of the fundamental steps in optimization-based methods for Task and Motion Planning (TAMP).
It involves sampling from a constraint manifold, which can be formulated as a nonlinear optimization problem without a cost term.
However, in cluttered environments with complex grasp models, these constraints become highly nonlinear, and local optimizers often get trapped in poor local optima, failing to find a feasible solution.

In \cref{ch:mcts}, we addressed this challenge with a meta-solver that automatically combines joint optimization and constrained sampling to find the best sequence of computations for generating solutions.
In a complementary line of research, this chapter\footnote{This chapter is based on the publication: Ortiz-Haro, J., Ha, J.
	S., Driess, D., \& Toussaint, M.
	(2022).
	Structured deep generative models for sampling on constraint manifolds in sequential manipulation.
	In the Conference on Robot Learning (pp.
	213-223).
	PMLR.
} presents a method to accelerate joint nonlinear optimization using a dataset of solutions from similar problems.

Our method, Deep Generative Constraint Sampling (DGCS), combines a deep generative model for sampling close to a constraint manifold with nonlinear constrained optimization to project onto the manifold.
The generative model, conditioned on the problem instance and taking a scene image as input, is trained with a dataset of solutions and a novel analytic constraint term.

An image-based representation of the problem instance (e.g., \cite{driess2020deep, xieimprovisation, ebert17self, paxton19visual}) provides a fixed-size parametrization that encodes obstacles and objects for interaction, eliminating the need to engineer features for the problem instance, and can accommodate a varying number of obstacle objects.

Generative Adversarial Networks (GANs) \cite{goodfellow2014generative,arjovsky2017wasserstein} and Variational Autoencoders (VAEs) \cite{kingma2013auto} have introduced a powerful methodology for training such generative models and have shown potential to represent complex distributions in high-dimensional spaces.
This work adopts the training objectives and methods of GANs, aiming to minimize the divergence between the generative and target distributions.
With a set of diverse solutions from similar problems as training data, the deep generative model is trained to produce samples close to the current problem's solution manifold.
These samples are then used as a warm start for the nonlinear optimizer, resulting in a combined data and model-based approach.

Despite the expressive power of function approximators and recent advancements in deep generative models, they still face limitations in generating samples from highly nonlinear and multimodal distributions.
This is critical in our application, as the solution manifolds of manipulation sequences are highly nonlinear and often disconnected.

To address multimodality, accuracy, and sample complexity, we propose an extension of our generative framework to leverage the structure of factored nonlinear programs in robotic sequential manipulation, as detailed in \cref{sec:bg:structure}.
Specifically, we decompose the sampling of a full solution into a sequence of smaller sampling operations and train a separate conditional generative model for each step.

We evaluate our approach on two problems of robotic sequential manipulation in cluttered environments.
Experimental results show that our deep generative model produces diverse and precise samples and outperforms heuristic warm start initialization.

\section{Related Work}

\paragraph{Generative models in robotics}
Recently, deep generative models have successfully been applied in robotics, especially in settings where problems are represented directly with images or point clouds.
For instance, 6DOF grasps of complex objects can be generated using a VAE \cite{mousavian2019dof,murali2020} conditioned on the object's point cloud.

In the context of motion and manipulation planning, generative models have been used to sample informative and collision-free configurations 
\cite{ichter2018learning,ha2018adaptive, 
sutanto2020learning, 
kurutach2018learning, 
kim2017guiding, 
simeonov2020long}.
These informed samples, as opposed to traditional uniform sampling, significantly improve the running times of sampling-based algorithms.
In this line of research, the general goal is to train a network to directly predict partial or full solutions to a problem.
In contrast, we combine learned generative models with model-based local optimization for constraint projection to achieve accurate sampling on high-dimensional nonlinear constraint manifolds.

\paragraph{Warm start in nonlinear optimization}
In robotics, nonlinear optimization is used to sample on constraint manifolds and to optimize trajectories, e.g., \cite{toussaint2018differentiable, mordatch2012discovery, winkler2018gait}.
Recent data-based approaches aim to learn a warm start to reduce online computation time \cite{mansard2018using,merkt2018leveraging}.
However, the mapping between problem instances and feasible manifolds is extremely nonlinear and discontinuous, presenting a fundamental challenge \cite{hauser2016learning,discontinuity2018tang}.

In settings where nonlinear programs can be represented with mixed-integer constraints, a successful approach is to learn an assignment for the integer constraints and conduct subsequent optimization
\cite{deits2019lvis,bertsimas2020online,cauligi2020learning,wells2019learning,driess2020deep}.
Compared to our method, integer formulations use a discriminative model that is easier to train, but their generalization to problems without a clear integer structure is challenging.
Recently, \cite{santoso2020Generative} applied GANs to inverse kinematics, incorporating forward kinematics into the network training.
In contrast, we use general analytical features in the cost term, and our framework can scale to longer sequential manipulation tasks by exploiting factorization.

\begin{figure}[t!]
	\centering
\setlength{\tabcolsep}{0.1cm} 
	\begin{tabular}{cccc}
    \includegraphics[trim={0 9cm 9cm 7cm},clip,width=0.23\textwidth]{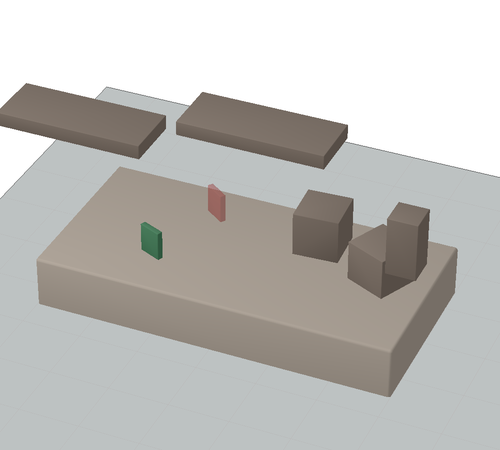}&
  \includegraphics[trim={0 9cm 9cm 7cm},clip,width=0.23\textwidth]{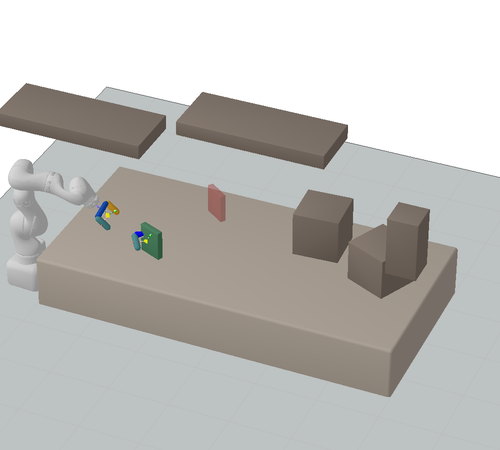}&
  \includegraphics[trim={0 9cm 9cm 7cm},clip,width=0.23\textwidth]{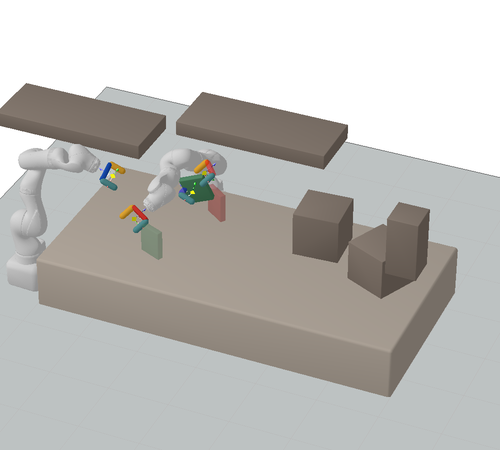}&
\includegraphics[trim={0 9cm 9cm 7cm},clip,width=0.23\textwidth]{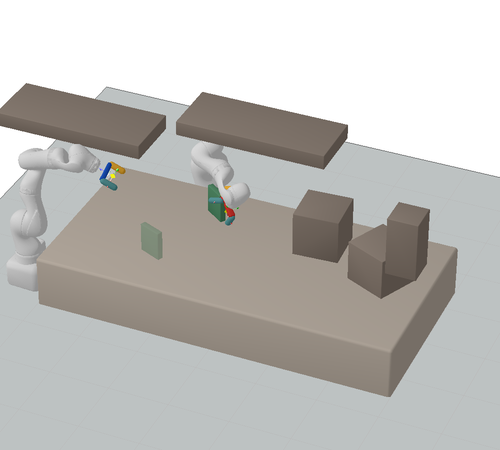} \\[.1cm]
		\hline            \\[.1cm]
	    \includegraphics[trim={0 7cm 9cm 7cm},clip,width=0.22\textwidth]{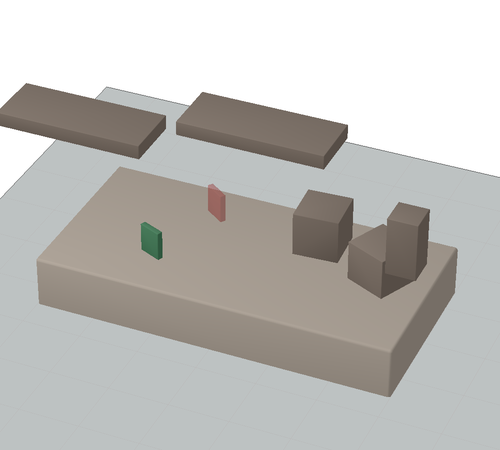}&
	  \includegraphics[trim={0 7cm 9cm 7cm},clip,width=0.22\textwidth]{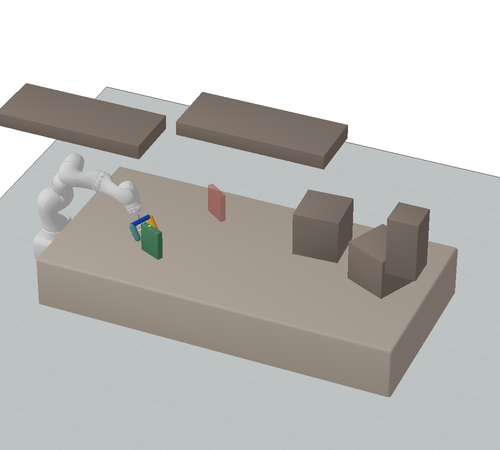}&
	  \includegraphics[trim={0 7cm 9cm 7cm},clip,width=0.22\textwidth]{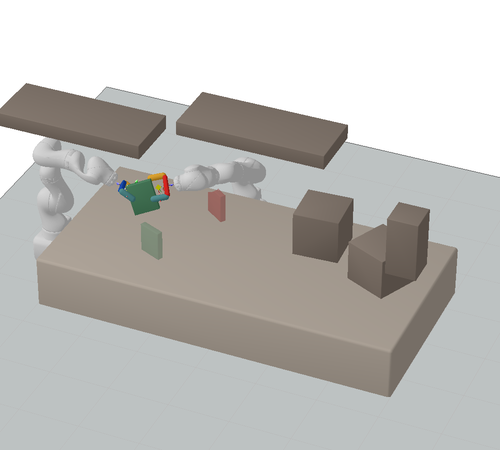}&
	\includegraphics[trim={0 7cm 9cm 7cm},clip,width=0.22\textwidth]{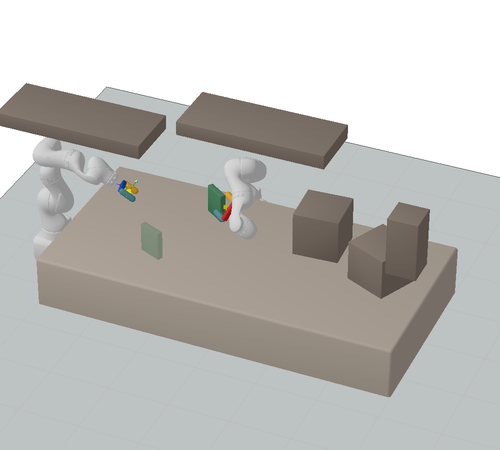} \\[.1cm]
	\end{tabular}
	\caption{
		Sequence of keyframes (pick - handover - place) for the \textit{Handover} problem.
		Our sampling framework (DGCS) combines a deep generative model that produces approximate samples (top row) conditioned on the scene (first column), with a nonlinear optimizer that projects them onto the constraint manifold (bottom row).
	}
	\label{fig:handover}
\end{figure}

\newpage

\section{Sampling on a Constraint Manifold}

The problem we address is to generate samples from a manifold \(\M_\tau\) parametrized by a fixed-dimensional (though potentially large, e.g., images) problem parameter \(\tau\in\mathbb{R}^m\),
\begin{equation}
	\M_{\tau} = \{ x \in \R^n \mid \feq(x; \tau)=0,~ \fineq(x;\tau) \le 0 \} ~,
\end{equation}
where \(\feq(x;\tau) : \R^n \times \R^m \to \R^{m_{\text{eq}}}\) and \(\fineq(x; \tau) : \R^n \times \R^m \to \R^{m_{\text{ineq}}}\) are the nonlinear equality and inequality constraints, respectively, and the parameter \(\tau \in \R^m\) represents the current problem instance and parametrizes all constraints.

In the context of TAMP, such manifolds arise once we have fixed the task plan (e.g., \cref{eq:lgp-nlp} in the LGP Formulation and \cref{eq:nlp} in the PNTC formulation).
The parameter \(\tau\) represents the properties of all objects in the scene, such as initial position, size, shape, and goal position.
The variable \(x\in \mathbb{R}^n\) concatenates all the degrees of freedom of the configurations of robots and objects in the entire manipulation sequence.
The (in)equality constraints \(\feq, \fineq\) describe the objectives of the problem, such as collision avoidance, grasping, kinematic and pose constraints, and are parametrized by the problem instance (e.g., the position of obstacles).

A generative model \(x \sim \PP(\tau)\) that produces feasible solutions, i.e., samples \(x^{(i)} \in \M_{\tau}\), is built from two components (see \cref{alg:gans:generate-solutions}): a randomized seed \(x_0\in \mathbb{R}^n\) generation and a constrained optimization algorithm that projects the seed \(x_0\) onto the solution manifold via the optimization problem \eqref{eq:projection},
\begin{equation}
	\label{eq:projection}
	\min_{x \in \mathbb{R}^n} \| x - x_0 \|^2 \quad \text{s.t.} \quad \feq(x;\tau) = 0 ,~ \fineq(x;\tau) \le 0.
\end{equation}
In our current implementation, we solve \eqref{eq:projection}
approximately by running a nonlinear optimizer from the starting point \(x_0\) on the feasibility problem:
\begin{equation}
	\text{find} \ x \in \mathbb{R}^n \quad \text{s.t.
	} \quad \feq(x;\tau) = 0 ,~ \fineq(x;\tau) \le 0,
\end{equation}
where the initialization and internal regularization of the optimizer (Augmented Lagrangian with Gauss-Newton) provide an implicit regularization with respect to \(x_0\).

Without using learning methods, the seed \(x_0\) is typically sampled from a uniform or Gaussian distribution over the ambient space \(\mathbb{R}^n\), resulting in an uninformed guess usually far from the solution manifold.

The projection step is a non-convex optimization problem with no guarantees of producing a feasible sample; especially for complex sequential manipulation problems, the optimization landscape often contains substantial infeasible local optima induced by the nonlinear constraints, making nonlinear projection prone to failure unless the seeds are close to the solution manifold.

\begin{center}
	\begin{minipage}{.8\linewidth}
		\begin{algorithm}[H]
			\caption{Sampling on a constraint manifold with deep generative models.}
			\label{alg:gans:generate-solutions}
			\begin{algorithmic}[1]
				\State \textbf{Input:}
				Problem parametrization \(\tau\), number of trials \(N\)
				\State \(L = \{\}\) \Comment{{\color{gray}{\small Empty list of samples}}}
				\For{\(i=1,2,\ldots,N\)}

        \State Sample \(x_0^{(i)} \sim \PP_\theta(\tau)\)
        \Comment{{\color{gray} \small Generate an approximate solution with deep generative models}}

				\State \(x^{(i)} \leftarrow \Pi(x_0^{(i)}) \quad\) 
        \Comment{{\color{gray} \small  Project \(x_0^{(i)}\) onto \(\M_{\tau}\) with nonlinear optimization, warm-started with \(x_0^{(i)}\) (\cref{eq:projection})}}
				\If {\(x^{(i)}\) is feasible }
				\State Append \(x^{(i)}\) to \(L\)
				\EndIf
				\EndFor
        \State \textbf{Output:} \(L\) \Comment{{\color{gray}  \small  {List of valid samples}}}
			\end{algorithmic}
		\end{algorithm}
		\vspace{.5cm}
	\end{minipage}
\end{center}

To address such difficulties, we train a seeding distribution \(\PP_{\theta}(\tau)\) to approximate a reference distribution of feasible solutions \(\PP_r(\tau)\), so that it can generate diverse samples close to the parametric manifold \(\M_{\tau}\).
These samples are then used as a warm start for the optimizer, projecting them onto \(\M_{\tau}\).
An example of our framework in the context of manipulation planning is shown in \cref{fig:handover}.

\section{Training Deep Generative Models to Sample on Constraint Manifolds}

Our deep generative model is denoted by \( x \sim \PP_{\theta}(\tau) \) with \( x = G_{\theta}(z, \tau), z \sim \PP_z \), where \( \PP_z \) is a multivariate Gaussian distribution, and \( G_{\theta} \) is a neural network parameterized by \( \theta \).

In contrast to the standard application of adversarial generative models for image generation, our setting also includes an analytical description of the target distribution's support, namely, the features \( \f(x;\tau) = [ \feq(x;\tau) , \max( 0 , \fineq(x;\tau) )] \) that characterize \( \M_\tau \) with \( \f(x;\tau) = 0 \).
We consider a data-free, gradient-based optimization of the analytical constraint violation, 
\begin{equation}
	\label{eq:data_free}
	\min \E_{\tau} \E_{ x \sim \PP_{\theta} } \norm{ \f( x ; \tau) }^2 ~.
\end{equation}
However, this approach is extremely ill-posed and can converge to a deterministic mapping \( G_{\theta}(z,\tau) = G_{\theta,\tau} \) for all \( z \), where the model disregards the noise \( z \) and loses the capacity to generate a diverse distribution.

To enforce sample diversity, we regularize with respect to a reference distribution \( \PP_r \) that is diverse and has its support on the manifold, satisfying \( \E_{ x \sim \PP_r } \norm{ \f( x ; \tau) }^2 = 0 \).
We formulate the problem as follows:
\begin{equation}
	\label{eq:distance+term}
	\min \E_{\tau} \ W( \PP_{\theta}(\tau), \PP_{r}(\tau)) + \beta \E_{ x \sim \PP_{\theta} } \norm{ \f( x ; \tau) }^2 ~,
\end{equation}
where \( W \) represents the Wasserstein distance, and \( \beta \) belongs to \( \R_{>0} \).
Although the analytical term does not provide additional information beyond the support of \( \PP_r \), its contribution is crucial in the context of function approximation and stochastic gradient descent in non-convex optimization, as demonstrated in the experiment section.

\subsection{Wasserstein Distance and Adversarial Formulation}

The Wasserstein-1 (Earth Mover's distance) between two probability distributions \(\PP_r\) and \(\PP_\theta\) as defined in \eqref{eq:wasser-def} is intuitively the cost of optimally transporting mass from one distribution to the other,
\begin{equation}
	\label{eq:wasser-def}
	W(\PP_{r}, \PP_{\theta}) = \inf _{\gamma \in \Pi(\PP_{r}, \PP_{\theta})} \E_{(x, y) \sim \gamma} \norm{x-y} ~,
\end{equation}
where \(\Pi(\PP_r,\PP_{\theta})\) denotes the set of all joint distributions with marginals \(\PP_r\) and \(\PP_{\theta}\).

Compared to other distance measures such as Jensen-Shannon Divergence or Total Variation, adversarial generative models employing Wasserstein distances have demonstrated superior practical stability and convergence, and are less susceptible to mode collapse \cite{arjovsky2017wasserstein,gulrajani2017improved}.

Furthermore, in our application, the Wasserstein distance provides a meaningful interpretation as it reflects the geometric distance between distributions—a critical factor in enhancing the success rate of subsequent nonlinear optimization.

We minimize the objective function in \eqref{eq:distance+term} using the Wasserstein GAN formulation \cite{arjovsky2017wasserstein,gulrajani2017improved}.
Employing Kantorovich duality, we transform the original definition in \eqref{eq:wasser-def} into a minimax game between the critic network \( D \) and the generator \( G \), training both concurrently via stochastic gradient descent.
Specifically, the minimax problem, incorporating our analytical feature, is:
\begin{equation}
	\label{eq:wgan-train}
	\min _{G} \max _{D} ~ \E_{\tau} ~ \E_{x \sim \PP_r }D(x;\tau)-\E_{x \sim \PP_{\theta}}D(x;\tau) - \lambda \E_{\hat{x} \sim \PP_{\hat{x}}}\left(\norm{\nabla D(\hat{x};\tau)}-1\right)^{2} +
	\beta \E_{x \sim \PP_{\theta}} \norm{ \f(x;\tau) }^2\,,
\end{equation}
where \( \beta, \lambda \) are positive real numbers, and \( \hat{x} \) denotes samples interpolated between \( \PP_r \) and \( \PP_{\theta} \).
Our reference distribution \( \PP_r(\tau) \) consists of a discrete set of solution/problem pairs \((x_i , \tau_i)\).
These data points are computed offline by resolving a large number of problems with nonlinear optimization or sequential constrained sampling, starting from randomized and uninformed initializations--a labor-intensive process that necessitates multiple attempts.

\section{Structured Generative Model by Exploiting Factorization}

\begin{figure}
	\centering
	\begin{subfigure}[b]{0.49\textwidth}
		\centering
		\begin{tikzpicture}[scale=0.8,every node/.style={transform shape}]
			\node[latent ] (t) {$t$} ; %
			\node[latent, right=of t ] (p2) {$p$} ; %
			\node[latent, below=of p2] (q2) {$q_2$} ; %
			\node[latent, left=of q2] (q1) {$q_1$} ; %
			\factor[above=of q1] {Kin1} {left:Kin} {t,q1} {};
			\factor[below=of p2] {Kin2} {left:Kin} {p2,t,q2} {};
			\factor[above=of t] {Grasp} {Grasp} {t} {};
			\factor[above=of p2] {Pose} {Pose} {p2} {};
			\factor[below=of q1,xshift=.5cm,color=brown] {col} {} {q1} {};
			\factor[right=of q2,color=brown] {col} {} {q2,p2} {};
		\end{tikzpicture} \\
		\caption{Factored-NLP.}
		\label{fig:graph-pp}
	\end{subfigure}
	\hfill
	\begin{subfigure}[b]{0.49\textwidth}
		\centering
		\begin{tikzpicture}[scale=0.8,every node/.style={transform shape}]
			\node[latent ] (t) {$t$} ; %
			\node[latent, right=of t ] (p2) {$p$} ; %
			\node[latent, below=of p2] (q2) {$q_2$} ; %
			\node[latent, left=of q2] (q1) {$q_1$} ; %
			\factoredge{}{p2}{t}
			\factoredge{}{t,p2}{q2}
			\factoredge{}{t}{q1}
		\end{tikzpicture} \\
		\caption{Sampling Network.}
		\label{fig:sample-pp}
	\end{subfigure}
	\caption{Factored-NLP and sampling network for the \textit{Pick and Place} problem.
		(a) Circles represent variables, and squares represent constraints (see the main text and \cref{sec:bg:structure} for an explanation).
		Brown squares indicate collision constraints.
		(b) Arrows indicate the factorization of the joint probability and the sequence of sampling operations.}
	\label{fig:graphs-robotics}
\end{figure}
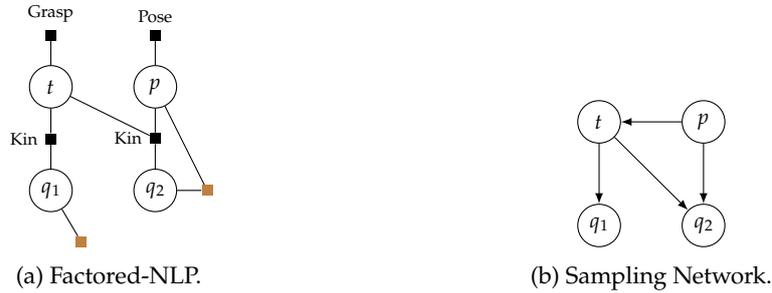

The previous sections treated \( x \in \mathbb{R}^n \) as a single vector variable.
While the approach for training a single generative model \( x = G_{\theta}(z;\tau), z \sim \PP_z \) is powerful, we can further improve scalability to large problems by exploiting a given factorization of \( x \) and sequentially decomposing the sampling procedure into a sequence of conditional sampling operations, as in Bayesian networks \cite{koller2009probabilistic}.

We now assume a Factored-NLP formulation of the nonlinear manifold (\cref{sec:bg:factored-nlp}).
The vector variable \( x \in \mathbb{R}^n \) is factored into a set \( X = \{x_1, \ldots, x_N\} \) of \( N \) smaller vector variables \( x_i \in \mathbb{R}^{n_i} \), and constraints are decomposed into a set of smaller nonlinear constraints \( \Phi = \{ \phi_1, \ldots, \phi_L \} \), each one depending only on a subset of the variables.

As illustrated in \cref{sec:bg:structure} for the case of TAMP, such factorization naturally arises in many applications, where each variable has some semantic and geometric meaning.

\cref{fig:graphs-robotics,fig:graphs-handover-assembly} show some examples of such Factored-NLPs in the context of robotic manipulation.
Note that the structure of these Factored-NLPs is slightly different from the ones shown in \cref{sec:bg:structure} and \cref{sec:planner:formulation}.
The key differences are: 1) here, all variables that are constrained to be fixed to some specific value are removed from the graph; 2) we use two distinct variables to represent object poses: \( t \) for the relative grasp and \( p \) for the absolute pose; and 3) consecutive variables that are constrained to be equal are merged into a single variable.

Both formulations of the Factored-NLPs are equivalent, and the differences serve only to expose a slightly different structure that is beneficial for decomposing the problem into a sequence of sampling operations.
However, in comparison to \cref{sec:planner:formulation}, we lose the clear temporal and local structures.

We quickly introduce the notation used in this chapter.
Variables for the configurations of robots \( Q \) and \( W \) are \( q,w \in \mathbb{R}^7 \) for the arm configuration, and \( r_q, r_w \in SE(2) \) for the pose of the mobile base.
Relative transformations between objects and grippers are \( t_q, t_w \in SE(3) \), and absolute object positions are \( p, \tilde{p} \in SE(3) \).
Nonlinear constraints (explained in detail in \cref{sec:bg:structure}), are kinematics (\textit{Kin}), grasp, position (\textit{Pose}), relative position (\textit{Relpose}), and collisions (in color brown).

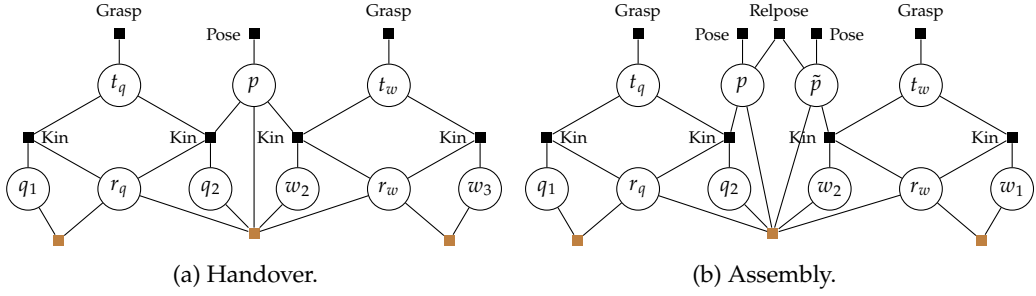
\begin{figure}
	\centering
	\begin{subfigure}{0.45\linewidth}
		\centering
		\begin{tikzpicture}[scale=0.8,every node/.style={transform shape}]
			\node[latent ] (p) {$p$} ; %
			\node[latent, left=1.5 of p , yshift=.0cm] (ta) {$t_q$} ; %
			\node[latent, right=1.5 of p, yshift=.0cm ] (tb) {$t_w$} ; %
			\node[latent, below=of ta, xshift=-1.5cm] (qa1) {$q_1$} ; %
			\node[latent, below=of ta] (ra) {$r_q$} ; %
			\node[latent, below=of ta, xshift=1.5cm] (qa2) {$q_2$} ; %
			\node[latent, below=of tb, xshift=-1.5cm] (qb2) {$w_2$} ; %
			\node[latent, below=of tb, xshift=1.5cm] (qb1) {$w_3$} ; %
			\node[latent, below=of tb]  (rb) {$r_w$} ; %
			\factor[above=of qa2] {Kina2} {left:Kin} {ta,qa2,p, ra} {};
			\factor[above=of qb2] {Kinb2} {left:Kin} {tb,qb2,p, rb} {};
			\factor[above=of qa1] {Kina1} {right:Kin} {ta,qa1,ra} {};
			\factor[above=of qb1] {Kinb1} {left:Kin} {tb,qb1,rb} {};
			\factor[above=of p] {Pose} {left:Pose} {p} {};
			\factor[above=of ta] {Grasp} {Grasp} {ta} {};
			\factor[above=of tb] {Grasp} {Grasp} {tb} {};

			\factor[below=of qa1,xshift=.5cm,color=brown] {col} {} {qa1,ra} {};

			\factor[below=of qb1,xshift=-.5cm,color=brown] {col} {} {qb1,rb} {};

			\factor[below= 2 of p, color=brown] {col} {} {qb2,qa2,p,ra,rb} {};

		\end{tikzpicture}
		\caption{Handover.}
		\label{fig:hand-constraint-graph}
	\end{subfigure}
	\hspace{10pt}
	\begin{subfigure}{0.45\linewidth}
		\centering
		\begin{tikzpicture}[scale=0.8,every node/.style={transform shape}]
			\node[latent ] (p) {$p$} ; %
			\node[latent , right=.5 of p] (p2) {$\tilde{p}$} ; %
			\node[latent, left=1 of p , yshift=.0cm] (ta) {$t_q$} ; %
			\node[latent, right=1 of p2, yshift=.0cm ] (tb) {$t_w$} ; %
			\node[latent, below=of ta, xshift=-1.5cm] (qa1) {$q_1$} ; %
			\node[latent, below=of ta] (ra) {$r_q$} ; %
			\node[latent, below=of ta, xshift=1.5cm] (qa2) {$q_2$} ; %
			\node[latent, below=of tb, xshift=-1.5cm] (qb2) {$w_2$} ; %
			\node[latent, below=of tb, xshift=1.5cm] (qb1) {$w_1$} ; %
			\node[latent, below=of tb]  (rb) {$r_w$} ; %
			\factor[above=of qa2] {Kina2} {left:Kin} {ta,qa2,p, ra} {};
			\factor[above=of qb2] {Kinb2} {left:Kin} {tb,qb2,p2, rb} {};
			\factor[above=of qa1] {Kina1} {right:Kin} {ta,qa1,ra} {};
			\factor[above=of qb1] {Kinb1} {left:Kin} {tb,qb1,rb} {};
			\factor[above=of p] {Pose} {left:Pose} {p} {};
			\factor[above=of p2] {Pose} {right:Pose} {p2} {};
			\factor[above=of p, xshift=.6cm] {Pose} {Relpose} {p,p2} {};
			\factor[above=of ta] {Grasp} {Grasp} {ta} {};
			\factor[above=of tb] {Grasp} {Grasp} {tb} {};

			\factor[below=of qa1,xshift=.5cm,color=brown] {col} {} {qa1,ra} {};

			\factor[below=of qb1,xshift=-.5cm,color=brown] {col} {} {qb1,rb} {};
			\factor[below=2 of p,xshift=.5cm,color=brown] {col} {} {p,p2,ra,rb,qa2,qb2} {};

		\end{tikzpicture}
		\caption{Assembly.}
	\end{subfigure}
	\caption{Factored-NLPs for the \textit{Handover} and \textit{Assembly} problems.\vspace{.5cm}}
	\label{fig:graphs-handover-assembly}
\end{figure}

\subsection{Directed Graphical Model and Sequential Sampling}

\begin{figure}
	\centering
	\begin{subfigure}[t]{0.49\textwidth}
		\centering
		\begin{tikzpicture}[scale=.8,every node/.style={transform shape}]
			\node[latent] (p) {$p$} ; %
			\node[latent, left=of p] (ra) {$r_q$} ; %
			\node[latent, right=of p] (rb) {$r_w$} ; %
			\node[latent, below=.5 of p, xshift=-1cm] (ta) {$t_q$} ; %
			\node[latent, below=.5 of p, xshift=1cm] (tb) {$t_w$} ; %
			\node[latent, below=2 of p,  xshift=1cm] (qb2) {$w_2$} ; %
			\node[latent, below=2 of p,  xshift=2cm] (qb1) {$w_3$} ; %
			\node[latent, below=2 of p,  xshift=-2cm] (qa1) {$q_1$} ; %
			\node[latent, below=2 of p,  xshift=-1cm] (qa2) {$q_2$} ; %
			\factoredge{}{p}{ra}
			\factoredge{}{p}{rb}
			\factoredge{}{ra}{ta}
			\factoredge{}{rb}{tb}
			\factoredge{}{p}{ta}
			\factoredge{}{ta, p}{tb}
			\factoredge{}{ta, p, ra}{qa2}
			\factoredge{}{ta, ra}{qa1}
			\factoredge{}{tb, rb}{qb1}
			\factoredge{}{tb, p, rb}{qb2}
		\end{tikzpicture}
		\caption{Handover.}
	\end{subfigure}
	\begin{subfigure}[t]{0.49\textwidth}
		\centering
		\begin{tikzpicture}[scale=.8,every node/.style={transform shape}]
			\node[latent] (p) {$p$} ; %
			\node[latent, right=of p] (p2) {$\tilde{p}$} ; %
			\node[latent, left=of p] (ra) {$r_q$} ; %
			\node[latent, right=of p2] (rb) {$r_w$} ; %
			\node[latent, below=.5 of p, xshift=-1cm] (ta) {$t_q$} ; %
			\node[latent, below=.5 of p2, xshift=1cm] (tb) {$t_w$} ; %
			\node[latent, below=2 of p2,  xshift=1cm] (qb2) {$w_2$} ; %
			\node[latent, below=2 of p2,  xshift=2cm] (qb1) {$w_1$} ; %
			\node[latent, below=2 of p,  xshift=-2cm] (qa1) {$q_1$} ; %
			\node[latent, below=2 of p,  xshift=-1cm] (qa2) {$q_2$} ; %
			\factoredge{}{p}{ra}
			\factoredge{}{p}{p2}
			\factoredge{}{p2}{rb}
			\factoredge{}{ra}{ta}
			\factoredge{}{rb, p2}{tb}
			\factoredge{}{p}{ta}
			\factoredge{}{ta, p, ra}{qa2}
			\factoredge{}{ta, ra}{qa1}
			\factoredge{}{tb, rb}{qb1}
			\factoredge{}{tb, p2, rb}{qb2}
		\end{tikzpicture}
		\caption{Assembly.}
	\end{subfigure}
	\caption{Sampling networks for the \textit{Assembly} and \textit{Handover} problems.}
	\label{fig:sampling-hand-asse}
\end{figure}
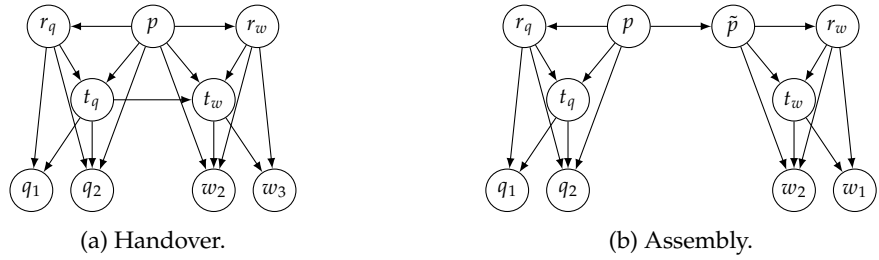

The factored structure directly suggests a factorization of the joint probability distribution $\PP(\tau)$ from which we want to sample.
As in Bayesian networks, we can sequence the sampling if we decide on an ordering of variables that corresponds to a directed acyclic graph.
Instead of learning a single $G_{\theta}$ to produce a full assignment with $x = G_{\theta}(z;\tau), ~ z \sim \PP_z$, we learn a \emph{conditional} model for each factor using the corresponding marginal distributions of the original data and the subset of corresponding analytical features in the Factored-NLP.

\newpage
We illustrate the benefits of this factorization with the \textit{Pick and Place} problem, as shown in \cref{fig:graphs-robotics}.
The joint probability density function \( P(p,t,q_1,q_2)  \)
for this problem is factorized into (we omit conditioning on $\tau$ for clarity):
\begin{equation}
	P(p,t,q_1,q_2) = P(p) \, P(t|p) \, P(q_1|t) \, P(q_2|t, p) ~,
\end{equation}
where $p$ is the placement position of the object, $t$ is the relative transformation between the object and the gripper, and $q_1$, $q_2$ are the robot joint configurations at the pick and place keyframes.
The factorization exploits conditional independence between $(q_1,p)$ given $t$ and $(q_1,q_2)$ given $t$.
We leverage this structure by training a sequence of conditional samplers: $p \sim \PP_p$, $t \sim \PP_t(p)$, $q_1 \sim \PP_{q_1}(t)$, and $q_2 \sim \PP_{q_2}(t,p)$, as illustrated in \cref{fig:sample-pp,fig:sequence}.
This factorization can be easily extended to longer manipulation sequences, with the ordering
$p \to r \to t \to q$ (object pose, robot base, grasp, robot joint values), as shown in \cref{fig:sampling-hand-asse} for the \textit{Handover} and \textit{Assembly} problems.

As a side note, using the marginal distributions of a jointly consistent dataset is necessary, as only the marginals contain useful information about whether a partial assignment will admit a full solution.
For example, when sampling \( p \sim \mathbb{P}_p \), we aim to fulfill local constraints (e.g., \textit{Pose}) and ensure that the value will be consistent with the future assignment of other variables.
The second type of constraints cannot be evaluated efficiently but is modeled by the marginal distribution of the data.

\subsection{The Advantage of Factorization for Modeling Multimodality}

Generating samples from a distribution with disconnected supports using a deep generative model that receives continuous input noise \( z \sim \mathbb{P}_z \) requires infinite gradients in \( G_{\theta} \) and can only be done approximately.
In these cases, training is unstable and very sensitive to hyperparameters and architecture.

We can model disconnected distributions more effectively by factoring the full joint probability into a sequence of smaller, conditional modules, as confirmed by our experimental results in \cref{sec:ablation}.
The sequencing still requires that each module has the ability to produce some degree of multimodality.
However, once a module in the sequence receives disconnected input in the form of conditioning, it can successfully produce a disconnected output.
As we chain modules with the ability to generate a small amount of disconnected components given continuous input, the number of possible disconnected components of the output grows exponentially with the number of modules in sequence.

Furthermore, from a practical perspective, training smaller modules has proved to be more effective.
For instance, we observed that the analytical feature \( \|\phi(x;\tau)\|^2 \) of the joint problem can provide badly conditioned gradients when evaluated far from the manifold.
This issue is alleviated when considering only subsets of constraints and variables.
In our preliminary experiments, we also evaluated GAN frameworks that have a mechanism to model disconnected distributions \cite{chen2016chen,khayatkhoei2018disconnected}, but we did not find significant improvements.

\begin{figure}[t!]
	\setlength{\tabcolsep}{1pt}
	\begin{tabular}{ccccc}
		1                                                                                                                           & 2                                                                             & 3 & 4 & 5 \\
		\includegraphics[trim={3cm 0  3cm 0},clip, width=0.19\textwidth]{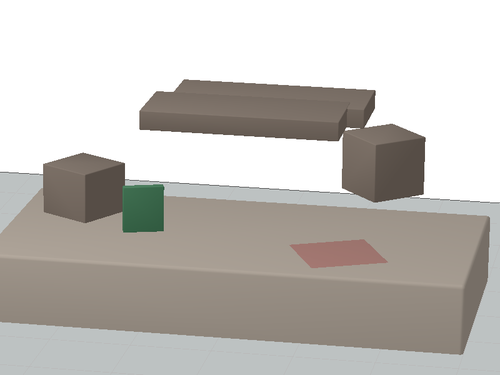} &
		\includegraphics[trim={3cm 0  3cm 0},clip,width=0.19\textwidth]{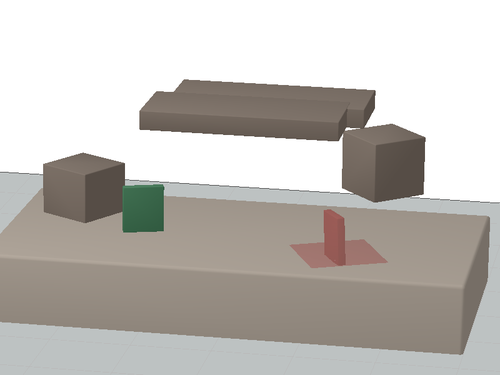}  &
		\includegraphics[trim={3cm 0  3cm 0},clip,width=0.19\textwidth]{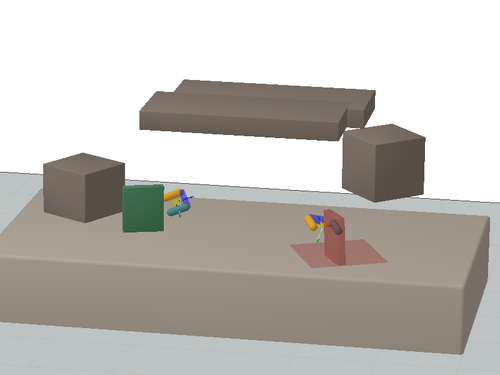}  &
		\includegraphics[trim={3cm 0  3cm 0},clip,width=0.19\textwidth]{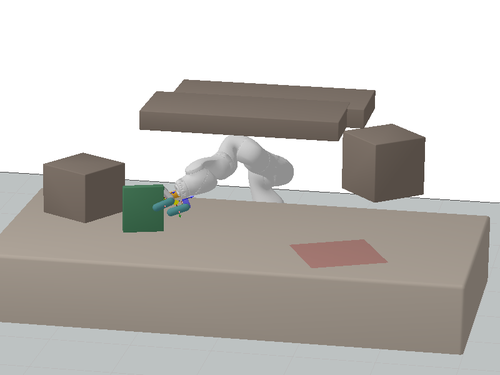}  &
		\includegraphics[trim={3cm 0  3cm 0},clip,width=0.19\textwidth]{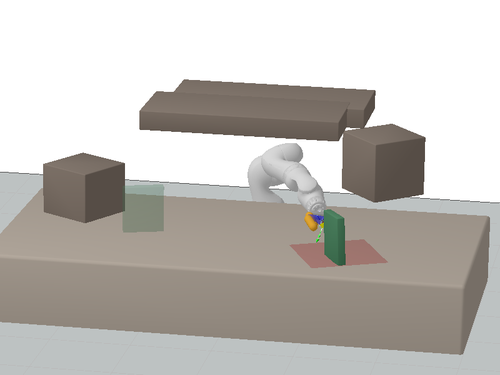}                                                                                              \\
		\includegraphics[width=0.12\textwidth]{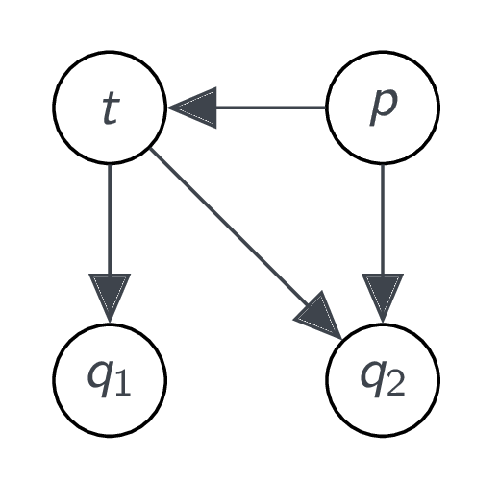}
		                                                                                                                            & \includegraphics[width=0.12\textwidth]{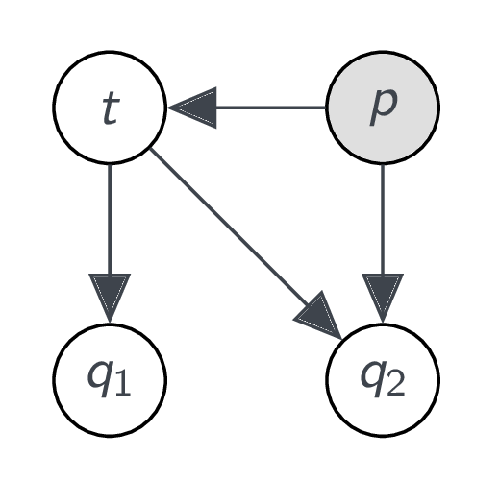}
		                                                                                                                            & \includegraphics[width=0.12\textwidth]{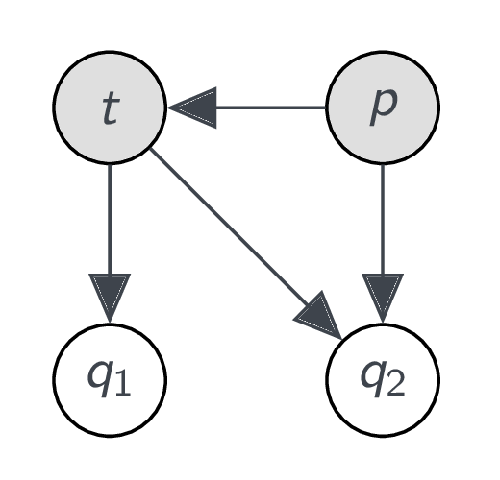}
		                                                                                                                            & \includegraphics[width=0.12\textwidth]{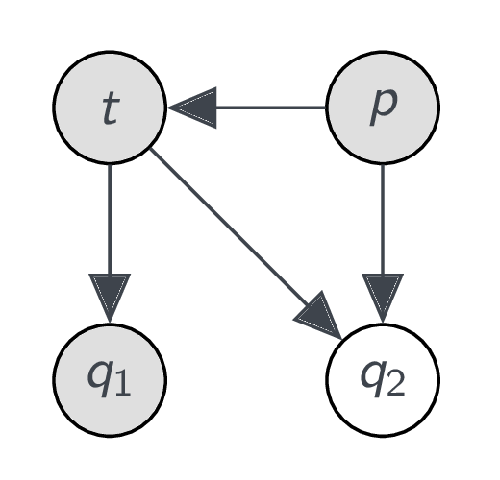}
		                                                                                                                            & \includegraphics[width=0.12\textwidth]{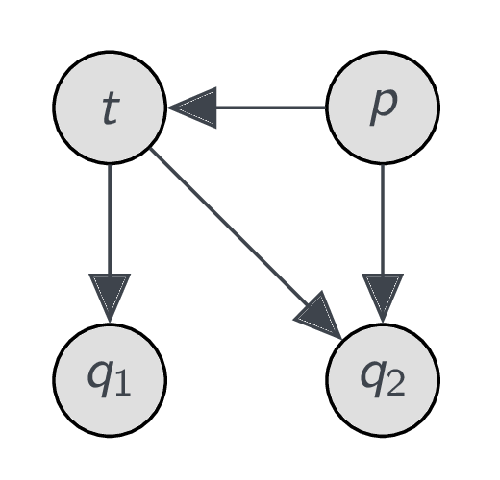}
	\end{tabular}
	\caption{Sequence of learned deep generative models for the \textit{Pick and Place} problem.}
	\label{fig:sequence}
\end{figure}

\section{Experiments}

\subsection{Image-Based Problem Representation}

We use an image-based representation of the problem instance \( \tau \) that consists of a depth image and masks.
Specifically, \( \tau = \{d , m_1 , m_2 , m_3\} \), where \( d \) is the depth image, and \( m_1, m_2, m_3 \) are three masks containing information from the initial object pose, the goal pose or placement region, and obstacles, respectively, as shown in \cref{fig:instance-assembly}.
In the factored approach, each generative module receives as input only the relevant masks; for example, the sampler for the robot pick configuration receives a mask of the obstacles and the initial configuration but not the goal pose.

The main strength of the image representation is that it can generalize to different object shapes and a varying number of obstacles and shapes.
Moreover, a depth camera is readily available, approximate masks can be computed with image segmentation techniques, and it provides a robust representation of sequential manipulation problems in tabletop scenarios.

\begin{figure}
	\def\namelist{18011,18017}
	\includegraphics[width=0.19\textwidth]{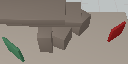}
	\includegraphics[width=0.19\textwidth]{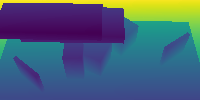}
	\includegraphics[width=0.19\textwidth]{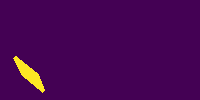}
	\includegraphics[width=0.19\textwidth]{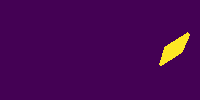}
	\includegraphics[width=0.19\textwidth]{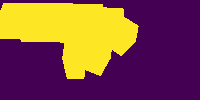} \\
	\includegraphics[width=0.19\textwidth]{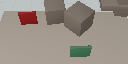}
	\includegraphics[width=0.19\textwidth]{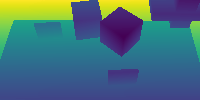}
	\includegraphics[width=0.19\textwidth]{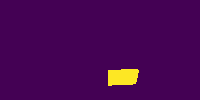}
	\includegraphics[width=0.19\textwidth]{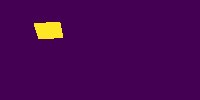}
	\includegraphics[width=0.19\textwidth]{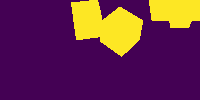}
	\caption{Image-based representation of problem instances in \textit{Assembly} and \textit{Handover}.
		\vspace{.5cm}}
	\label{fig:instance-assembly}
\end{figure}

\subsection{Scenarios}

We consider three different manipulation tasks that involve object manipulation with stable grasps:

\begin{itemize}
	\item \textit{Pick and Place}: A robot must pick up an object and place it within a specified rectangular region on the table.
	      Refer to \cref{fig:sequence,fig:gan:pp_six} for illustrations.

	\item \textit{Assembly}: Two mobile robots are required to each pick up an object and join them together.
	      The assembly process is not prescriptive and is characterized as a manifold with constraints on both rotation and position: the objects must align perpendicularly and establish stable contact with predetermined faces of the cubes, forming a 'T' shape.
	      Refer to \cref{fig:assembly,fig:gan:asse_seven} for illustrations.

	\item \textit{Handover}: Two mobile robots collaborate to transfer the object from its starting to its target position via a handover.
	      Refer to \cref{fig:handover,fig:gan:hand_four} for illustrations.
\end{itemize}

These tasks are executed on a cluttered table with obstacles varying in number from three to five.
The grasp between the gripper and the object is defined by a two-fingered gripper (e.g., the gripper of the Franka Panda robot), which restricts both position and orientation.
Robots and movable objects are required to avoid collisions with each other, as well as with any obstacles and the table itself.
The training dataset comprises 4,000 pairs of problem-solution scenarios, computed offline using a user-defined sampling sequence to ensure diversity.
Each problem is depicted by a 64x128x4 image (the network's input) and the corresponding environment (used to calculate the analytical error term during training and to conduct nonlinear optimization in the benchmark).

Variability in instances is introduced by differences in the number, position, and size of obstacles, the dimensions and location of the objects, and the goal configuration (for example, see \cref{fig:gan:pp_six}).
The instances in both the evaluation and training datasets are generated from the same distribution.

\begin{figure}[t]
	\begin{tabular}{c|c}
		\includegraphics[trim={2cm 6cm 4cm 6cm},clip,width=.22\textwidth]{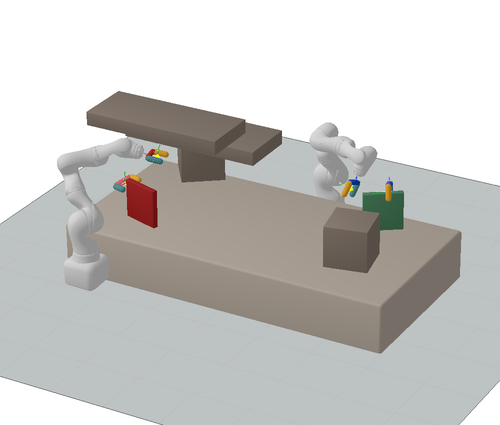}
		\includegraphics[trim={2cm 6cm 4cm 6cm},clip,width=.22\textwidth]{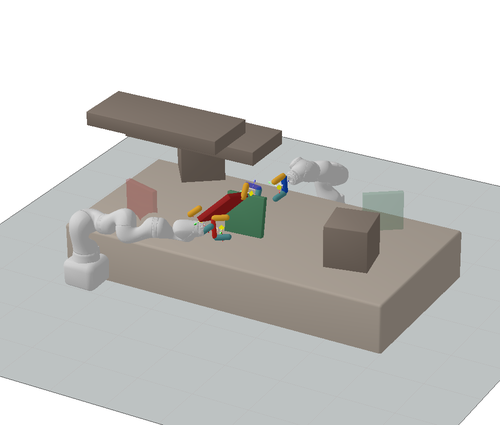} &
		\includegraphics[trim={2cm 6cm 4cm 6cm},clip,width=.22\textwidth]{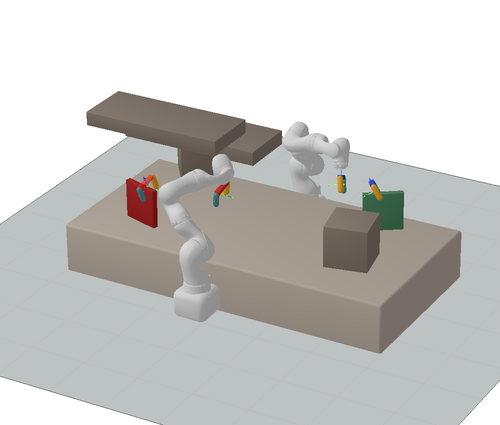}
		\includegraphics[trim={2cm 6cm 4cm 6cm},clip,width=.22\textwidth]{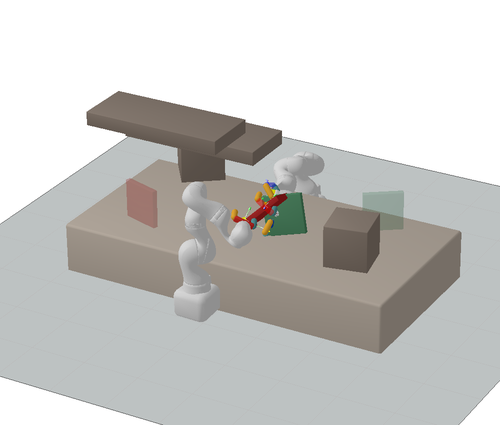}
	\end{tabular}
	\caption{
    Two samples from the deep generative model in the same instance of the \textit{Assembly} problem are displayed.
		Each sample is depicted with two keyframes (pick and assembly).
		Both seeds lead to a feasible solution.
	}
	\label{fig:assembly}
  \vspace{.2cm}
\end{figure}

\subsection{Ablation Study}
\label{sec:ablation}

The \textit{Pick and Place} scenario is utilized for an ablation study of the proposed generative model.
We assess the contribution of the factored structure (\textit{Big NN} vs \textit{Structured NN}) and the analytical error term (\textit{+analytical}).
We evaluate the precision and coverage of seed samples (output from the deep generative models) and solutions (after projection with nonlinear optimization) by generating \num{4,000} samples for each new instance (\num{30} in total).
Results are presented in \cref{tab:ablation}.
We report the following metrics:

\begin{itemize}
	\item \textit{Error}: the constraint violation $||\f(x;\tau)||^2$ (unitless, lower is better).
	\item \textit{Precision}: the average nearest neighbour distance to a reference dataset.
	      It models how close samples are to the real data (lower is better).
	\item \textit{Coverage}: the average nearest neighbour distance from the reference dataset to the computed samples.
	      It describes how well the learned distribution covers the reference dataset (lower is better).
	\item \textit{Success rate}: the success rate of the nonlinear optimization, signifying how many samples result in a feasible solution after optimization (higher is better).
\end{itemize}

When evaluating the seeds (output of the generative model), our two contributions are vital for achieving a minimal constraint violation (analytical term) and optimal coverage (structure).
Seeds from the model with both structure and the analytical term are more likely to lead to a solution (success rate).
After the projection step, only samples from networks with structure exhibit good coverage.

\begin{table}[t!]
	\centering
	\small
	\setlength{\tabcolsep}{4pt}
	\begin{tabular}{@{}lllllll@{}}
		\toprule
		                                                & \multicolumn{3}{c}{Seeds}                       & \multicolumn{3}{c}{Solutions}                                                                                                  \\
		\cmidrule(lr){2-4} \cmidrule(lr){5-7}
		                                                & Coverage                                        & Precision                                       & Error                                       & Coverage & Precision & Success \\ \midrule
		Big NN                                          & \rawtwo[f_dist2][0]$\pm$\rawtwo[f_dist2][1]     & \rawtwo[f_dist1][0]$\pm$\rawtwo[f_dist1][1]     & \rawtwo[error][0]$\pm$\rawtwo[error][1]     &
		\soltwo[f_dist2][0]$\pm$\soltwo[f_dist2][1]     & \soltwo[f_dist1][0]$\pm$\soltwo[f_dist1][1]     & 0.46                                                                                                                           \\
		Big NN+analytical                               & \rawzero[f_dist2][0]$\pm$\rawzero[f_dist2][1]   & \rawzero[f_dist1][0]$\pm$\rawzero[f_dist1][1]   & \rawzero[error][0]$\pm$\rawzero[error][1]   &
		\solzero[f_dist2][0]$\pm$\solzero[f_dist2][1]   & \solzero[f_dist1][0]$\pm$\solzero[f_dist1][1]   & 0.43                                                                                                                           \\
		Structure NN                                    & \rawthree[f_dist2][0]$\pm$\rawthree[f_dist2][1] & \rawthree[f_dist1][0]$\pm$\rawthree[f_dist1][1] & \rawthree[error][0]$\pm$\rawthree[error][1] &
		\solthree[f_dist2][0]$\pm$\solthree[f_dist2][1] & \solthree[f_dist1][0]$\pm$\solthree[f_dist1][1] & 0.56                                                                                                                           \\
		Structure NN+analytical                         & \rawone[f_dist2][0]$\pm$\rawone[f_dist2][1]     & \rawone[f_dist1][0]$\pm$\rawone[f_dist1][1]     & \rawone[error][0]$\pm$\rawone[error][1]     &
		\solone[f_dist2][0]$\pm$\solone[f_dist2][1]     & \solone[f_dist1][0]$\pm$\solone[f_dist1][1]     & 0.78                                                                                                                           \\ \bottomrule
	\end{tabular}
	\caption{Ablation study in the \textit{Pick and Place} scenario.}
	\label{tab:ablation}
\end{table}

\subsection{Benchmark: Generative Models in Nonlinear Optimization}

The \textit{Assembly} and \textit{Handover} scenarios are used to compare our generative model against two baseline methods for warm-starting (seeding) nonlinear optimizers.
We analyze the number of solved problems and the number of necessary optimization runs.
Measuring the number of solved nonlinear programs is an indirect way to evaluate coverage and sample quality, as both are fundamental to solving a diverse set of problems with a nonlinear optimizer and preventing convergence to infeasible points.
We compare our complete model (deep generative model with structure and analytical error term), in short, \textit{Deep}, with the baselines:
\begin{itemize}
	\item \textit{Rand:} Randomized initial guess around a reference value.
	\item \textit{Rand Data:}
	      Choosing samples from the training dataset at random.
	      The initial point is a feasible sample for another problem of the same family.
	      This is actually a strong baseline because it provides diverse, informative initial seeds.
\end{itemize}

We evaluate the generative model (deep generative model + optimization) on \num{200} problems from the evaluation dataset.
The experiments are repeated 10 times, and we report the mean and variance.
We first report how many optimization trials (each trial has an independent starting point) are necessary to solve each of the test instances and plot the histogram of the mean value in \cref{fig:trials}.
Unsolved problems are assumed to be solved with \num{10} trials (maximum number of trials).

In both scenarios, the proposed deep generative model outperforms the baseline warm starts, significantly reducing the number of trials required to solve the instances.
First, note that \textit{Rand} can only solve \num{10}\% of the problems when using a maximum of \num{10} trials.
This is because the initial guess is not informative, and the nonlinear optimizer often converges to infeasible points.
We now compare \textit{Deep} against \textit{Rand Data} and observe that \textit{Deep} provides a \num{1.5}-\num{2}x improvement.
On average (across problems), from (Rand Data, $3.86\pm1.29$) to (Deep, $2.77\pm1.79$) in \textit{Handover} and from (Rand Data, $3.99\pm1.44$) to (Deep, $2.07\pm1.38$) in \textit{Assembly}.
To complete the analysis, we also show the cumulative number of problems solved as we increase the number of optimization trials in \cref{fig:cumul}.

\begin{figure}[t]
	\centering
	\begin{subfigure}[t]{0.49\textwidth}
		\centering
		\includegraphics[width=0.8\textwidth]{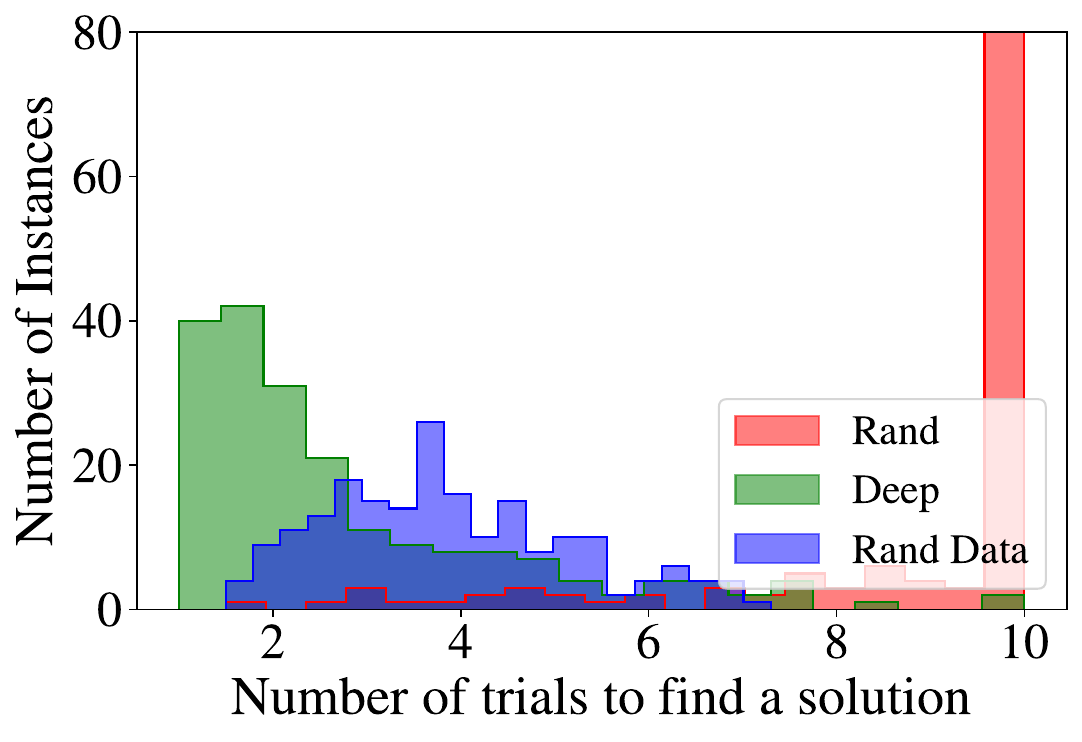}
		\caption{Handover.}
	\end{subfigure}
	\begin{subfigure}[t]{0.49\textwidth}
		\centering
		\includegraphics[width=0.8\textwidth]{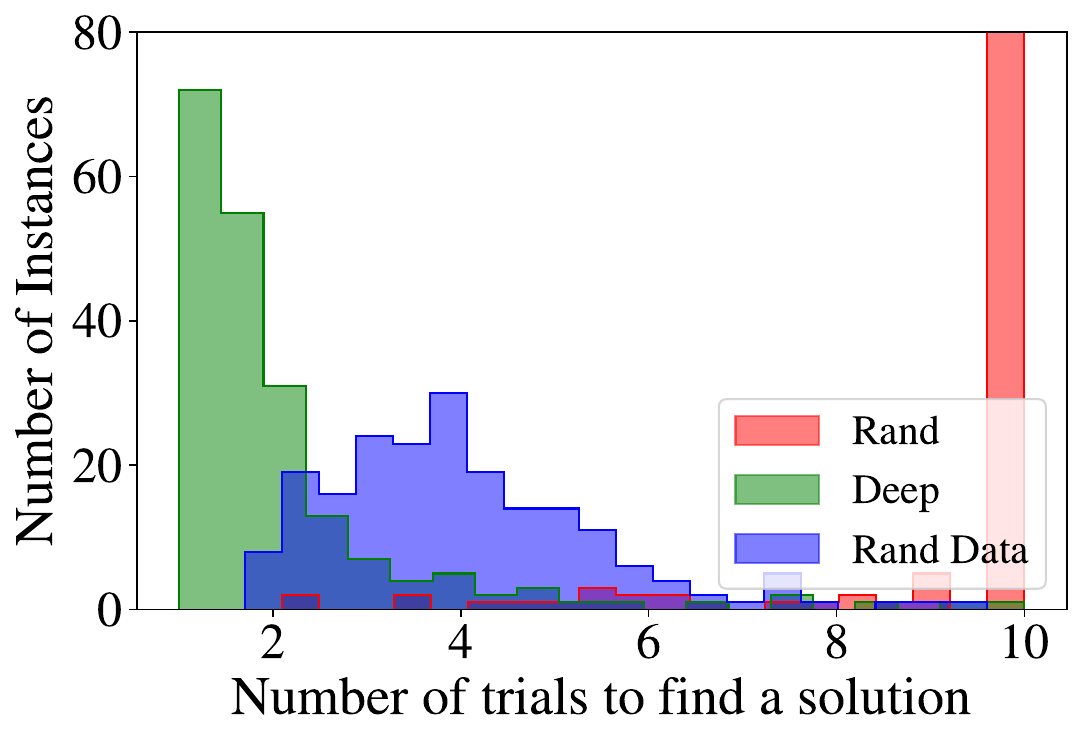}
		\caption{Assembly.}
	\end{subfigure}
	\caption{Histogram of the estimated number of trials necessary to solve an instance (lower is better). \vspace{.5cm}
	}
	\label{fig:trials}
\end{figure}
\begin{figure}
	\centering
	\begin{subfigure}[t]{0.49\textwidth}
		\centering
		\includegraphics[width=0.8\textwidth]{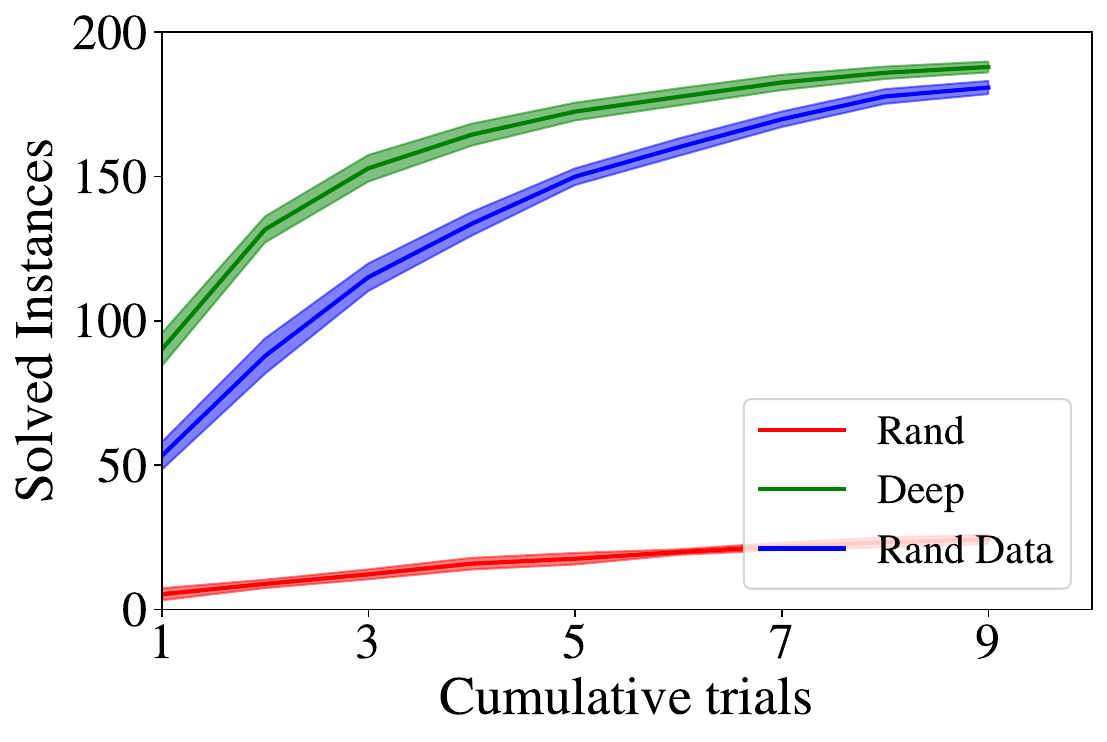}
		\caption{Handover.}
	\end{subfigure}
	\begin{subfigure}[t]{0.49\textwidth}
		\centering
		\includegraphics[width=0.8\textwidth]{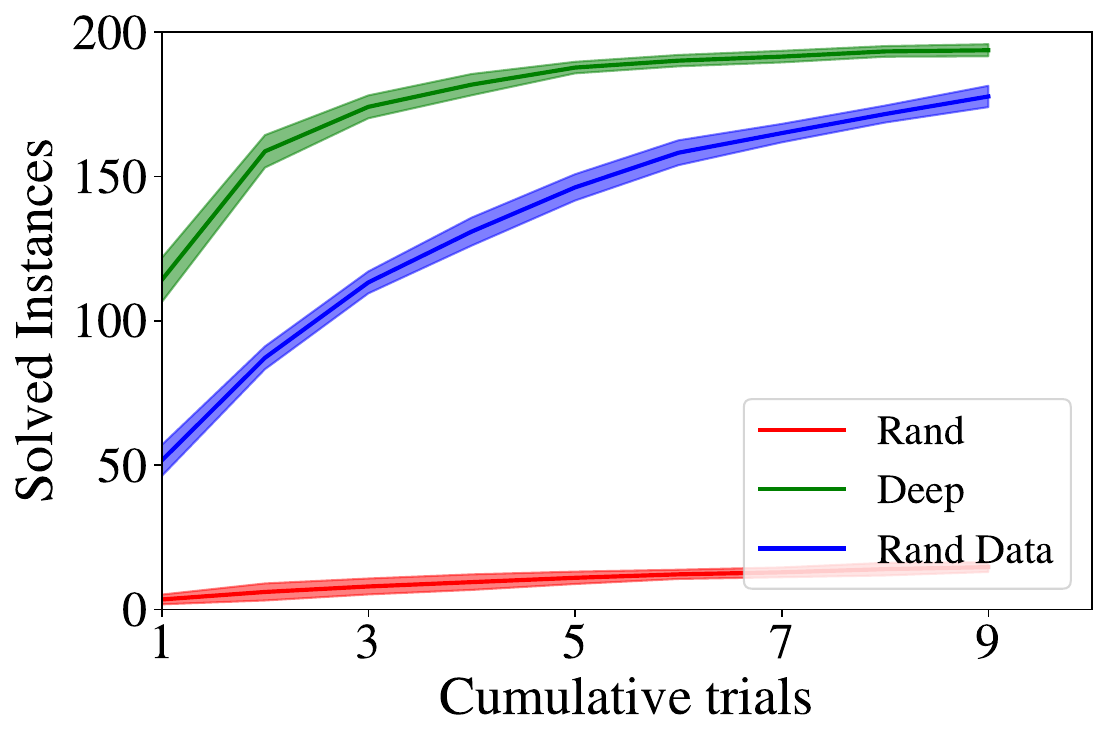}
		\caption{Assembly.}
	\end{subfigure}
	\caption{Cumulative number of solved problems (higher is better).}
	\label{fig:cumul}
\end{figure}

The computational overhead of evaluating the neural network is small (we produce \num{10} samples in \SI{8}{ms} with a GPU), while most of the time is spent in optimization runs that converge to infeasible points.

\begin{figure}
	\centering
	\includegraphics[width=.95\textwidth]{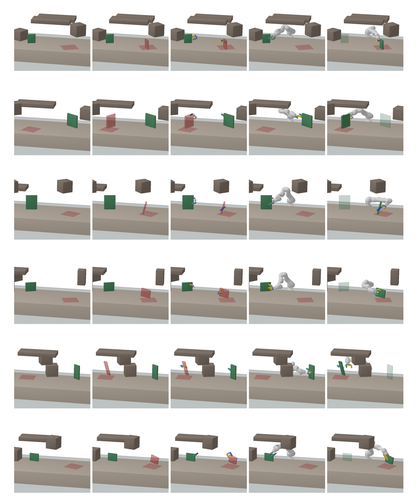}
	\caption{Sequence of sampling operations (output of the deep generative model) in \emph{Pick and Place} across six different instances.
		Variables are sampled in the following order: object pose on the table, grasp, robot pick configuration, and robot place configuration.
	}
	\label{fig:gan:pp_six}
\end{figure}

\begin{figure}
	\centering
	\includegraphics[width=.82\textwidth]{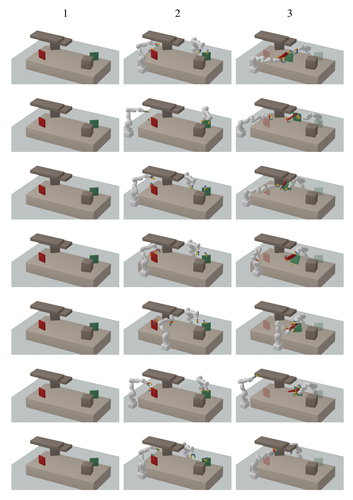}
	\caption{Seven different samples (output of the deep generative model) within the same instance of the \textit{Assembly} problem.
		Each row represents a different sample.
		Column 1 shows the problem instance; column 2, the pick keyframe; and column 3, the assembly keyframe.
	}
	\label{fig:gan:asse_seven}
\end{figure}

\begin{figure}
	\centering
	\includegraphics[width=.75\textwidth]{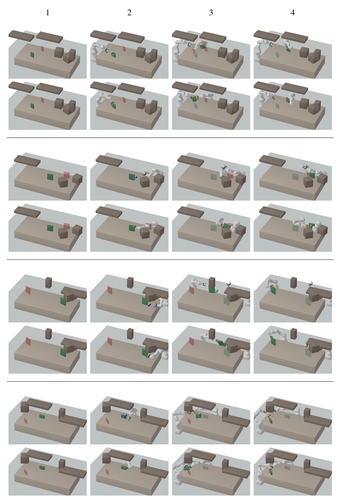}
	\caption{Four pairs of initial guesses (deep generative model outputs) and solutions (post-optimization) across different instances of the \textit{Handover} problem.
		In each pair of rows, the top row displays the approximate sample, and the bottom row shows the optimized solution.
		The first column presents the scene, while columns 2, 3, and 4 display the pick, handover, and place keyframes, respectively.
	}
	\label{fig:gan:hand_four}
\end{figure}

\section{Limitations}

This chapter highlights that finding a good warm start for nonlinear optimization is an intricate problem in itself.
Although the proposed method clearly outperforms randomized initialization, we observe that one of the baselines, namely warm-starting with a solution from the dataset chosen at random (and thus without considering the conditioning on the scene), also offers competitive results, given its simplicity.
Thus, a fundamental question is to understand what constitutes a good warm start for nonlinear optimization and how this might depend on the problem instance.
In our work, the underlying assumption is that a sample close to the solution manifold is a good warm start.
In practice, we observe that it is often sufficient and beneficial to have diverse samples on different bases of attraction of the optimizer, which might vary based on the problem instance.

A future research direction is to analyze the expressivity and practical performance of different formulations and architectures for deep generative models, comparing explicit generative models (such as GANs and VAEs), energy-based implicit models, and diffusion-based models.

An inherent limitation of our approach is that generative models are only applicable to similar problems with the same number of variables and constraints.
Specifically, for each class of problems evaluated in the experiments--\textit{Pick and Place}, \textit{Handover}, and \textit{Assembly}--we have generated a different dataset of solutions and trained different models.

In contrast, in the following \cref{ch:learn-feas}, we use the structure of the Factored-NLP to provide generalization across different problem classes (i.e., different types and numbers of constraints and variables) by sharing and combining small modules, resulting in a single universal model for different task plans.
The neural architectures and applications differ, as we predict subsets of infeasible nonlinear constraints instead of generating solutions.
However, we see great potential in adapting these insights back to the generative framework.

\section{Conclusion}
\label{sec:conclusion}

In this work, we propose Deep Generative Constraint Sampling (DGCS), a novel approach to sampling from a constraint manifold to address challenges in robotic sequential manipulation.
Our framework combines a deep generative sampling model, conditioned on an image-based representation of the problem, with a nonlinear optimizer to project samples onto the manifold.
Additionally, we extend the approach to exploit a given factorization of the problem by training a sequence of conditional generative models rather than a single joint generator.
Our empirical results confirm that the trained generative models outperform heuristic warm start strategies.
Moreover, incorporating analytic constraints into the training of the generative model, as well as exploiting the factorization of a given problem, significantly enhances the efficiency, diversity, and precision of the sampling approach.

Our current framework integrates generative sampling using a neural network with subsequent projection through constrained optimization.
A promising future direction involves exploring the possibility of embedding the optimization algorithm as the last layer of the generative model, while still maintaining good coverage and multimodality.

In this chapter, we fix the sequence of sampling operations by design when transforming the Factored-NLP into a directed sampling network.
As discussed in \cref{ch:mcts}, the choice of the sampling operations sequence can significantly impact performance. However, sampling operations with deep learning are conditioned on the problem scene, providing flexibility to capture interdependencies in the manipulation sequences and mitigating the drawbacks of uninformed uniform sampling.

Looking forward, we aim to develop deep generative models that directly work with the Factored-NLP, eliminating the need for manually designed directed graphical models, as this would allow for the automatic learning of the most effective decompositions.

%% file: learn_conflicts.tex
\chapter{\nameChapterSix}
\label{ch:learn-feas}

\section{Introduction}

When a Factored Nonlinear Program (Factored-NLP) is over-constrained or infeasible, a fundamental challenge is to extract a minimal conflict -- a minimal subset of constraints that can never be fulfilled.
Traditional approaches require solving several nonlinear programs, incrementally adding and removing constraints, and are thus computationally expensive.

In this chapter\footnote{
	This chapter is based on the publication:
	Ortiz-Haro, J., Ha, J.
	S., Driess, D., Karpas, E., \theAnd Toussaint, M.
	(2023).
	Learning Feasibility of Factored Nonlinear Programs in Robotic Manipulation Planning.
	IEEE International Conference on Robotics and Automation (ICRA) (pp.
  3729-3735).
},
we propose a graph neural architecture that predicts which subsets of variables and constraints are infeasible in a Factored-NLP.
The model is trained with a dataset of labeled subgraphs from Factored-NLPs and can make useful predictions on larger problems than those seen during training.

As an application, we evaluate our method in robotic sequential manipulation and integrate this model into our novel TAMP solver presented in \cref{ch:bid}.
The objective is to quickly determine which constraints fail in the trajectory optimization problems of candidate task plans, which is one of the computational bottlenecks of our solver.

Beyond factored nonlinear programs in TAMP, our framework is applicable for detecting infeasibility in constraint satisfaction problems, combinatorial optimization, and Boolean satisfaction (SAT), which have broad applications in robotics, planning, and scheduling.

The foundation of our framework is built on the factored structure of manipulation planning problems, which has been presented in \cref{sec:bg:structure}, and further developed and analyzed with our new factored TAMP formulation in \cref{ch:bid}.

An overview of our approach is shown in \cref{fig:overview}.
The input to our model is directly the graph representation of the Factored-NLP, including semantic information about variables and constraints (e.g., a class label), and a continuous feature for each variable that encodes geometric information about the scene.
Finding the minimal infeasible subgraph (i.e., a subset of variables and constraints of the Factored-NLP) is cast as a graph node classification problem, and the predicted infeasible subsets are extracted with a connected component analysis.

By leveraging the factored structure, our model is able to predict infeasibility in longer manipulation sequences involving more objects and robots, as well as different geometric environments -- a broader generalization than our deep generative models presented in \cref{ch:gans}, which were limited to a fixed high-level task plan.

Our experiments show that the model accelerates general algorithms for conflict extraction by a factor of \num{50}, and our previous heuristic algorithm for conflict detection in TAMP by a factor of \num{4}.

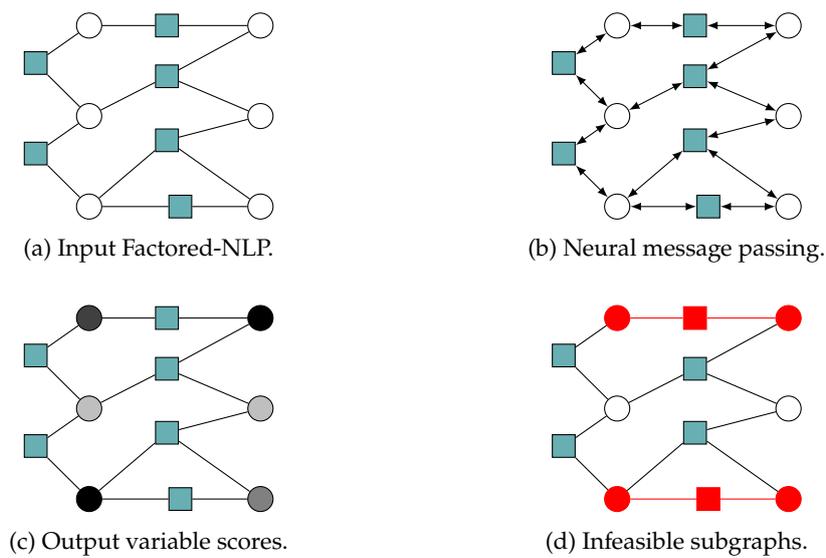
\begin{figure}
	\begin{subfigure}[c]{0.49\columnwidth}

		\centering

		\begin{tikzpicture}[node distance=1.2cm,
				every node/.style={fill=white, font=\small}, align=center]
			\node (var1)     [gnn_var]          {$$};
			\node (var2)     [right of=var1, xshift=3em,  gnn_var]          {$$};

			\node (var3)     [below of=var1, gnn_var]          {$$};
			\node (var4)     [below of=var2, gnn_var]          {$$};
			\node (var5)     [below of=var3, gnn_var]          {$$};
			\node (var6)     [below of=var4, gnn_var]          {$$};

			\node (con12)     [right of = var1 , xshift=-.5em,gnn_con ]          {$$};
			\node (con13)     [below of = var1 , yshift=2em, xshift=-2em, gnn_con]          {$$};
			\node (con35)     [below of = var3 , yshift=2em, xshift=-2em, gnn_con]          {$$};
			\node (con234)     [below of = con12 ,yshift=1.5em, gnn_con]          {$$};

			\node (con546)     [below of = con234 , yshift=1em, gnn_con]          {$$};

			\node (con56)     [right of = var5 , gnn_con]          {$$};

			\draw[-]             (con12) -- (var1);
			\draw[-]             (con12) -- (var2);

			\draw[-]             (con234) -- (var3);
			\draw[-]             (con234) -- (var4);
			\draw[-]             (con234) -- (var2);

			\draw[-]             (con546) -- (var5);
			\draw[-]             (con546) -- (var4);
			\draw[-]             (con546) -- (var6);

			\draw[-]             (con13) -- (var3);
			\draw[-]             (con13) -- (var1);

			\draw[-]             (con35) -- (var3);
			\draw[-]             (con35) -- (var5);

			\draw[-]             (con56) -- (var5);
			\draw[-]             (con56) -- (var6);

		\end{tikzpicture}
		\subcaption{Input Factored-NLP.}
	\end{subfigure}
	\begin{subfigure}[c]{0.49\columnwidth}
		\centering
		\begin{tikzpicture}[node distance=1.2cm,
				every node/.style={fill=white, font=\small}, align=center]
			\node (var1)     [gnn_var]          {$$};
			\node (var2)     [right of=var1, xshift=3em,  gnn_var]          {$$};

			\node (var3)     [below of=var1, gnn_var]          {$$};
			\node (var4)     [below of=var2, gnn_var]          {$$};
			\node (var5)     [below of=var3, gnn_var]          {$$};
			\node (var6)     [below of=var4, gnn_var]          {$$};

			\node (con12)     [right of = var1 , xshift=-.5em,gnn_con ]          {$$};
			\node (con13)     [below of = var1 , yshift=2em, xshift=-2em, gnn_con]          {$$};
			\node (con35)     [below of = var3 , yshift=2em, xshift=-2em, gnn_con]          {$$};
			\node (con234)     [below of = con12 ,yshift=1.5em, gnn_con]          {$$};

			\node (con546)     [below of = con234 , yshift=1em, gnn_con]          {$$};

			\node (con56)     [right of = var5 , gnn_con]          {$$};

			\draw[latex-latex]             (con12) -- (var1);
			\draw[latex-latex]             (con12) -- (var2);

			\draw[latex-latex]             (con234) -- (var3);
			\draw[latex-latex]             (con234) -- (var4);
			\draw[latex-latex]             (con234) -- (var2);

			\draw[latex-latex]             (con546) -- (var5);
			\draw[latex-latex]             (con546) -- (var4);
			\draw[latex-latex]             (con546) -- (var6);

			\draw[latex-latex] (con13) -- (var3);
			\draw[latex-latex]             (con13) -- (var1);

			\draw[latex-latex]             (con35) -- (var3);
			\draw[latex-latex]             (con35) -- (var5);

			\draw[latex-latex]             (con56) -- (var5);
			\draw[latex-latex]             (con56) -- (var6);

		\end{tikzpicture}

		\subcaption{Neural message passing.}
	\end{subfigure}
	\vspace{.5cm}

	\begin{subfigure}[c]{0.49\columnwidth}
		\centering
		\begin{tikzpicture}[node distance=1.2cm,
				every node/.style={fill=white, font=\small}, align=center]
			\node (var1)     [gnn_var, fill=darkgray]          {$$};
			\node (var2)     [right of=var1, xshift=3em,  gnn_var, fill=black]          {$$};

			\node (var3)     [below of=var1, gnn_var, fill=lightgray]          {$$};
			\node (var4)     [below of=var2, gnn_var, fill=lightgray]          {$$};
			\node (var5)     [below of=var3, gnn_var, fill=black]          {$$};
			\node (var6)     [below of=var4, gnn_var, fill=gray]          {$$};

			\node (con12)     [right of = var1 , xshift=-.5em,gnn_con]          {$$};
			\node (con13)     [below of = var1 , yshift=2em, xshift=-2em, gnn_con]          {$$};
			\node (con35)     [below of = var3 , yshift=2em, xshift=-2em, gnn_con]          {$$};
			\node (con234)     [below of = con12 ,yshift=1.5em, gnn_con]          {$$};

			\node (con546)     [below of = con234 , yshift=1em, gnn_con]          {$$};

			\node (con56)     [right of = var5 , gnn_con]          {$$};

			\draw[-]             (con12) -- (var1);
			\draw[-]             (con12) -- (var2);

			\draw[-]             (con234) -- (var3);
			\draw[-]             (con234) -- (var4);
			\draw[-]             (con234) -- (var2);

			\draw[-]             (con546) -- (var5);
			\draw[-]             (con546) -- (var4);
			\draw[-]             (con546) -- (var6);

			\draw[-]             (con13) -- (var3);
			\draw[-]             (con13) -- (var1);

			\draw[-]             (con35) -- (var3);
			\draw[-]             (con35) -- (var5);

			\draw[-]             (con56) -- (var5);
			\draw[-]             (con56) -- (var6);
		\end{tikzpicture}
		\subcaption{Output variable scores.}
	\end{subfigure}
	\begin{subfigure}[c]{0.49\columnwidth}
		\centering
		\begin{tikzpicture}[node distance=1.2cm,
				every node/.style={fill=white, font=\small}, align=center]
			\node (var1)     [gnn_var, draw=red,fill=red]          {$$};
			\node (var2)     [right of=var1, xshift=3em,  gnn_var,draw=red, fill=red]          {$$};

			\node (var3)     [below of=var1, gnn_var]          {$$};
			\node (var4)     [below of=var2, gnn_var]          {$$};
			\node (var5)     [below of=var3, gnn_var, draw=red,fill=red]          {$$};
			\node (var6)     [below of=var4, gnn_var, draw=red, fill=red]          {$$};

			\node (con12)     [right of = var1 , xshift=-.5em,gnn_con, draw=red, fill=red]          {$$};
			\node (con13)     [below of = var1 , yshift=2em, xshift=-2em, gnn_con]          {$$};
			\node (con35)     [below of = var3 , yshift=2em, xshift=-2em, gnn_con]          {$$};
			\node (con234)     [below of = con12 ,yshift=1.5em, gnn_con]          {$$};

			\node (con546)     [below of = con234 , yshift=1em, gnn_con]          {$$};

			\node (con56)     [right of = var5 , gnn_con, color=red]          {$$};

			\draw[-,color=red]             (con12) -- (var1);
			\draw[-,color=red]             (con12) -- (var2);

			\draw[-]             (con234) -- (var3);
			\draw[-]             (con234) -- (var4);
			\draw[-]             (con234) -- (var2);

			\draw[-]             (con546) -- (var5);
			\draw[-]             (con546) -- (var4);
			\draw[-]             (con546) -- (var6);

			\draw[-]             (con13) -- (var3);
			\draw[-]             (con13) -- (var1);

			\draw[-]             (con35) -- (var3);
			\draw[-]             (con35) -- (var5);

			\draw[-,color=red]             (con56) -- (var5);
			\draw[-,color=red]             (con56) -- (var6);

		\end{tikzpicture}

		\subcaption{Infeasible subgraphs.}
	\end{subfigure}
	\caption{Overview of our approach to detecting minimal infeasible subgraphs in a Factored-NLP. \textit{(a)}
		The input of the model is the graph representation of the Factored-NLP.
		Circles represent variables, and squares represent constraints.
		\textit{(b)}
		We perform several iterations of neural message passing using the structure of the Factored-NLP.
		\textit{(c)}
		The network outputs the probability that a variable belongs to a minimal infeasible subgraph.
		\textit{(d)}
		We extract several minimal infeasible subgraphs using a connected component analysis.
	}
	\label{fig:overview}
\end{figure}

\section{Related Work}

\paragraph{Minimal infeasible subsets of constraints}

In the discrete SAT and CSP literature, a minimal infeasible subset of constraints (also called a Minimal Unsatisfiable Subset of Constraints or a Minimal Unsatisfiable Core) is usually computed by solving a sequence of SAT and MAX-SAT problems \cite{liffiton2008algorithms,marques2021conflict, hemery2006extracting}.

In continuous domains, a minimal infeasible subset can be found by solving a sequence of feasibility problems, adding and removing constraints, with linear complexity in the number of constraints \cite{amaldi1999some}.
This search can be accelerated with a divide and conquer strategy, with logarithmic complexity \cite{junker2004preferred}.
In convex and nonlinear optimization, we can find approximate minimal subsets by solving one optimization problem with slack variables \cite{shoukry2018smc}.

In contrast, our method uses learning to directly predict minimal infeasible subsets of variables and constraints and can be combined with these previous approaches to reduce computational time.

\paragraph{Graph Neural Networks in combinatorial optimization}

We use Graph Neural Networks (GNN) \cite{kipf2016semi, battaglia2018relational, ma_tang_2021} for learning in graph-structured data.
Different message passing and convolutions have been proposed, e.g., \cite{gilmer2017neural,velivckovic2017graph}.
Our architecture, targeted toward inference in factored nonlinear programs, is inspired by previous works that approximate belief propagation in factor graphs \cite{zhang2020factor, garcia2020neural, kuck2020belief}.

Recently, GNN models have been applied to solve NP-hard problems \cite{schuetz2021combinatorial}, Boolean Satisfaction \cite{selsam2018learning}, Max cut \cite{yao2019experimental}, constraint satisfaction \cite{toenshoff2021graph}, and discrete planning \cite{shen2020learning,rivlin2020generalized,DBLP:conf/socs/NirSK21}.
Compared to state-of-the-art solvers, learned models achieve competitive solution times and scalability but are outperformed in reliability and accuracy.
To our knowledge, this is the first work to use a GNN model to predict minimal infeasible subsets of constraints in a continuous domain.

\paragraph{Graph Neural Networks in manipulation planning}

In manipulation planning, Graph Neural Networks
are a popular architecture to represent the relations between movable objects because they provide a strong relational bias and a natural generalization to include additional objects in the scene.

For example, they have been used as problem encodings to learn policies for robotic assembly \cite{pmlr-v164-funk22a,Ghasemipour2022blocks} and manipulation planning \cite{li2020towards}, to learn object importance and guide task and motion planning \cite{silver2021planning}, and to learn dynamical models and interactions between objects \cite{driess2022learning}, \cite{paus2020predicting}.
Previous works often use task-specific, object-centric representations, where the vertices of the graph represent the objects, and the task is encoded in the initial feature vector of each variable.
Alternatively, our model performs message passing using the structure of the nonlinear program of the manipulation sequence, achieving better generalization to different task plans that fulfill different goals.

\section{Formulation}

\subsection{Minimal Infeasible Subgraph in a Factored-NLP}

Given an infeasible or over-constrained Factored-NLP \( G=(X_G \cup \Phi_G,E_G) \) with variables \( X_G \) and constraints \( \Phi_G \) (refer to \cref{eq:factored-nlp,eq:factpred-nlp-graph}), our intention is to identify a minimal infeasible subgraph, i.e., a subset of variables and constraints that are jointly infeasible and cannot be reduced further.

To define it formally, a minimal infeasible subgraph \( M = (X_M \cup \Phi_M, E_M) \) of a Factored-NLP \( G=(X_G \cup \Phi_G,E_G) \), is a subset of variables \( X_M \subseteq X_G \) and constraints \( \Phi_M \subseteq \Phi_G \) that is infeasible; yet, any proper subset of it is feasible:
\begin{equation}
	M \subseteq G, ~\mathcal{F}(M)=0, ~ \mathcal{F}(M')= 1, ~ \forall M' \subset M,
\end{equation}
where \(\mathcal{F}(M)\) denotes the feasibility of the Factored-NLP (see \cref{eq:fac:feas}), holding the value \( 1 \) if it is feasible and \( 0 \) otherwise.

In this chapter, we consider only minimal subgraphs in the form of \textit{variable-induced} subgraphs because they enable a more compact representation.
Given a graph \( G \) and a subset of variables \( X' \subseteq X_G \), a \emph{variable-induced} subgraph \( M = G[X'] = (X' \cup \Phi', E') \), where \( \Phi' = \{ \phi \in \Phi_G \mid \text{Neigh}_G(\phi) \subseteq X' \} \), is the subgraph spanned by the variables \( X' \).
Intuitively, \( G[X'] \) contains the variables \( X' \) and all the constraints that can be evaluated with these variables.
Our approach can be adapted to predict general subgraphs if required, by modifying the proposed variable classification to constraint classification in \cref{sec:asnodeclassif}.

A Factored-NLP can contain multiple infeasible subgraphs, and a variable \( x_i \in X_G \) can belong to multiple infeasible subgraphs.
Recall that a minimal infeasible subgraph is connected, and a supergraph \( \tilde{M} \supseteq M \) of an infeasible subgraph \( M \) is also infeasible.

\subsection{Minimal Infeasible Subgraph as Variable Classification}
\label{sec:asnodeclassif}

Let \( \Omega_G = \{ M_r ~|~ M_r \subseteq G \text{~minimal infeasible} \} \) be the set of minimal infeasible subgraphs of a Factored-NLP \( G \).
Instead of learning the mapping \( \omega : G \mapsto \Omega_G \) directly, we propose to learn an over-approximation \( \tilde{\omega} \) that can efficiently be framed as binary variable classification.

We first introduce the \textit{variable-feasibility} function \( \psi(x_i;G) \) that assigns a label \( y_i \in \{0,1\} \) to each variable \( x_i \in X_G \):
\( y_i=0 \) if \( x_i \) belongs to some infeasible subgraph and \( y_i=1 \) otherwise.
Given such a labeled graph, we can recover the infeasible subgraphs approximately by computing the connected components on the subgraph induced by the variables labeled \( 0 \), i.e., \( G\left[\{ x_i \in X_G \mid y_i = 0 \}\right] \).
Thus, we define the approximate mapping as
\begin{equation}
	\tilde{\omega}(G) = \text{CCA}\left(G\left[\{ x_i \in X_G \mid y_i = 0 \}\right]\right),
\end{equation}
where \( \text{CCA} \) denotes a connected component analysis.

The approximate mapping \( \tilde{\omega} \) is exact, i.e., \( \tilde{\omega} = \omega \), if the infeasible subgraphs are disconnected.
If two or more infeasible subgraphs are connected, it returns their union as a minimal infeasible subgraph, i.e., \( \cup \tilde{\omega} = \cup \omega \), which over-approximates the size of the original minimal infeasible subgraph.
Our neural model will be trained to emulate the labels of the \textit{variable-feasibility} function \( \psi \).

We emphasize that learning the approximate function \( \tilde{\omega} \) is not a real limitation.
First, because the prediction will be integrated into an algorithm that can further reduce the size of the infeasible subgraph, if not already minimal, as shown later in \cref{sec:algorithm}.
Second, because finding small infeasible subgraphs, as opposed to strictly minimal, is already useful in many applications.
Finally, note that \( \omega \) could be transformed into a multiclass variable classification \( f(x_i;G) = r_i \subseteq \{1,\ldots,R\} \), where each variable may belong to multiple classes -- but this would require a complex and potentially intractable permutation invariant formulation.

\subsection{GNN with the Structure of a Factored-NLP}

A fundamental idea of our method is to use the structure of the Factored-NLP for message passing with Graph Neural Networks (GNN) to learn the \emph{variable-feasibility} \( \psi(x_i;G) \).

In neural message passing, each variable vertex \( x_i \in X_G \) has a feature vector \( z_i \in \mathbb{R}^{n_z} \) that is updated with the incoming messages of the neighboring constraints.
Each \( z_i \) is initialized with \( z_i^0 \) to encode semantic and continuous information of the variable \( x_i \) (an example of how to initialize the features in manipulation planning is shown in \cref{sec:encoding}).
The update rule follows a two-step process: first, each constraint computes and sends back a message to each neighboring variable, which depends on the current features of all the neighboring variables.
Second, each variable aggregates the information of the incoming messages from the constraints and updates its feature vector,
\begin{subequations}\label{eq:message_pass}
	\begin{align}
		[ \oplus \mu_{a \to i} ]_{i \in N(a)} = \text{Message}_{a}( [ \oplus z_{i} ]_{i \in N(a)}), \\
		z_i' = \text{Update} ( \text{AGG}_{a \in N(i)} ~ \mu_{a \to i} , z_i ),
	\end{align}
\end{subequations}
where \( \mu_{a\to i } \in \mathbb{R}^{n_{\mu}} \) is the message from constraint \( a \) to variable \( i \).
The operator \( [\oplus \sbullet ]_{i} \) denotes concatenation.
\( N(a) = \text{Neigh}_G(\phi_a) \) is the ordered set of variables connected to the constraint \( \phi_a \).
Conversely, \( N(i) = \text{Neigh}_G(x_i) \) is the set of constraints connected to variable \( x_i \).
AGG is an aggregation function, e.g., max, sum, mean, or weighted average.
We use max (element-wise) in our implementation.
A graphical representation is shown in \cref{fig:message_pass}.

\texttt{Update} and \( \texttt{Message}_a \) are small MLPs (Multilayer Perceptron) with learnable parameters.
As the nonlinear constraints in the Factored-NLP are not permutation invariant or symmetric, the features \( z_i \) must be concatenated in a predefined order \( N(a) \) when evaluating \( \texttt{Message}_a \).
The function \texttt{Update} is shared by all vertices (which generalizes to Factored-NLPs with additional variables).
The function \( \texttt{Message}_a \) is shared between different constraints of the Factored-NLP that represent the same mathematical function, i.e., \( \texttt{Message}_a = \texttt{Message}_b \) iff \( \phi_a(x) = \phi_b(x) \forall x \) (which generalizes to Factored-NLPs with additional constraints).
For example, in manipulation planning, all constraints that model collisions between objects will share the same \( \texttt{Message} \) MLP.

The message passing update \eqref{eq:message_pass} is performed \( K \) times, starting from the initial feature vectors \( z_i^0 \).
The feature vectors after \( K \) iterations are used for feasibility prediction with a small MLP classifier,
\begin{equation}
	\hat{y}_i = \text{Classifier}(z_i^K)\,.
	\label{eq:node_classif}
\end{equation}
The number of iterations is a hyperparameter of the model, and the weights of the MLPs may differ between message passing iterations \( k=1,\ldots,K \) (e.g., \( \texttt{Message}_a \) at \( k=1 \), denoted with
\( \texttt{Message}_a^1 \), is different from \( \texttt{Message}_a^k \) at iteration \( k \)).
The parameters of the classifier, message, and update networks
are trained end-to-end to minimize the weighted binary cross-entropy loss between the prediction \( \hat{y}_i \) and the \textit{variable-feasibility} labels \( y_i \).

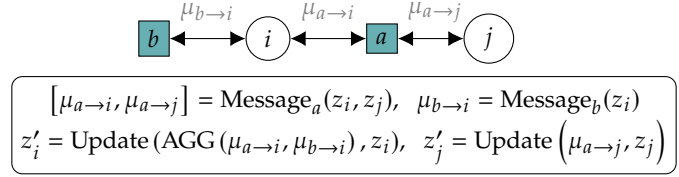
\begin{figure}
	\centering
	\begin{tikzpicture}[node distance=1.5cm,
			every node/.style={fill=white, font=\small}, align=center]
		\node (vari)     [gnn_var]          {$i$};
		\node (varj)     [right of=vari, xshift=4em,  gnn_var]          {$j$};

		\node (cona)     [right of = vari , gnn_con2]          {$a$};
		\node (conb)     [left of = vari,  gnn_con2]          {$b$};

		\draw[<->]             (cona) -- (vari);
		\draw[<->]              (cona) -- (varj) ;

		\draw[<->]             (conb) -- (vari);

		\node (mu1)     [right of = vari , xshift = -2em, yshift=1em,  sec]          {$\mu_{a \to i}$};
		\node (mu2)     [right of = vari , xshift = +2em, yshift=1em, sec]          {$\mu_{a\to j}$};

		\node (mu3)     [left of = vari , xshift = +2em, yshift=1em, sec]          {$\mu_{b\to i}$};

		\node (planner)     [process,below of=vari,yshift=1em, xshift=1cm]          {
			$\left[ \mu_{a \to i} , \mu_{a \to j } \right] = \text{Message}_a(  z_i , z_j  )$,~~
			$\mu_{b \to i}  = \text{Message}_b(  z_i   ) $  \\
			$z_i' = \text{Update}\left( \text{AGG}\left( \mu_{a\to i } , \mu_{b \to i } \right) , z_i \right) $,~~
			$z_j' = \text{Update}\left( \mu_{ a \to j}  , z_j \right) $
		};

	\end{tikzpicture}

	\caption{Message passing in a Factored-NLP with two variables ($i,j$) and two constraints $(a,b$).}
	\label{fig:message_pass}
\end{figure}

\begin{center}
	\begin{minipage}{.8\linewidth}
		\begin{algorithm}[H]
			\caption{Conflict Extraction with a Graph Neural Network.}
			\label{alg:overview}
			\begin{algorithmic}[1]
				\State \textbf{Input:}
        \State Factored-NLP $G=(X_G\cup \Phi_G, E_G)$ \Comment{{\small \color{gray} Infeasible factored nonlinear program}}
        \State \texttt{GNN\_Model} = \{$\texttt{Message}_a^k, \texttt{Update}^k, \texttt{Classifier}$\}\quad \Comment{{\small \color{gray} Learned GNN model}}
        \State \texttt{Solve} \Comment{{\small \color{gray} Algorithm provided by the user}}
      \State \texttt{Reduce} \Comment{{\small \color{gray} Algorithm provided by the user}}
    \State \textbf{Output:} $M \subseteq G$ \quad \Comment{{\small \color{gray} Minimal infeasible subgraph}}
				\State $\{\hat{y}_i\} = \texttt{GNN\_Model}(G)$
				\State $\delta \leftarrow 0.5, ~ \delta_r \leftarrow 1.2$
				\While{ \textbf{True}}
        \State $ X_{\delta} = \{ x_i \in X_G \mid \hat{y}_i < \delta\} \quad$ \Comment{{\small \color{gray} Candidate infeasible variables}}
        \For{$g \in \texttt{CCA}(G[X_{\delta}])$} \Comment{{\small \color{gray} Connected component analysis}}
				\State $ \texttt{feasible} \leftarrow \texttt{Solve}(g) $ \; \label{lbl:solve}
				\If{\textup{not \texttt{feasible}}}
				\State $ M \leftarrow \texttt{Reduce}(g) $ \;
				\State \textbf{Return} $M$
				\EndIf
				\EndFor

				\State $\delta \leftarrow \delta \times \delta_r \;$
				\EndWhile

			\end{algorithmic}

		\end{algorithm}
	\end{minipage}
\end{center}

\subsection{Algorithm to Detect Minimal Infeasible Subgraphs}
\label{sec:algorithm}

To account for the approximation in our variable classification formulation and small prediction errors, we integrate the learned classifier into a classical algorithm to detect minimal infeasible subgraphs.

We assume the user provides the $\texttt{Solve}$ and $\texttt{Reduce}$ routines, which check if a Factored-NLP is feasible and compute a minimal infeasible subset of constraints, respectively.
$\texttt{Reduce}$ is an expensive routine, as it involves solving several nonlinear programs by adding and removing constraints.
The number of evaluated NLPs—and therefore the computation time—depends on the size of the input graph: linear with the total number of variables according to \cite{amaldi1999some}, or logarithmic according to \cite{junker2004preferred}.

Our algorithm is outlined in \cref{alg:overview}.
The GNN model is evaluated once on the input Factored-NLP and computes feasibility scores, $\hat{y_i}$, for each variable.
By iteratively increasing the classification threshold, $\delta$, we select the candidate infeasible variables, $X_{\delta}$, with a score lower than the current threshold, $\delta$.
We then generate candidate infeasible subgraphs with connected component analysis on the \textit{variable-induced} subgraph, $G[X_{\delta}]$,
which are evaluated with $\texttt{Solve}$.
Once an infeasible subgraph is found, we use $\texttt{Reduce}$ to obtain a minimal infeasible subgraph.

A traditional conflict extraction approach would run $\texttt{Solve}$ and $\texttt{Reduce}$ directly on the input Factored-NLP.
The acceleration in our algorithm, therefore, comes from evaluating these routines on small (ideally minimal) candidates.
\cref{alg:overview} can be extended to compute multiple minimal infeasible subgraphs by omitting the return statement and adding a special check to avoid solving a supergraph of an infeasible subgraph identified in a previous iteration.

\section{Factored-NLP for Manipulation Planning}

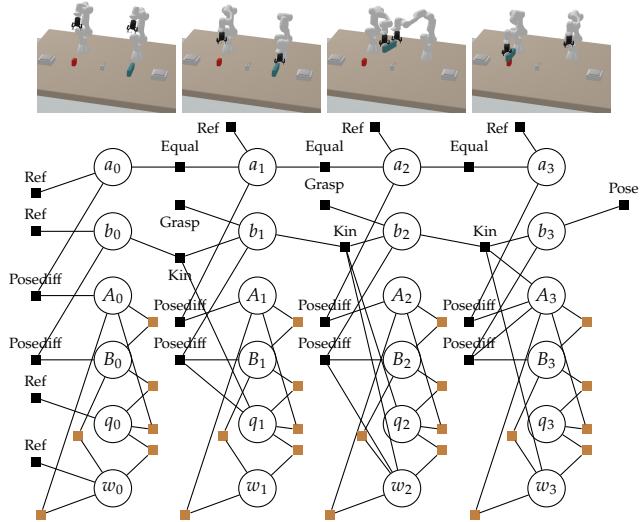
\begin{figure}[t]
	\centering
	\includegraphics[width=.13\textwidth]{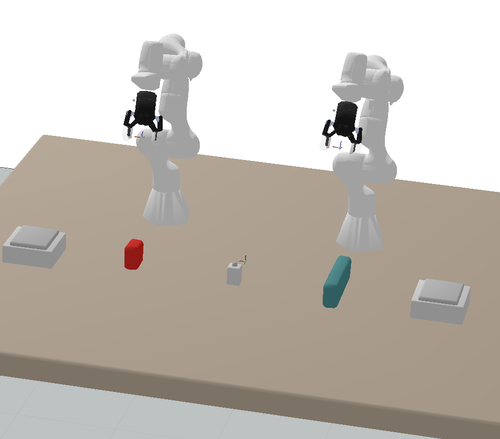}
	\includegraphics[width=.13\textwidth]{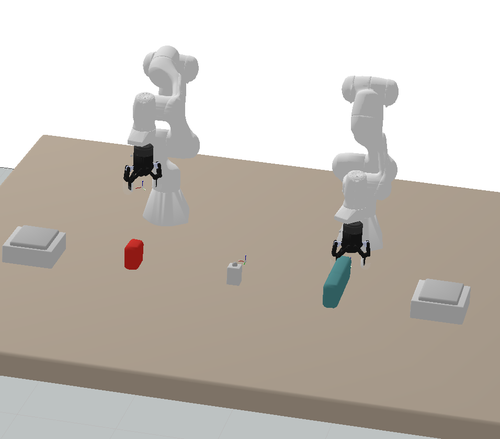}
	\includegraphics[width=.13\textwidth]{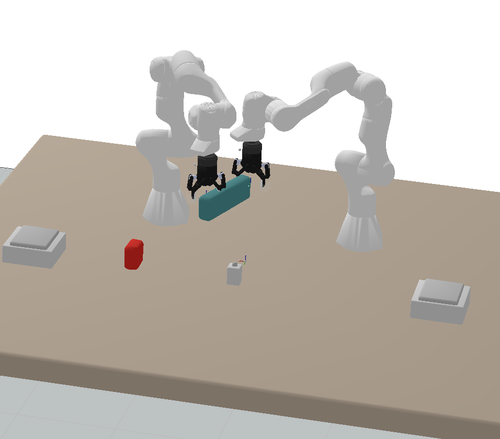}
	\includegraphics[width=.13\textwidth]{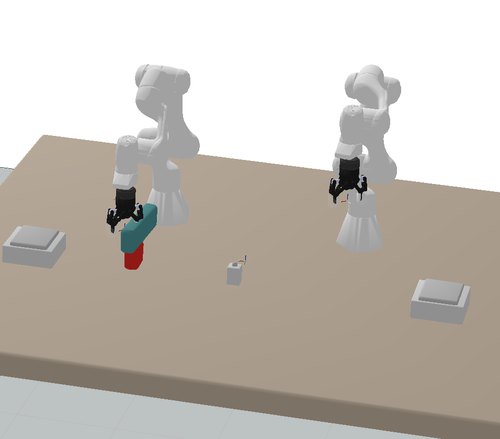} \\
	\begin{tikzpicture}[scale=0.7,every node/.style={transform shape}]

		\node[latent] (a0) {$a_0$} ;

		\node[latent,below=.5 of a0 ] (b0) {$b_0$} ;

		\node[latent, ,below=.5 of b0] (A0) {$A_0$} ;

		\node[latent,below=.5 of A0 ] (B0) {$B_0$} ;

		\node[latent,below=.5 of B0] (q0) {$q_0$} ;
		\node[latent,below=.5 of q0] (w0) {$w_0$} ;

		\node[latent,right=2 of a0 ] (a1) {$a_1$} ;
		\node[latent,below=.5 of a1 ] (b1) {$b_1$} ;

		\node[latent, ,below=.5 of b1] (A1) {$A_1$} ;

		\node[latent,below=.5 of A1 ] (B1) {$B_1$} ;

		\node[latent,below=.5 of B1] (q1) {$q_1$} ;
		\node[latent,below=.5 of q1] (w1) {$w_1$} ;

		\node[latent,right=2 of a1 ] (a2) {$a_2$} ;
		\node[latent,below=.5 of a2 ] (b2) {$b_2$} ;
		\node[latent,below=.5 of b2 ] (A2) {$A_2$} ;
		\node[latent,below=.5 of A2 ] (B2) {$B_2$} ;
		\node[latent,below=.5 of B2] (q2) {$q_2$} ;
		\node[latent,below=.5 of q2] (w2) {$w_2$} ;

		\node[latent,right=2 of a2 ] (a3) {$a_3$} ;
		\node[latent,below=.5 of a3 ] (b3) {$b_3$} ;

		\node[latent,below=.5 of b3 ] (A3) {$A_3$} ;
		\node[latent,below=.5 of A3 ] (B3) {$B_3$} ;

		\node[latent,below=.5 of B3] (q3) {$q_3$} ;
		\node[latent,below=.5 of q3] (w3) {$w_3$} ;

		\factor[left=1 of A0, yshift=0.0cm] {trajp0} { Posediff } {A0, a0} {};
		\factor[left=1 of B0, yshift=0.0cm] {trajp0} { Posediff } {B0, b0} {};
		\factor[left=1 of w0, yshift=0.5cm] {trajp0} { Ref } {w0} {};
		\factor[left=1 of q0, yshift=0.5cm] {trajp0} { Ref } {q0} {};
		\factor[left=1 of a0, yshift=-0.5cm] {trajp0} { Ref } {a0} {};
		\factor[above=.3 of a1,xshift=-.5cm] {trajp0} { left:Ref } {a1} {};
		\factor[above=.3 of a2, xshift=-.5cm] {trajp0} { left:Ref } {a2} {};

		\factor[left=1 of A1, yshift=-0.5cm] {trajp0} { Posediff } {A1, a1} {};

		\factor[left=1 of B1, yshift=0.0cm] {trajp0} { Posediff } {B1, b1, q1} {};

		\factor[left=1 of b0, yshift=0.0cm] {trajp0} { Ref } {b0} {};
		\factor[left=1 of b1, yshift=0.5cm] {trajp0} { below:Grasp } {b1} {};
		\factor[left=1 of b2, yshift=0.5cm] {trajp0} { Grasp } {b2} {};
		\factor[right=1 of b3, yshift=0.5cm] {trajp0} { Pose } {b3} {};
		\factor[above=.3 of a3,xshift=-.5cm] {trajp0} {left:Ref} {a3} {};

		\factor[left=1 of b1, yshift=-.5cm] {trajp0} {below:Kin} {b0, q1,b1} {};

		\factor[right=1.2 of b1, yshift=-.3cm] {trajp0} {Kin} {b1, q2,b2,w2} {};

		\factor[left=.7 of b3, yshift=-.3cm] {trajp0} {Kin} {w3,b3,A3,b2} {};

		\factor[left=1 of a3] {trajp0} {Equal} {a2,a3} {};

		\factor[left=1 of a1] {trajp0} {Equal} { a0, a1} {};

		\factor[left=1 of a2] {trajp0} {Equal} { a1, a2} {};

		\factor[left=1 of A2, yshift=-0.5cm] {trajp0} { Posediff } {A2, a2} {};
		\factor[left=1 of A3, yshift=-0.5cm] {trajp0} { Posediff } {A3, a3} {};
		\factor[left=1 of B2, yshift=0.0cm] {trajp0} { Posediff } {B2, b2, w2} {};
		\factor[left=1 of B3, yshift=0.0cm] {trajp0} { Posediff } {B3, b3, A3} {};

		\factor[right=.3 of A0, yshift=-0.5cm,color=brown] {} {} {A0,B0} {};
		\factor[right=.3 of B0, yshift=-0.5cm,color=brown] {} {} {B0,q0} {};
		\factor[right=.3 of q0, yshift=-0.1cm,color=brown] {} {} {A0,q0} {};

		\factor[right=.3 of q0, yshift=-0.5cm,color=brown] {} {} {q0,w0} {};
		\factor[left=.9 of w0, yshift=-0.5cm,color=brown] {} {} {A0,w0} {};
		\factor[left=.2 of w0, yshift=+1.0cm,color=brown] {} {} {w0,B0} {};

		\factor[right=.3 of A1, yshift=-0.5cm,color=brown] {} {} {A1,B1} {};
		\factor[right=.3 of B1, yshift=-0.5cm,color=brown] {} {} {B1,q1} {};
		\factor[right=.3 of q1, yshift=-0.1cm,color=brown] {} {} {A1,q1} {};

		\factor[right=.3 of q1, yshift=-0.5cm,color=brown] {} {} {q1,w1} {};
		\factor[left=.9 of w1, yshift=-0.5cm,color=brown] {} {} {A1,w1} {};
		\factor[left=.2 of w1, yshift=+1cm,color=brown] {} {} {w1,B1} {};

		\factor[right=.3 of A2, yshift=-0.5cm,color=brown] {} {} {A2,B2} {};
		\factor[right=.3 of B2, yshift=-0.5cm,color=brown] {} {} {B2,q2} {};
		\factor[right=.3 of q2, yshift=-0.1cm,color=brown] {} {} {A2,q2} {};

		\factor[right=.3 of q2, yshift=-0.5cm,color=brown] {} {} {q2,w2} {};
		\factor[left=.9 of w2, yshift=-0.5cm,color=brown] {} {} {A2,w2} {};
		\factor[left=.3 of w2, yshift=+1cm,color=brown] {} {} {w2,B2} {};

		\factor[right=.3 of A3, yshift=-0.5cm,color=brown] {} {} {A3,B3} {};
		\factor[right=.3 of q3, yshift=-0.1cm,color=brown] {} {} {A3,q3} {};

		\factor[right=.3 of q3, yshift=-0.5cm,color=brown] {} {} {q3,w3} {};
		\factor[left=.9 of w3, yshift=-0.5cm,color=brown] {} {} {A3,w3} {};
		\factor[left=.2 of w3, yshift=+1cm,color=brown] {} {} {w3,B3} {};
		\factor[right=.3 of B3, yshift=-0.5cm,color=brown] {} {} {B3,q3} {};

	\end{tikzpicture}
	\caption{Factored-NLP for the task plan $\langle$\textit{pick object B with robot Q from B\_init}, \textit{pick object B with robot W from robot Q}, \textit{place object B with robot W on object A}$\rangle$.
		Circles represent variables, and squares represent constraints.
		Each column symbolizes a keyframe of the manipulation sequence.
		$q,w$ are the configurations of the two robots; $A,B$ are the absolute positions of the two objects, and $a,b$ are the relative poses of these objects with respect to their parent in the kinematic tree (e.g., the table, a robot, or another object as indicated by the task plan).
		See the main text and \cref{sec:bg:structure} for an explanation of variables and constraints.
	}
	\label{fig:cg_example2}
\end{figure}

\subsection{Structure of the Factored-NLP}

As an application within TAMP, we use our model to predict minimal infeasibility when computing the keyframe configurations that fulfill a high-level task plan.

When the optimization problem is infeasible, finding a minimal subset of infeasible constraints is crucial for understanding the cause of the infeasibility and providing valuable feedback to the task planner, as demonstrated in our conflict-based TAMP planner (\cref{ch:bid}).

The Factored-NLPs used in this chapter are generated using the formulation Planning with Nonlinear Transition Constraints (PNTC,
\cref{sec:planner:formulation}) presented in \cref{ch:bid}, which ensures a consistent local and repeatable structure to enable generalization across different nonlinear programs.

However, we employ a different, yet equivalent, formulation of the continuous space within PNTC, using two continuous variables for each movable object: one indicating the absolute position and another for the relative, along with additional constraints.

This makes the Factored-NLPs more redundant, as the absolute positions of the objects can be deduced from their relative positions and the positions of the parent frame.
Nevertheless, now the Factored-NLP can be formulated using a smaller number of distinct types of nonlinear constraints.
Since each type of constraint corresponds to a unique $\texttt{Message}$ network, this formulation becomes vital for generalization in scenes with more objects.

Thus, Factored-NLPs in this chapter contain three types of variables: robot configurations, object absolute positions, and object relative positions with respect to the parent frame.
Beyond the nonlinear constraints highlighted throughout the thesis, we now incorporate a new constraint type that ensures the geometric consistency between relative and absolute poses of objects (\textit{Posediff}).

In 
\cref{fig:cg_example2}, we display the Factored-NLP corresponding to the sequence
$\langle$\textit{pick object B with robot Q from B\_init}, \textit{pick object B with robot W from robot Q}, \textit{place object B with robot W on object A}$\rangle$,
in an environment with two robots, $Q$ and $W$, and two objects, $A$ and $B$.
For comparison, the Factored-NLP using the original PNTC formulation for the same task plan is presented in \cref{fig:graph-in-planner-q}.

Lastly, it is worth noting that in this chapter, we do not consider trajectory variables (e.g., $\tau_q$ in \cref{fig:graph-in-planner-q} in \cref{ch:bid}) because the keyframe variables already provide very informative information for evaluating geometric infeasibility.

\subsection{Encoding of the Problem in the Initial Feature Vectors}\label{sec:encoding}

The structure of the Factored-NLP encodes the number of objects, robots, and the task plan.
The geometric description of the environment is encoded locally in the initial feature vector of each variable $z_i^0$.
Specifically, the initial feature vector includes the information of unary constraints (i.e., constraints evaluated only on a single variable, which are then not added to the message passing architecture), additional semantic class information (for example, whether the variable represents an object or a robot, but without including a notion of a time index or entity), and geometric information that is relevant for the constraints (for example, the size of the objects).
The dimension of $z_i^0$ is fixed, and shorter feature vectors are padded with zeros.

For example, suppose that the Factored-NLP of 
\cref{fig:cg_example2} is evaluated in a scene where robot $Q$ is at pose { \small $T_Q = [0.32 , \,0.41 , \,0.56, \,0.707 , \, 0 ,\, 0 ,\, 0.707 ] $ }, the start position of object $A$ is { \small $T_A = [ 0.35, \, 0.4, \, 0.5 , \, 0.707 , \,  0 , \, 0 , \, 0.707  ]$ }, and object $A$ is a box of size { \small $S_A = [0.2 , \, 0.3 , \, 0.2 ]$}.
Then the $z^0$ of variables $\{q_0,q_1,q_2,q_3\}$ is { \small $[ 1,\, 0,\,0,\,0,\,0,\,0 , T_Q ]$ }, where the first six components indicate that it is a robot, and $T_Q$ is the base pose.
The $z^0$ of $\{a_0,a_1,a_2,a_3\}$ is {\small $ [0, \, 1,\, 0,\, 0, \, 0, \,0 , \, T_A ]$ }, where the first components indicate that it is a relative pose with respect to the reference position $T_A$.
The $z^0$ of $\{A_0,A_1,A_2,A_3\}$ is { \small
		$[ 0, \, 0, \, 1, \, 0,\, 0, \, 0 , \, S_A ,\, 0 ,\, 0 , \, 0, \, 0 ]$ } to indicate that it is an absolute position of an object of size $S_A$.

\section{Experimental Results}

\begin{figure}
	\setlength{\tabcolsep}{0.2em} %
	\centering
	\begin{tabular}{ccc}
		\includegraphics[width=.25\columnwidth]{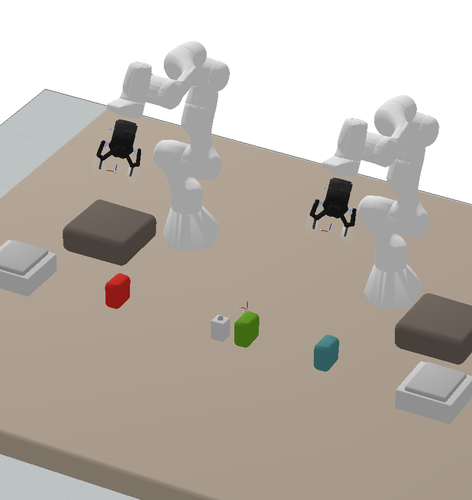} &
		\includegraphics[width=.25\columnwidth]{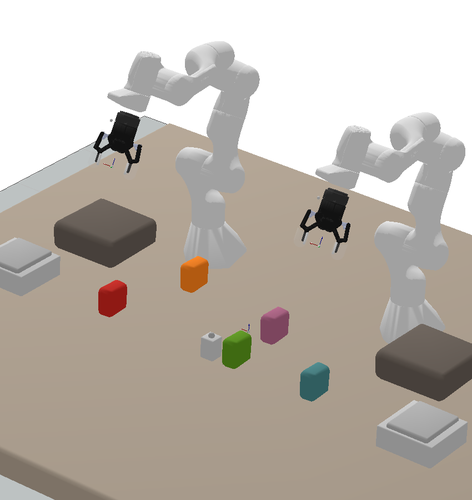}  &
		\includegraphics[width=.25\columnwidth]{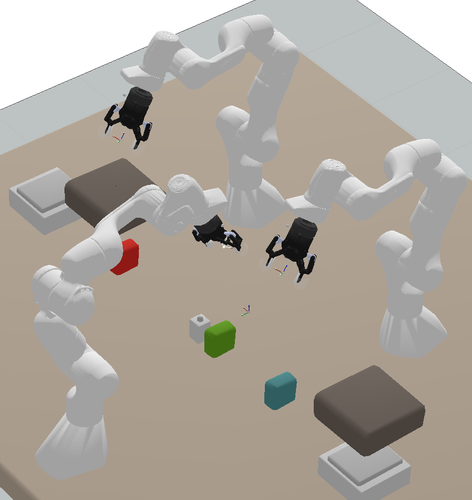}
	\end{tabular}
	\caption{TAMP scenarios.
		Obstacles are brown, blocks are colorful and tables are white.
		\textit{Left}: Training Data, \textit{Middle}: + Blocks dataset, \textit{Right:} + Robots dataset.\vspace{.5cm}}
	\label{fig:view_data}
\end{figure}

\begin{figure}
	\centering
	\includegraphics[width=.19\textwidth]{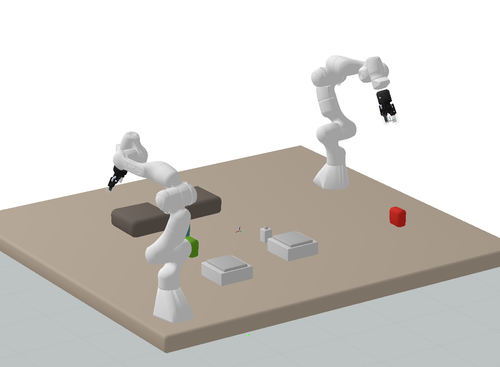}
	\includegraphics[width=.19\textwidth]{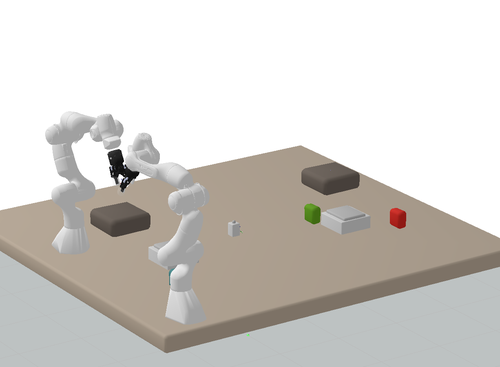}
	\includegraphics[width=.19\textwidth]{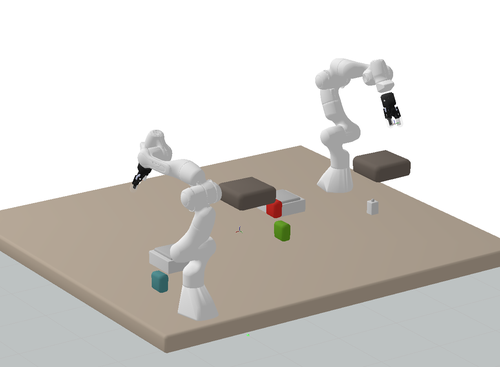}
	\includegraphics[width=.19\textwidth]{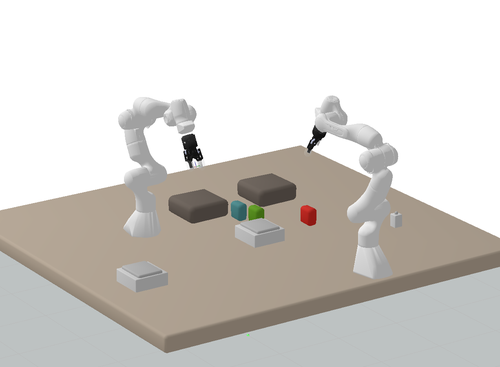}
	\includegraphics[width=.19\textwidth]{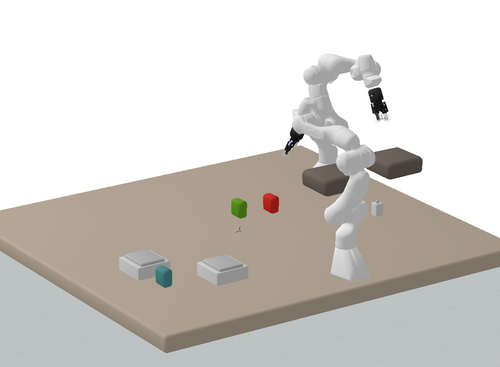}
	\caption{Training Data -- Different Scenes.
		The positions of robots, obstacles, and blocks are randomized.
	}
\end{figure}

\paragraph{Scenario}

We evaluate our model in robotic sequential manipulation.
The objective is to predict minimal infeasibility of task plans that build towers and rearrange blocks into different configurations, in scenarios containing a varying number of blocks, robots, and movable obstacles, in different positions.
See 
\cref{fig:view_data,fig:view_skeleton}.
The following settings are used to generate the training dataset (\SI{4800}{Factored-NLPs}):
\begin{enumerate}
	\item[--] Five movable objects: 3 blocks and 2 obstacles.
		Both types of objects have collision constraints, but obstacles are larger and usually block grasps or placements.
	\item[--] Two robots: 7-DOF Panda robot arms, which can pick and place objects using a top grasp.
	\item[--] Different geometric scenes: the positions of the objects, robots, and tables are randomized.
	\item[--] Manipulation sequences of lengths 4 to 7 (length of the high-level task plan).
\end{enumerate}

To evaluate the generalization capabilities of the learned model, we consider three additional datasets:
\begin{enumerate}
	\item[--] \texttt{+Robots}: we add an additional robot.
	\item[--] \texttt{+Blocks}: we add two additional blocks.
	\item[--] \texttt{+Actions}: this dataset contains Factored-NLPs from longer task plans (lengths of 8 to 10).
\end{enumerate}

\subsection{Data Generation}

For training the GNN model, we need a set of Factored-NLPs with labeled variables to indicate whether they belong to a minimal infeasible subset.
First, we generate a set of interesting task plans.
Second, we evaluate the manipulation sequences on random geometric scenes.
To compute the feasibility labels, we adapt the conflict extraction algorithm of \cref{ch:bid} to find up to 10 minimal infeasible subgraphs.

\subsection{Accuracy of the GNN Classifier}

We compare our model (\emph{GNN}) against a Multilayer Perceptron (\emph{MLP}) and a sequential model (\emph{MLP-SEQ}), trained with the same dataset.

The \emph{MLP} computes $\hat{y}_i = \text{MLP}( \tilde{z}^0_i, A, C)$, where $\tilde{z}_i^0 = [ z_i^0 , t_i , e_i]$ is the feature vector of the variable we want to classify.
It concatenates the feature vector $z_i^0$ used in the \emph{GNN}, with the time index $t_i$ of the variable, and a parametrization that defines the \textit{name} of the variable $e_i$ (for instance, we represent an object with its starting pose).
Note that $t_i$ and $e_i$ are not used in the \emph{GNN} model because this information is encoded in the structure of the graph.
$A$ is the encoding of the whole task plan, using small vectors to encode each token, e.g., \{\qquote{pick}, \qquote{block1}, \qquote{l\_gripper}, \qquote{table}\}.
To account for sequences of different lengths, we fix a maximum length and add padding.
$C$ is the scene parametrization and contains the position and shapes of all possible objects and robots.

We also evaluate \emph{MLP-SEQ}, a sequential model $\text{MLP}( \tilde{z}^0_i, \text{SEQ} (A), C)$ that encodes the action sequence with a recurrent network (Gated Recurrent Units).

We first evaluate the accuracy of the models to predict if a variable belongs to a minimal infeasible subset, see \cref{tab-theaccuracy}.
Our \emph{GNN} model outperforms the alternative architectures, both in the original \textit{Train Data} and, especially, in the extension datasets.
Our model maintains a constant $\sim$\SI{95}{\percent} success rate across all datasets, while the performance of \emph{MLP} and \emph{MLP-SEQ} drops to \SI{48}{\percent} and \SI{75}{\percent}, respectively.
We also evaluate the accuracy of our model to predict infeasible subgraphs, using the proposed method that combines variable classification and connected component analysis, with the initial threshold for classification set at $\delta = 0.5$.
Our model outperforms \emph{MLP} and \emph{MLP-SEQ}, finding between \SI{70}{\percent} and \SI{57}{\percent} of the infeasible subgraphs, and \SI{30}{\percent}-\SI{50}{\percent} of the predicted subgraphs are minimal, see \cref{tab-theaccuracygraph}.
Between \SI{34}{\percent}-\SI{48}{\percent} of the predicted infeasible graphs are actually feasible.
As shown later, these levels of accuracy, together with our iterative threshold strategy, result in a strong acceleration.

\begin{table}[t]
	\small
	\caption{Classification accuracy.
		Each pair indicates the accuracy of predicting feasible and infeasible variables.
	}
	\label{tab-theaccuracy}
	\centering
	\renewcommand{\arraystretch}{1.4} %
	\setlength{\tabcolsep}{0.4em} %
	\begin{tabular}{  l c c c c   } \toprule
		               & Train Data                & + Blocks                & + Robots                 & + Actions               \\
		\midrule
		\emph{GNN}     & $(94.7, ~95.4)  $         & $(    96.1, ~95.2  )$   & $(    95.7, ~95.3   )$   & $(    94.6, ~94.1    )$ \\
		\emph{MLP}     & $(    93.0,~82.2       )$ & $(   93.4,~80.8     )$  & $(   93.0,~80.8       )$ & $(   91.0,~48.0    )$   \\
		\emph{MLP-SEQ} & $(   83.5,~88.1    )$     & $(    82.3,~88.8     )$ & $(  82.1,~88.8   )$      & $(   74.0,~75.3      )$ \\
		\bottomrule
	\end{tabular}
\end{table}

\begin{table}
	\small
	\caption{Prediction of infeasible subgraphs.
		Each pair indicates the ratio ``found / total'' (higher is better) and ``minimal / found'' (higher is better).
	}
	\label{tab-theaccuracygraph}
	\centering
	\renewcommand{\arraystretch}{1.4}
	\setlength{\tabcolsep}{0.4em}
	\begin{tabular}{  l c c c c   } \toprule
		               & Train Data        & + Blocks         & + Robots            & + Actions           \\
		\midrule
		\emph{GNN}     & $ (71.2,~54.1)  $ & $ (58.9,~33.3) $ & $  (~70.2,~55.3)  $ & $  (~57.1,~41.9)  $ \\
		\emph{MLP}     & $(58.5,~54.6)$    & $(34.5,~53.2)$   & $(55.2,~37.6)$      & $(22.1,~35.5)$      \\
		\emph{MLP-SEQ} & $(65.7,~26.0)$    & $(28.6,~21.2)$   & $(61.3,~09.5)$      & $(36.3,~11.0)$      \\
		\bottomrule
	\end{tabular}
\end{table}

\emph{MLP}, \emph{MLP-SEQ}, and \emph{GNN} have the same information to make the predictions because the Factored-NLP is a deterministic mapping of the action sequence and the geometric scene.
Although the unstructured \emph{MLP} and \emph{MLP-SEQ} baselines could potentially learn this mapping, our experiments show that the representation does not emerge naturally, confirming that a structured model yields better generalization.

\subsection{Finding Minimal Infeasible Subgraphs}

\begin{table}
\small
\caption{Finding one minimal infeasible subgraph, evaluated on
	100 different Factored-NLPs.
	Each pair indicates the average number of evaluated NLPs (lower is better) and the compute time (lower is better), normalized by the results of \emph{GNN+g1}.
}
\label{tab:infeas_subgraph}
\begin{center}
\renewcommand{\arraystretch}{1.4} %
\setlength{\tabcolsep}{0.4em} %
\begin{tabular}{  l c c c c   } \toprule
	                 & Train Data     & +   Blocks     & + Robots       & + Actions      \\
	\midrule
	\emph{GNN+e}     & $(1.57,~2.25)$ & $(1.44,~2.09)$ & $(1.66,~2.14)$ & $(1.50,~2.19)$ \\
	\emph{GNN+g1}    & $(1,~1)$       & $(1,~1)$       & $(1,~1)$       & $(1,~1)$       \\
	\emph{Oracle}    & $(0.83,~0.97)$ & $(0.62,~0.79)$ & $(0.83,~0.84)$ & $(0.71,~0.86)$ \\
	\emph{Expert}    & $(3.66,~4.32)$ & $(3.13,~5.06)$ & $(4.33,~4.62)$ & $(3.33,~4.56)$ \\
	\emph{General 2} & $(3.50,~64.1)$ &
	$(3.30,~163)$    & $(3.50,~66.5)$ & $(3.83,~128)$                                    \\
	\bottomrule
\end{tabular}
\end{center}
\end{table}

We analyze the time required to find one minimal infeasible subgraph in an infeasible Factored-NLP with the following algorithms:

\begin{itemize}

	\item \textit{Oracle}, which knows beforehand the minimal infeasible subgraph and executes only a single call to $\texttt{Solve}$ and $\texttt{Reduce}$ on this minimal infeasible subgraph.
	      This provides a lower bound on the compute time.

	\item \textit{General \{1,2\}}, which are generic algorithms for conflict extraction: \emph{General 1} uses constraint filtering \cite{amaldi1999some}, and \emph{General 2} uses \textit{QuickXplain} \cite{junker2004preferred}.

	\item \textit{Expert} is the heuristic algorithm for conflict extraction in manipulation planning presented in \cref{ch:bid}.

	      It exploits the temporal structure, domain relaxations, and the convergence of the optimizer to quickly discover the conflicts.

	\item \textit{GNN+\{e,g1\}} combines the prediction of our \emph{GNN} model with either \emph{Expert} or \emph{General 1},
	      which are used as the \texttt{Reduce} routine in \cref{alg:overview}.

\end{itemize}

Results are shown in \cref{tab:infeas_subgraph}.
\emph{GNN+g1} is \num{60}-\num{120}x faster than \emph{General 2} (which is faster than \emph{General 1}).
This highlights the benefits of our approach in domains where we can compute a dataset using \textit{General} offline and train the model to get an order-of-magnitude improvement in new problems.
\emph{GNN+g1} is \num{4}-\num{5}x faster than the \textit{Expert} algorithm and only \num{1.2}x slower than an oracle.
Moreover, the acceleration provided by \textit{GNN} is maintained in all the datasets.
This confirms the good accuracy and generalization of the architecture seen in the classification results.
As a side note,
\textit{Expert} is faster than \textit{General 2} because it solves many small feasible NLPs first until it finds one that is infeasible (which is faster than solving infeasible NLPs).

\begin{figure}

	\setlength{\tabcolsep}{0.2cm} %
	\centering
	\begin{tabular}{ccc}
		\includegraphics[width=.22\columnwidth]{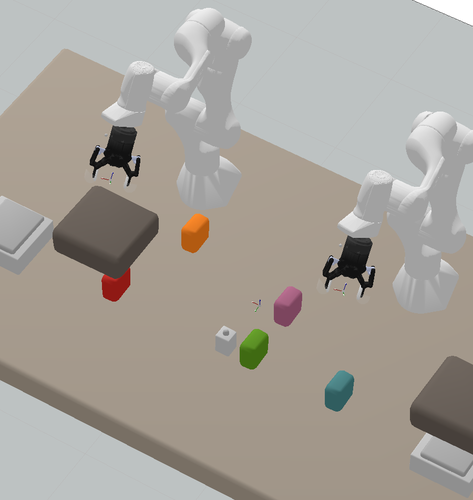} &
		\includegraphics[width=.22\columnwidth]{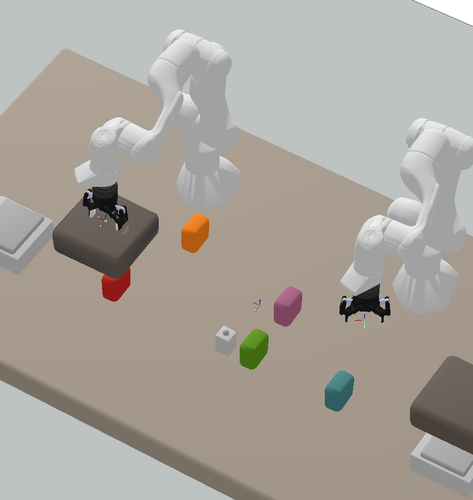} &
		\includegraphics[width=.22\columnwidth]{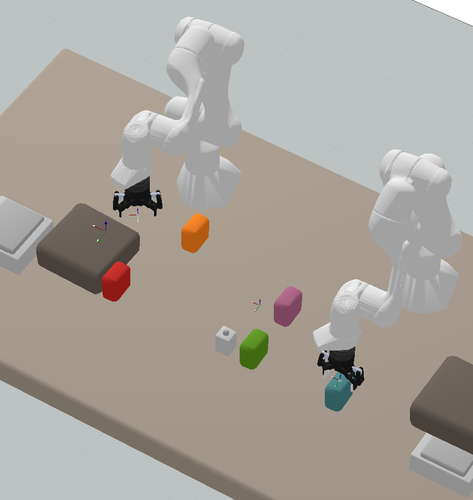}   \\
		\includegraphics[width=.22\columnwidth]{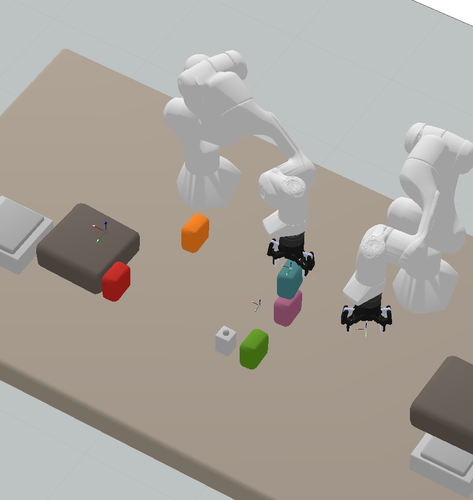} &
		\includegraphics[width=.22\columnwidth]{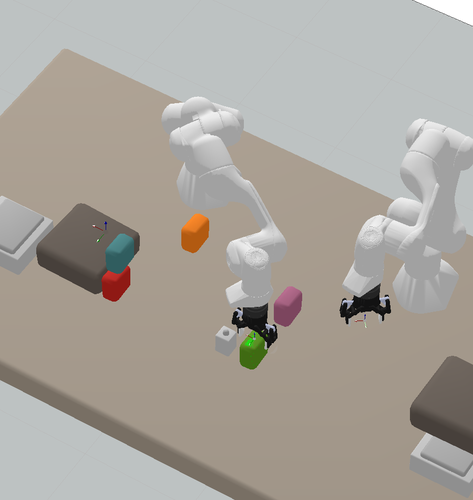} &
		\includegraphics[width=.22\columnwidth]{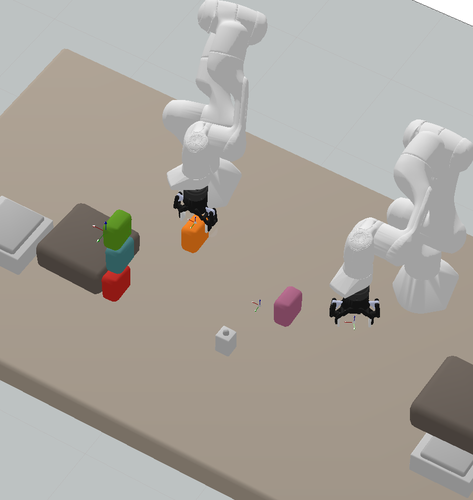}
	\end{tabular}
	\caption{Keyframes of a task plan in the evaluation dataset \textit{+ Actions}.
		Robots build a tower of blocks [\textit{red}, \textit{blue}, \textit{green}, \textit{orange}], moving first an obstacle.
	}
	\label{fig:view_skeleton}
\end{figure}

\subsection{Integration in a Conflict-Based TAMP Planner}

We demonstrate the benefits of using deep learning to accelerate conflict extraction inside the Factored-NLP Planner (\cref{ch:bid}).
Our planner iteratively generates candidate task plans, detects infeasible subgraphs of the Factored-NLP, and encodes this information back into the discrete description of the problem.
For this evaluation, we define \num{10} TAMP problems for a reference environment in each setting: \texttt{+Actions}, \texttt{+Robots}, and \texttt{+Blocks}.
We report the sum (across all \num{10} TAMP problems) of the number of solved NLPs (lower is better) and the computational time in the conflict extraction component of the TAMP solver.
\textit{GNN+e} (which is more robust than \textit{GNN+g1} in this setting) takes only (\SI{8.33}{s}, \SI{511}{NLPs}) in \texttt{+Actions}, (\SI{9.83}{s}, \SI{603}{NLPs}) in \texttt{+Robots}, and (\SI{63.9}{s}, \SI{1979}{NLPs}) in \texttt{+Blocks}, and is between \si{2} and \si{3} times faster than the \textit{expert} algorithm, which requires (\SI{24.2}{s}, \SI{731}{NLPs}), (\SI{38.7}{s}, \SI{1116}{NLPs}), and (\SI{137.9}{s}, \SI{2554}{NLPs}).

\section{Limitations}

From a practical perspective, the limitation of our approach is that the neural model primarily serves as a component within an extensive model-based pipeline for TAMP.
Conflict extraction alone is not capable of generating robot motion or suggesting the next potential task plan.

Furthermore, our graph-based classifier, despite its potential for accurate predictions, requires integration into a traditional
conflict extraction algorithm to refine the identified conflicts.
Looking ahead, future research might branch out in three distinct and potentially conflicting directions: improving the neural models to achieve impeccable predictions, adapting the TAMP solver to handle approximate predictions, or progressively transitioning from model-based TAMP solvers with integrated learning components to fully learning-based TAMP solvers.

\section{Conclusion}

In this chapter, we have presented a neural model to predict the minimal infeasible subsets of variables and constraints in a factored nonlinear program.
The structure of the nonlinear program is used for neural message passing, providing generalization to problems with more variables and constraints.

Our model achieves high accuracy, and the predictions can be integrated to guide and accelerate classical and heuristic algorithms for detecting minimal conflicts.
As confirmed in the experiments, the neural model is directly applicable as a submodule of the Factored-NLP Planner (\cref{ch:bid}) to accelerate the conflict detection pipeline.

When compared to Deep Generative Constraint Sampling (\cref{ch:gans}), these graph neural models exhibit superior generalization across varied task plans.
As such, a new research question is how to extend these ideas and insights to the generative setting.

Similar to foundational models in computer vision or natural language processing, we see great potential in investigating general models for manipulation planning, where these models can serve as backbones for different downstream tasks in any TAMP solver.
In this setting, graph neural networks are an attractive architecture to combine logic and geometric information in an end-to-end, yet interpretable, manner.
The structure of the graph is used to model discrete information such as the task plan and the number of objects and robots, and the feature vectors of nodes and edges are used to encode geometric information.

%% file: conclusion.tex
\makeatletter
\def\toclevel@chapter{-1}
\makeatother

\chapter{Conclusions}
\label{ch:conclusion}

We conclude this thesis by summarizing the main contributions and results of this work and by discussing some possible future directions and open challenges.

\section{Summary of Contributions}

\paragraph{Factored structure of Task and Motion Planning}

In \cref{sec:bg:structure}, we present a refined factored representation of optimization problems within Task and Motion Planning.
This factorization naturally arises from the temporal, object-centric, and robot-centric representations of the problem, but it had not been studied in detail in previous optimization-based TAMP solvers.

The formal definition, properties, and methods to generate these factored nonlinear programs are presented in \cref{ch:bid}.
During the development of this thesis, we demonstrate the advantages and generality of this representation, illustrating how it can represent a wide range of diverse task plans in a unified manner, including problems involving multiple robots, objects, and diverse task plans.

Unintentionally, our graph representation has come to closely resemble the factored representation of PDDLStream \cite{garrett2020pddlstream}.
Exploiting this factorization is crucial for designing efficient interfaces between continuous optimization and task planning.
From a sampling perspective, we can design good sampling operations that incrementally compute the motion, and from an optimization perspective, we can detect infeasible subsets of constraints and encode this information back into the task planner (\cref{ch:bid}).

We strongly believe that the factored representation of TAMP can be used to bridge the gap between optimization and sample-based approaches to TAMP (\cref{ch:mcts,ch:meta-solver}), and to improve the generalization and efficiency of deep learning methods for TAMP (\cref{ch:gans,ch:learn-feas}).

\paragraph{\namePartOneSentence}

The objective of the first part of the thesis is to combine trajectory optimization and discrete task planning into a unified and efficient framework for solving TAMP problems.
A fundamental challenge in solving large-scale problems with hard geometric and physical constraints is to automatically inform the task planner about motion feasibility.

In \cref{ch:diverse_planning,ch:bid}, we present two new conflict-based solvers for TAMP that attempt to solve the TAMP problem by evaluating candidate task plans, identifying why plans fail when considering the continuous constraints, and encoding this information back into the discrete planning problem.

Our first new solver, \emph{Diverse Planning for LGP} (\cref{ch:diverse_planning}), detects and encodes infeasible prefixes of the task plan.
Additionally, ideas from diverse planning and meta-reasoning are used to choose the most promising plans to test next and to decide how much compute effort to allocate to searching for conflicts.

In \emph{\nameChapterTwo} (\cref{ch:bid}), we introduce a new factored hybrid planning formulation for TAMP, which provides a more efficient interface between discrete planning and optimization.
Based on this formulation, our second solver, the \textit{Factored-NLP Planner}, can now detect any infeasible subset of nonlinear constraints and encode this information back into the task planner, resulting in a highly efficient bidirectional interface between task planning and motion planning.

\paragraph{\namePartMetaSolverSentence}

In the second part of the thesis, we propose two meta-solvers: algorithms that can automatically combine two different techniques—optimization and sampling—to compute the robot motion in TAMP problems.

In \emph{\nameChapterThree} (\cref{ch:mcts}), we begin by considering the problem of computing the motion for a fixed task plan, which corresponds to solving challenging factored nonlinear programs.
Using a Monte Carlo Tree Search formulation, our method automatically discovers the best sequence of sampling and/or optimization operations to solve the problem.
This adaptive algorithm outperforms both fixed sampling sequences and full nonlinear optimization used in previous work, in terms of the diversity and number of solutions found within a fixed amount of compute time.

In \emph{\nameChapterFour} (\cref{ch:meta-solver}), we present a preliminary study towards a complete TAMP meta-solver that optimizes the task plan and the robot motion while automatically deciding whether it is better to use sequential sampling or joint optimization.
To this end, we first introduce a notion of a compute state that extends the traditional discrete-continuous states with additional free states subject to constraints that have not yet been computed.

Our first TAMP meta-solver is an informed search algorithm on this computational space.
We show that this simple search strategy can outperform sample-based and optimization-based solvers on average compute time across diverse TAMP settings with few objects and robots.
However, further research is needed to enhance the TAMP meta-solver's performance for tackling large-scale problems.

\paragraph{\namePartLearningSentence}

In the third part of the thesis, we propose two novel ways to use learning to accelerate expensive operations in a TAMP solver.
We assume that a dataset of solutions to similar problems can be computed offline using a model-based solver, which requires multiple expensive computations, such as solving nonlinear programs with multiple restarts or combinatorial optimization.
At runtime, the learned models are used to accelerate our model-based solvers, providing speed-ups on new, unseen problems.
In both contributions, exploiting the Factored-NLP representation of the problem is crucial to achieving good generalization and performance.

\textit{Deep Generative Constraint Sampling} (\cref{ch:gans}) combines a deep learning model with nonlinear optimization to generate keyframes of manipulation sequences faster.
In particular, we use Generative Adversarial Networks to produce a good warm start for nonlinear optimization, outperforming alternative warm start initialization strategies.
Here, we transform the Factored-NLP into a directed graphical model, to reduce sample complexity and increase the expressivity and multimodality of the generative model.

In \textit{\nameChapterSix} (\cref{ch:learn-feas}), we propose a classifier that predicts which constraints in a Factored-NLP cannot be fulfilled, a fundamental step in our conflict-based TAMP solver (\cref{ch:bid}).
The model takes the structure of the Factored-NLP as direct input, together with a local encoding of the manipulation scene, providing good generalization across different task plans.
Using this graph-based classifier, we can detect conflicts an order of magnitude faster than classical conflict extraction strategies.

\section{Open Challenges and Future Work}

In this section, we discuss open challenges and outline several promising ideas for further research in the field of Task and Motion Planning in robotics.

\paragraph{TAMP benchmarks}

Different TAMP methods utilize slightly varied formulations and various benchmark problems, making it challenging to compare and analyze algorithms, utilize common tools, and draw on ideas from other research groups.
Additionally, similar problem formulations to TAMP are often studied under different names in robotics, such as multimodal motion planning and manipulation planning, again with slightly varied formulations and heuristics tailored to different environments.
There is a clear need for standardized benchmarks.
Drawing inspiration from the discrete planning and reinforcement learning communities, it is evident that we need a unified way to define the problem (e.g., PDDL \cite{mcdermott1998pddl} or the OpenAI-Gym interface \cite{brockman2016openai}), and a set of different standardized scenarios (e.g., domains in PDDL and benchmarks or datasets in RL) to evaluate our methods.

While there have been some attempts to establish TAMP benchmarks \cite{lagriffoul2018platform}, the utilization of different software tools and slightly varied problem formulations has hindered broad adoption.
A successful TAMP benchmark should define a common interface for stating a problem and its solution, a simulator to validate any provided solution, and interfaces that could optionally be utilized in the solvers, such as discrete abstractions and differentiable constraints.

Moreover, our research experience underscores the need to develop superior motion planning and optimization tools.
Although certain algorithmic ideas, such as sample-based motion planning and trajectory optimization, have reached a mature state, the research community still requires high-quality, standalone, and open-source implementations.

\paragraph{Optimization-based solvers for TAMP}

In this thesis, we have presented two TAMP solvers that combine discrete planning and optimization with a conflict-based approach, effectively scaling to large-scale problems involving multiple robots and objects.

However, a fundamental issue with optimization-based approaches remains, namely, convergence to local minima.
Because our solvers are conflict-based, a failure to find a solution on a feasible problem due to an unfortunate initialization compromises the completeness of our approach.

In practice, our solvers have performed very well in tabletop environments and with simple geometric shapes, where there are few local optima in the trajectory optimization problems.
This assumption holds true for many relevant manipulation problems.
However, for more general applications, we need to address the issue of local minima, making our solvers more robust against failed optimization attempts and allowing them to try the same problem again with a different initialization.
As discussed in the limitations of our solvers, a promising future direction is to use soft or probabilistic conflict formulations.

\paragraph{Nonlinear optimization in robotics}

Nonlinear optimization methods are a powerful tool in robotics and have shown great success in solving robotics problems in high-dimensional spaces with complex constraints, e.g., \cite{winkler2018gait,mordatch2012discovery,toussaint2018differentiable}.
Unfortunately, these results are often difficult to reproduce for non-experts, as they require technical knowledge and experience to formulate the problem correctly (e.g., the selection of variables and constraints, the scaling of each term, and the warm start).

In practice, we observe that sequential conditional sampling and sample-based motion planning are more robust to the choice of hyperparameters and the exact problem formulation.
In contrast to optimization, where bad hyperparameters often lead to failure, in sampling-based algorithms, choosing hyperparameters incorrectly often means longer solution times, but solvers manage to find a solution.

A key difference is that sampling-based algorithms are often anytime algorithms, improving the success rate and solution cost with more compute time.
The na\"{i}ve way to convert optimization methods into an anytime algorithm is to add random restarts.
However, we observe that random restarts are not informative enough to solve hard problems within a reasonable timeframe.
We believe there is great potential in more intelligent restart strategies that, for instance, use both the structure of the problem and previous computations, as explored in this thesis.

Moving forward, to increase the influence and adoption of nonlinear optimization in robotics, it is essential to incorporate these advanced restart strategies into nonlinear solvers.
Coupled with automatic hyperparameter tuning and appropriate scaling of costs and constraints, this opens up interesting avenues for both research and software development.

\paragraph{Development of meta-solvers for TAMP -- research and software infrastructure}

When deploying TAMP systems in the real world, we aim to ensure that our planning algorithms consistently perform quickly across all potential problems.
Achieving this robustness is only possible with TAMP solvers designed to identify the most efficient computing method for solving current problems.
Our algorithms, presented in \cref{ch:mcts,ch:meta-solver}, represent a foundational step in this direction, yet they have some limitations in terms of scalability and applicability.

Thinking in terms of computational space and optimization over computing decisions is a very powerful idea, and we believe that this is a promising direction to pursue.
However, we require more complex models to reason about computation cost and success, to share information between similar task plans, and to find intelligent ways to reuse previous computations.

Further, our research on TAMP meta-solvers highlights that, in addition to research contributions, there is a need for better software infrastructure and tools to combine sampling, discrete search, and optimization into a unified framework.

\paragraph{Perception for TAMP}

A limitation of our work is that we assume an extremely accurate perception module, which provides a perfect model of the world in a form that is convenient for planning.
For instance, we assume knowledge of the objects' positions in the world, differentiable model-based nonlinear constraints, and a low-dimensional representation (e.g., parametric shape) of the objects we intend to manipulate.
Additionally, we presume these objects have simple geometric shapes, such as cubes, cylinders, and spheres.

An attractive and structured approach to extend our TAMP solvers to include perception is to utilize neural-based perception modules that can directly map high-dimensional sensor input to the required representation for planning.
Recent deep learning methods have shown great success in perception tasks, such as object detection, and pose estimation and segmentation \cite{he2017mask,kirillov2023segment,redmon2016you,labbe2020cosypose}.

However, small perception errors can lead to large failures in the planning and execution of manipulation tasks.
An alternative approach is to learn directly manipulation features or discrete states for planning.
For instance, this could involve generating a discrete state representation directly from images for high-level task plan computation \cite{yuan2022sornet}, or learning manipulation features that can be used directly to synthesize motion using trajectory optimization \cite{ha2022deep,simeonov2022neural}.

\paragraph{Learning universal policies for robotic manipulation}

In Part III of this thesis, we have demonstrated how to leverage learning to accelerate expensive operations in a TAMP solver.
Although the proposed learned modules are integrated as small components in model-based algorithms, they have limited applications as standalone components.

A natural extension is to train models that can either solve the entire TAMP problem or, at the very least, substantial components of it, such as computing the complete motion for a given task plan.
The scope of research in this area is vast, with numerous contributions, such as \cite{driess2021learning,kase2020transferable,fang2019dynamics,gupta2020relay,ichter2020broadly}, among many others.

Drawing from the contributions in our thesis, we aim to utilize our structured representation to enhance efficiency and generalization across a wide array of manipulation tasks.
Our future work involves applying this graph representation to develop neural universal manipulation policies, which directly (or indirectly) map the current state to the subsequent action that robots should take, all while conditioned on the high-level task plan.
By leveraging the graph structure of the optimization problem, a single policy—trainable either through imitation or reinforcement learning—can be adapted to various task plans in different scenes.
Graph-structured policies could be trained with diverse data across different manipulation tasks.
The resulting policies could then be applied to larger problems (e.g., those involving more robots or objects), potentially outperforming recent transformer-based architectures such as \cite{shridhar2023perceiver,brohan2023rt}.

We envision that a combination of data, structure, and models will be necessary to create a universal policy with exceptional generalization capabilities, enabling it to tackle various tasks and environments.
Structured policies hold the potential to merge learned features from perception and contact models—which are inherently challenging to model—with precise model-based features, such as the robot's kinematics, joint limits, or self-collision avoidance.

\section{Final Remarks}

In this thesis, we have studied Task and Motion Planning from a multidisciplinary perspective, combining ideas from optimization, discrete planning, and learning.

Improving TAMP solvers is not just key for the future of robotics but also a fascinating research topic.
I hope this thesis has shown that TAMP is a very interesting field, requiring hybrid planning with continuous and discrete variables and constraints, and advanced reasoning about abstraction and structure.
All these concepts are essential for any robot operating in our continuous world while leveraging discrete abstractions and decompositions for more effective long-term planning.

The question of whether to integrate learning into model-based reasoning is highly pertinent in the context of TAMP.
On one hand, good approximate forward models and model-based solvers are readily available; on the other hand, planning with such models might be time-consuming, often too slow for real-time planning.

From a learning perspective, TAMP constitutes a compelling challenge, given the high dimensionality, requirement for long-term planning, and multimodality.
Generating new data or training new neural networks for every distinct TAMP problem is impractical, underscoring the essential need for generalization abilities in learning-based TAMP.

Finding the right balance between model-based approaches and learning can become even more challenging when we consider the real world, with its complex object shapes and perception through images or point clouds.
While today it is clear that a combination of learning and model-based methods in TAMP is required to solve complex manipulation problems, this equilibrium could shift in the future.
Even if TAMP systems were to become predominantly based on data and neural networks, model-based TAMP would continue to play a central role, providing both a dataset of diverse solutions and the correct understanding and inductive bias for the design of efficient learning-based systems.

Beyond advancements in TAMP solvers, real-world applications will also require deeper exploration into enhanced perception, dexterous manipulation, and dynamic replanning.
However, a deep understanding of model-based TAMP is instrumental, even when some real challenges are not considered or are simplified.

From a research perspective, the journey through task and motion planning in robotics has been enlightening, touching upon multiple paradigms in robotics.
I hope the algorithms, formulations, and discussions presented in this thesis have piqued your interest and can serve as a starting point for further research in the field.

%% file: all_publications.tex
\chapter{Complete List of Publications}\label{app:publications}

\subsection*{Journal Papers}

\begin{enumerate}
	\item \underline{Ortiz-Haro, J.}, Karpas, E., Katz, M., \theAnd Toussaint, M. (2022). A Conflict-Driven Interface Between Symbolic Planning and Nonlinear Constraint Solving. IEEE Robotics and Automation Letters, 7(4), (pp. 10518-10525).
\end{enumerate}

\subsection*{Conferences Papers}

\begin{enumerate}
	\item Hartmann, V. N., \underline{Ortiz-Haro, J.}, \theAnd Toussaint, M. (2023). Efficient Path Planning In Manipulation Planning Problems by Actively Reusing Validation Effort. IEEE/RSJ International Conference on Intelligent Robots and Systems (IROS).

	\item \underline{Ortiz-Haro, J.}, Ha, J. S., Driess, D., Karpas, E., \theAnd Toussaint, M. (2023). Learning Feasibility of Factored Nonlinear Programs in Robotic Manipulation Planning. IEEE International Conference on Robotics and Automation (ICRA) (pp. 3729-3735). 

	\item Braun, C. V., \underline{Ortiz-Haro, J.}, Toussaint, M., \theAnd Oguz, O. S. (2022). Rhh-lgp: Receding Horizon and Heuristics-Based Logic-Geometric Programming for Task and Motion Planning. IEEE/RSJ International Conference on Intelligent Robots and Systems (IROS) (pp. 13761-13768).

	\item Kamat, J., \underline{Ortiz-Haro, J.}, Toussaint, M., Pokorny, F. T., \theAnd Orthey, A. (2022). Bitkomo: Combining Sampling and Optimization for Fast Convergence in Optimal Motion Planning. IEEE/RSJ International Conference on Intelligent Robots and Systems (IROS) (pp. 4492-4497). 

	\item \underline{Ortiz-Haro, J.}, Karpas, E., Toussaint, M., \theAnd Katz, M. (2022). Conflict-Directed Diverse Planning for Logic-Geometric Programming. In Proceedings of the International Conference on Automated Planning and Scheduling (Vol. 32, pp. 279-287).

	\item Hoenig, W., \underline{Ortiz-Haro, J.}, \theAnd Toussaint, M. (2022). db-A*: Discontinuity-Bounded Search for Kinodynamic Mobile Robot Motion Planning. IEEE/RSJ International Conference on Intelligent Robots and Systems (IROS) (pp. 13540-13547). 

	\item \underline{Ortiz-Haro, J.}, Ha, J. S., Driess, D., \theAnd Toussaint, M. (2022). Structured Deep Generative Models for Sampling on Constraint Manifolds in Sequential Manipulation. In Conference on Robot Learning (pp. 213-223). PMLR.

	\item \underline{Ortiz-Haro, J.}, Hartmann, V. N., Oguz, O. S., \theAnd Toussaint, M. (2021). Learning Efficient Constraint Graph Sampling for Robotic Sequential Manipulation. IEEE International Conference on Robotics and Automation (ICRA) (pp. 4606-4612). 
\end{enumerate}

\subsection*{Preprints}

\begin{enumerate}

  \item 
\underline{Ortiz-Haro, J.}, Hoenig, W., Hartmann, V., \theAnd Toussaint, M. (2023). iDb-A*: Iterative Search and Optimization for Optimal Kinodynamic Motion Planning. Submitted to IEEE Transactions on Robotics (T-RO). 

	\item Moldagalieva, A., \underline{Ortiz-Haro, J.}, Toussaint, M., \theAnd Hoenig, W. (2023). db-CBS: Discontinuity Bounded Conflict-Based Search for Multi-Robot Kinodynamic Motion Planning. arXiv preprint arXiv:2309.16445. Submitted to ICRA.

	\item Grote, P., \underline{Ortiz-Haro, J.}, Toussaint, M., \theAnd Oguz, O. S. (2023). Neural Field Representations of Articulated Objects for Robotic Manipulation Planning. arXiv preprint arXiv:2309.07620. Submitted to ICRA. 

	\item Wahba, K., \underline{Ortiz-Haro, J.}, Toussaint, M., \theAnd Hoenig, W. (2023). Kinodynamic Motion Planning for a Team of Multirotors Transporting a Cable-Suspended Payload in Cluttered Environments. arXiv preprint arXiv:2310.03394. Submitted to ICRA.

	\item Toussaint, M., \underline{Ortiz-Haro, J.}, Hartmann, V., Karpas, E., \theAnd Hoenig, W. (2023). Effort Level Search in Infinite Completion Trees with Application to Task-and-Motion Planning. Submitted to ICRA.

	\item Levit, S., \underline{Ortiz-Haro, J.}, \theAnd Toussaint, M. (2023). Solving Sequential Manipulation Puzzles by Finding Easier Subproblems. Submitted to ICRA.
\end{enumerate}

\subsection*{Workshop Papers}

\begin{enumerate}
	\item Rehberg, W., \underline{Ortiz-Haro, J.}, Toussaint, M., \theAnd Hoenig, W. (2023). Comparison of Optimization-Based Methods for Energy-Optimal Quadrotor Motion Planning. arXiv preprint arXiv:2304.14062. Aerial Robotics Workshop ICRA.

	\item Hoenig, W., \underline{Ortiz-Haro, J.}, \theAnd Toussaint, M. (2022). Benchmarking Sampling-, Search-, and Optimization-based Approaches for Time-Optimal Kinodynamic Mobile Robot Motion Planning. Motion Planning Workshop IROS.

\end{enumerate}

\subsection*{Master's and Bachelor's Theses}

\begin{enumerate}
	\item Weingart, A. (2023) Efficient Kinodynamic Motion Planning with Reinforcement Learning Policies. Master's Thesis in Computer Science (TU-Berlin). Co-supervision: \underline{Ortiz-Haro, J.}, \theAnd Hoenig, W.

	\item Groete, P. (2023) Neural Scene Representations for Sequential Reasoning.
	      Master's Thesis in Computer Science (TU-Berlin). Co-supervision: \underline{Ortiz-Haro, J.}, \theAnd Oguz, O.

	\item Rehberg, W. (2022) SCP and k-Order Motion Optimization for Cooperative Multirotor Teams.
	      Master's Thesis in Computer Science (TU-Berlin). Co-supervision: \underline{Ortiz-Haro, J.}, \theAnd Hoenig, W.

	\item Kamat, J. (2022) Combining Sampling and Optimization for Optimal Motion Planning. Master's Thesis in Mathematics (BITS and TU-Berlin). Co-supervision: \underline{Ortiz-Haro, J.}, \theAnd Orthey, A.

	\item Oedi, P. (2022) Constrained Sampling - A Study on Methods for Sampling from Constraint Manifolds. Master's Thesis in Computer Science (TU-Berlin). Co-supervision: \underline{Ortiz-Haro, J.}, \theAnd Ha, J.

\end{enumerate}